\def\year{2021}\relax
\documentclass[letterpaper]{article} 
\usepackage{aaai21}  
\usepackage{times}  
\usepackage{helvet} 
\usepackage{courier}  
\usepackage[hyphens]{url}  
\usepackage{graphicx} 
\usepackage{graphicx}  
\usepackage{natbib}  
\usepackage{caption} 
\usepackage{booktabs}       
\usepackage{nicefrac}       
\usepackage{microtype}      
\usepackage{subcaption}
\usepackage{algorithm}
\usepackage{algorithmic}
\usepackage{amsmath,amssymb,amsfonts,amstext,amsthm,mathrsfs}
\usepackage{multirow}
\usepackage{makecell}
\usepackage[switch]{lineno}

\newtheorem{theorem}{Theorem}
\newtheorem{definition}{Definition}

\newtheorem{lemma}{Lemma}

\newtheorem{assum}{Assumption}
\newtheorem{remark}{Remark}

\newcommand{\ie}{\emph{i.e.}}
\newcommand{\ff}{follows from }
\newcommand{\ffe}{follows from Eq.~}
\newcommand{\ffl}{follows from Lemma~}

\allowdisplaybreaks[2]
\usepackage[symbol]{footmisc}

\title{Secure Bilevel Asynchronous Vertical Federated Learning \\with Backward Updating}
\author{
 Qingsong Zhang\textsuperscript{\rm 1,2}, Bin Gu \textsuperscript{\rm 3,4},
  Cheng Deng\textsuperscript{\rm 1,$\ast$}, \setcounter{footnote}{0}
  and Heng Huang\textsuperscript{\rm 3,5}\footnote{Corresponding Authors}
}
\affiliations{
 \\ \textsuperscript{\rm 1}School of Electronic Engineering, Xidian University, Xi'an 710071, China;
 \textsuperscript{\rm 2}JD Tech, Beijing 100176, China.\\
 \textsuperscript{\rm 3}JD Finance America Corporation, Mountain View, CA, USA;
 \textsuperscript{\rm 4}MBZUAI, United Arab Emirates\\
 \textsuperscript{\rm 5}Electrical and Computer Engineering, University of Pittsburgh, PA, USA\\
 \{qszhang1995, jsgubin, chdeng.xd, henghuanghh\}@gmail.com
}
\begin{document}

\maketitle
\begin{abstract}
Vertical federated learning (VFL) attracts increasing attention due to the emerging demands of multi-party collaborative modeling and concerns of privacy leakage. In the real VFL applications, usually only one or partial parties hold labels, which makes it challenging for all parties to collaboratively learn the model without privacy leakage. Meanwhile, most existing VFL algorithms are trapped in the synchronous computations, which leads to
inefficiency in their real-world applications. To address these challenging problems, we propose a novel {\bf VF}L framework integrated with new {\bf b}ackward updating mechanism and {\bf b}ilevel asynchronous parallel architecture (VF{${\textbf{B}}^2$}), under which  three new algorithms, including  VF{${\textbf{B}}^2$}-SGD, -SVRG, and -SAGA, are proposed. We derive the theoretical results of the convergence rates of these three algorithms under both strongly convex and nonconvex conditions. We also prove the security of VF{${\textbf{B}}^2$} under semi-honest threat models. Extensive experiments on benchmark datasets demonstrate that our algorithms are efficient, scalable and lossless.
\end{abstract}
\section{Introduction}
Federated learning \cite{mcmahan2016communication,smith2017federated,kairouz2019advances} has emerged as a paradigm for collaborative modeling with privacy-preserving. A line of recent works \cite{mcmahan2016communication,smith2017federated} focus on the horizontal federated learning, where each party has a subset of samples with complete features. There are also some works \cite{gascon2016secure,yang2019federated,dang2020large} studying the vertical federated learning (VFL), where each party holds a disjoint subset of features for all samples. In this paper, we focus on VFL that has attracted much attention from the academic and industry due to its wide applications to emerging multi-party collaborative modeling with privacy-preserving.

Currently, there are two mainstream methods for VFL, including homomorphic encryption (HE) based methods and exchanging the raw computational results (ERCR) based methods. The HE based methods \cite{hardy2017private,cheng2019secureboost} leverage HE techniques to encrypt the raw data and then use the encrypted data (ciphertext) for training model with privacy-preserving. However, there are two major drawbacks of HE based methods. First, the complexity of homomorphic mathematical operation on ciphertext field is very high,
thus HE is extremely time consuming for modeling \cite{liu2015encrypted,liu2019federated}. Second, approximation is required for HE to support operations of non-linear functions, such as Sigmoid and Logarithmic functions, which inevitably causes loss of the accuracy for various machine learning models using non-linear functions \cite{kim2018secure,yang2019quasi}. Thus, the inefficiency and inaccuracy of HE based methods dramatically limit their wide applications to realistic VFL tasks.

ERCR based methods \cite{zhang2018feature,hu2019fdml,gu2020Privacy} leverage labels and the raw intermediate computational results transmitted from the other parties to compute stochastic gradients, and thus use distributed stochastic gradient descent (SGD) methods to train VFL models efficiently. Although ERCR based methods circumvent aforementioned drawbacks of HE based methods, existing ERCR based methods are designed with only considering that all parties have labels, which is not usually the case in real-world VFL tasks. In realistic VFL applications, usually only one or partial parties (denoted as active parties) have the labels, and the other parties (denoted as passive parties) can only provide extra feature data but do not have labels. When these ERCR based methods are applied to the real situation with both active and passive parties, \textbf{the algorithms even cannot guarantee the convergence} because only active parties can update the gradient of loss function based on labels but the passive parties cannot, \emph{i.e.}  partial model parameters are not optimized during the training process. Thus, it comes to the crux of designing the proper algorithm for solving real-world VFL tasks with only one or partial parties holding labels.

Moreover, algorithms using synchronous computation \cite{gong2016private,zhang2018feature} are inefficient when applied to real-world VFL tasks, especially, when computational resources in the VFL system are unbalanced. Therefore, it is desired to design the efficient asynchronous algorithms for real-world VFL tasks.
Although there have been several works studying asynchronous VFL algorithms
\cite{hu2019fdml,gu2020Privacy}, it is still an open problem to design asynchronous algorithms for solving real-world VFL tasks with only one or partial parties holding labels.

In this paper, we address these challenging problems by proposing a novel framework (VF{${\textbf{B}}^2$}) integrated with the novel backward updating mechanism (BUM) and bilevel asynchronous parallel architecture (BAPA). Specifically, the BUM
enables all parties, rather than only active parties, to collaboratively update the model securely and also makes the final model lossless; the BAPA is designed for efficiently asynchronous backward updating.
Considering the advantages of SGD-type algorithms in optimizing machine learning models, we thus propose three new SGD-type algorithms, \ie, VF{${\textbf{B}}^2$}-SGD, -SVRG and -SAGA, under that framework.
\begin{figure*}[!t]
	\centering
\begin{subfigure}{0.38\linewidth}
		\includegraphics[width=\linewidth]{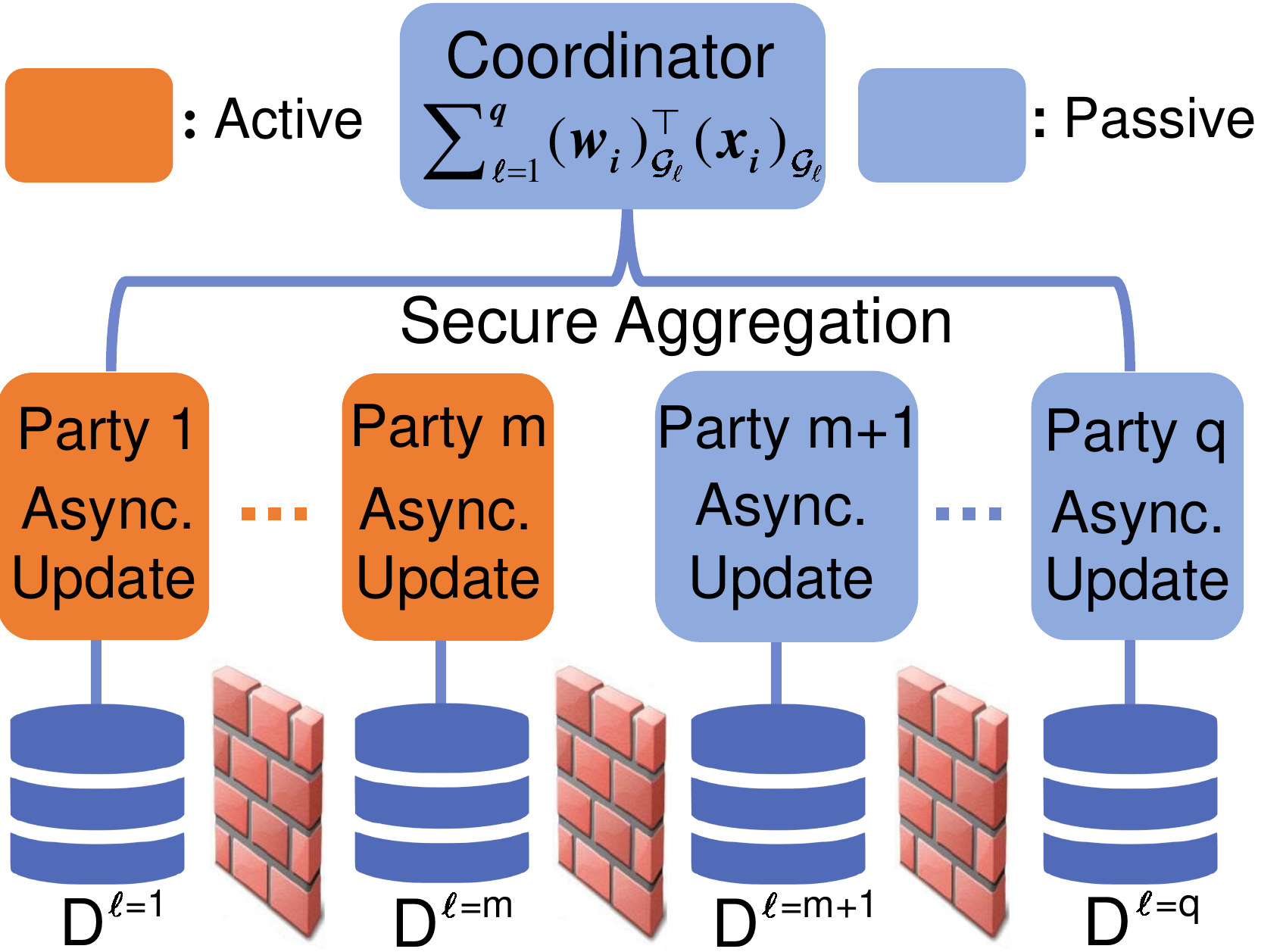}
		\caption{}
\label{struca}
	\end{subfigure}%
	\qquad
	\begin{subfigure}{0.4\linewidth}
		\includegraphics[width=\linewidth]{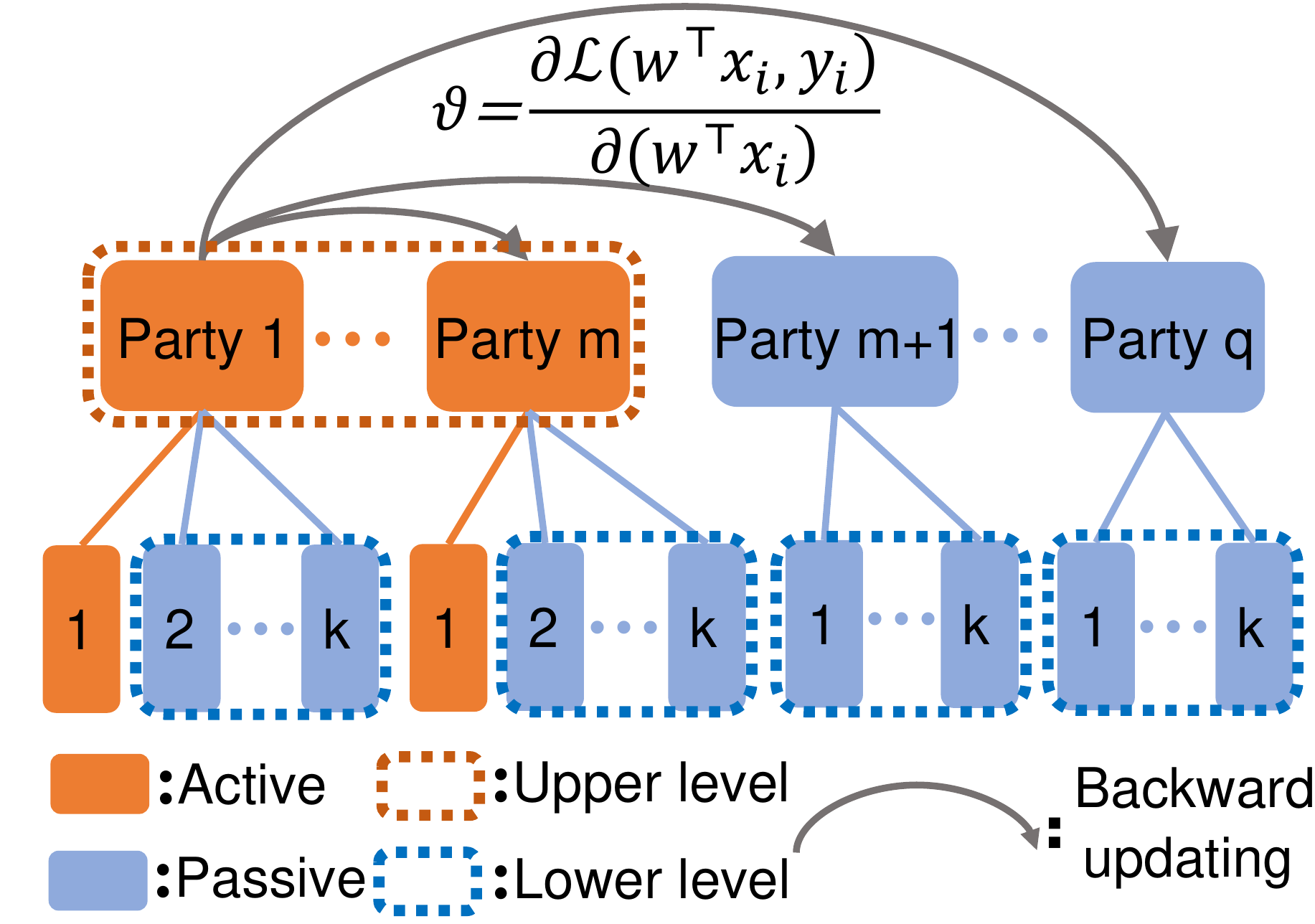}
		\caption{}
\label{strucb}
	\end{subfigure}%
	\caption{(a): System structure of VF{${\textbf{B}}^2$} framework. (b): Illustration of the BUM and BAPA, where $k$ is defined in Section \ref{secexp}.}
\label{struc}
\end{figure*}
We summarize the contributions of this paper as follows.
\begin{itemize}
\item
We are the first to propose the novel backward updating mechanism for ERCR based VFL algorithms, which enables all parties, rather than only parties holding labels, to collaboratively learn the model with privacy-preserving and without hampering the accuracy of final model.
\item
We design a bilevel asynchronous parallel architecture that enables all parties asynchronously update the model through backward updating, which is efficient and scalable.
\item
We propose three new algorithms for VFL, including VF{${\textbf{B}}^2$}-SGD, -SVRG, and -SAGA under VF{${\textbf{B}}^2$}. Moreover, we theoretically prove their convergence rates for both strongly convex and nonconvex problems.
\end{itemize}
\noindent\textbf{Notations.}
$\widehat{w}$ denotes the inconsistent read of $w$.
 $\bar{w}$ denotes $w$ to compute local stochastic gradient of loss function for collaborators, which maybe stale due to communication delay.
 $\psi(t)$ is the corresponding party performing  the $t$-th global iteration.
Given a finite set $S$, $|S|$ denotes its cardinality.

\section{Problem Formulation}
Given a training set $\{x_i,y_i\}_{i=1}^n$, where $y_i \in \{-1, +1\}$ for binary classification task or $y_i \in \mathbb{R}$ for regression problem and $x_i \in \mathbb{R}^d $, we consider the model in a linear form of $w^{\top} x$, where $w \in \mathbb{R}^d $ corresponds to the model parameters. For VFL, $x_i$ is vertically distributed among $q\geq2$ parties, \ie, $x_i=[(x_i)_{\mathcal{G}_1}; \cdots; (x_i)_{\mathcal{G}_q}]$, where $(x_i)_{\mathcal{G}_\ell} \in \mathbb{R}^{d_\ell}$ is stored on the $\ell$-th party and $\sum_{\ell=1}^{q}d_{\ell} = d$. Similarly, there is $w=[w_{\mathcal{G}_1}; \cdots; w_{\mathcal{G}_q}]$. Particularly, we focus on the following regularized empirical risk minimization problem.
\begin{equation}\label{P}
\min_{w\in \mathbb{R}^{d}} f(w) := \frac{1}{n} \sum_{i=1}^{n} \underbrace{\mathcal{L}\left(w^{\top} x_{i}, y_{i}\right)+\lambda \sum_{\ell=1}^{q} g(w_{\mathcal{G}_\ell})}_{f_{i}(w)}, \tag{P}
\end{equation}
where $w^{\top}x_i =\sum_{\ell=1}^{q}{w}_{\mathcal{G}_{\ell }}^{\top}\left(x_{i}\right)_{\mathcal{G}_{\ell}}$, $\mathcal{L}$ denotes the loss function, $\sum_{\ell=1}^{q} g(w_{\mathcal{G}_\ell})$ is the regularization term, and $f_i: \mathbb{R}^d \to \mathbb{R}$ is smooth and possibly nonconvex. Examples of problem~\ref{P} include models  for binary classification tasks \cite{conroy2012fast,wang2017stochastic} and  models for regression tasks \cite{shen2013novel,wang2019spiderboost}.

In this paper, we introduce two types of parties: {\bf{active party}} and {\bf{passive party}}, where the former denotes data provider holding labels  while the latter does not. Particularly, in our problem setting, there are $m$ ($1\leq m \leq q$) active parties. Each active party can play the role of dominator in model updating by actively launching updates.
All parties, including both active and passive parties, passively launching updates play the role of collaborator.
To guarantee the model security, only active parties know the form of the loss function.
Moreover, we assume that the labels can be shared by all parties finally. Note that this does not obey our intention that only active parties hold the labels before training.
The problem studied in this paper is stated as follows:\\
\noindent{\bf{Given}}: Vertically partitioned data $\{x_{\mathcal{G}_\ell}\}_{\ell =1}^{q}$ stored in $q$ parties and the labels only held by active parties. \\
\noindent{\bf{Learn}}: A machine learning model {\bf M} collaboratively learned by both active and passive parties without leaking privacy.\\
\noindent{\bf{Lossless Constraint}}: The accuracy of {\bf M} must be comparable to that of model {\bf M$'$} learned under non-federated learning.

\section{VF${{\text{B}}}^2$ Framework}
In this section, we propose the novel VF${{\text{B}}}^2$ framework. VF${{\text{B}}}^2$ is composed of three components and its systemic structure is illustrated in Fig.~\ref{struca}. The details of these components are presented in the following.

The key of designing the proper algorithm for solving real-world VFL tasks with both active and passive parties  is to make the passive parties utilize the label information for model training. However, it is challenging to achieve this because direct using the labels hold by active parties leads to privacy leakage of the labels without training. To address this challenging problem, we design the BUM with painstaking.\\
{\noindent\bf{Backward Updating Mechanism:}}
The key idea of BUM is to make passive parties indirectly use labels to compute stochastic gradient without directly accessing the raw label data. Specifically, the BUM embeds label $y_i$ into an intermediate value $\vartheta: = \frac{\partial \mathcal{L}\left({w}^{\top} x_{i}, y_{i}\right)}{\partial\left({w}^{\top} x_{i}\right)}$. Then $\vartheta$ and $i$ are distributed backward to the other parties. Consequently, the passive parties can also compute the stochastic gradient and update the model by using the received $\vartheta$ and $i$ (please refer to Algorithms~\ref{AFSGD-A} and \ref{AFSGD-P} for details). Fig.~\ref{strucb} depicts the case where $\vartheta$ is distributed from party $1$ to the rest parties.
In this case, all parties, rather than only active parties, can collaboratively learn the model without privacy leakage.
 \begin{algorithm}[!t]
\caption{Safe algorithm of obtaining $w^Tx_i$.}\label{safer_tree}
\begin{algorithmic}[1]
\REQUIRE {$\{w_{\mathcal{G}_{\ell'}}\}_{{\ell '}=1}^{q}$ and $ \{{(x_i)}_{\mathcal{G}_{\ell '}}\}_{{\ell '}=1}^{q}$ allocating at each party, index $i$.} \\
{ \bf{Do this in parallel}}
\FOR {$\ell '=1, \cdots, q$}
\STATE Generate a  ramdon number $\delta_{\ell '}$ and calculate $w_{\mathcal{G}_{\ell '}}^{\top} {(x_i)}_{\mathcal{G}_{\ell '}}+ \delta_{\ell '}$,
\ENDFOR
\STATE Obtain $\xi_1 = \sum_{\ell '=1}^{q} (w_{\mathcal{G}_{\ell '}}^{\top} {(x_i)}_{\mathcal{G}_{\ell '}} +\delta_{\ell '})$ through tree structure $T_1$.
\STATE Obtain $\xi_2 = \sum_{\ell '=1}^{q} \delta_{\ell '} $ through totally different tree structure $T_2\neq T_1$.
\ENSURE  {${w}^{\top}x_i =\xi_1-\xi_2$}
\end{algorithmic}
\end{algorithm}

For VFL algorithms with BUM, dominated updates in different active parties are performed in distributed-memory parallel, while collaborative updates within a party are performed in shared-memory parallel. The difference of parallelism fashion leads to the challenge of developing a new parallel architecture instead of just directly adopting the existing asynchronous parallel architecture for VFL. To tackle this challenge, we elaborately design a novel BAPA.\\
{\bf{Bilevel Asynchronous Parallel Architecture:}}
The BAPA includes two levels of parallel architectures, where the upper level denotes the inner-party parallel and the lower one is the intra-party parallel. More specifically, the inner-party parallel denotes distributed-memory parallel between active parties, which enables all active parties to asynchronously launch dominated updates; while the intra-party one denotes the shared-memory  parallel of collaborative updates within each party, which enables multiple threads within a specific party
to asynchronously perform the collaborative updates. Fig.~\ref{strucb} illustrates the BAPA with $m$ active parties.

To utilize feature data provided by other parties, a party need obtain $w^Tx_i=\sum_{\ell=1}^{q} w_{\mathcal{G}_{\ell}}^{\top} {(x_i)}_{\mathcal{G}_{\ell}}$.  Many recent works achieved this by aggregating the local intermediate computational results securely \cite{hu2019fdml,gu2020federated}. In this paper, we use the efficient tree-structured communication scheme \cite{zhang2018feature} for secure aggregation, whose security  was proved in \cite{gu2020federated}. \\
{\noindent\bf{Secure Aggregation Strategy:}}  The details are summarized in Algorithm~~\ref{safer_tree}. Specifically, at step~2, $w_{\mathcal{G}_\ell}^{\top} {(x_i)}_{\mathcal{G}_\ell}$ is computed locally on the $\ell$-th party to prevent the direct leakage of $w_{\mathcal{G}_\ell}$ and ${(x_i)}_{\mathcal{G}_\ell}$. Especially, a random number $\delta_{\ell}$ is added to $w_{\mathcal{G}_{\ell }}^{\top} {(x_i)}_{\mathcal{G}_{\ell }}$ to mask the value of $w_{\mathcal{G}_{\ell }}^{\top} {(x_i)}_{\mathcal{G}_{\ell }}$, which can enhance the security during aggregation process. At steps~4 and 5, $\xi_1$ and $\xi_2$ are aggregated through tree structures $T_1$ and $T_2$, respectively. Note that $T_2$ is totally different from $T_1$ that can prevent the random value being removed under threat model 1 (defined in section \ref{securitysec}).
Finally, value of $w^{\top} x_i=\sum_{\ell=1}^{q} (w_{\mathcal{G}_{\ell }}^{\top} {(x_i)}_{\mathcal{G}_{\ell }}$ is recovered by removing term $\sum_{\ell =1}^{q} \delta_{\ell }$ from $\sum_{\ell =1}^{q} (w_{\mathcal{G}_{\ell }}^{\top} {(x_i)}_{\mathcal{G}_{\ell }} +\delta_{\ell})$ at the output step. Using such aggregation strategy, ${(x_i)}_{\mathcal{G}_{\ell }}$ and $w_{\mathcal{G}_{\ell }} $ are prevented from leaking during the aggregation.

\section{Secure Bilevel Asynchronous VFL Algorithms with Backward Updating}
\begin{algorithm}[!t]
\caption{VF{${\textbf{B}}^2$}-SGD for active party $\ell$ to actively launch dominated updates.}\label{AFSGD-A}
\begin{algorithmic}[1]
\REQUIRE {Local data $\{{(x_i)}_{\mathcal{G}_{\ell}},y_i\}_{i=1}^{n}$ stored on the $\ell$-th party, learning rate $\gamma$}.
\STATE Initialize the necessary parameters.\\
{ \bf{Keep doing in parallel (distributed-memory parallel for multiple active parties)}}
  \STATE \quad  Pick up an index $i$ randomly from $\{1,...,n\}$.
\STATE \quad Compute $\widehat{w}^{\top} x_{i}=\sum_{\ell^{\prime}=1}^{q}\widehat{w}_{\mathcal{G}_{\ell '}}^{\top}\left(x_{i}\right)_{\mathcal{G}_{\ell^{\prime}}}$  based on Al\\
\quad gorithm~\ref{safer_tree}.
  \STATE \quad Compute $\vartheta = \frac{\partial \mathcal{L}\left(\widehat{w}^{\top} x_{i}, y_{i}\right)}{\partial\left(\widehat{w}^{\top} x_{i}\right)}$.
   \STATE \quad Send $\vartheta$ and index $i$ to collaborators.
  \STATE  \quad Compute $\widetilde{v}^{\ell}=\nabla_{\mathcal{G}_{\ell}} f_{i}(\widehat{w})$.
  \STATE \quad Update $w_{\mathcal{G}_{\ell}} \leftarrow w_{\mathcal{G}_{\ell}}-\gamma \widetilde{v}^{\ell}$. \\
  { \bf{End parallel }}
\end{algorithmic}
\end{algorithm}
SGD \cite{bottou2010large} is a popular method for learning machine learning (ML) models. However, it has a poor convergence rate due to the intrinsic variance of stochastic gradient. Thus, many popular variance reduction techniques have been proposed, including the SVRG, SAGA, SPIDER \cite{johnson2013accelerating,defazio2014saga,wang2019spiderboost} and their applications to other problems \cite{huang2019faster,huang2020accelerated,zhang2020faster,dang2020large,yang2020learning,yang2020adversarial,li2020towards,wei2019adversarial}. In this section we raise three SGD-type algorithms, \emph{i.e.} the SGD, SVRG and SAGA,
which are the most popular ones among SGD-type  methods for the appealing performance in practice. We summarize the detailed steps of VF$\bf B^2$-SGD in Algorithms \ref{AFSGD-A} and \ref{AFSGD-P}. For VF$\bf B^2$-SVRG and -SAGA, one just needs to replace the update rule with corresponding one.

As shown in Algorithm~\ref{AFSGD-A}, at each dominated update, the dominator (an active party) calculates $\vartheta$ and then distributes $\vartheta$ together with $i$ to the collaborators (the rest $q-1$ parties). As shown in algorithm~\ref{AFSGD-P}, for party $\ell$, once it has received the $\vartheta$ and $i$, it will launch a new collaborative update asynchronously.  As for the dominator, it computes the local stochastic gradient as $\nabla_{\mathcal{G}_{\ell}} f_{i}(\widehat{w}) = \nabla_{\mathcal{G}_\ell} \mathcal{L}(\widehat{w})+ \lambda \nabla g(\widehat{w}_{\mathcal{G}_{\ell }})$. While, for the collaborator, it uses the received $\vartheta$ to compute $\nabla_{\mathcal{G}_\ell} \mathcal{L}$ and local $\widehat{w}$ to compute $\nabla_{\mathcal{G}_{\ell }} g$ as shown at step 3 in Algorithm~\ref{AFSGD-P}. Note that active parties also need perform Algorithm~\ref{AFSGD-P} to collaborate with other dominators to ensure that the model parameters of all parties are updated.
\begin{algorithm}[!t]
\caption{VF{${\textbf{B}}^2$}-SGD for the $\ell$-th party to passively launch collaborative updates.}\label{AFSGD-P}
\begin{algorithmic}[1]
\REQUIRE {Local data $\{{(x_i)}_{\mathcal{G}_{\ell}},y_i\}_{i=1}^{n}$ stored on the $\ell$-th party, learning rate $\gamma$}.
\STATE Initialize the necessary parameters (for passive parties).\\
 { \bf{Keep doing in parallel (shared-memory parallel for multiple threads)}}
\STATE \quad Receive $\vartheta$ and the index $i$ from the dominator.
  \STATE \quad Compute  $\widetilde{v}^{\ell } = \nabla_{\mathcal{G}_{\ell}} \mathcal{L}(\bar{w})+\lambda \nabla_{\mathcal{G}_{\ell }} g(\widehat{w}) = \vartheta \cdot (x_i)_{\mathcal{G}_{\ell }} +$ \\
  \quad $\lambda \nabla g(\widehat{w}_{\mathcal{G}_{\ell }})$.
  \STATE  \quad Update  $w_{\mathcal{G}_{\ell}} \leftarrow w_{\mathcal{G}_{\ell }}-\gamma \widetilde{v}^{\ell}$.
\STATE{ \bf{End parallel}}
\end{algorithmic}
\end{algorithm}
\section{Theoretical Analysis}
In this section, we provide the convergence analyses. Please see the arXiv version for more details.  We first present  preliminaries for strongly convex and nonconvex problems.
\begin{assum}\label{assum1}
For $f_i(w)$ in problem \ref{P}, we assume the following conditions hold:
\begin{enumerate}
  \item {\bf{Lipschitz Gradient:}} Each  function $f_i$, $i=1,\ldots,n$,  there exists  $L>0$ such that for  $ \forall \ w,w'\in \mathbb{R}^d$, there is
\begin{equation}
\|\nabla f_i (w) - \nabla f_i (w')\| \le L\|w-w'\|.
\end{equation}
  \item {\bf{Block-Coordinate Lipschitz Gradient:}} For $i=1,\ldots,n$, there exists an $L_\ell>0$ for the $\ell$-th block $\mathcal{G}_\ell$, where $\ell=1,\cdots,q$ such that
\begin{equation}
\|\nabla_{\mathcal{G}_\ell} f_i (w+U_\ell\Delta_\ell) - \nabla_{\mathcal{G}_\ell} f_i (w)\| \le L_\ell\|\Delta_\ell\|,
\end{equation}
where $\Delta_\ell \in \mathbb{R}^{d_\ell}$, $U_\ell \in \mathbb{R}^{d\times d_\ell}$ and $[U_1, \cdots, U_q] = I_d$.
  \item {\bf{Bounded Block-Coordinate Gradient:}} There exists a constant $G$ such that for $f_i,\ i=1,\cdots,n$ and block $\mathcal{G}_\ell$, $\ell =1,\cdots,q$, it holds that $\|\nabla_{\mathcal{G}_\ell} f_i(w)\|^2\leq G $.
\end{enumerate}
\end{assum}
\begin{assum}\label{assum2}
The regularization term $g$ is $L_g$-smooth, which  means that there exists an $L_g>0$ for $\ell = 1,\dots,q$ such that  $\forall  w_{\mathcal{G}_\ell}, w_{\mathcal{G}_\ell}' \in \mathbb{R}^{d_\ell}$ there is
\begin{equation}
\|\nabla g(w_{\mathcal{G}_\ell}) - \nabla g(w_{\mathcal{G}_\ell}')\| \le L_g\|w_{\mathcal{G}_\ell}-w_{\mathcal{G}_\ell}'\|.
\end{equation}
\end{assum}
Assumption \ref{assum2} imposes the smoothness on $g$, which is necessary for the convergence analyses. Because, as for a specific collaborator, it uses the received $\widehat{w}$ (denoted as $\bar{w}$) to compute $\nabla_{\mathcal{G}_{\ell }} \mathcal{L}$ and local $\widehat{w}$ to compute $\nabla_{\mathcal{G}_{\ell }} g=\nabla g(w_{\mathcal{G}_\ell})$, which makes it necessary  to track the behavior of $g$ individually.  Similar to previous research works \cite{lian2015asynchronous,huo2017asynchronous,leblond2017asaga}, we introduce the bounded delay as follows.
\begin{assum}\label{assum4}{\bf{Bounded Delay:}}
Time delays of inconsistent reading and communication between dominator and its collaborators are upper bounded by $\tau_1$ and $\tau_2$, respectively.
\end{assum}
Given $\widehat{w}$ as the inconsistent read of $w$, which is used to compute the stochastic gradient in dominated updates, following the analysis in \cite{gu2020Privacy}, we have
\begin{equation}\label{Dt1}
\widehat{w}_t-w_t = \gamma \sum_{u\in D(t)}U_{\psi(u)}\widetilde{v}_u^{\psi(u)},
\end{equation}
where  $D(t)=\{t-1,\cdots,t-\tau_0\}$ is a subset of non-overlapped previous iterations with $\tau_0\leq \tau_1$. Given $\bar{w}$ as the parameter used to compute the $\nabla_{\mathcal{G}_{\ell }} \mathcal{L}$ in collaborative updates, which is the steal state of $\widehat{w}$ due to the communication delay between the specific dominator and its corresponding collaborators. Then, following the analyses in \cite{huo2017asynchronous}, there is
\begin{equation}\label{Dt2}
   \bar{w}_t = \widehat{w}_{t-\tau_0} = \widehat{w}_t + \gamma \sum_{t'\in D^\prime(t)}U_{\psi(t')}\widetilde{v}_{t'}^{\psi(t')},
\end{equation}
where $D'(t)=\{t-1,\cdots,t-\tau_0\}$ is a subset of previous iterations performed during the communication and $\tau_0\leq \tau_2$.
\subsection{Convergence Analysis for Strongly Convex Problem}
\begin{assum}\label{assumc1}
Each  function $f_i$, $i=1,\ldots,n$, is $\mu$-strongly convex, i.e., $\forall \ w,\  w'\in \mathbb{R}^d$ there exists a $\mu>0$ such that
\begin{equation}
f_i(w)\geq f_i(w') +  \langle \nabla f_i(w'), w- w' \rangle + \frac{\mu}{2}\|w-w'\|^2.
\end{equation}
\end{assum}
For strongly convex problem, we introduce notation $K(t)$ that denotes a minimum set of successive iterations fully visiting all coordinates from global iteration number $t$. Note that this is necessary for the asynchronous convergence analyses of the global model. Moreover, we assume that the size of $K(t)$ is upper bounded by $\eta_1$, \ie, $|K(t)|\leq \eta_1$. Based on $K(t)$, we introduce the epoch number $v(t)$ as follow.
\begin{definition}\label{definc2}
Let $P(t)$ be a partition of $\{0,1,\cdots, t-\sigma'\}$, where $\sigma'\geq0$. For any $\kappa\subseteq P(t)$ we have that there exists $t'\leq t$ such that $K(t')=\kappa$, and $\kappa_1 \subseteq P(t)$ such that $K(0)=\kappa_1$. The epoch number for the $t$-th global iteration, i.e., $v(t)$ is defined as the maximum cardinality of $P(t)$.
\end{definition}
Given the definition of epoch number $v(t)$, we have the following theoretical results for $\mu$-strongly convex problem.
\begin{theorem}\label{thm-sgdconvex}
Under Assumptions~\ref{assum1}-\ref{assum4} and \ref{assumc1}, to achieve the accuracy $\epsilon$ of problem~\ref{P} for VF{${\textbf{B}}^2$}-SGD, i.e., $\mathbb{E}(f(w_t)-f(w^*))\leq \epsilon$, let $\gamma\leq \frac{\epsilon\mu^{1/3}}{(G{96L_*^2})^{1/3}}$, if $\tau\leq {\text {min}}\{\epsilon^{-4/3}, \frac{(GL_*^2)^{2/3}}{\epsilon^2\mu^{2/3}}\}$
, the epoch number $v(t)$ should satisfy
$
v(t) {\geq} \frac{44(GL^2_*)^{1/3}}{ \mu^{4/3}\epsilon} log (\frac{2(f(w_0)-f(w^*))}{\epsilon})
$
, where $L_{*}=\text{max}\{L, \{L_{\ell}\}_{\ell=1}^{q}, L_g\}$, $\tau={\text{max}}\{\tau_1^2,\tau_2^2,\eta_1^2\}$, $w^0$ and $w^*$ denote the initial point and optimal point, respectively.
\end{theorem}
\begin{theorem}\label{thm-svrgconvex}
Under Assumptions~\ref{assum1}-\ref{assum4} and \ref{assumc1}, to achieve the accuracy $\epsilon$ of problem~\ref{P} for VF{${\textbf{B}}^2$}-SVRG, let $C=(L_*^2 \gamma+L_*)\frac{\gamma^2}{2}$ and $\rho=\frac{\gamma\mu}{2}- \frac{16L_*^2\eta_1C}{\mu}$, we can carefully choose $\gamma$  such that
\begin{eqnarray}
\nonumber
  && 1)  \ 1-2L_*^2\gamma^2\tau>0;
  \ \ 2)\ \rho>0;
  \ \ 3)\ \frac{8L_*^2\tau^{1/2}C}{\rho\mu} \leq 0.05;\\
  && 4)\ L_{*}^{2} \gamma^{2}\tau^{3/2}(28 C + 5 {\gamma})  \frac{2\lambda_{\gamma}G}{\rho} \leq \frac{\epsilon}{8},
\end{eqnarray}
where $\lambda_{\gamma}=\frac{18}{1-2L_{*}^2\gamma^2\tau}$, the inner epoch number $v(t)$ should satisfy $v(t)\geq \frac{{\text {log}} 0.25}{{\text {log}}(1-\rho)}$ and the outer loop number $S$ should satisfy $S\geq \frac{1}{{\text {log}}\frac{4}{3}}{{\text {log}}\frac{2f(w_0)-f(w^*)}{\epsilon}}$.
\end{theorem}
\begin{theorem}\label{thm-sagaconvex}
Under Assumptions~\ref{assum1}-\ref{assum4} and \ref{assumc1}, to achieve the accuracy $\epsilon$ of problem~\ref{P} for VF{${\textbf{B}}^2$}-SAGA, let
$c_0=\left(2 {  \gamma^3 \tau^{3/2}} + ({L_*^2 \gamma^3 \tau} + {L_*\gamma^2}) 180 \gamma^2\tau^{3/2} + 8 \gamma^{2} \tau\right)\frac{18GL_{*}^2}{1 - 72L_{*}^2\gamma^2\tau }$,
$c_1=2L_*^2\tau({L_*^2 \gamma^3 \tau}+ {L_*\gamma^2})$,
$c_2=4 ({L_*^2 \gamma^3 \tau}+ {L_*\gamma^2})\frac{L_*^{2} \tau}{n} $, and $\rho\in (1-\frac{1}{n},1)$, we can choose $\gamma$  such that
\begin{eqnarray}
   &&1) \  1-72L_*^2\gamma^2\tau>0;\ 2)\ 0<1-\frac{\gamma \mu}{4}<1;
  \nonumber \\
   &&3)\ \frac{4 c_{0}}{\gamma \mu(1-\rho)\left(\frac{\gamma \mu^{2}}{4}-2 c_{1}-c_{2}\right)} \leq \frac{\epsilon}{2};
  \nonumber \\
   &&4)\ -\frac{\gamma \mu^{2}}{4}+2 c_{1}+c_{2}\left(1+(1-\frac{1-\frac{1}{n}}{\rho})^{-1}\right) \leq 0;
  \nonumber \\
  &&5)\ -\frac{\gamma \mu^{2}}{4}+c_{2}+c_{1}\left(2+(1-\frac{1-\frac{1}{n}}{\rho})^{-1}\right) \leq 0,
\end{eqnarray}
the epoch number $v(t)$ should satisfy $v(t) \geq \frac{1}{\log \frac{1}{\rho}}\log \frac{2\left(2 \rho-1+\frac{\gamma \mu}{4}\right) \left(f(w_0)-f(w^*)\right)}{\epsilon\left(\rho-1+\frac{\gamma \mu}{4}\right)\left(\frac{\gamma \mu^{2}}{4}-2 c_{1}-c_{2}\right)}$.
\end{theorem}
\begin{remark}
For strongly convex problems, given the assumptions and  parameters in corresponding theorems, the convergence rate of VF{${\textbf{B}}^2$}-SGD is $\mathcal{O} (\frac{1}{{\epsilon}}\text{log}(\frac{1}{\epsilon}))$, and those of VF{${\textbf{B}}^2$}-SVRG and VF{${\textbf{B}}^2$}-SAGA are $\mathcal{O} (\text{log}(\frac{1}{\epsilon}))$.
\end{remark}
\subsection{Convergence Analysis for Nonconvex Problem}
\begin{assum}\label{assumnc1}
Nonconvex function $f(w)$ is bounded below,
\begin{equation}\label{inf}
f^*:=\inf_{w\in \mathbb{R}^d} f(w) > -\infty.
\end{equation}
\end{assum}
Assumption \ref{assumnc1} guarantees the feasibility of nonconvex problem (P).
For nonconvex problem, we introduce the notation $K'(t)$  that denotes a set of $q$ iterations fully visiting all coordinates, \ie, $K'(t) = \{\{t, t+\bar{t}_1, \cdots, t+\bar{t}_{q-1}\}: \psi(\{t, t+\bar{t}_1,\cdots, t+\bar{t}_{q-1}\}) = \{1,\cdots,q\}\}$, where the $t$-th global iteration denotes a dominated update. Moreover, these  iterations are performed respectively on a dominator and $q-1$ different collaborators receiving $\vartheta$ calculated at the $t$-th global iteration. Moreover, we assume that $K'(t)$ can be completed in $\eta_2$ global iterations, \ie, for $\forall t' \in \mathcal{A}(t)$, there is $\eta_2 \geq {\text{max}}\{u|u\in K'(t')\}-t'$. Note that, different from $K(t)$, there is $|K'(t)|=q$ and the definition of $K'(t)$ does not emphasize on ``successive iterations'' due to the difference of analysis techniques between strongly convex and nonconvex problems.  Based on $K'(t)$, we introduce the epoch number $v'(t)$ as follow.
\begin{definition}\label{definnc2}
$\mathcal{A}(t)$ denotes a set of global iterations, where for $\forall$ $t' \in \mathcal{A}(t)$ there is the $t'$-th global iteration denoting a dominated update and  $\cup_{\forall t' \in \mathcal{A}(t)} K'(t')=\{0,1,\cdots,t\}$. The epoch number $v'(t)$ is defined as $|\mathcal{A}(t)|$.
\end{definition}
Give the definition of epoch number $v'(t)$, we have the following theoretical results for nonconvex problem.
\begin{theorem}\label{thm-sgdnonconvex}
Under Assumptions~\ref{assum1}-\ref{assum4} and \ref{assumnc1}, to achieve the $\epsilon$-first-order stationary point of problem~\ref{P}, i.e. $\mathbb{E}\|\nabla f(w)\| \le \epsilon$ for stochastic variable $w$, for VF{${\textbf{B}}^2$}-SGD, let
 $\gamma = \frac{ \epsilon}{{L_{*}qG}}$, if $\tau\leq\frac{512qG}{\epsilon^2}$, the total epoch number $T$ should satisfy
 \begin{equation}
   T \geq  {\frac{ {\mathbb{E}\left[ f (w^0) - f^* \right]L_{*}qG }}{\epsilon^2}},
 \end{equation}
 where $L_{*}=\text{max}\{L, \{L_{\ell}\}_{\ell=1}^{q}, L_g\}$, $\tau={\text{max}}\{\tau_1^2,\tau_2^2,\eta_2^2\}$, $f(w^0)$ is the initial function value and $f^*$ is defined in Eq.~\ref{inf}.
\end{theorem}
\begin{theorem}\label{thm-svrgnonconvex}
Under Assumptions~\ref{assum1}-\ref{assum4} and \ref{assumnc1},
to solve problem~\ref{P} with VF{${\textbf{B}}^2$}-SVRG, let $\gamma = \frac{m_0}{L_{*}n^\alpha}$, where $0<m_0<\frac{1}{8}$, $0<\alpha \leq 1$, if epoch number $N$ in an outer loop satisfies  $ N \leq \lfloor  \frac{n^{{\alpha}}}{2m_0}  \rfloor$, and $\tau < \text{min} \{\frac{n^{2\alpha}}{20m_0^2},\frac{1-8m_0}{40m_0^2} \} $, there is
{\begin{equation}
\small{\frac{1}{T}\sum\limits_{s=1}^{S}\sum\limits_{t=0}^{N-1}\mathbb{E}  ||\nabla f(w^s_{t_0})||^2  \leq \frac{L_{*}n^{\alpha}\mathbb{E}\left[  f( w_{0})  -  f( w^{*}) \right] }{T \sigma }},
\end{equation}
}
where $T$ is the total number of epoches, $t_0$ is the start iteration of epoch $t$, $\sigma$ is a small value independent of $n$.
\end{theorem}
\begin{theorem}\label{thm-saganonconvex}
Under Assumptions~\ref{assum1}-\ref{assum4} and \ref{assumnc1},
to solve problem~\ref{P} with VF{${\textbf{B}}^2$}-SAGA, let
$\gamma = \frac{m_0}{L_{*}n^\alpha}$, where $0<m_0<\frac{1}{20}$, $0<\alpha \leq 1$, if total  epoch number $T$ satisfies $T \leq \lfloor  \frac{n^{{\alpha}}}{4m_0}  \rfloor$
and  $\tau < \text{min} \{\frac{n^{2\alpha}}{180m_0^2},\frac{1-20m_0}{40m_0^2} \}$, there is
\begin{equation}
  \frac{1}{T}\sum\limits_{t=0}^{T-1}\mathbb{E}  ||\nabla f(w_{t_0})||^2  \leq \frac{L_{*}n^{\alpha}\mathbb{E}\left[  f( w_{0})  -  f( w^{*}) \right] }{T \sigma }.
\end{equation}
\end{theorem}
\begin{remark}
For  nonconvex problems, given conditions in the theorems, the convergence rate of VF{${\textbf{B}}^2$}-SGD is $\mathcal{O} (1/{\sqrt{T}})$, and those of VF{${\textbf{B}}^2$}-SVRG and VF{${\textbf{B}}^2$}-SAGA are  $\mathcal{O} (1/{T})$.
\end{remark}
\vspace{-0.4cm}
\section{Security Analysis}\label{securitysec}
 We discuss the data security and model security of VF{${\textbf{B}}^2$} under two semi-honest threat models commonly used in security analysis \cite{cheng2019secureboost,xu2019hybridalpha,gu2020federated}. Specially, these two threat models have different threat abilities, where threat model 2 allows collusion between parties while threat model 1 does not.
\begin{figure}[!t]
	\centering
	\begin{subfigure}{0.3\linewidth}
		\includegraphics[width=\linewidth]{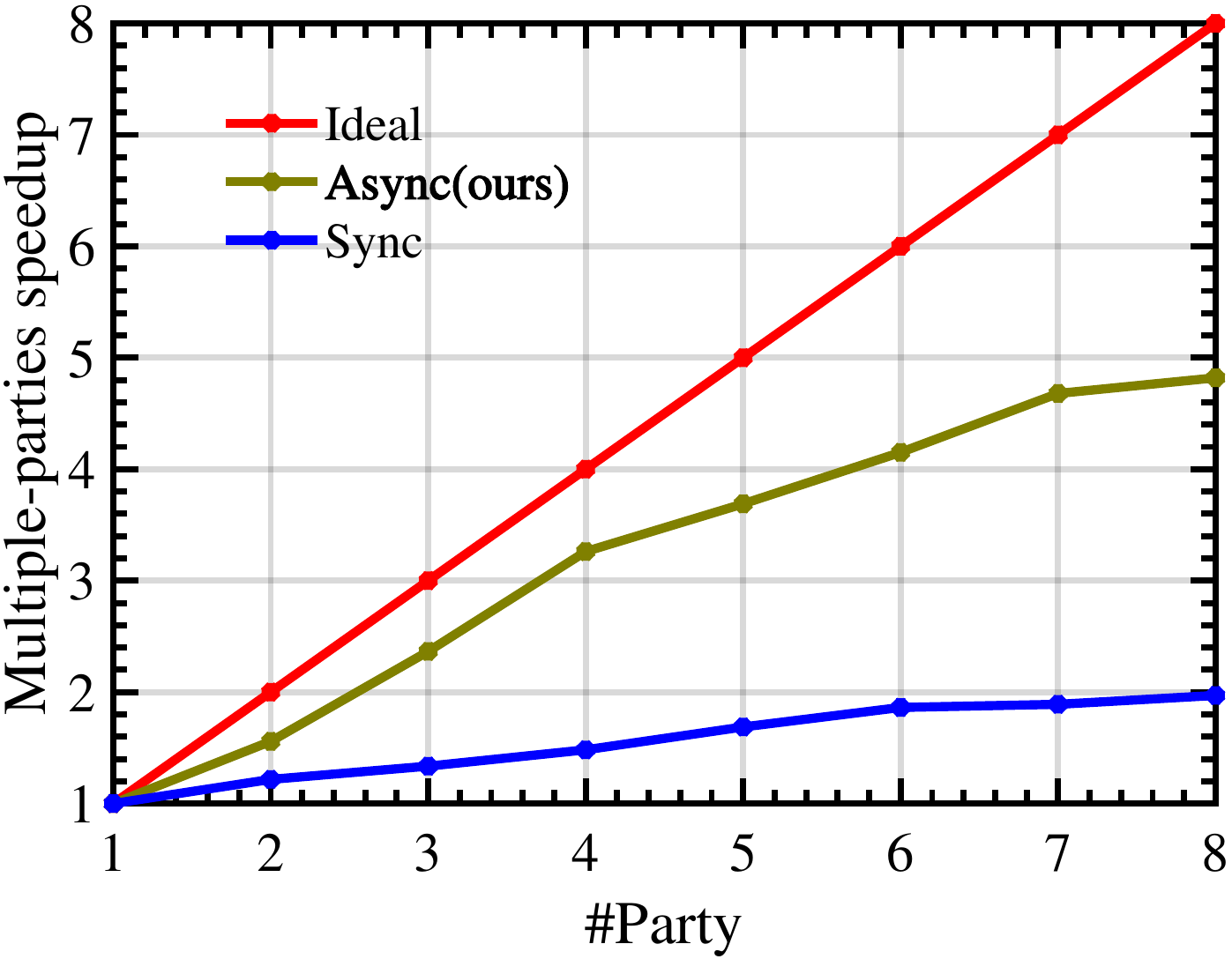}
		\caption{SGD-based}
	\end{subfigure}
\begin{subfigure}{0.3\linewidth}
		\includegraphics[width=\linewidth]{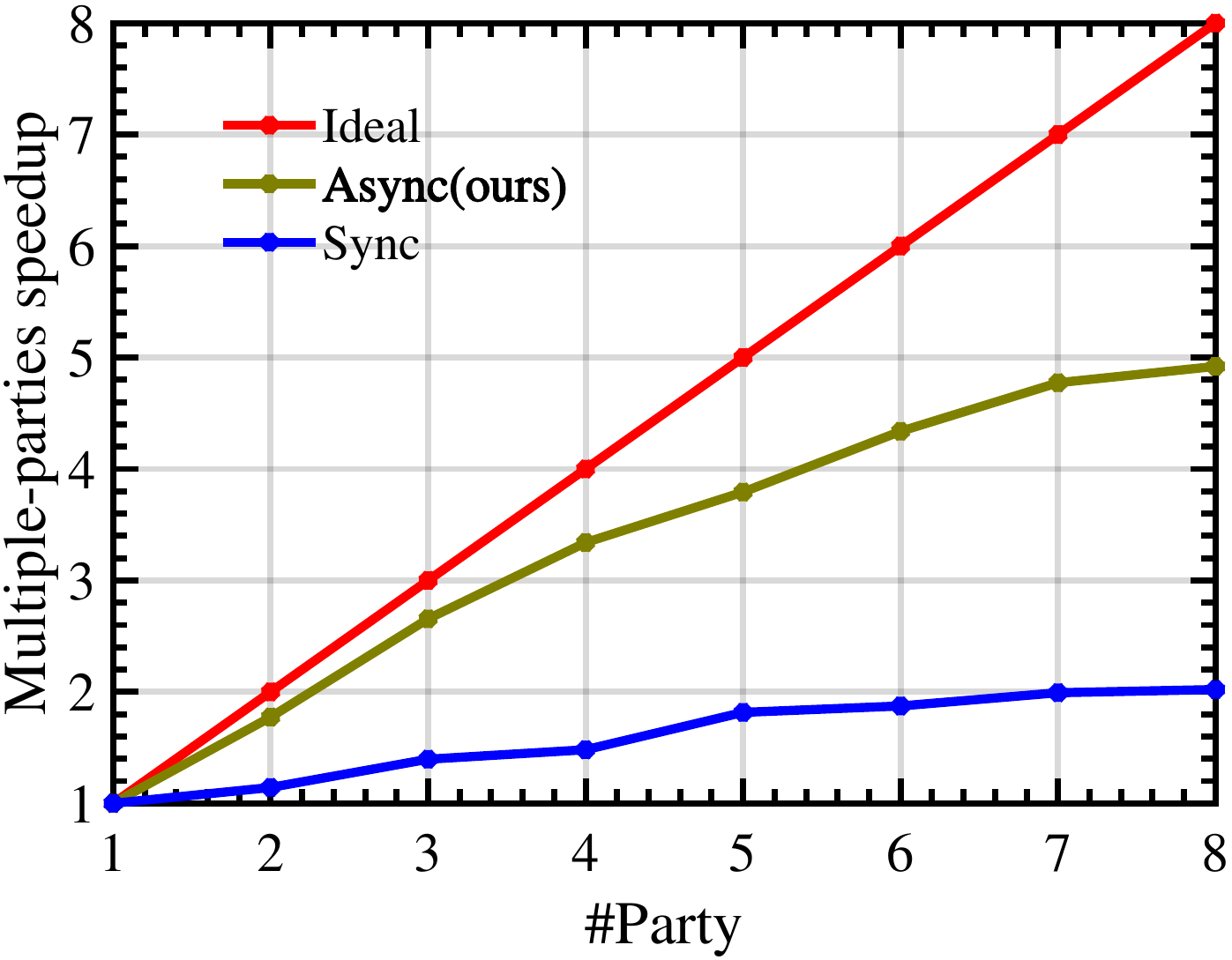}
		\caption{SVRG-based}
	\end{subfigure}
\begin{subfigure}{0.3\linewidth}
		\includegraphics[width=\linewidth]{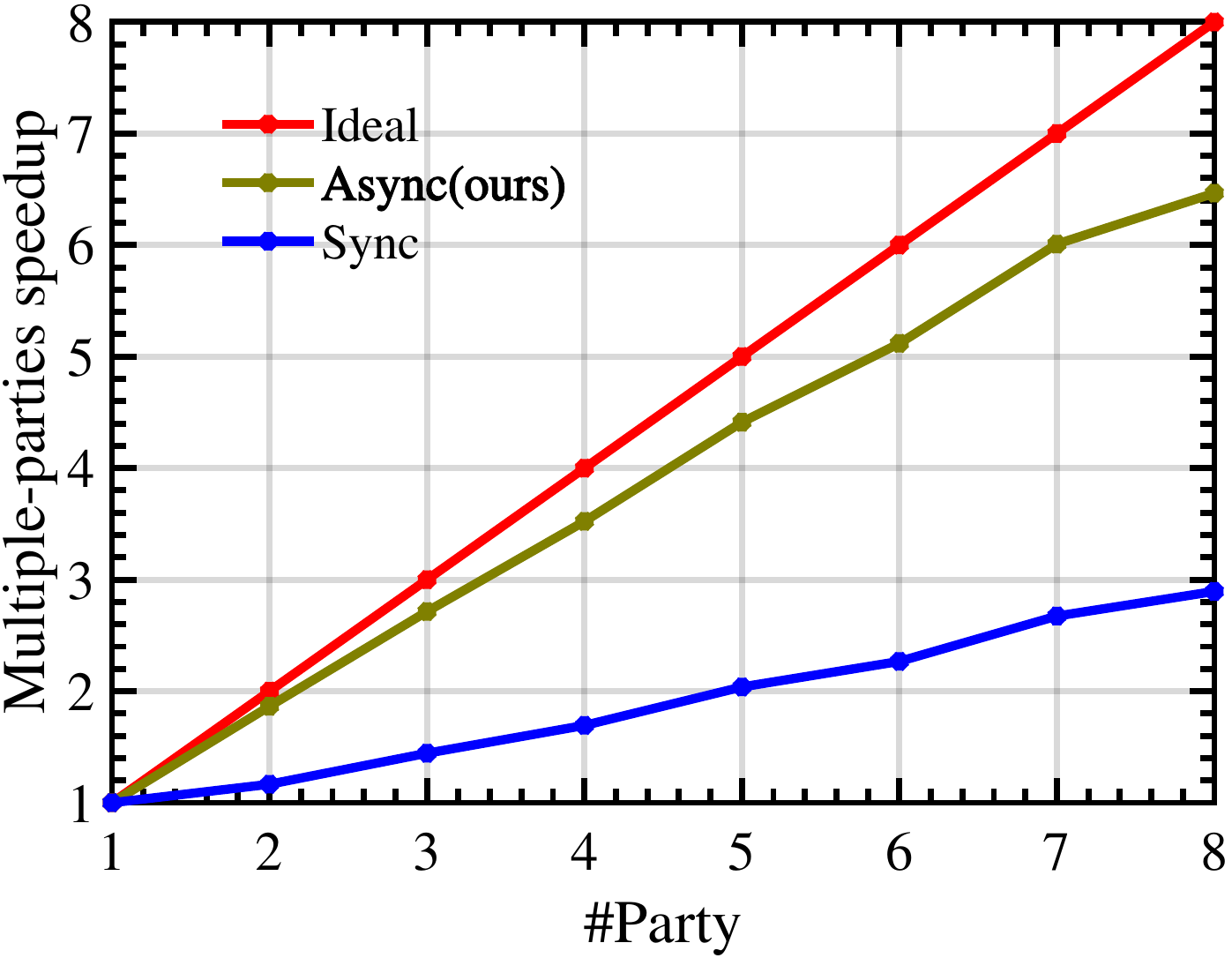}
		\caption{SAGA-based}
	\end{subfigure}%
	\caption{ $q$-parties speedup scalability  with $m=2$ on  $D_4$.}
\label{Exp-sca}
\end{figure}
\begin{itemize}
  \item {\bf{Honest-but-curious}} (threat model 1): All workers will follow the algorithm to perform the correct computations. However, they may use their own retained records of the intermediate computation result to infer other worker's data and model.
  \item {\bf{Honest-but-colluding}} (threat model 2): All workers will follow the algorithm to perform the correct computations. However, some workers may collude to infer other worker's data and model by sharing their retained records of the intermediate computation result.
\end{itemize}
Similar to \cite{gu2020federated}, we prove the security of VF{${\textbf{B}}^2$} by analyzing and proving its ability to prevent inference attack defined as follows.
\begin{definition}[Inference attack]\label{definatt1}
 An inference attack on the $\ell$-th party is to infer $(x_i)_{\mathcal{G}_\ell}$ (or $w_{\mathcal{G}_\ell}$) belonging to other parties or $y_i$ hold by active parties without directly accessing them.
\end{definition}
\begin{lemma}\label{infinite}
Given an equation $o_i = w_{\mathcal{G}_{\ell }}^{\top} {(x_i)}_{\mathcal{G}_{\ell }}$ or $o_i = \frac{\partial \mathcal{L}\left(\widehat{w}^{\top} x_{i}, y_{i}\right)}{\partial\left(\widehat{w}^{\top} x_{i}\right)}$  with only $o_i$ being known, there are infinite different solutions to this equation.
\end{lemma}
The proof of  lemma~\ref{infinite} is shown in the arXiv version. Based on lemma~\ref{infinite}, we obtain the following theorem.
\begin{theorem}\label{security}
Under two semi-honest threat models, VF{${\textbf{B}}^2$} can prevent the inference attack.
\end{theorem}
\noindent{\bf Feature and model security:}
 During the aggregation, the value of $o_i = w_{\mathcal{G}_{\ell }}^{\top} {(x_i)}_{\mathcal{G}_{\ell }}$ is masked by $\delta_{\ell }$ and  just the value of  $w_{\mathcal{G}_{\ell }}^{\top} {(x_i)}_{\mathcal{G}_{\ell }}+ \delta_{\ell }$ is transmitted. Under threat model 1, one even can not access the true value of $o_i$, let alone using relation $o_i = w_{\mathcal{G}_{\ell }}^{\top} {(x_i)}_{\mathcal{G}_{\ell}}$ to refer $w_{\mathcal{G}_{\ell }}^{\top}$ and ${(x_i)}_{\mathcal{G}_{\ell}}$.  Under threat model 2, the random value $\delta_{\ell }$ has risk of being removed from  term $w_{\mathcal{G}_{\ell }}^{\top} {(x_i)}_{\mathcal{G}_{\ell }}+ \delta_{\ell }$ by colluding with other parties. Applying lemma~\ref{infinite} to this circumstance, and we have that even if the random value is removed  it is still impossible to exactly refer $w_{\mathcal{G}_{\ell }}^{\top}$ and ${(x_i)}_{\mathcal{G}_{\ell }}$. Thus, the aggregation process can prevent inference attack under two semi-honest threat models.\\
{\bf Label security}: When analyze the security of label, we do not consider the collusion between active parties and passive parties, which will make preventing labels from leaking meaningless. In the backward updating process, if a passive party $\ell$ wants to infer $y_i$ through the received $\vartheta$, it must solve the equation
 $\vartheta = \frac{\partial \mathcal{L}\left(\widehat{w}^{\top} x_{i}, y_{i}\right)}{\partial\left(\widehat{w}^{\top} x_{i}\right)}$. However, only $\vartheta$ is known to party $\ell$. Thus, following from lemma~\ref{infinite}, we have that it is impossible to exactly infer the labels. Moreover, the collusion between passive parties has no threats to the security of labels.  Therefore, the backward updating can prevent inference attack under two semi-honest threat models.

 From above analyses, we have that the feature security, label security and model security are guaranteed in VFB$^{2}$.
\section{Experiments}\label{secexp}
In this section, extensive experiments are conducted to demonstrate the efficiency, scalability and losslessness of our algorithms.  More experiments are presented in the arXiv version.\\
\noindent {\bf Experiment Settings:} All experiments are implemented on a machine with four sockets, and each sockets has 12 cores. To simulate the environment with multiple machines (or parties), we arrange an extra thread for each party to schedule its $k$ threads and support communication with (threads of) the other parties. We use MPI to implement the communication scheme. The data are partitioned  vertically and randomly  into $q$  non-overlapped parts with nearly equal number of features. The number of threads within each parties, \ie \ $k$, is set as $m$.  We use the training dataset or randomly select 80\% samples  as the training data, and the testing dataset or the rest as the testing data. An optimal learning rate $\gamma$ is chosen  from $\{5e^{-1},1e^{-1},5e^{-2},1e^{-2},\cdots\}$  with regularization coefficient  $\lambda=1e^{-4}$ for all experiments.
\begin{table}[!t]
\centering
\begin{tabular}{@{}ccccc@{}}
\toprule
\multirow{2}{*}{}
 & \multicolumn{2}{c}{Financial} & \multicolumn{2}{c}{Large-Scale}  \\ \cmidrule(l){2-3} \cmidrule(l){4-5}
 & $D_1$ & $D_2$ & $D_3$ & $D_4$ \\
 \midrule
\#Samples & 24,000 & 96,257 & 17,996 & 175,000  \\
\#Features & 90 & 92 & 1,355,191 & 16,609,143 \\ \bottomrule
\end{tabular}
\caption{Dataset Descriptions.}
\label{dataset}
\end{table}

\begin{figure*}[!t]
	\centering
	\begin{subfigure}{0.24\linewidth}
		\includegraphics[width=\linewidth]{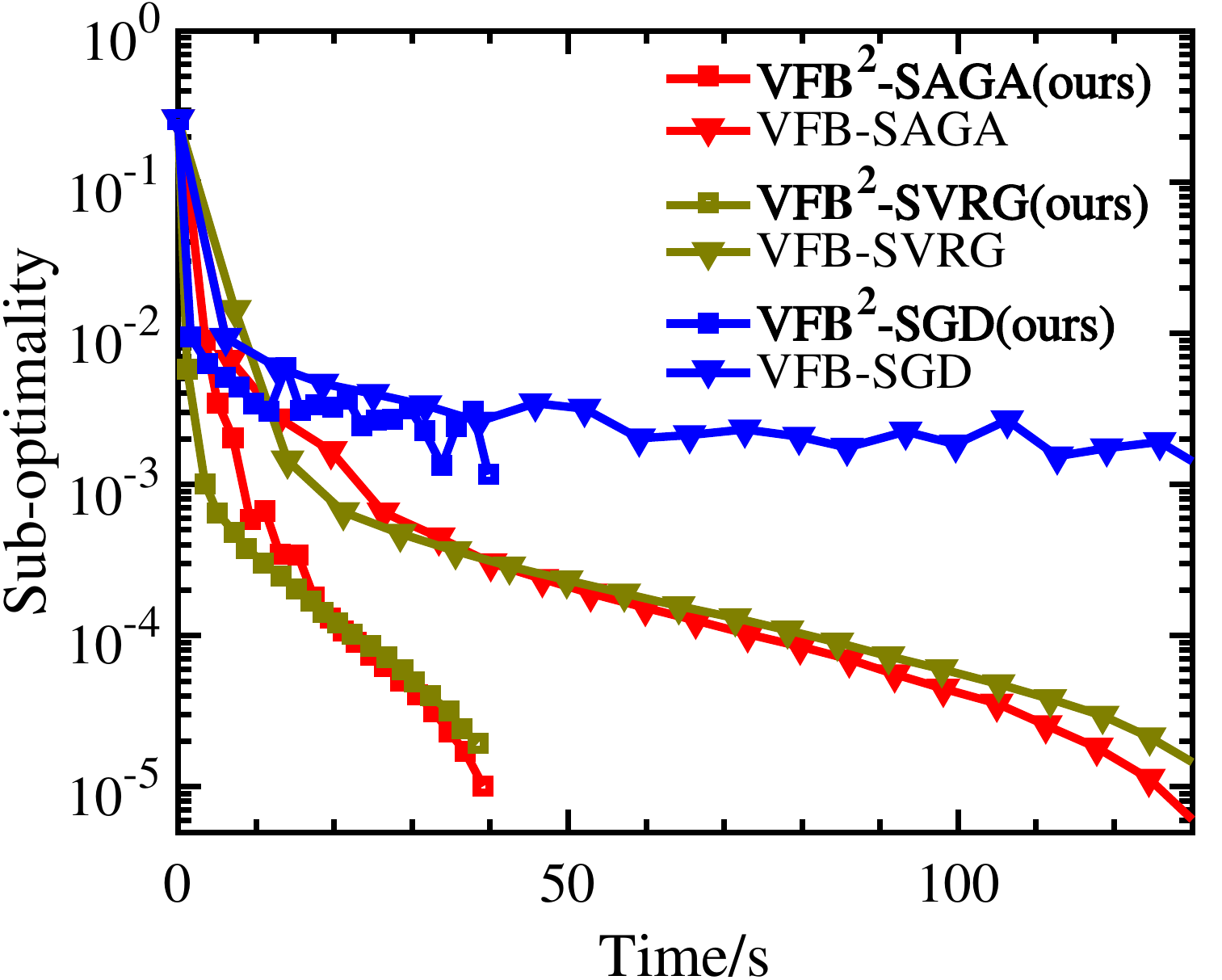}
		\caption{Data: $D_1$}
	\end{subfigure}
\begin{subfigure}{0.24\linewidth}
		\includegraphics[width=\linewidth]{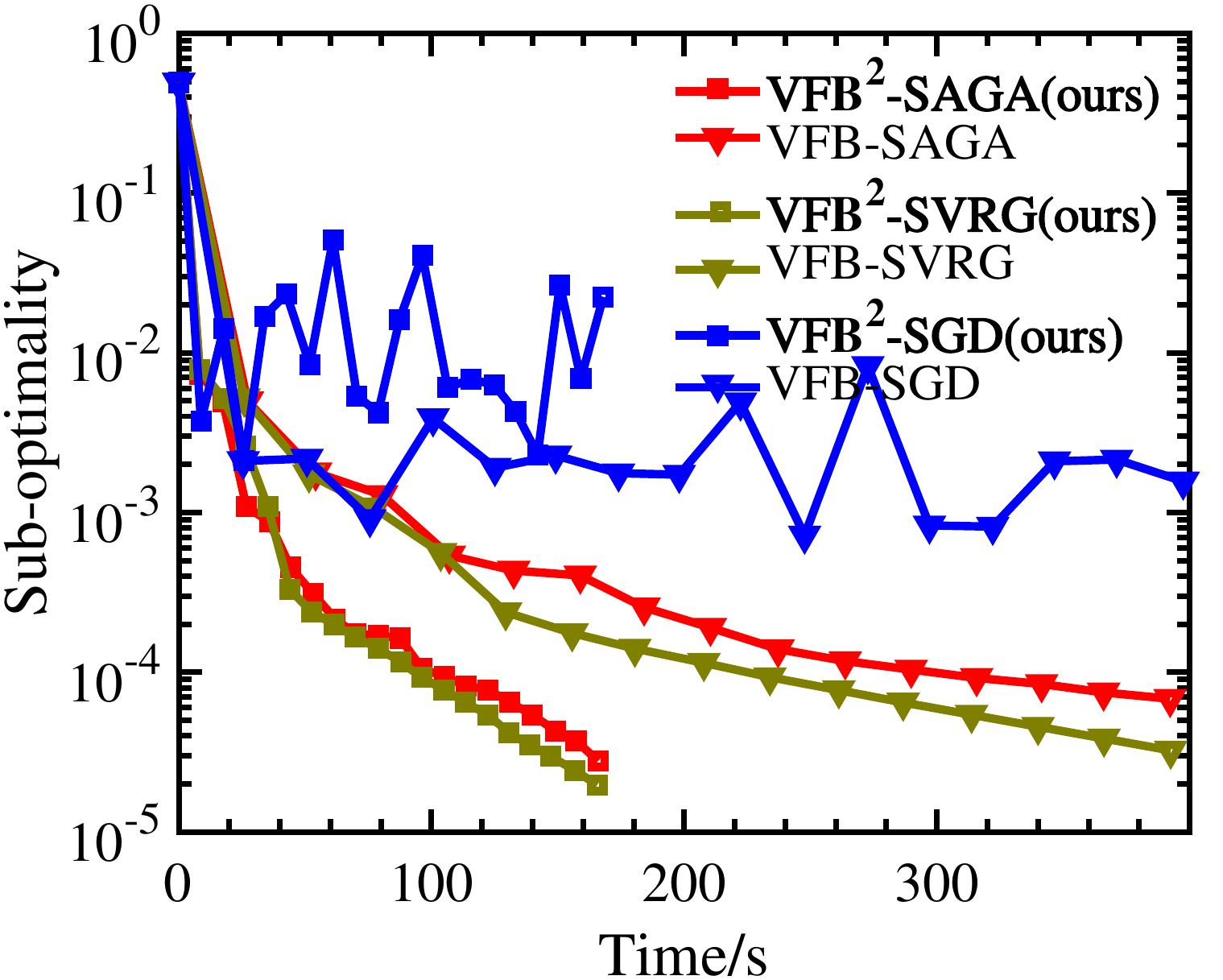}
		\caption{Data: $D_2$}
	\end{subfigure}
\begin{subfigure}{0.24\linewidth}
		\includegraphics[width=\linewidth]{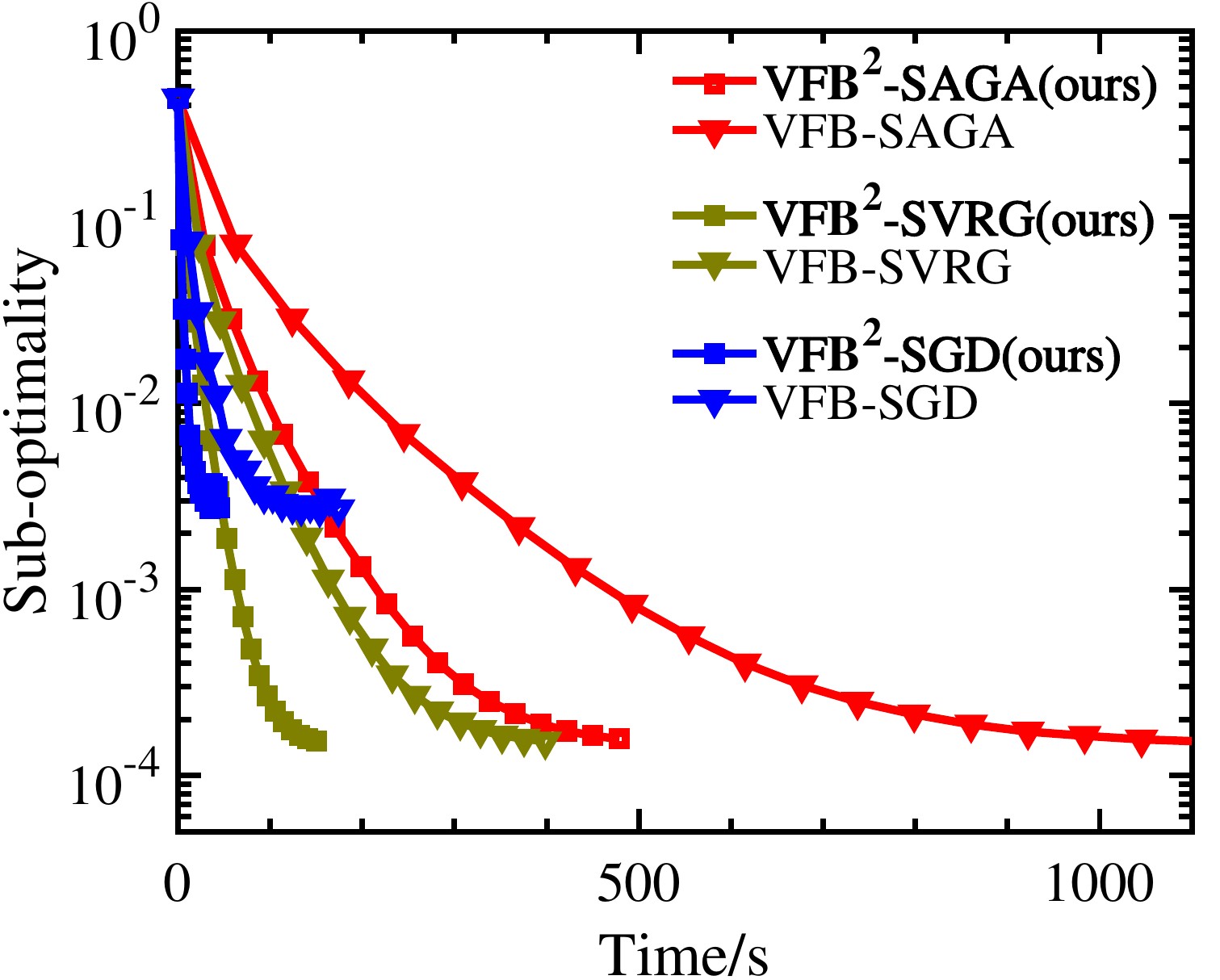}
		\caption{Data: $D_3$}
	\end{subfigure}%
	\begin{subfigure}{0.24\linewidth}
		\includegraphics[width=\linewidth]{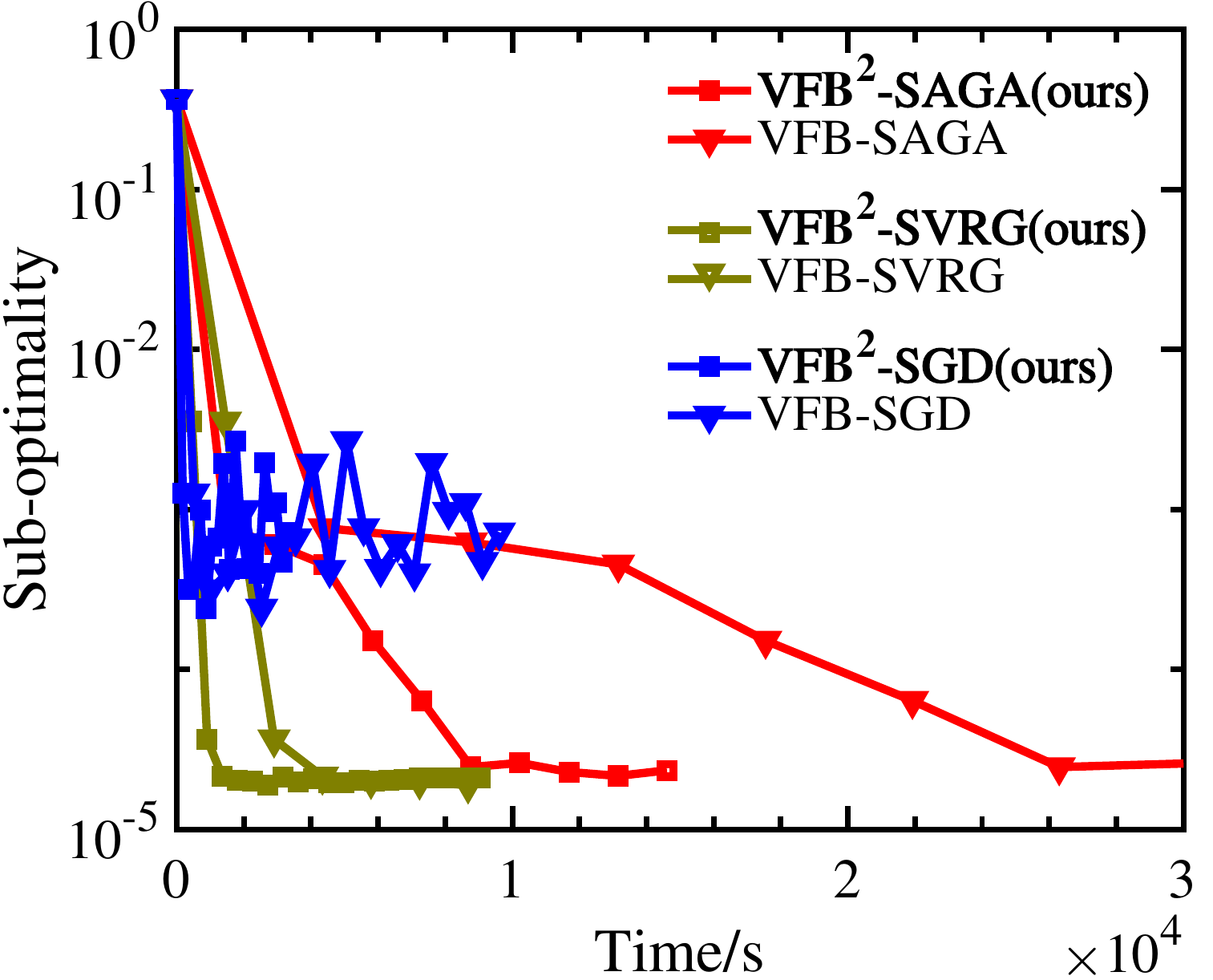}
		\caption{Data: $D_4$}
	\end{subfigure}
	\caption{Results  for solving $\mu$-strongly convex VFL models (Problem \ref{P1}), where the number of epoches (points) denotes how many passes over the dataset the algorithm makes.}
\label{Exp-con}
\end{figure*}
\begin{figure*}[!t]
	\centering
	\begin{subfigure}{0.24\linewidth}
		\includegraphics[width=\linewidth]{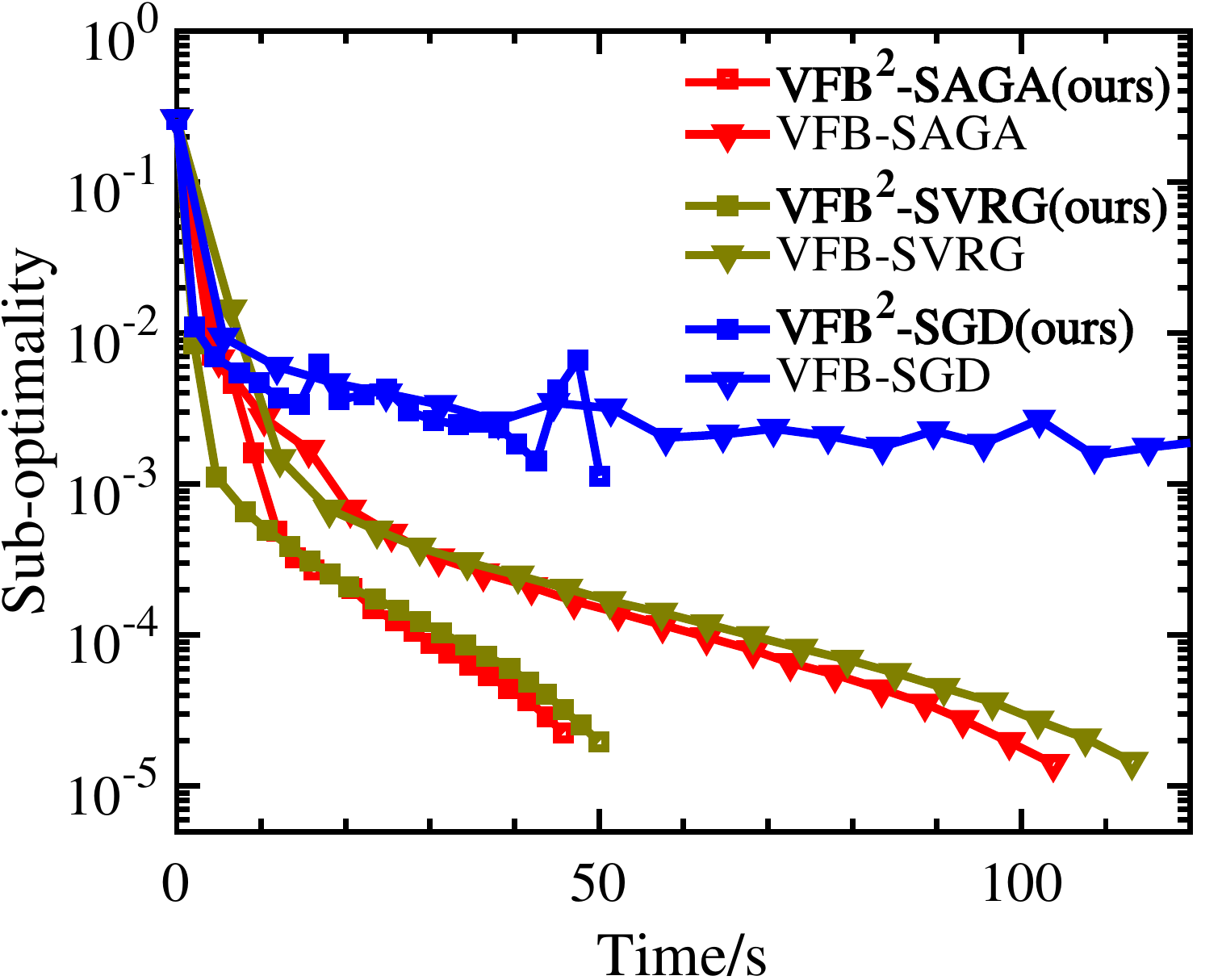}
		\caption{Data: $D_1$}
	\end{subfigure}
\begin{subfigure}{0.24\linewidth}
		\includegraphics[width=\linewidth]{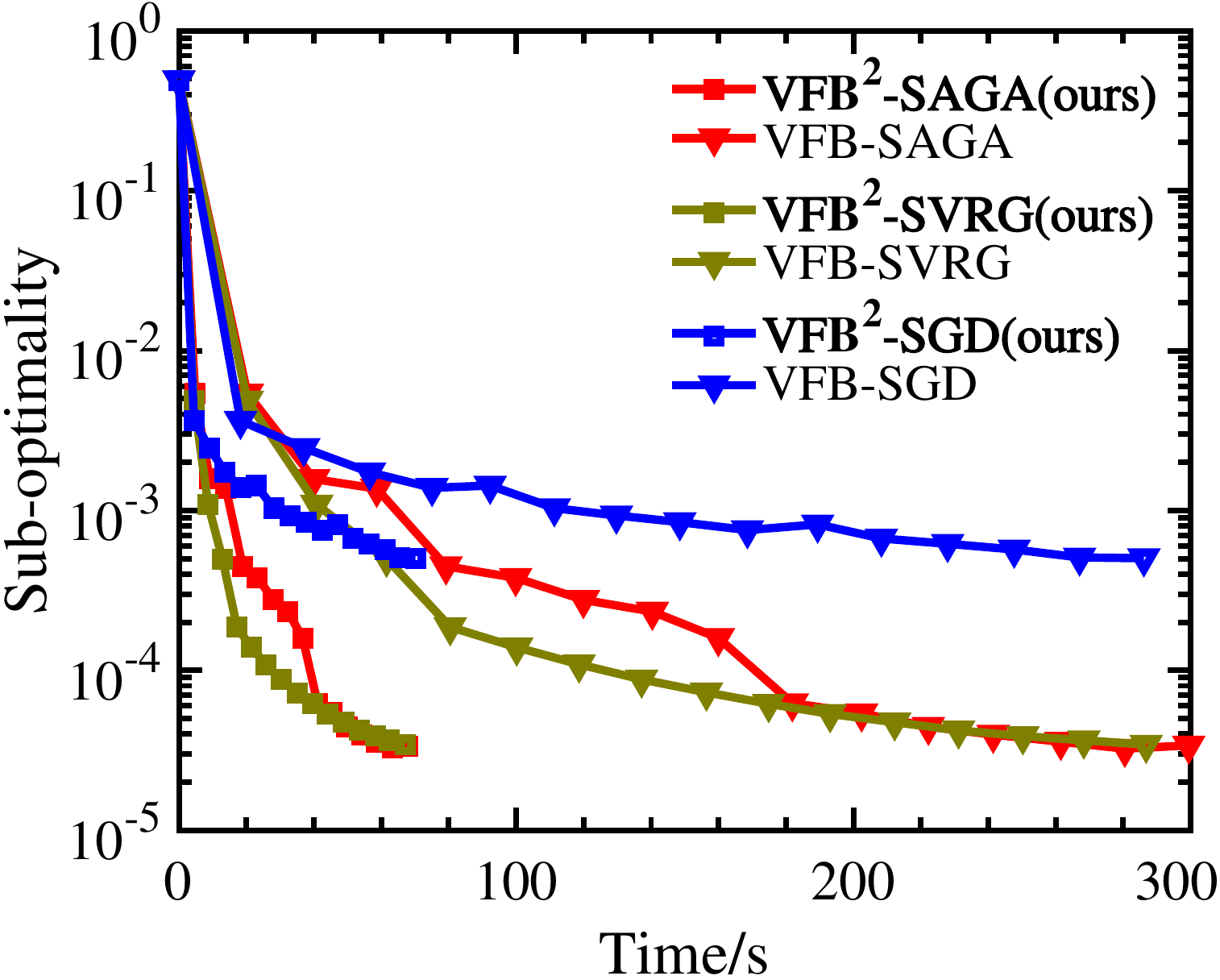}
		\caption{Data: $D_2$}
	\end{subfigure}
\begin{subfigure}{0.24\linewidth}
		\includegraphics[width=\linewidth]{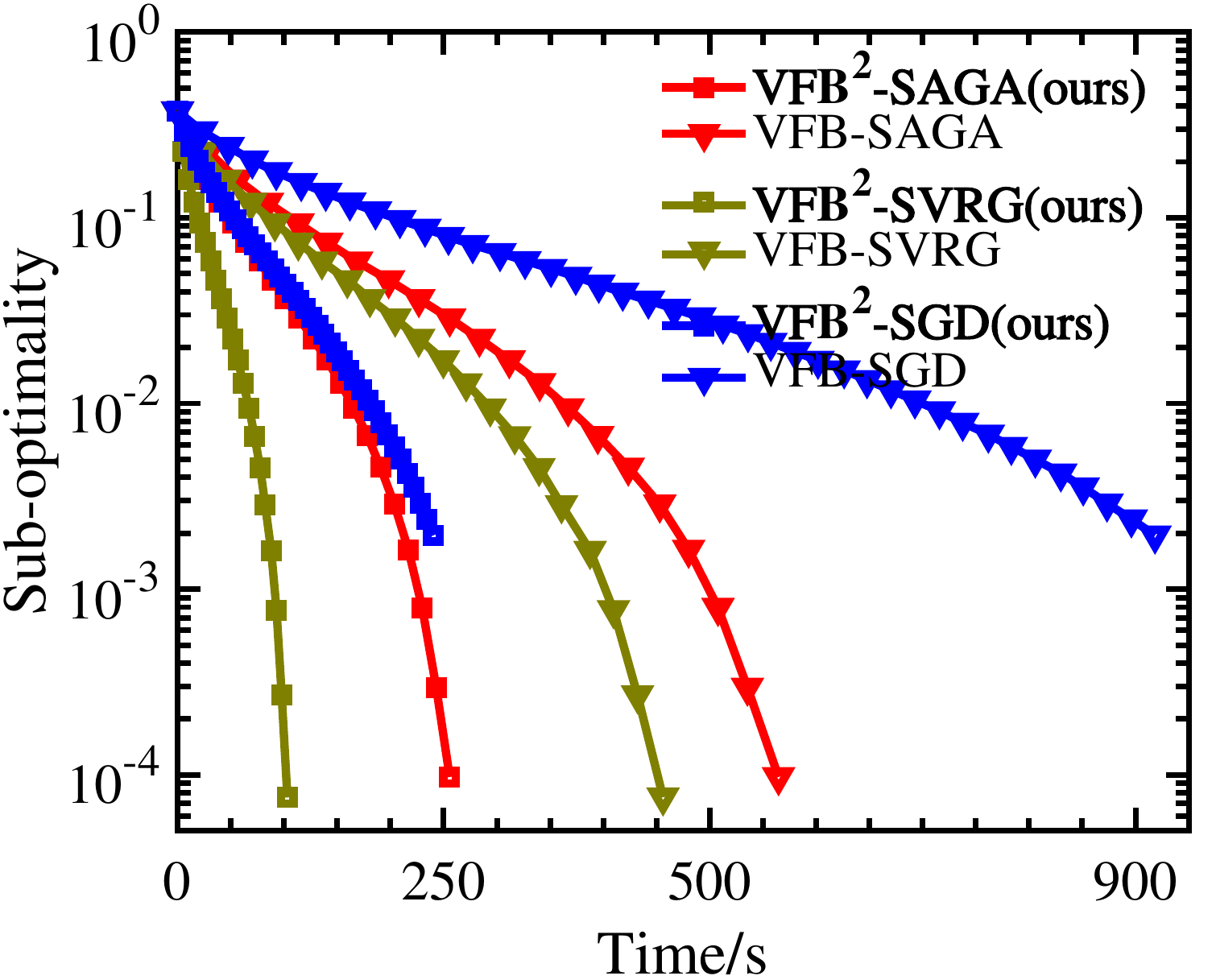}
		\caption{Data: $D_3$}
	\end{subfigure}%
	\begin{subfigure}{0.24\linewidth}
		\includegraphics[width=\linewidth]{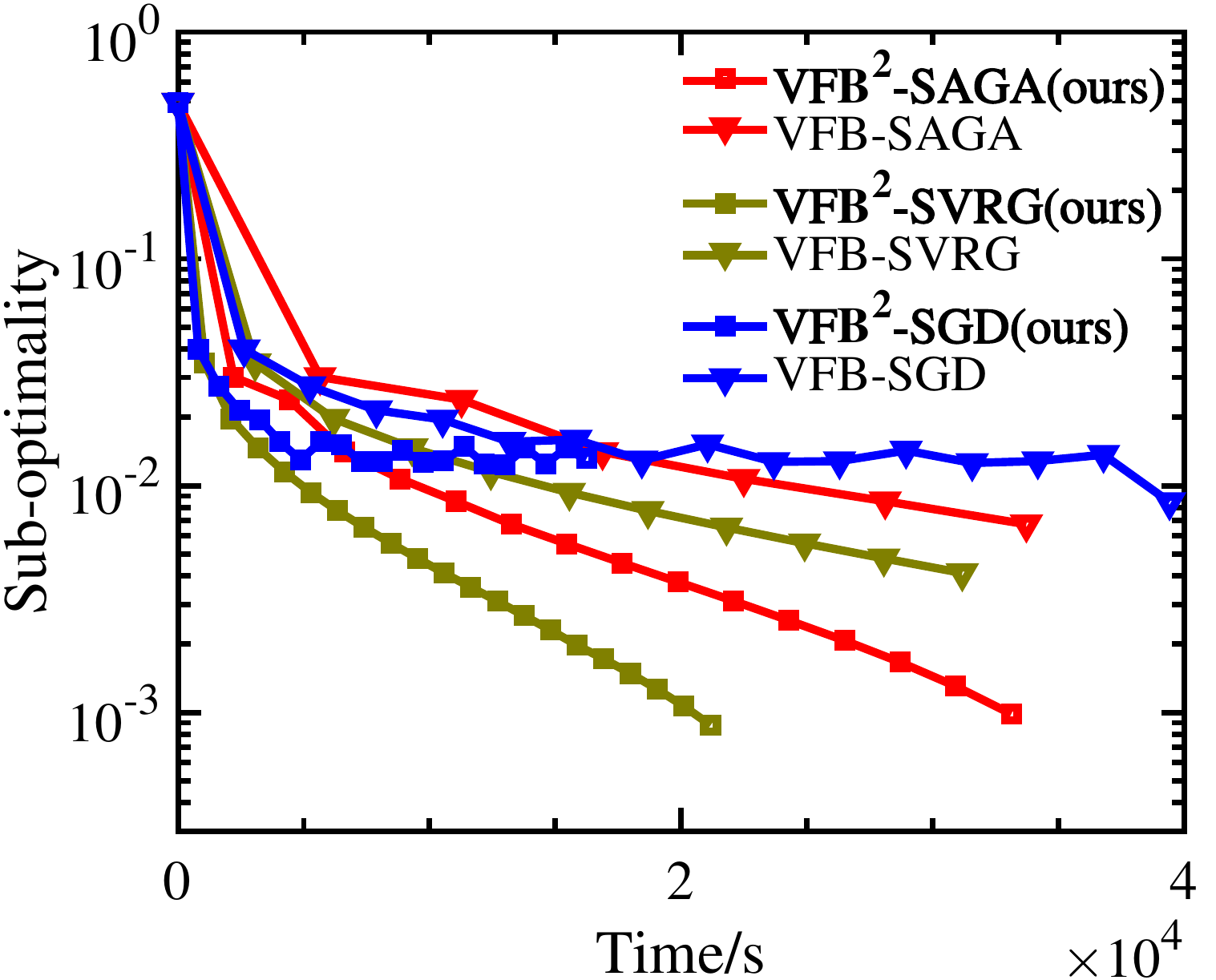}
		\caption{Data: $D_4$}
	\end{subfigure}
	\caption{Results for solving nonconvex VFL models (Problem \ref{P2}), where the number of epoches (points) denotes how many passes over the dataset the algorithm makes.}
\label{Exp-ncon}
\end{figure*}

\noindent {\bf Datasets:} We use four classification datasets summarized in Table~\ref{dataset} for evaluation. Especially, $D_1$ (UCICreditCard) and $D_2$ (GiveMeSomeCredit) are the real financial datsets from the Kaggle website\footnote{\url{https://www.kaggle.com/datasets}}, which can be used to demonstrate the ability to address real-world tasks; $D_3$ (news20) and $D_4$ (webspam) are the large-scale ones from the LIBSVM \cite{chang2011libsvm} website\footnote{\url{https://www.csie.ntu.edu.tw/cjlin/libsvmtools/datasets/}}. Note that we apply one-hot encoding to categorical features of $D_1$ and $D_2$ , thus the number of features become 90 and 92, respectively. \\
\noindent {\bf Problems:} We consider $\ell_2$-norm regularized logistic regression problem for $\mu$-strong convex case
\begin{equation}\label{P1}
\min_{w \in \mathbb{R}^d} f(w):=\frac{1}{n} \sum_{i=1}^{n}  {\text{log}}(1+e^{-y_iw^{\top} x_i}) + \frac{\lambda}{2} \|w\|^2,
\end{equation}
and the nonconvex logistic regression  problem
\begin{equation}\label{P2}
\min_{w \in \mathbb{R}^d} f(w):=\frac{1}{n} \sum_{i=1}^{n}  {\text{log}}(1+e^{-y_iw^{\top} x_i}) + \frac{\lambda}{2} \sum_{i=1}^{d} \frac{w_i^2}{1+w_i^2} \nonumber.
\end{equation}

\subsection{Evaluations of Asynchronous Efficiency and Scalability}
To demonstrate the asynchronous efficiency, we introduce the synchronous counterparts of our algorithms (\emph{i.e.}, synchronous {VF}L algorithms with BUM, denoted as {VF\bf{B}}) for comparison.  When implementing the synchronous algorithms, there is a synthetic straggler party which may be 30\% to 50\% slower than the faster party to simulate the real application scenario with unbalanced computational resource.\\
{\noindent {\bf Asynchronous Efficiency:}} In these experiments, we set $q=8$, $m=3$ and fix the $\gamma$ for algorithms with a same SGD-type but in different parallel fashions.
As  shown in Figs.~\ref{Exp-con} and \ref{Exp-ncon}, the loss v.s. run
time curves demonstrate that our algorithms consistently outperform their synchronous counterparts regarding the efficiency.
\begin{table*}[!t]
\centering
\begin{tabular}{@{}cccccccc@{}}
\toprule
&Algorithm& $D_1$ & $D_2$ & $D_3$ & $D_4$ \\ \midrule
\multirow{3}{*}{Problem  (\ref{P1})}
 & NonF & 81.96\%$\pm$0.25\% & 93.56\%$\pm$0.19\% & 98.29\%$\pm$0.21\% & 92.17\%$\pm$0.12\% \\
  & AFSVRG-VP & 79.35\%$\pm$0.19\% & 93.35\%$\pm$0.18\% & 97.24\%$\pm$0.11\% & 89.17\%$\pm$0.10\% \\
 &{\bf{ Ours}} & 81.96\%$\pm$0.22\% & 93.56\%$\pm$0.20\% & 98.29\%$\pm$0.20\% & 92.17\%$\pm$0.13\% \\
 \midrule
\multirow{3}{*}{Problem (\ref{P2})}
& NonF & 82.03\%$\pm$0.32\% & 93.56\%$\pm$0.25\% & 98.45\%$\pm$0.29\% & 92.71\%$\pm$0.24\% \\
 & AFSVRG-VP & 79.36\%$\pm$0.24\% & 93.35\%$\pm$0.22\% & 97.59\%$\pm$0.13\% & 89.98\%$\pm$0.14\% \\
 & {\bf{ Ours}} & 82.03\%$\pm$0.34\% & 93.56\%$\pm$0.24\% & 98.45\%$\pm$0.33\% & 92.71\%$\pm$0.27\% \\
 \bottomrule
\end{tabular}
\caption{Accuracy of different algorithms to evaluate the losslessness  of our algorithms (10 trials).}
\label{exp-lossless}
\end{table*}

Moreover, from the perspective of  loss v.s. epoch number, we have that algorithms based on SVRG and SAGA have the better convergence rate than that of SGD-based algorithms which is consistent to the theoretical results.\\
{\noindent {\bf Asynchronous Scalability:}} We also consider the asynchronous speedup scalability in terms of the number of total parties $q$. Given a fixed $m$, $q$-parties speedup is defined as
\begin{equation}
  \text{$q$-parties speedup} =\frac{\text{Run time  of using  1 party}}{\text{Run time of using $q$ parties}},
\end{equation}
where run time is defined as time spending on reaching a certain precision of sub-optimality, \ie, $1e^{-3}$ for $D_4$. We implement experiment for {Problem (\ref{P2})}, results of which are shown in Fig.~\ref{Exp-sca}. As depicted in Fig.~\ref{Exp-sca}, our asynchronous algorithms has much better $q$-parties speedup scalability than synchronous ones and can achieve near linear speedup.
\subsection{Evaluation of Losslessness}
To demonstrate the losslessness of our algorithms, we compare VF${ {\textbf{B}}}^2$-SVRG with its non-federated (NonF) counterpart (all data are integrated together for modeling) and ERCR based algorithm but without BUM, \ie, AFSVRG-VP proposed in \cite{gu2020Privacy}. Especially, AFSVRG-VP also uses distributed SGD method but can not optimize the parameters corresponding to passive parties due to lacking labels. When implementing AFSVRG-VP, we assume that only half parties have labels, \ie, parameters corresponding to the features held by the other parties are not optimized. Each comparison is repeated 10 times with $m=3$, $q=8$, and a same stop criterion, \emph{e.g.}, $1e^{-5}$ for $D_1$. As shown  in Table~\ref{exp-lossless}, the accuracy of our algorithms are the same with those of NonF algorithms and are much better than those of AFSVRG-VP, which are consistent to our claims.

\section{Conclusion}
In this paper, we proposed a novel backward updating mechanism for the real VFL system where only one or partial parties have labels for training models. Our new algorithms enable all parties, rather than only active parties, to collaboratively update the model and also guarantee the algorithm convergence, which was not held in other recently proposed ERCR based VFL methods under the real-world setting. Moreover, we proposed a bilevel asynchronous parallel architecture to make ERCR based algorithms with backward updating more efficient in real-world tasks. Three practical SGD-type of algorithms were also proposed with theoretical guarantee.
\bibliography{ref}
\bibliographystyle{aaai21}
\clearpage
\onecolumn
\appendix 
\section*{\LARGE{Supplementary Materials}}
We present the related supplements in following sections.
\section{Explanation of the Bilevel Asynchronous Parallel Architecture }
When $m=1$, we just need to set the number of threads within each party as 1, then Bilevel Asynchronous Parallel Architecture (BAPA) reduces to a parallel architecture with multiple parties.  While, the updates on passive parties rely on the $\vartheta$ received from the only active party. In this case, the BAPA behaves likely (just behaves likely not the same as) the server-worker distributed-memory architecture in \cite{huo2017asynchronous} for there is a communication delay between the active and passive parties. The difference is that in our BAPA with $m=1$ the worker (\ie, passive parties $\ell$ in our BAPA) passively send the local $w_{\mathcal{G}_{\ell}}^{\top} {(x_i)}_{\mathcal{G}_{\ell}}$ to the other parties when $w^Tx_i$ is required instead of just actively sending the local $w_{\mathcal{G}_{\ell}}^{\top} {(x_i)}_{\mathcal{G}_{\ell}}$ to the only server (\ie, active in our BAPA). When $m=q$, then all parties hold labels and the BAPA reduces to the general shared-memory parallel architecture from the perspective of analysis.
\section{Supplements Related to Tree-Structured Communication}
\subsection{The definition and illustration of  totally different tree structures}
First, we present the definition of significantly different tree structures mentioned at step 5 in Algorithm \ref{safer_tree}.
\begin{definition}[Two significantly different tree structures\cite{gu2020federated}]\label{definatt3}
For two tree structures $T_1$ and $T_2$ on all parties $\{1,\cdots,q\}$, they are significantly different if there does not exist a subtree $\widehat{T}_1$ of  $T_1$ and a subtree $\widehat{T}_2$ of  $T_2$ whose size are larger than 1 and smaller than $T_1$ and $T_2$, respectively, such that leaf ($\widehat{T}_1$) = leaf ($\widehat{T}_2$).
\end{definition}
Then we present an illusion of the totally different tree structures in Fig.~\ref{tree_struc}.
\begin{figure}[H]
	\centering	
\begin{subfigure}{0.49\linewidth}
		\includegraphics[width=\linewidth]{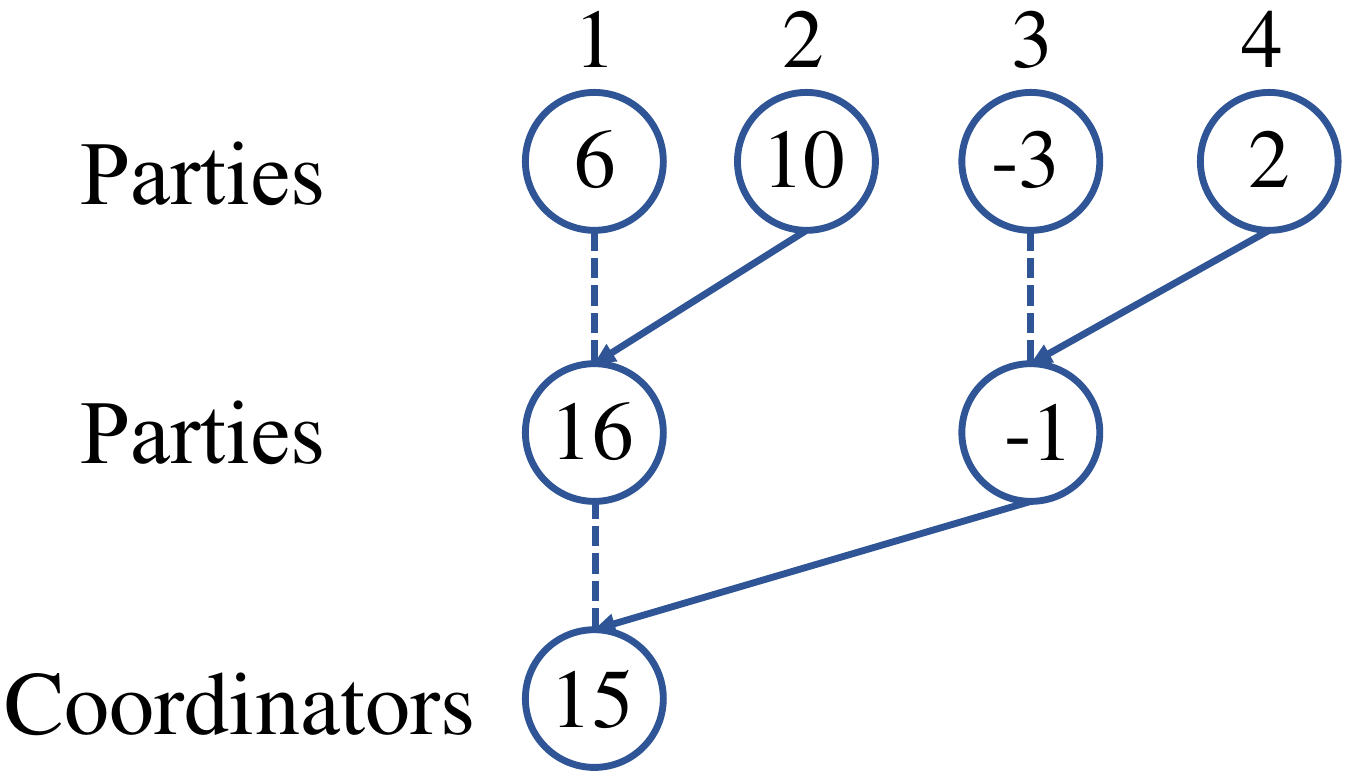}
		\caption{Tree structure $T_1$}
	\end{subfigure}%
	\begin{subfigure}{0.49\linewidth}
		\includegraphics[width=\linewidth]{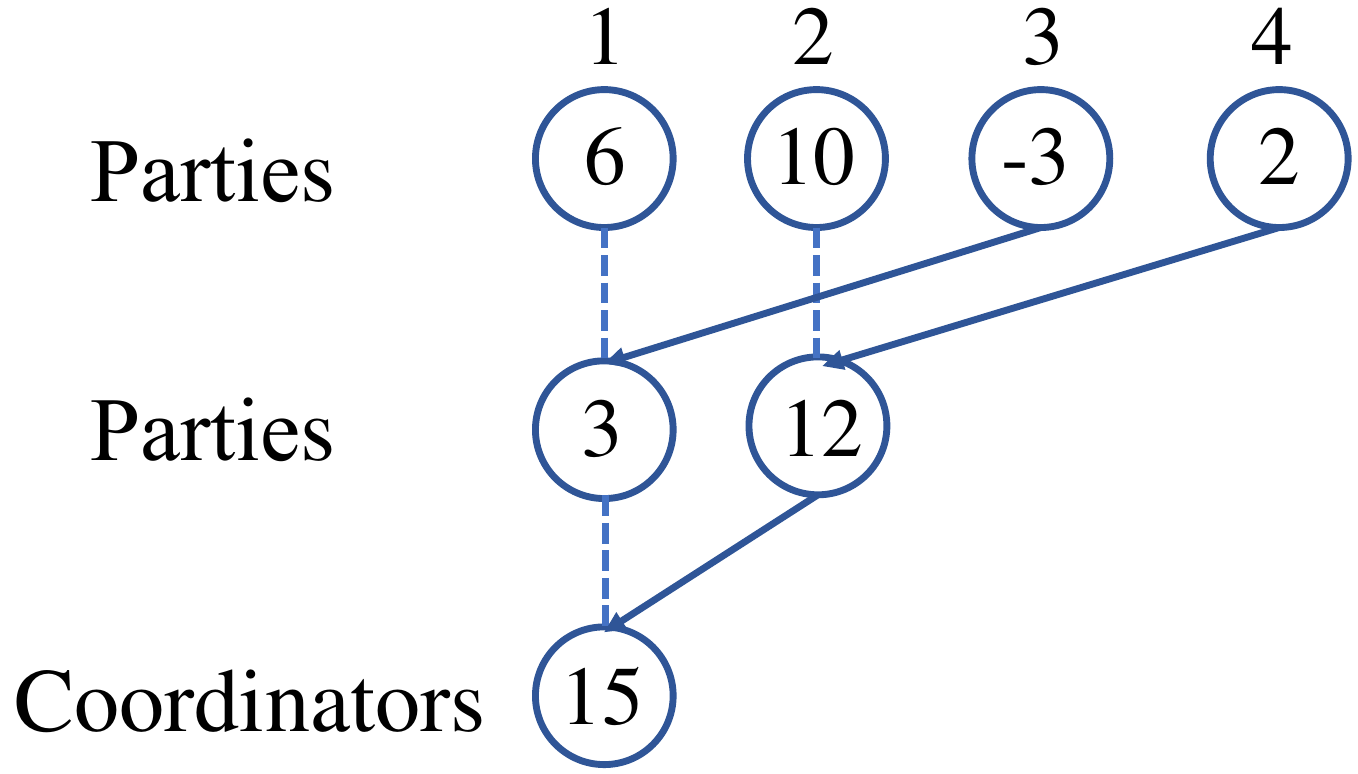}
		\caption{Tree structure $T_2$}
	\end{subfigure}%
	\caption{Illustration of tree-structured communication based on two totally different tree structures $T_1$ and $T_2$.}
\label{tree_struc}
\end{figure}
As depicted in Fig.~\ref{tree_struc} (a), party $1$ aggregates values from parties $1$  and $2$; party $3$ aggregates values from parties $3$ and $4$; and then party $1$, \ie, the aggregator, aggregates these two aggregated values from parties $1$ and $3$.  While, as depicted in Fig.~\ref{tree_struc} (b), party $1$ aggregates values from parties $1$ and $3$; party $2$ aggregates values from parties $2$ and $4$; and then the aggregated values are aggregated from parties $1$ and $2$ to party $1$, \ie, the aggregator. From the aggregation process describe above, it is easily to conclude that aggregation through such significantly different tree structures can prevent the leakage of the random value $\delta_{{\ell}}$ when there are no collusion between parties.
\subsection{An example showing collusion between parties}
Then we present an example to show that collusion between parties can remove the random value  $\delta_{\ell}$ added to $w_{\mathcal{G}_\ell}^{\top} {(x_i)}_{\mathcal{G}_\ell}$. Assume that $\{w_{\mathcal{G}_{\ell'}}^{\top} {(x_i)}_{\mathcal{G}_{\ell '}} + \delta_{\ell'}\}_{{\ell '}=1}^{q}$ are aggregated through tree structure $T_1$ and $\{ \delta_{\ell'}\}_{{\ell '}=1}^{q}$ are aggregated through tree structure $T_2$. In this case, party $3$ knows the value of $w_{\mathcal{G}_{\ell=4}}^{\top} {(x_i)}_{\mathcal{G}_{\ell=4}} + \delta_{\ell=4}$ and  party $2$ knows the value of $\delta_{\ell=4}$. Then if there is collusion between parties 2 and 3,  $\delta_{\ell=4}$  added to party $4$ can be removed from $w_{\mathcal{G}_{\ell=4}}^{\top} {(x_i)}_{\mathcal{G}_{\ell=4}} + \delta_{\ell=4}$.
\subsection{Proof of Lemma \ref{infinite}}
\begin{proof}
  First, we consider the equation $o_i = w_{\mathcal{G}_{\ell }}^{\top} {(x_i)}_{\mathcal{G}_{\ell }}$ with two cases, including $d_{\ell}\geq 2$ and $d_{\ell}=1$. For $\forall d_{\ell}\geq 2$, given an arbitrary non-identity orthogonal matrix $U\in \mathbb{R}^{d_\ell\times d_\ell}$, we have
   \begin{equation}\label{proofinfinite1}
     (w_{\mathcal{G}_{\ell }}^{\top}U^{\top})(U {(x_i)}_{\mathcal{G}_{\ell }}) = w_{\mathcal{G}_{\ell }}^{\top}(U^{\top}U) {(x_i)}_{\mathcal{G}_{\ell }} = w_{\mathcal{G}_{\ell }}^{\top} {(x_i)}_{\mathcal{G}_{\ell }}=o_i
   \end{equation}
 From Eq.~\ref{proofinfinite1}, we have that given an equation $o_i = w_{\mathcal{G}_{\ell }}^{\top} {(x_i)}_{\mathcal{G}_{\ell }}$ with only $o_i$ being known, the solutions corresponding to   $w_{\mathcal{G}_{\ell '}}$ and  ${(x_i)}_{\mathcal{G}_{\ell '}}$ can be represented as $(w_{\mathcal{G}_{\ell '}}^{\top}U^{\top})$ and  $(U {(x_i)}_{\mathcal{G}_{\ell '}})$, respectively. However, $U$ can be an arbitrary different non-identity orthogonal matrix, the solutions are thus infinite. If $d_\ell=1$, give an arbitrary real number $u\neq1$, we have
    \begin{equation}\label{proofinfinite2}
     (w_{\mathcal{G}_{\ell }}^{\top}u)(\frac{1}{u} {(x_i)}_{\mathcal{G}_{\ell }}) = w_{\mathcal{G}_{\ell }}^{\top}(u\frac{1}{u}) {(x_i)}_{\mathcal{G}_{\ell }} =  w_{\mathcal{G}_{\ell }}^{\top} {(x_i)}_{\mathcal{G}_{\ell }}=o_i
   \end{equation}
   Similar to above analysis, we have that the solutions of equation  $o_i = w_{\mathcal{G}_{\ell }}^{\top} {(x_i)}_{\mathcal{G}_{\ell }}$ are infinite when $d_\ell=1$. As for $o_i = \frac{\partial \mathcal{L}\left({w}^{\top} x_{i}, y_{i}\right)}{\partial\left({w}^{\top} x_{i}\right)}$, both ${w}^{\top} x_{i}$ and loss function are unknown , it is thus impossible to exactly infer the $y_{i}$. This completes the proof.
\end{proof}
\section{Detailed Algorithmic Steps of VF{${\textbf{B}}^2$}-SVRG and -SAGA}
In the following, we present the detailed algorithmic steps of VF{${\textbf{B}}^2$}-SVRG and -SAGA.
\subsection{VF{${\textbf{B}}^2$}-SVRG}
The proposed VF${\textbf{B}}^2$-SVRG with an improved convergence rate than VF${\textbf{B}}^2$-SGD is shown in Algorithms~\ref{afsvrg-active} and \ref{afsvrg-passive}. Different from VF${\textbf{B}}^2$-SGD directly using the stochastic gradient for updating, VF${\textbf{B}}^2$-SVRG  adopts the variance reduction technique to control the intrinsic variance of stochastic gradient.  Algorithm~\ref{afsvrg-active} thus computes $\widetilde{v}^{\ell}: = \nabla_{\mathcal{G}_{\ell}} f_{i}(\widehat{w})-\nabla_{\mathcal{G}_{\ell}} f_{i}\left(w^{s}\right)+\nabla_{\mathcal{G}_{\ell}} f\left(w^{s}\right)$. While for Algorithm~\ref{afsvrg-passive}, there is  $\widetilde{v}^{\ell}=\vartheta_1 \cdot (x_i)_{\mathcal{G}_{\ell}} + \nabla g(\widehat{w}_{\mathcal{G}_{\ell }}) - \left(\vartheta_{i,0} \cdot (x_i)_{\mathcal{G}_{\ell}}
  + \nabla g(w_{\mathcal{G}_{\ell}}^{s})\right)
  +\nabla_{\mathcal{G}_{\ell}} f\left(w^{s}\right)$., where $\nabla_{\mathcal{G}_{\ell}} f_{i}\left(w^{s}\right)$ is computed as $\left(\vartheta_2 \cdot (x_i)_{\mathcal{G}_{\ell }}+\nabla_{\mathcal{G}_\ell}g(w_{\mathcal{G}_{\ell}}^{s})\right)$.
\begin{algorithm}[h]
\caption{VF{${\textbf{B}}^2$}-SVRG for active party $\ell$ to actively launch dominated update}\label{afsvrg-active}
\begin{algorithmic}[1]
\REQUIRE {Local data $\{{(x_i)}_{\mathcal{G}_{\ell}},y_i\}_{i=1}^{n}$ stored on the $\ell$-th party, learning rate $\gamma$}.
\STATE Initialize $w_{\mathcal{G}_{\ell}} \in \mathbb{R}^{d_{\ell}}$.\\
\FOR {$s=0, 1, \ldots, S-1$}
\STATE  Compute $\left(w^{s}\right)^{\top} x_{i}$ for $i=1,\cdots,n$ based on Algorithm~\ref{safer_tree}.
  \STATE  Compute $\vartheta_{0,i} = \frac{\partial \mathcal{L}\left(({w}^s)^{\top} x_{i}  y_{i}\right)}{\partial\left(({w}^s)^{\top} x_{i}\right)}$ for $i=1, \cdots, n$ and the full local gradient $\nabla_{\mathcal{G}_{\ell}} f\left(w^{s}\right)=\frac{1}{n} \sum_{i=1}^{n} \nabla_{\mathcal{G}_{\ell}} f_{i}\left( {w}^{s}\right)$ , and then distribute all $\vartheta_0$ to the rest parties.
  \STATE $w_{\mathcal{G}_{\ell}}=w_{\mathcal{G}_{\ell}}^{s}$.\\
  { \bf{Keep doing in parallel (distributed-memory parallel for multiple active parties)}}
  \STATE \ \ Pick an index $i$ randomly from  $\{1,...,n\}$.
  \STATE \ \ Compute $\widehat{w}^{\top} x_{i}$ based on tree-structured communication.
  \STATE \ \ Compute $\vartheta_1 = \frac{\partial \mathcal{L}\left(\widehat{w}^{\top} x_{i}, y_{i}\right)}{\partial\left(\widehat{w}^{\top} x_{i}\right)}$.
  \STATE \ \ Sned $\vartheta_1$ and index $i$ to the rest parties.
  \STATE \ \ Compute $\widetilde{v}^{\ell}=\nabla_{\mathcal{G}_{\ell}} f_{i}(\widehat{w})-\nabla_{\mathcal{G}_{\ell}} f_{i}\left(w^{s}\right)+\nabla_{\mathcal{G}_{\ell}} f\left(w^{s}\right)$.
  \STATE \ \ Update $w_{\mathcal{G}_{\ell}} \leftarrow w_{\mathcal{G}_{\ell}}-\gamma \widetilde{v}^{\ell}$.\\
 \textbf{End parallel}
  \STATE $w_{\mathcal{G}_{\ell}}^{s+1}=w_{\mathcal{G}_{\ell}}$.
  \ENDFOR
\end{algorithmic}
\end{algorithm}
\begin{algorithm}[h]
\caption{VF{${\textbf{B}}^2$}-SVRG for the $\ell$-th party to passively launch collaborative updates.}\label{afsvrg-passive}
\begin{algorithmic}[1]
\REQUIRE {Local data $D^{\ell}$ stored on the $\ell$-th party, learning rate $\gamma$}
\STATE Initialize $w_{\mathcal{G}_{\ell}} \in \mathbb{R}^{d_{\ell}}$ (only performed on passive parties).
  \FOR {$s=0, 1, \ldots, S-1$}
  \STATE Receive all $\vartheta_{0,i}$ from the dominator and use them to compute the full local gradient $\nabla_{\mathcal{G}_{\ell}} f\left(w^{s}\right)=\frac{1}{n} \sum_{i=1}^{n} \nabla_{\mathcal{G}_{\ell}} f_{i}\left( {w}^{s}\right) = \frac{1}{n} \sum_{i=1}^{n} (\vartheta_{0,i} \cdot (x_i)_{\mathcal{G}_{\ell}}
  + \nabla g({w}^s_{\mathcal{G}_{\ell }}))$.
  \STATE $w_{\mathcal{G}_{\ell}}=w_{\mathcal{G}_{\ell}}^{s}$.
  \STATE  { \bf{Keep doing in parallel (shared-memory parallel for multiple threads)}}
  \STATE  \quad Receive $\vartheta_1$ and index $i$ from the dominator.
  \STATE \quad  Compute $\widetilde{v}^{\ell}=\vartheta_1 \cdot (x_i)_{\mathcal{G}_{\ell}}
  + \nabla g((\widehat{w})_{\mathcal{G}_{\ell }})
  -\left(\vartheta_{i,0} \cdot (x_i)_{\mathcal{G}_{\ell}}
  + \nabla g(w_{\mathcal{G}_{\ell}}^{s})\right)
  +\nabla_{\mathcal{G}_{\ell}} f\left(w^{s}\right)$.
  \STATE  \quad Update $w_{\mathcal{G}_{\ell}} \leftarrow w_{\mathcal{G}_{\ell}}-\gamma \widetilde{v}^{\ell}$.
  \STATE {\bf{End Parallel}}
  \ENDFOR
\end{algorithmic}
\end{algorithm}
\subsection{VF{${\textbf{B}}^2$}-SAGA}
VF${\textbf{B}}^2$-SAGA enjoying the same convergence rate with VF${\textbf{B}}^2$-SVRG is shown in Algorithms~\ref{afsaga-active} and \ref{afsaga-passive}. Different from VF${\textbf{B}}^2$-SVRG using ${w}^s$ as the reference gradient, VF${\textbf{B}}^2$-SAGA uses the average of history gradients stored in a table. In Algorithm~\ref{afsaga-active}, there is  $\widetilde{v}^{\ell}=\nabla_{\mathcal{G}_{\ell}} f_{i}(\widehat{w})-\widetilde{\alpha}_{i}^{\ell} + \frac{1}{n} \sum_{j=1}^{n} \widetilde{\alpha}_{j}^{\ell}$. While, in Algorithm~\ref{afsaga-passive}, there is $\widetilde{v}^{\ell} = \vartheta \cdot (x_i)_{\mathcal{G}_{\ell }}
  + \nabla g((\widehat{w})_{\mathcal{G}_{\ell }})
  - \widetilde{\alpha}_{i}^{\ell}
  +\frac{1}{n} \sum_{j=1}^{n} \widetilde{\alpha}_{j}^{\ell}$.
\begin{algorithm}[h]
\caption{VF${\textbf{B}}^2$-SAGA for active party $\ell$ to actively launch dominated update}\label{afsaga-active}
\begin{algorithmic}[1]
\REQUIRE {Local data $\{{(x_i)}_{\mathcal{G}_{\ell}},y_i\}_{i=1}^{n}$ stored on the $\ell$-th party, learning rate $\gamma$}.
\STATE Initialize $w_{\mathcal{G}_{\ell}} \in \mathbb{R}^{d_{\ell}}$.\\
  \STATE Compute the local gradient $\widehat{\alpha}_{i}^{\ell}=\nabla_{\mathcal{G}_{\ell}} f_{i}(\widehat{w})$, for $\forall i \in\{1, \ldots, n\}$ and $\ell = 1,\cdots,q$ through tree-structured communication.
  (this is performed only at the $1$-th global iteration)\\
 { \bf{Keep doing in parallel (distributed-memory parallel for multiple active parties)}}
  \STATE \quad Pick an index $i$ randomly from ${1,...,n}$.
  \STATE \quad Compute $\widehat{w}^{\top} x_{i}=\sum_{\ell^{\prime}=1}^{q}(\widehat{w})_{\mathcal{G}_{\ell^\prime}}^{\top} \left(x_{i}\right)_{\mathcal{G}_{\ell^{\prime}}}$
   based on tree-structured communication.
   \STATE \quad Compute $\vartheta = \frac{\partial \mathcal{L}\left(\widehat{w}^{\top} x_{i}y_{i}\right)}{\partial\left(\widehat{w}^{\top} x_{i}\right)}$
   \STATE  \quad Sent $\vartheta$ and index $i$ to collaborators.
  \STATE \quad Compute $\widetilde{v}^{\ell}=\nabla_{\mathcal{G}_{\ell}} f_{i}(\widehat{w})-\widetilde{\alpha}_{i}^{\ell}+\frac{1}{n} \sum_{j=1}^{n} \widetilde{\alpha}_{j}^{\ell}$.
  \STATE \quad Update $w_{\mathcal{G}_{\ell}} \leftarrow w_{\mathcal{G}_{\ell}}-\gamma \widetilde{v}^{\ell}$.
  \STATE \quad Update $\widetilde{\alpha}_{i}^{\ell} \leftarrow \nabla_{\mathcal{G}_{\ell}} f_{i}(\widehat{w})$.\\
\textbf{End parallel}
  \ENSURE {$w_{\mathcal{G}_\ell}$}
\end{algorithmic}
\end{algorithm}
\begin{algorithm}[h]
\caption{VF{${\textbf{B}}^2$}-SAGA for the $\ell$-th party to passively launch collaborative updates.}\label{afsaga-passive}
\begin{algorithmic}[1]
 \STATE { \bf{Keep doing in parallel (shared-memory parallel for multiple threads)}}
  \STATE  \quad Receive $\vartheta$ and index $i$ from the dominator.
  \STATE \quad Compute $\widetilde{v}^{\ell}=\vartheta \cdot (x_i)_{\mathcal{G}_{\ell }}
  + \nabla_{\mathcal{G}_\ell}g((\widehat{w})_{\mathcal{G}_{\ell }})
  - \widetilde{\alpha}_{i}^{\ell}
  +\frac{1}{n} \sum_{j=1}^{n} \widetilde{\alpha}_{j}^{\ell}$.
  \STATE  \quad Update $w_{\mathcal{G}_{\ell}} \leftarrow w_{\mathcal{G}_{\ell}}-\gamma \widetilde{v}^{\ell}$.
  \STATE  \quad Update $\widetilde{\alpha}_{i}^{\ell} \leftarrow \vartheta \cdot (x_i)_{\mathcal{G}_{\ell }}
  + \nabla_{\mathcal{G}_\ell}g((\widehat{w})_{\mathcal{G}_{\ell }})$.
\STATE {\bf{End parallel}}
\end{algorithmic}
\end{algorithm}
\section{Additional Experiments on Regression Task}
These experiments are conducted on two datasets for regression task: $D_5$ (E2006-tfidf) and $D_6$ (YearPredictitionMSD) from the LIBSVM \cite{chang2011libsvm}. $D_5$ has $16,087$ training samples and $150,306$ features. $D_6$ has $463,715$ training samples and $90$ features. Moreover, we apply the min-max normalization technique to the target variables $y$ of $D_6$.\\
 {\noindent {\bf Problems:}
 We consider $\ell_2$-norm regularized  regression problem for $\mu$-strong convex case
\begin{align}\label{P3}
\vspace{-0.05cm}
\min_{w \in \mathbb{R}^d} f(w):=\frac{1}{n} \sum_{i=1}^{n}  ({w^{\top} x_i}-y_i)^2 + \frac{\lambda}{2} \|w\|^2,
\vspace{-0.05cm}
\end{align}
and the robust linear regression  for nonconvex problem
\begin{align}\label{P4}
\min_{w \in \mathbb{R}^d} f(w):=\frac{1}{n} \sum_{i=1}^{n}  \mathcal{L}(y_i -\left\langle {x_i}, {w} \right\rangle),
\end{align}
where $\mathcal{L} (x) :=\log (\frac{x^2}{2} + 1)$. \\
{\noindent {\bf Asynchronous efficiency:}}
In these experiments, we set $q=12$, $m=2$ and fix the $\gamma$ for algorithms with a same SGD-type but in different parallel fashions.
As  shown in Fig.~\ref{Exp-regression}, the loss v.s. running
time curves demonstrate that our algorithms consistently outperform the corresponding synchronous counterparts in terms of the efficiency. \\
\begin{figure}[h]
	\centering
\begin{subfigure}{0.24\linewidth}
		\includegraphics[width=\linewidth]{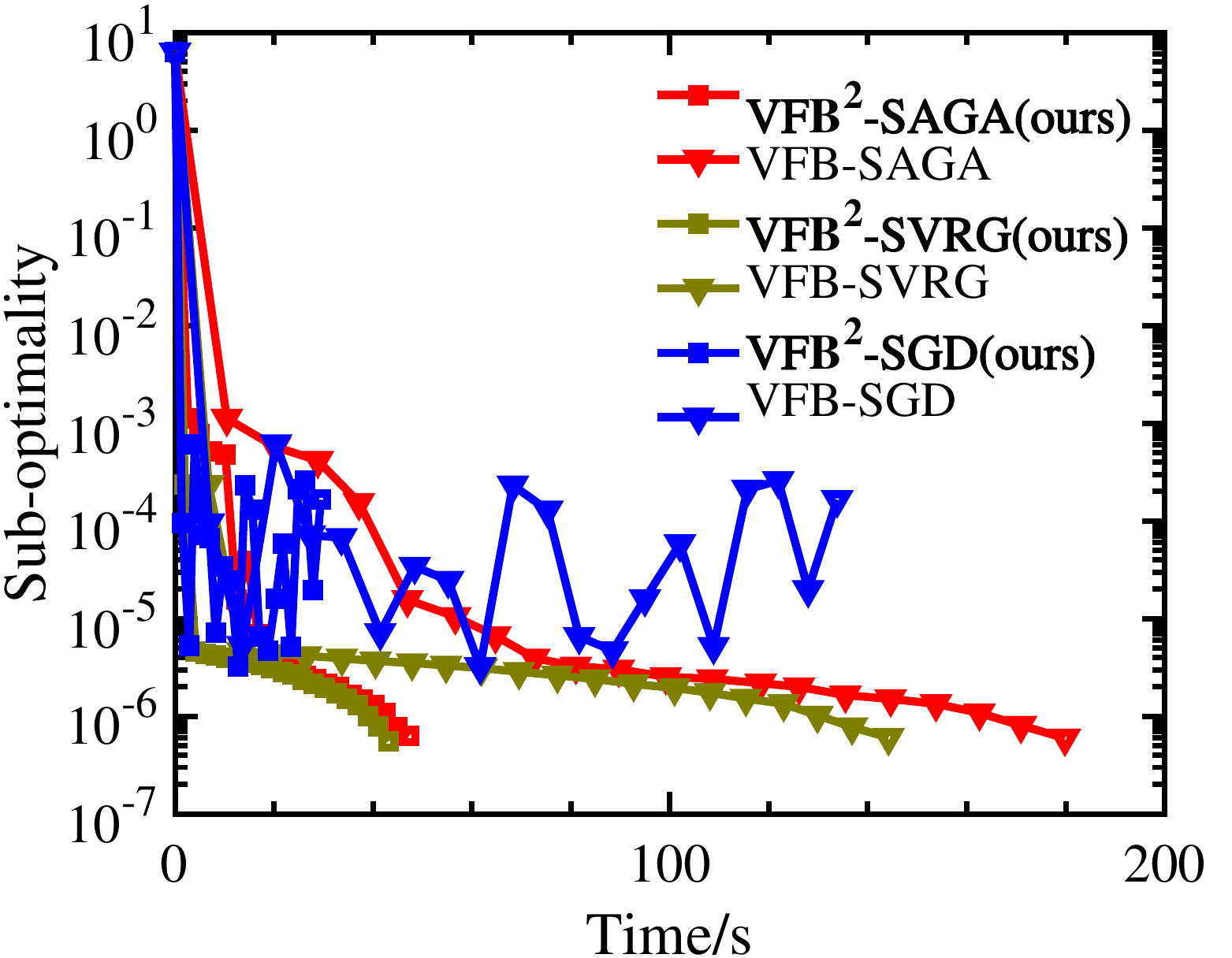}
		\caption{$D_5$ for Problem (\ref{P3})}
	\end{subfigure}
	\begin{subfigure}{0.24\linewidth}
		\includegraphics[width=\linewidth]{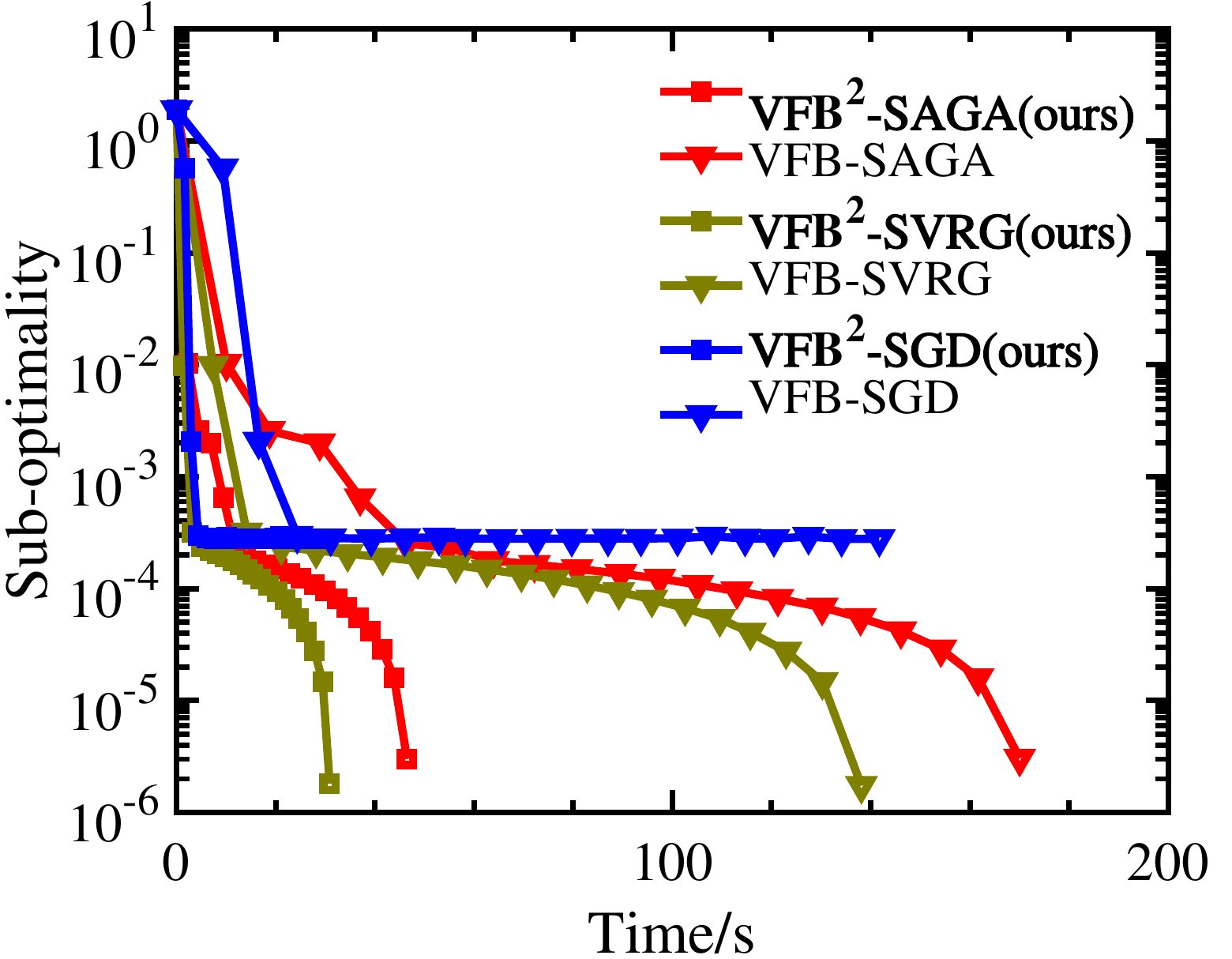}
		\caption{$D_5$ for Problem (\ref{P4})}
	\end{subfigure}
	\begin{subfigure}{0.24\linewidth}
		\includegraphics[width=\linewidth]{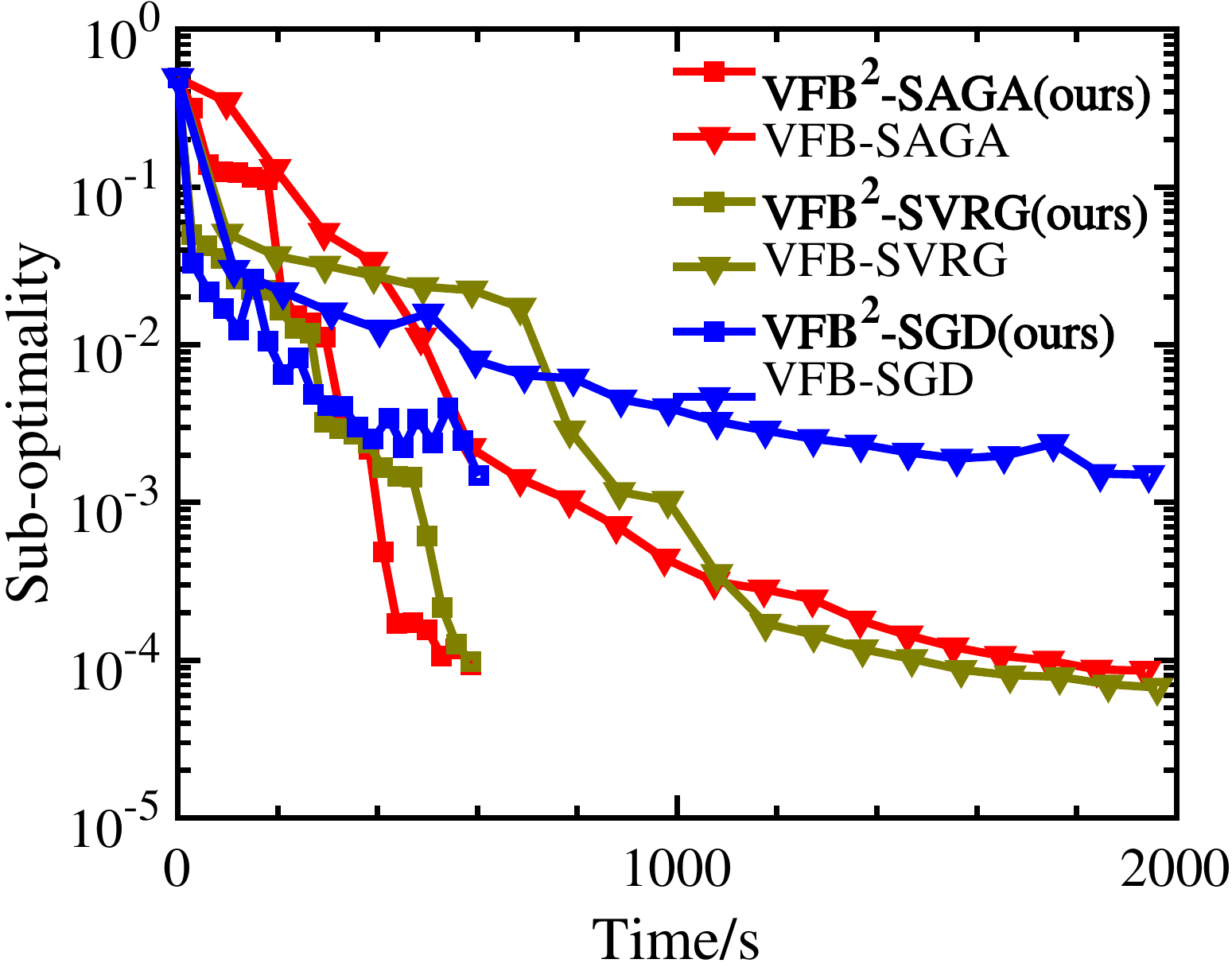}
		\caption{$D_6$ for Problem (\ref{P3})}
	\end{subfigure}
\begin{subfigure}{0.24\linewidth}
		\includegraphics[width=\linewidth]{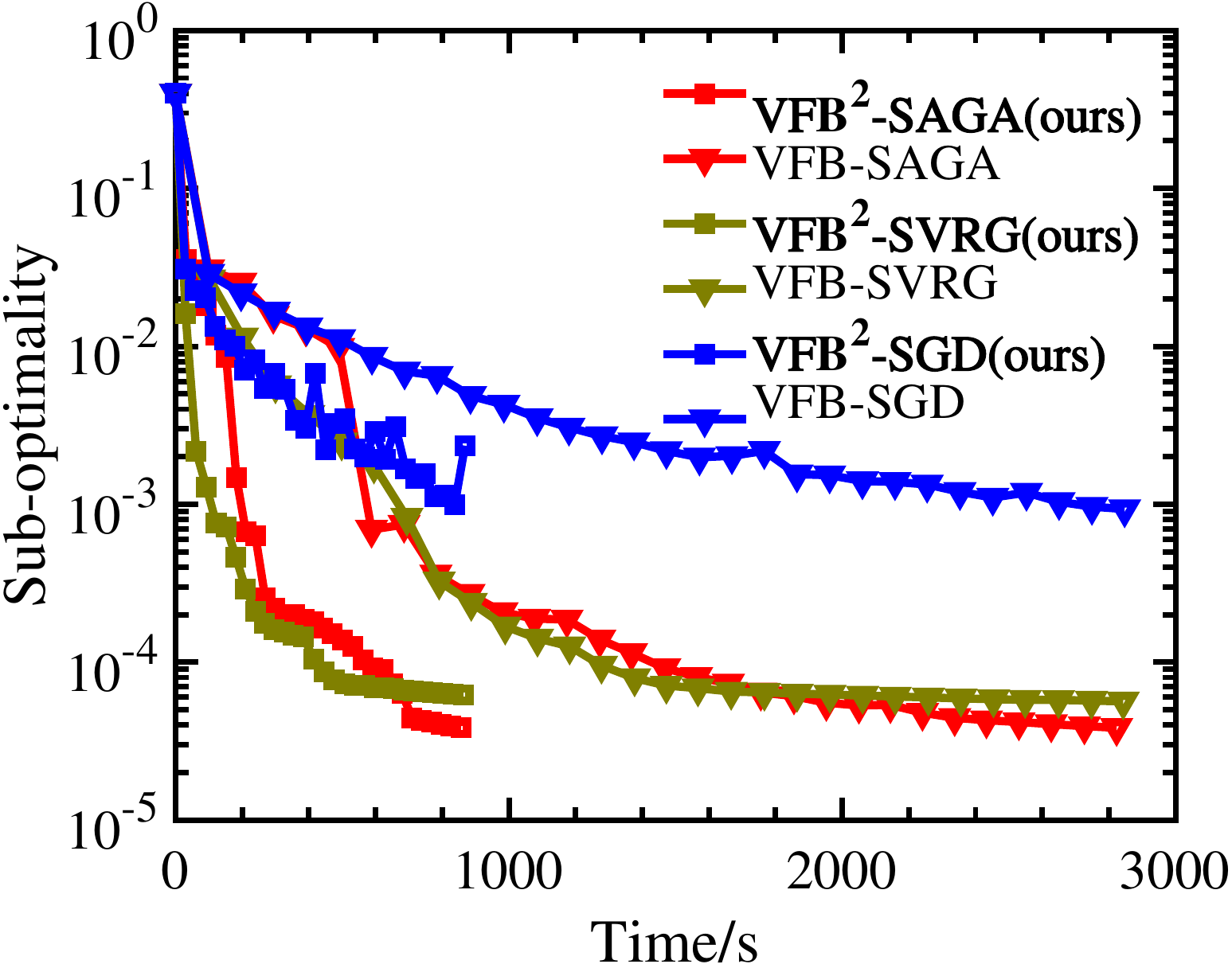}
		\caption{$D_6$ for Problem (\ref{P4})}
	\end{subfigure}%
	\caption{Results for solving regression tasks, where the number of epoches (points) denotes how many passes over the dataset the algorithm makes.}
\label{Exp-regression}
\end{figure}
\noindent{\bf{Evaluations of the losslessness}}
To demonstrate that our algorithms are lossless, we compare them with the corresponding non-federated (NonF) algorithms, \ie, all data were integrated together for modeling. For datasets without testing data, we split the data set into $5$ parts, and use one of them for testing. Moreover, we use the metric root mean square error (RMSE) for evaluation
\begin{align}\label{rmse}
 RMSE=\sqrt{\frac{1}{n}\sum_{i=1}^{n}(\hat{y}-y)^2},
\end{align}
where $\hat{y}$ denotes the prediction value and $y$ is the true value. As shown in Table~\ref{exp-lossless1}, the results of our algorithms are the same with those of NonF algorithms and are much better than those of AFSVRG-VP, which are consistent to our claims.
\begin{table}[!t]
\centering
\begin{tabular}{@{}ccccc@{}}
\toprule
& Algorithm & $D_5$(RMSE)  & $D_6$(RMSE)  \\ \midrule
\multirow{3}{*}{Problem  (\ref{P3})}
 & NonF & 0.389$\pm$0.012 & 0.069$\pm$0.004  \\
 &{{ AFSVRG-VP}} & 0.417$\pm$0.010 & 0.084$\pm$0.003  \\
  &{\bf{ Ours}} & 0.389$\pm$0.013 & 0.069$\pm$0.005  \\
 \midrule
\multirow{3}{*}{Problem (\ref{P4})}
 & NonF & 0.382$\pm$0.014 & 0.068$\pm$0.004  \\
 & {{ AFSVRG-VP}} & 0.415$\pm$0.009 & 0.084$\pm$0.004  \\
  &{\bf{ Ours}} & 0.382$\pm$0.013 & 0.068$\pm$0.005  \\
 \bottomrule
\end{tabular}
\caption{Evaluation of the losslessness  for regression task (10 trials).}
\label{exp-lossless1}
\end{table}
\subsection{Asynchronous scalability in terms of $q$}
W present a more clear illusion of the asynchronous scalability in terms of $q$ shown in Fig.~\ref{Exp-sca-1}.
\begin{figure}[!htb]
	\centering
	\begin{subfigure}{0.3\linewidth}
		\includegraphics[width=\linewidth]{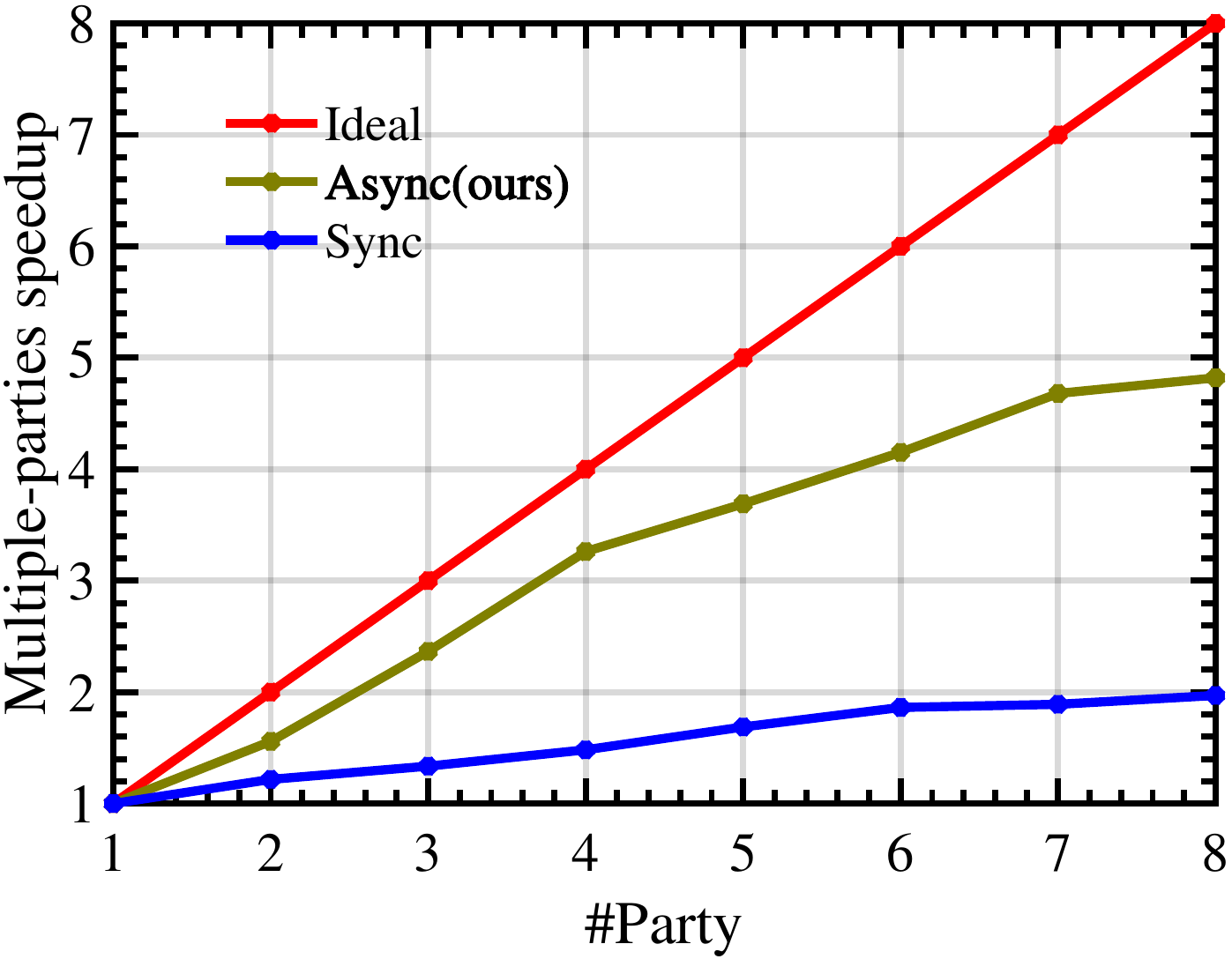}
		\caption{SGD-based}
	\end{subfigure}
\begin{subfigure}{0.3\linewidth}
		\includegraphics[width=\linewidth]{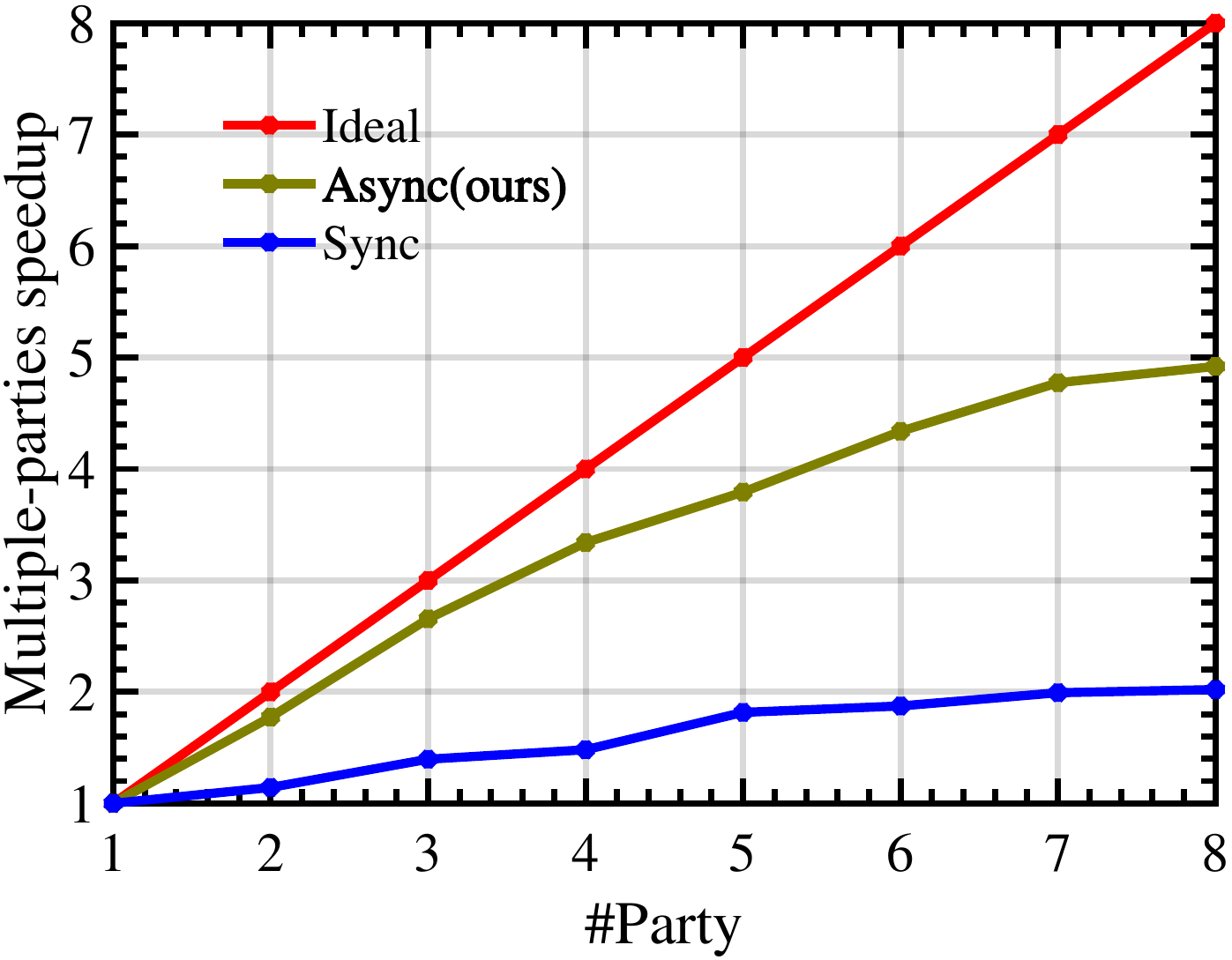}
		\caption{SVRG-based}
	\end{subfigure}
\begin{subfigure}{0.3\linewidth}
		\includegraphics[width=\linewidth]{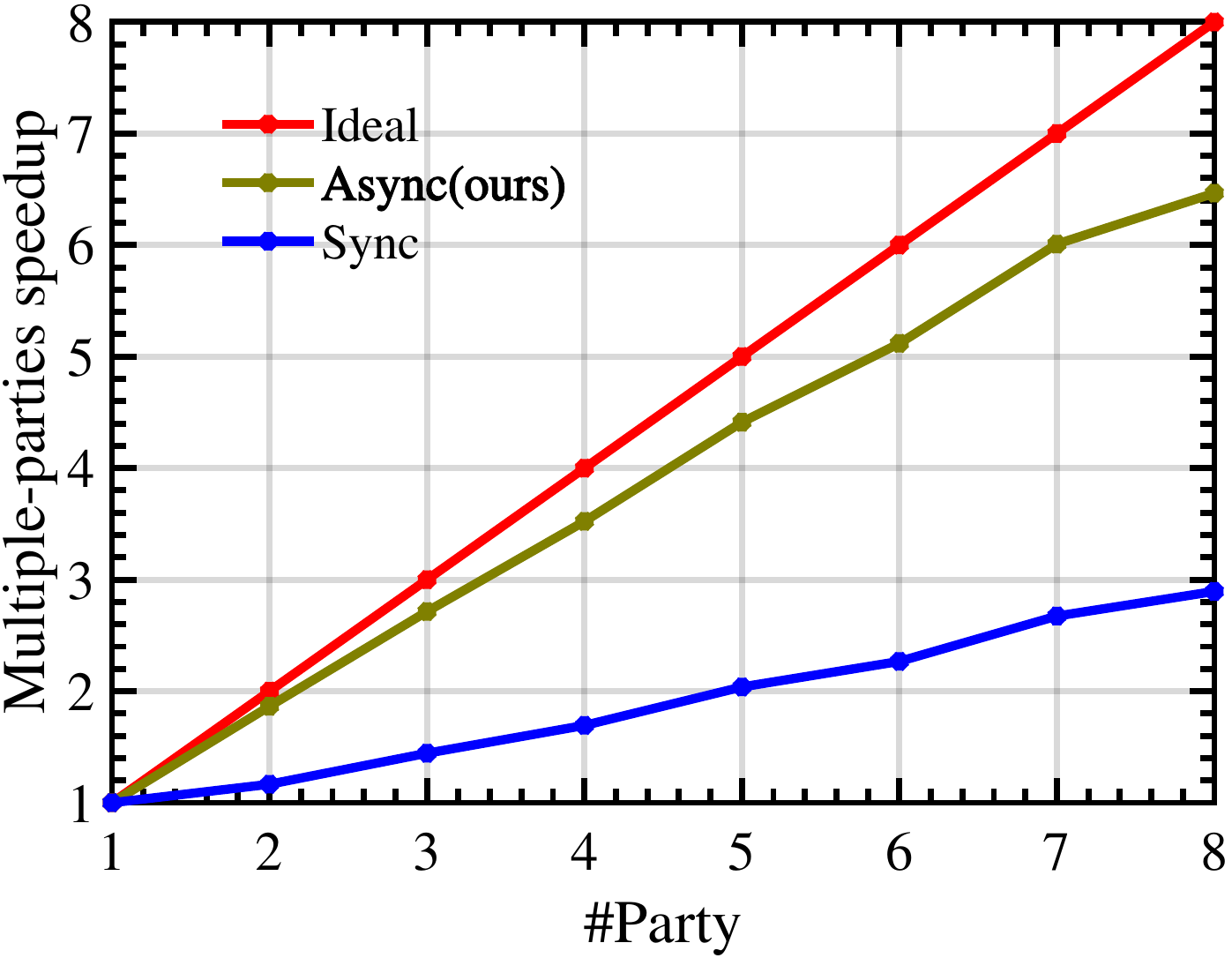}
		\caption{SAGA-based}
	\end{subfigure}%
	\caption{ $q$-parties speedup scalability  with $m=2$ on  $D_4$.}
\label{Exp-sca-1}
\end{figure}
\section{Preliminaries for Convergence Analysis (corresponding to line 254 in the manuscript)}
In this section, we present some preliminaries which are helpful for readers to understand the analysis.
\\
\noindent{\bf{Globally labeling the iterates:}} As shown in the algorithms, we do not globally label the iterates from different parties. While, how to define the global iteration counter $t$ to label an iterate $w_t$ matters in the convergence analysis. In this paper, we adopts the ``after read'' labeling strategy \cite{leblond2017asaga}, where the global iterate counter is updated as one dominator  finishes computing $\widehat{w}_t^\top x_i$ or as one collaborator  finishes reading local parameters $(\widehat{w}_t)_{\mathcal{G}_{\psi(t)}} $ (this reading operation is performed after having received information from a specific dominator, \emph{e.g.,} step~3 in Algorithm~\ref{AFSGD-P}). It means that $\widehat{w}_t$ on a specific dominated parties is the $t+1$-th fully completed computation of $\widehat{w}^\top x_i$ and $(\widehat{w}_t)_{\mathcal{G}_{\psi(t)}}$ on a collaborative party is the $t+1$-th fully completed read of $({w}_t)_{\mathcal{G}_{\psi(t)}}$. Importantly, such a labeling strategy guarantees that $i_t$ and $\widehat{w}_t$ are independent \cite{leblond2017asaga}, which simplifies the  convergence analyses, especially, for VF{${\textbf{B}}^2$}-SAGA.
\\
\noindent{\bf{Global updating rule:}}  Here we introduce the global updating rule as
\begin{equation}\label{global-up1}
  w_{t+1} = w_t -\gamma U_{\psi(t)}\widetilde{v}_t^{\psi(t)}
\end{equation}
where $\widetilde{v}_t^{\psi(t)}$ has a different definition on different type of roles (dominator or collaborator). Although the definitions of $\widetilde{v}_t^{\psi(t)}$ are different on different type of roles, we will build uniform analyses for them.
\\
\noindent{ \bf{Relationship between $w_t$ and $\widehat{w}_t$:}} For dominators, $\widehat{w}^{T} x_{i}=\sum_{\ell^{\prime}=1}^{q}(\widehat{w})_{\mathcal{G}_{\ell '}}^{T}\left(x_{i}\right)_{\mathcal{G}_{\ell^{\prime}}}$ is obtained based on Algorithm~\ref{safer_tree} in an asynchronous parallel fashion, where $\widehat{w}$ denotes $w$ inconsistently read from different data parties. It means that, vector $(\widehat{w}_t)_{\mathcal{G}_{\ell '}}$ (where $\ell ' \neq \ell$) may be inconsistent to $({w}_t)_{\mathcal{G}_{\ell '}}$, i.e., some blocks of $\widehat{w}_t$ are the same with the ones in $w_t$ (e.g., $({w_t})_{\mathcal{G}_{\ell '}}=(\widehat{w}_t)_{\mathcal{G}_{\ell '}}$), but others are different.  Thus we introduce a set $D(t)$ in Eq.~\ref{Dt1} and the upper bound of its size is introduced in Assumption~\ref{assum4}.
\\
\noindent{ \bf{Relationship between $\bar{w}_t$ and $\widehat{w}_t$:}} For a collaborative party, it use $\vartheta$ received from dominated party to compute $\nabla_{\mathcal{G}_{\psi(t)}} \mathcal{L}$, and we donate  $\vartheta \cdot (x_i)_{\mathcal{G}_{\psi(t)}}$ at global iteration $t$ as $\nabla_{\mathcal{G}_{\psi(t)}} \mathcal{L}(\bar{w}_t)$. Since there is a communication delay between dominator  and collaborators, $\bar{w}_t$ maybe an old $\widehat{w}_u$ ($u\leq t$). To describe the relation between $\bar{w}_t$  and $\widehat{w}_u$, we thus introduce a set $D^\prime(t)$ in Eq.~\ref{Dt2} (when $u = t$, $D^\prime(t)$ denotes an empty set). Meanwhile, we introduce an upper bound to the communication delay in Assumption~\ref{assum4}.
\\
 {\noindent{\bf Introduction of $\widetilde{v}_t$  and  $\widehat{v}_t$}:} In Algorithms~\ref{AFSGD-A} and \ref{AFSGD-P}, we have that for a dominator, there is $\widetilde{v}_t=\widehat{v}_t$. While  for collaborators, there is $\widetilde{v}_t^{\psi(t)}=\vartheta \cdot (x_i)_{\mathcal{G}_\ell} + \nabla_{\mathcal{G}_\ell}g(\widehat{w})$ which can be rewritten as $\widetilde{v}_t^{\psi(t)}=\bar{v}_t^{\psi(t)}+ \nabla_{\mathcal{G}_\ell}g(\widehat{w})-\nabla_{\mathcal{G}_\ell}g(\bar{w})$.
\section{Convergence Analyses for Strongly Convex problems}
 \subsection{Convergence Analysis of Theorem~\ref{thm-sgdconvex}}
\begin{lemma}\label{lem-csgd-1}
For VF{${\textbf{B}}^2$}-SGD,  for  $\forall t$, there is
\begin{equation}\label{lemequ-csgd-1}
  \mathbb{E} ||\widetilde{v}_{t}^{\psi(t)}||^2 \leq \frac{2G}{1-\lambda_{1}}
\end{equation}
where there is $\lambda_{1}=2L_*^2\gamma^2\tau$.
\end{lemma}
\begin{proof}[\textbf{Proof of Lemma \ref{lem-csgd-1}:}] If the $t$-th global iteration is a collaborative update we have
\begin{eqnarray}\label{csgd-0}
\mathbb{E} ||\widetilde{v}_{t}^{\psi(t)}||^2  &=& \mathbb{E} || \vartheta \cdot\left(x_{i}\right)_{\mathcal{G}_{\psi(t)}}
+ \nabla_{\mathcal{G}_{\psi(t)}} g((\widehat{w}_t)_{\mathcal{G}_{\psi(t)}} )||^2 \nonumber \\
&=&  \mathbb{E} ||\vartheta \cdot\left(x_{i}\right)_{\mathcal{G}_{\psi(t)}}
 + \nabla_{\mathcal{G}_{\psi(t)}} g((\bar{w}_{t})_{\mathcal{G}_{\psi(t)}})
 - \nabla_{\mathcal{G}_{\psi(t)}} g((\bar{w}_{t})_{\mathcal{G}_{\psi(t)}})
 + \nabla_{\mathcal{G}_{\psi(t)}} g((\widehat{w}_t)_{\mathcal{G}_{\psi(t)}}) ||^2 \nonumber \\
 &\stackrel{(a)}{\leq}& 2  \mathbb{E} || \bar{v}_{t}^{\psi(t)}\|^2
+ 2 \mathbb{E}\|\nabla_{\mathcal{G}_{\psi(t)}} g((\bar{w}_{t})_{\mathcal{G}_{\psi(t)}})
- \nabla_{\mathcal{G}_{\psi(t)}} g((\widehat{w}_t)_{\mathcal{G}_{\psi(t)}})||^2
\nonumber \\
&\stackrel{(b)}{\leq}& 2  \mathbb{E} || \bar{v}_{t}^{\psi(t)}\|^2
 + 2{L_{g}^2} \mathbb{E}\|(\bar{w}_{t})_{\mathcal{G}_{\psi(t)}}
 - (\widehat{w}_t)_{\mathcal{G}_{\psi(t)}}||^2
\nonumber \\
&\stackrel{(c)}{=}& 2  \mathbb{E} || \bar{v}_{t}^{\psi(t)}\|^2
+ 2{L_{g}^2}\gamma^2 \mathbb{E}\|\sum_{t^\prime\in D'(t), \psi(t^\prime)=\psi(t)} \widetilde{v}_{t^\prime}^{\psi(t^\prime)}||^2
\nonumber \\
&\stackrel{(d)}{\leq}& 2  \mathbb{E} || \bar{v}_{t}^{\psi(t)}\|^2
+ 2{L_{g}^2}\gamma^2 \tau_2\sum_{t^\prime\in D'(t)} \mathbb{E}\|\widetilde{v}_{t^\prime}^{\psi(t^\prime)}||^2
 \nonumber \\
&\stackrel{(e)}{\leq}& 2 \mathbb{E} || \bar{v}_{t}^{\psi(t)}\|^2
+ 2{L_{*}^2}\gamma^2 \tau_2 \sum_{t^\prime\in D'(t)}\mathbb{E}\|\widetilde{v}_{t^\prime}^{\psi(t^\prime)}||^2
 \nonumber \\
&\stackrel{(f)}{\leq}& 2 G
+ 2{L_{*}^2}\gamma^2 \tau_2 \sum_{t^\prime\in D'(t)}\mathbb{E}\|\widetilde{v}_{t^\prime}^{\psi(t^\prime)}||^2
\end{eqnarray}
where (a) follows from $\|a+b\|^2\leq 2\|a\|^2 + 2\|b\|^2$, (b) follows from Assumption~\ref{assum2}, (c) follows from the Eq.~\ref{Dt2}, (d) follows from Assumption~\ref{assum4} and $\|\sum_{i=1}^{n}a_i\|^2 \leq n \sum_{i=1}^{n} \|a_i\|^2$, (e) follows from definition of $L_{*}$, (f) follows from the definition of $\bar{v}_t$  and Assumption~\ref{assum1}.

If the $t$-th global iteration is a dominated update, there is
\begin{align}\label{csgd-1}
\mathbb{E} ||\widetilde{v}_{t}^{\psi(t)}||^2  = \mathbb{E} ||\widehat{v}_{t}^{\psi(t)}||^2 \leq G
\end{align}
Then for $\forall t$, according to Eqs.~\ref{csgd-0} and \ref{csgd-1}, we have
\begin{align}\label{csgd-2}
\mathbb{E} ||\widetilde{v}_{0}^{\psi(0)}||^2  \stackrel{(a)}{\leq}& G \leq 2G
\nonumber \\
 \mathbb{E} ||\widetilde{v}_{1}^{\psi(1)}||^2  \leq&
  2G + 2{L_{*}^2}\gamma^2 \tau_2 \sum_{t^\prime\in D'(1)}\mathbb{E}\|\widetilde{v}_{t^\prime}^{\psi(t^\prime)}||^2
  \stackrel{(b)}{\leq}
    2G + 2{L_{*}^2}\gamma^2 \tau_2^2 \mathbb{E} ||\widetilde{v}_{0}^{\psi(0)}||^2
  \stackrel{(c)}{=}
  2G\frac{1-k^{ 1 +1}}{1-k}
  \nonumber \\
  ...
  \nonumber \\
  \mathbb{E} ||\widetilde{v}_{t}^{\psi(t)}||^2 \leq &
   2G + 2{L_{*}^2}\gamma^2 \tau_2 \sum_{t^\prime\in D'(t)}\mathbb{E}\|\widetilde{v}_{t^\prime}^{\psi(t^\prime)}||^2
   \leq
   2G + 2{L_{*}^2}\gamma^2 \tau_2^2 (2G\frac{1-k^{t}}{1-k})
     \stackrel{(d)}{=}
    2G\frac{1-k^{t+1}}{1-k}
\end{align}
where (a) follows from that the $0$-th global iteration must be a dominated update, (b) follows from that for all $t' \in D'(t)$, there is $t' \leq t$, (c) follows from that $k:=2{L_{*}^2}\gamma^2 \tau_2^2$, (d) follows from the summation formula of equal ratio sequence. According to Eq.~\ref{csgd-2}, it holds that for $\forall t$ there is
\begin{align}\label{csgd-3}
  \mathbb{E} ||\widetilde{v}_{t}^{\psi(t)}||^2 \leq
    2G\frac{1-k^{t+1}}{1-k} \leq \frac{2G}{1-k}\leq \frac{2G}{1-2{L_{*}^2}\gamma^2 \tau}
\end{align}
where the last inequality follows from the definition of $\tau$.
This completes the proof.
\end{proof}
\begin{lemma}\label{lem-csgd-2}
For all $\forall t$, there is
\begin{equation}\label{1}
 \mathbb{E} \| {v}_{t}^{\psi(t)} - \widetilde{v}_{t}^{\psi(t)} \|^2 \leq2{ L_{{*}}^2  \gamma^2 \tau_1}  \sum_{t' \in D(t)} \mathbb{E} \|   \widetilde{v}^{\psi(t')}_{t'} \|^2
+ 8{ L_{*}^2  \gamma^2 \tau_2  \sum_{t' \in D^\prime(t)}} \mathbb{E} \|   \widetilde{v}^{\psi(t')}_{t'} \|^2
\end{equation}
\end{lemma}
\begin{proof}[\textbf{Proof of  Lemma \ref{lem-csgd-2}:}]
First, we give the bound of $ \mathbb{E} \| \widehat{v}_{t}^{\psi(t)} - \widetilde{v}_{t}^{\psi(t)} \|^2$ as follow
\begin{align}\label{csgd-4}
 \mathbb{E} \| \widehat{v}_{t}^{\psi(t)} - \widetilde{v}_{t}^{\psi(t)} \|^2
 &  \stackrel{ (a) }{\leq} \mathbb{E} \| \nabla_{\mathcal{G}_{\psi(t)}} f(\bar{w}_t) - \nabla_{\mathcal{G}_{\psi(t)}} f(\widehat{w}_t) + \nabla_{\mathcal{G}_{\psi(t)}} g((\widehat{w}_t)_{\mathcal{G}_{\psi(t)}}) - \nabla_{\mathcal{G}_{\psi(t)}} g((\bar{w}_t)_{\mathcal{G}_{\psi(t)}}\|^2
\nonumber \\
& \stackrel{ (b) }{\leq} 2\mathbb{E} \| \nabla_{\mathcal{G}_{\psi(t)}} f_{i_t} (\bar{w}_t) - \nabla_{\mathcal{G}_{\psi(t)}} f_{i_t} (\widehat{w}_t) \|^2
+ 2 \mathbb{E} \| \nabla_{\mathcal{G}_{\psi(t)}} g ((\bar{w}_t)_{\mathcal{G}_\psi(t)})
- \nabla_{\mathcal{G}_{\psi(t)}} g ((\widehat{w}_t))_{\mathcal{G}_\psi(t)} \|^2
\nonumber \\
& \stackrel{ (c) }{\leq} 2{L^2} \mathbb{E} \| \bar{w}_t - \widehat{w}_t \|^2
 + 2{L_{g}^2} \mathbb{E} \| (\bar{w}_t)_{\mathcal{G}_\psi(t)} - (\widehat{w}_t)_{\mathcal{G}_\psi(t)} \|^2
\nonumber \\
& \stackrel{ (d) }{=} 2{ L^2 \gamma^2}  \mathbb{E} \|  \sum_{t' \in D'(t)} \textbf{U}_{\psi(t')} \widetilde{v}^{\psi(t')}_{t'} \|^2 + 2{ L_{g}^2 \gamma^2 }  \mathbb{E} \| \sum_{t' \in D^\prime(t), \psi(t^\prime)=\psi(t)} \textbf{U}_{\psi(t')} \widetilde{v}^{\psi(t')}_{t'} \|^2
\nonumber \\
& \stackrel{ (e) }{\leq} 2{ L^2  \gamma^2 \tau_2}  \sum_{t' \in D'(t)} \mathbb{E} \|   \widetilde{v}^{\psi(t')}_{t'} \|^2 + 2{ L_{g}^2  \gamma^2 \tau_2  \sum_{t' \in D^\prime(t)}} \mathbb{E} \|   \widetilde{v}^{\psi(t')}_{t'} \|^2
\nonumber \\
& \stackrel{ (f) }{\leq} 4{ L_{{*}}^2  \gamma^2 \tau_2}  \sum_{t' \in D'(t)} \mathbb{E} \|   \widetilde{v}^{\psi(t')}_{t'} \|^2
\end{align}
where (a) follows from the definition of $\bar{v}_{t}^{\psi(t)}$ and the definitions of $\widetilde{v}_{t}^{\psi(t)}$ for different types of the $t$-th global iteration (i.e., dominated or collaborative), (b) follows from $\|a+b\|^2 \leq 2\|a\|^2 + 2\|b\|^2$, (c) follows from  Assumptions~\ref{assum1} and \ref{assum2}, (d) follows from Eq.~\ref{Dt1}, (e) follows from Assumption~\ref{assum4} and $\| \sum_{i=1}^{n} a_i \|^2 \leq n\sum_{i=1}^{n} \|a_i\|^2$, (f) follows from the definition of $L_{*}$.
 Then we consider the bound of $\mathbb{E} \| v_{t}^{\psi(t)} - \widetilde{v}_{t}^{\psi(t)} \|^2$:
\begin{align}\label{csgd-5}
 \mathbb{E} \| v_{t}^{\psi(t)} - \widetilde{v}_{t}^{\psi(t)} \|^2 & = \mathbb{E} \| v_{t}^{\psi(t)} - \widehat{v}_{t}^{\psi(t)} + \widehat{v}_{t}^{\psi(t)} - \widetilde{v}_{t}^{\psi(t)} \|^2
\nonumber \\
 & \stackrel{ (a) }{\leq} 2\mathbb{E} \| v_{t}^{\psi(t)} - \widehat{v}_{t}^{\psi(t)}\|^2+ 2 \mathbb{E} \| \widehat{v}_{t}^{\psi(t)} - \widetilde{v}_{t}^{\psi(t)} \|^2
\nonumber \\
& \leq 2\mathbb{E} \| \nabla_{\mathcal{G}_{\psi(t)}} f_{i_t} ({w}_t) - \nabla_{\mathcal{G}_{\psi(t)}} f_{i_t} (\widehat{w}_t) \|^2
+ 2 \mathbb{E} \| \widehat{v}_{t}^{\psi(t)} - \widetilde{v}_{t}^{\psi(t)} \|^2
\nonumber \\
& \stackrel{ (b) }{\leq} 2{L^2} \mathbb{E} \| {w}_t - \widehat{w}_t \|^2
 + 2 \mathbb{E} \| \widehat{v}_{t}^{\psi(t)} - \widetilde{v}_{t}^{\psi(t)} \|^2
\nonumber \\
& \stackrel{ (c) }{=} 2{ L^2 \gamma^2}  \mathbb{E} \|  \sum_{t' \in D(t)} \textbf{U}_{\psi(t')} \widetilde{v}^{\psi(t')}_{t'} \|^2
+2 \mathbb{E} \| \widehat{v}_{t}^{\psi(t)} - \widetilde{v}_{t}^{\psi(t)} \|^2
\nonumber \\
& \stackrel{ (d) }{\leq} 2{ L^2  \gamma^2 \tau_1}  \sum_{t' \in D(t)} \mathbb{E} \|   \widetilde{v}^{\psi(t')}_{t'} \|^2
+ 2 \mathbb{E} \| \widehat{v}_{t}^{\psi(t)} - \widetilde{v}_{t}^{\psi(t)} \|^2
\nonumber \\
& \stackrel{ (e) }{\leq} 2{ L_{{*}}^2  \gamma^2 \tau_1}  \sum_{t' \in D(t)} \mathbb{E} \|   \widetilde{v}^{\psi(t')}_{t'} \|^2
+ 8{ L_{*}^2  \gamma^2 \tau_2  \sum_{t' \in D^\prime(t)}} \mathbb{E} \|   \widetilde{v}^{\psi(t')}_{t'} \|^2
\end{align}
where (a) follows from $\|a+b\|^2 \leq 2\|a\|^2 + 2\|b\|^2$, (b) follows from Assumptions~\ref{assum1}, (c) follows from Eq.~\ref{Dt1}, inequalities, (d) follows from Assumptions~\ref{assum4} and $\| \sum_{i=1}^{n} a_i \|^2 \leq n\sum_{i=1}^{n} \|a_i\|^2$, (e) follows from the definition of $L_{*}$ and Eq. \ref{csgd-4}. This completes the proof.
\end{proof}

\begin{lemma}\label{lem-csgd-3}
For VF{${\textbf{B}}^2$}-SGD, we have
\begin{equation*}\label{lemeq-csgd3}
\sum_{u\in K(t)}\mathbb{E}  \|\nabla_{\mathcal{G}_{\psi(u)}} f({w}_{u})\|^2
 \geq
 \frac{1}{2}  \sum_{u\in K(t)} \mathbb{E}\| \nabla_{\mathcal{G}_{\psi(u)}} f({w}_{t})\|^2
  -  L^2 \gamma^2 \eta_1 \sum_{u\in K(t)} \sum_{u' \in \{t,\cdots,u\}} \mathbb{E} \|\widetilde{v}_{u'}^{\psi(u')}\|^2
\end{equation*}
\end{lemma}
\begin{proof}[\textbf{Proof of  Lemma \ref{lem-csgd-3}:}]
For any $u \in K(t)$, there is
\begin{align}\label{csgd-6}
  \mathbb{E} \|\nabla_{\mathcal{G}_{\psi(u)}} f(w_{t})\|^{2}
 & = \mathbb{E} \| \nabla_{\mathcal{G}_{\psi(u)}} f({w}_{t})
  - \nabla_{\mathcal{G}_{\psi(u)}} f({w}_{u})
   + \nabla_{\mathcal{G}_{\psi(u)}} f({w}_{u})\|^2
 \nonumber \\
 & \stackrel{(a)}{\leq}  2\mathbb{E} \| \nabla_{\mathcal{G}_{\psi(u)}} f({w}_{t})
  - \nabla_{\mathcal{G}_{\psi(u)}} f({w}_{u})\|^2
 +  2 \mathbb{E} \|\nabla_{\mathcal{G}_{\psi(u)}} f({w}_{u})\|^2
  \nonumber \\
 & \leq  2\mathbb{E} \| \nabla f({w}_{t}) - \nabla f({w}_{u})\|^2
 +  2 \mathbb{E} \|\nabla_{\mathcal{G}_{\psi(u)}} f({w}_{u})\|^2
   \nonumber \\
 & \stackrel{(b)}{\leq}  2 L^2 \gamma^2 \mathbb{E} \| {w}_{t} - {w}_{u}\|^2
 +  2 \mathbb{E} \|\nabla_{\mathcal{G}_{\psi(u)}} f({w}_{u})\|^2
    \nonumber \\
 & \stackrel{(c)}{=}  2 L^2 \gamma^2 \mathbb{E} \|\sum_{u' \in \{t,...,u\}}\textbf{U}_{\psi(u')}\widetilde{v}_{u'}^{\psi(u')}\|^2
 +  2 \mathbb{E} \|\nabla_{\mathcal{G}_{\psi(u)}} f({w}_{u})\|^2
     \nonumber \\
 & \stackrel{(d)}{\leq}  2 L^2 \gamma^2 \eta_1 \sum_{u \in \{t,...,u\}} \mathbb{E} \|\widetilde{v}_{u}^{\psi(u)}\|^2
 +  2 \mathbb{E} \|\nabla_{\mathcal{G}_{\psi(u)}} f({w}_{u})\|^2
\end{align}
where (a) follows from $\|a+b\|^2 \leq 2\|a\|^2 + 2\|b\|^2$, (b) follows from Assumptions~\ref{assum1}, (c) follows from Eq.~\ref{global-up1}, (d) follows from The bound of $|K(t)|$ and $\| \sum_{i=1}^{n} a_i \|^2 \leq n\sum_{i=1}^{n} \|a_i\|^2$. According to Eq.~\ref{csgd-6} we have
\begin{align}\label{csgd-7}
 \mathbb{E}  \|\nabla_{\mathcal{G}_{\psi(u)}} f({w}_{u})\|^2
 \geq     \frac{1}{2} \mathbb{E}\| \nabla_{\mathcal{G}_{\psi(u)}} f({w}_{t})\|^2 -  L^2 \gamma^2 \eta_1 \sum_{u' \in \{t,\cdots,u\}} \mathbb{E} \|\widetilde{v}_{u'}^{\psi(u')}\|^2
\end{align}
Summing above equality for all $u \in K(t)$ we obtain the conclusion. This completes the proof.
\end{proof}
\begin{proof}[\textbf{Proof of Theorem \ref{thm-sgdconvex}:}]
 For $\forall u \in K(t)$ we have that
\begin{eqnarray}\label{csgd-8}
&& \mathbb{E} f (w_{u+1})
\\ \nonumber &\stackrel{ (a) }{\leq}&  \mathbb{E} \left ( f (w_{u}) + \langle \nabla f(w_{u}), w_{u+1}-w_{u}  \rangle + \frac{L}{2} \|w_{u+1}-w_{u}   \|^2  \right )
\\ \nonumber &=&  \mathbb{E} \left ( f (w_{u}) -  \gamma \langle \nabla f(w_{u}),
\widetilde{v}^{\psi(u)}_{u}  \rangle + \frac{L\gamma^2}{2} \|  \widehat{v}^{\psi(u)}_{u}  \|^2  \right )
\\ \nonumber &{=}&  \mathbb{E} \left ( f (w_{u}) -  \gamma \langle \nabla f(w_{u}),  \widetilde{v}^{\psi(u)}_{u} + {v}^{\psi(u)}_{u}- {v}^{\psi(u)}_{u} \rangle  + \frac{L \gamma^2}{2} \|  \widetilde{v}^{\psi(u)}_{u}  \|^2  \right )
\\ \nonumber &\stackrel{(b)}{=}&  \mathbb{E}  f (w_{u}) -  \gamma \mathbb{E} \langle \nabla f(w_{u}),  \nabla_{\mathcal{G}_{\psi(u)}} f ({w}_{u}) \rangle  + \frac{L \gamma^2}{2} \mathbb{E} \|  \widetilde{v}^{\psi(u)}_{u}  \|^2  + \gamma \mathbb{E} \langle \nabla f(w_{u}), {v}^{\psi(u)}_{u} - \widetilde{v}^{\psi(u)}_{u} \rangle
 \\ \nonumber &\stackrel{ (c) }{\leq}&  \mathbb{E}  f (w_{u}) -  \gamma \mathbb{E} \|  \nabla_{\mathcal{G}_{\psi(u)}} f ({w}_{u}) \|^2  + \frac{\gamma}{2} \mathbb{E} \| \nabla_{\mathcal{G}_{\psi(u)}} f ({w}_{u}) \|^2 + \frac{L \gamma^2}{2} \mathbb{E} \|  \widetilde{v}^{\psi(u)}_{u}  \|^2 + \frac{\gamma}{2} \mathbb{E} \| \widetilde{v}^{\psi(u)}_{u} - {v}^{\psi(u)}_{u} \|^2
  \\ &\stackrel{(d)}{\leq}& \mathbb{E}  f (w_{u}) -  \frac{\gamma}{2} \mathbb{E} \|  \nabla_{\mathcal{G}_{\psi(u)}} f ({w}_{u}) \|^2  + \frac{L_* \gamma^2}{2} \mathbb{E} \|  \widetilde{v}^{\psi(u)}_{u}  \|^2
  \nonumber \\ &&
  + { L_{{*}}^2  \gamma^3 \tau_1}  \sum_{t' \in D(t)} \mathbb{E} \|   \widetilde{v}^{\psi(t')}_{t'} \|^2
+ 4{ L_{*}^2  \gamma^3 \tau_2  \sum_{t' \in D^\prime(t)}} \mathbb{E} \|   \widetilde{v}^{\psi(t')}_{t'} \|^2 \nonumber
 \end{eqnarray}
where the  inequalities (a) follows form Assumption~\ref{assum2}, (b) follows from that $ {v}^{\psi(u)}_u = \nabla_{\mathcal{G}_{\psi(u)}} f_{i_u} ({w}_u)$ for a specific party, (c) follows from $\langle a,b \rangle\leq \frac{1}{2}(\|a\|^2+\|b\|^2)$, (d) follows from Lemma~\ref{lem-csgd-2} and the definition of $L_*$.
Summing  Eq.~(\ref{csgd-8}) over all $ u \in K(t) $, we obtain
\begin{eqnarray}\label{csgd-9}
&& \mathbb{E} \left[f (w_{t+|K(t)|} - f (w_{t}) \right]
\\ \nonumber
&\leq&
 -\frac{\gamma}{2}\sum_{u \in K(t)}  \mathbb{E} \|  \nabla_{\mathcal{G}_{\psi(u)}} f ({w}_{u}) \|^2
 + \frac{L_* \gamma^2}{2}\sum_{u \in K(t)} \mathbb{E} \|  \widetilde{v}^{\psi(u)}_{ u }\|^2
 \nonumber \\
&+& ( L_{{*}}^2  \gamma^3 \tau_1 + 4L_{*}^2  \gamma^3 \tau_2 ) \sum_{u \in K(t)} \sum_{u' \in D^\prime(u)} \mathbb{E} \|   \widetilde{v}^{\psi(u')}_{u'} \|^2
 \nonumber \\
 &\stackrel{(a)}{\leq}&
 -\frac{\gamma}{2}\left( \frac{1}{2}  \sum_{u\in K(t)} \mathbb{E}\| \nabla_{\mathcal{G}_{\psi(u)}} f({w}_{t})\|^2
  -  L^2 \gamma^2 \eta_1 \sum_{u\in K(t)} \sum_{u' \in \{t,\cdots,u\}} \mathbb{E} \|\widetilde{v}_{u'}^{\psi(u')}\|^2 \right)
 \nonumber \\&&
 + \frac{L_* \gamma^2}{2}\sum_{u \in K(t)}  \mathbb{E} \|\widetilde{v}^{\psi(u)}_{ u }\|^2
+ ( L_{{*}}^2  \gamma^3 \tau_1 + 4L_{*}^2  \gamma^3 \tau_2 ) \tau_2  \sum_{u \in K(t)}  \sum_{u' \in D^\prime(u)} \mathbb{E} \|   \widetilde{v}^{\psi(u')}_{u'} \|^2
 \nonumber \\&=&
 -\frac{\gamma}{4} \sum_{u\in K(t)} \mathbb{E}\| \nabla_{\mathcal{G}_{\psi(u)}} f({w}_{t})\|^2
 +\frac{L^2 \gamma^3 \eta_1}{2} \sum_{u \in K(t)}\sum_{u' \in \{t,...,u\}} \mathbb{E} \|\widetilde{v}_{u'}^{\psi(u')}\|^2
\nonumber \\ &&
 + \frac{L_* \gamma^2}{2}\sum_{u \in K(t)} \mathbb{E} \|\widetilde{v}^{\psi(u)}_{ u }\|^2
+ ( L_{{*}}^2  \gamma^3 \tau_1 + 4L_{*}^2  \gamma^3 \tau_2 )  \sum_{u \in K(t)}  \sum_{u' \in D^\prime(u)} \mathbb{E} \|   \widetilde{v}^{\psi(u')}_{u'} \|^2
 \nonumber \\ &\stackrel{(b)}{\leq}&
 -\frac{\gamma\mu}{2}(f(w_{t}) - f(w^*)) + \underbrace{(\frac{L_*^2 \gamma^3 \tau^{3/2}}{2}
 + \frac{L_*\gamma^2\tau^{1/2}}{2} + 5 { L_{*}^2 \gamma^3 \tau^{3/2}} )\frac{2G}{1-2L_*^2\gamma^2\tau}}_{C} \nonumber
\end{eqnarray}
where (a) follows from Lemma~\ref{lem-csgd-3}, (b) follows from Assumption~\ref{assumc1}. According to Eq.~\ref{csgd-9}, we have
\begin{align}\label{csgd-10}
\mathbb{E} \left[f (w_{t + |K(t)|}) - f(w^{*}) \right]\leq (1-\frac{\gamma\mu}{2})(f(w_{t}) - f(w^*)) + C
\end{align}
Assuming that $\cup_{\kappa \in P(t)}=\{0,1, \ldots, t\}$, applying Eq.~\ref{csgd-10}, we have that
\begin{align}\label{csgd--11}
  &\mathbb{E} \left[f (w_{t}) - f (w^*)) \right]
\\
  &\leq (1-\frac{\gamma\mu}{2})^{v(t)}(f(w_{0}) - f(w^*)) + C\sum_{i=0}^{v(t)}(1-\frac{\gamma\mu}{2})^{i}
  \nonumber \\ & \leq
  (1-\frac{\gamma\mu}{2})^{v(t)}(f(w_{0}) - f(w^*)) + C\frac{2(1-(1-\frac{\gamma\mu}{2})^{v(t)})}{\gamma \mu}
  \nonumber \\ &\leq
 (1-\frac{\gamma\mu}{2})^{v(t)}(f(w_{0}) - f(w^*)) + C\frac{2(1-(1-\frac{\gamma\mu}{2})^{v(t)})}{\gamma \mu}
 \nonumber \\ &\stackrel{(a)}{\leq}
 (1-\frac{\gamma\mu}{2})^{v(t)}(f(w_{0}) - f(w^*)) + (\frac{L_*^2 \gamma^3 \tau^{3/2}}{2}
 + \frac{L_*\gamma^2\tau^{1/2}}{2} + 5 { L_{*}^2 \gamma^3 \tau^{3/2}} )\frac{2G}{1-2L_*^2\gamma^2\tau}\frac{2}{\gamma \mu} \nonumber,
\end{align}
where (a) follows form the definition of $C$. To obtain the $\epsilon$ solution one can choose suitable $\gamma$, such that
\begin{align}\label{csgd-91-1}
1-2L_*^2\gamma^2\tau & >0
\end{align}
\begin{align}\label{csgd-91-2}
  (1-\frac{\gamma\mu}{2})^{v(t)} \left(f(w_{0}) - f(w^*)\right) & \leq \frac{\epsilon}{2}
\end{align}
\begin{align}\label{csgd-91-3}
  (\frac{L_*^2 \gamma^3 \tau^{3/2}}{2}
 + \frac{L_*\gamma^2\tau^{1/2}}{2} + 5 { L_{*}^2 \gamma^3 \tau^{3/2}} )\frac{2G}{1-2L_*^2\gamma^2\tau}\frac{2}{\gamma \mu} & \leq \frac{\epsilon}{2}.
\end{align}
 According to Eq.~\ref{csgd-91-1}, there is $\gamma^2<\frac{1}{2L_*^2\tau}$, which implies that $ \frac{L_*^2 \gamma^3 \tau^{3/2}}{2}
 + \frac{L_*\gamma^2\tau^{1/2}}{2} + 5 { L_{*}^2 \gamma^3 \tau^{3/2}} \leq 6 L_*^2\gamma^3\tau^{3/2}$ (here we assume that $L_*$ can be chosen a value $\geq 1$, this is reasonable from the definition of $L_*$). Thus, we can rewrite Eq.~\ref{csgd-91-3} as
  \begin{align}\label{1}
  6L_*^2\gamma^3\tau^{3/2}\frac{4G}{\mu(1-2L_*^2\gamma^2\tau)} & \leq \frac{\epsilon}{2}.
 \end{align}
 which implies that if $\tau$ is upper bounded, \ie, $\tau\leq {\text {min}}\{\epsilon^{-4/3}, \frac{(GL_*^2)^{2/3}}{\epsilon^2\mu^{2/3}}\}$, we can carefully choose $\gamma\leq \frac{\epsilon\mu^{1/3}}{(G{96L_*^2})^{1/3}}$
 such that Eq.~\ref{csgd-91-3} holds. According to Eq.~\ref{csgd-91-2}, there is
 \begin{align}\label{1}
   \text{log} (\frac{2(f(w_0)-f(w^*))}{\epsilon}) \leq v(t) \text{log}(\frac{1}{1-\frac{\gamma \mu}{2}})
 \end{align}
Because $log(\frac{1}{x}) \geq 1-x$ for $0< x \leq 1 $, we have
 \begin{align}\label{11}
   v(t) \geq \frac{2}{\gamma \mu} log (\frac{2(f(w_0)-f(w^*))}{\epsilon})
   \stackrel{(a)}{\geq} \frac{44(GL^2_*)^{1/3}}{ \mu^{4/3}\epsilon} log (\frac{2(f(w_0)-f(w^*))}{\epsilon})
 \end{align}
 This complets the proof.
 \end{proof}
\subsection{Proof of Theorem~\ref{thm-svrgconvex}}
\begin{lemma}\label{lem-csvrg-1}
For VF{${\textbf{B}}^2$}-SVRG, let $u\in K(t)$ for $\forall t$, we have that
one can get:
\begin{eqnarray}\label{lemeq-csvrg-1}
 \mathbb{E}  \left \|   \widetilde{v}^{\psi(u) }_u \right \|^2
\leq  \frac{18G}{1-2  L_*^2 \gamma^2\tau}
\end{eqnarray}
\end{lemma}
\begin{proof}[\textbf{Proof of  Lemma~\ref{lem-csvrg-1}:}]
First, we prove the relation between $\mathbb{E}\|\widetilde{v}_u^{\psi(u)}\|^2$ and $\mathbb{E}\|\widehat{v}_u^{\psi(u)}\|^2$.
\begin{eqnarray}\label{svrg_1}
\mathbb{E} ||\widetilde{v}_{u}^{\psi(u)}||^2  &=&  \mathbb{E} || \widetilde{v}^{\psi(u)}_u - \widehat{v}^{\psi(u)}_u +\widehat{v}^{\psi(u)}_u||^2 \nonumber \\
&\stackrel{a}{\leq}&  2 \mathbb{E} ||\widetilde{v}^{\psi(u)}_u - \widehat{v}^{\psi(u)}_u ||^2  +  2 \mathbb{E} ||\widehat{v}^{\psi(u)}_u||^2 \nonumber \\
\end{eqnarray}
where (a) follows from $\|a+b\|^2\leq 2\|a\|^2+ 2\|b\|^2$. The upper bound to $\mathbb{E}   \left \|   \widetilde{v}^{\psi(u)}_u - \widehat{v}^{\psi(u)}_u \right  \|^2$ can be obtained as follows.
\begin{eqnarray}\label{svrg_2}
  \mathbb{E}   \left \|   \widetilde{v}^{\psi(u)}_u - \widehat{v}^{\psi(u)}_u \right  \|^2
 &=&\mathbb{E}   \left \|  \left(\nabla_{\mathcal{G}_\ell} f_i (\widetilde{w}_{u}^s)  - \nabla_{\mathcal{G}_\ell} f_i (\widehat{w}_{u}^s)\right) \right  \|^2 \nonumber
\\ \nonumber &\stackrel{a}{\leq}& {L^2}  \mathbb{E}   \left \| \widetilde{w}_{{u}}^s  - \widehat{w}_{u}^s \right  \|^2
\\ \nonumber &\stackrel{b}{=}& {L^2 \gamma^2}   \mathbb{E}   \left \|  \sum_{u' \in D(u)}    \textbf{U}_{\psi(u')} \widetilde{v}^{\psi(u')}_{u'}  \right  \|^2
\\ \nonumber &\stackrel{c}{\leq} & {\tau_2 L^2 \gamma^2} \mathbb{E}  \sum_{u' \in D(u)}   \left \|    \widetilde{v}^{\psi(u')}_{u'}  \right  \|^2 \nonumber \\
\end{eqnarray}
where (a) follows from Assumption~2, (b) follows from Eq.~\ref{Dt1}, (c) follows from Assumption~\ref{assum4}. Combining  Eqs.~(\ref{svrg_1}) and (\ref{svrg_2}), we have that
\begin{eqnarray}\label{svrg_3}
 && \mathbb{E}  \left \|   \widetilde{v}^{\psi(u)}_u \right \|^2
\\ \nonumber  &\leq &    2 \mathbb{E}   \left \|   \widetilde{v}^{\psi(u)}_u - \widehat{v}^{\psi(u)}_u \right  \|^2
+   2 \mathbb{E} \left  \| \widehat{v}^{\psi(u)}_u \right \|^2
\\ \nonumber &\leq &   2 \tau_2 L^2 \gamma^2 \mathbb{E}  \sum_{u' \in D(u)}   \left \|    \widetilde{v}^{\psi(u')}_{u'}  \right  \|^2
 + 2 \mathbb{E} \left  \| \widehat{v}^{\psi(u)}_u \right \|^2
\end{eqnarray}
Then following the analyses of \ffl \ref{lem-csgd-1} , we have
\begin{eqnarray}\label{svrg_3}
 \mathbb{E}  \left \|   \widetilde{v}^{\psi(u)}_u \right \|^2
\leq  \frac{18G}{1-2  L_*^2 \gamma^2\tau_2^2}
\end{eqnarray}
This completes the proof
\end{proof}
\begin{lemma}\label{lem-csvrg-2}
Given the conditions in Theorem~\ref{thm-svrgconvex}, let $u\in K(t)$, we have that:
\begin{align}\label{svrg-lemeq-1}
  &\mathbb{E}\|\widetilde{v}_{u}^{{\psi(u)}} \|^2   \leq
  \frac{16 L^{2}}{\mu} \mathbb{E}\left(f\left(w_{t}^{s}\right)-f\left(w^{*}\right)\right)+\frac{8 L^{2}}{\mu} \mathbb{E}\left(f\left(w^{s}\right)-f\left(w^{*}\right)\right)+ 8 L^{2} \gamma^{2} \eta_{1} \sum_{u' \in\{t, \ldots, u\}} \mathbb{E}\left\|\widetilde{v}_{u'}^{\psi(u')}\right\|^{2}
  \nonumber \\
   & + 4{ L_{{*}}^2  \gamma^2 \tau_1}  \sum_{u' \in D(u)} \mathbb{E} \|   \widetilde{v}^{\psi(u')}_{u'} \|^2
 + 16 { L_{*}^2  \gamma^2 \tau_2  \sum_{u' \in D^\prime(u)}} \mathbb{E} \|   \widetilde{v}^{\psi(u')}_{u'} \|^2
\end{align}
\end{lemma}
\begin{proof}[\textbf{Proof of  Lemma~\ref{lem-csvrg-2}:}]
Define $\mathbb{E}\|v_{u}^{{\psi(u)}}\|^{2}=\mathbb{E}\left\|\nabla_{\mathcal{G}_{\psi(u)}} f_{i}\left(w_{u}^{s}\right)-\nabla_{\mathcal{G}_{\psi(u)}} f_{i}\left(w^{s}\right)+\nabla_{\mathcal{G}_{\psi(u)}} f\left(w^{s}\right)\right\|^{2}$, we have that $\mathbb{E}\|\widetilde{v}_{u}^{{\psi(u)}}\|^2 = \mathbb{E}\|\widetilde{v}_{u}^{{\psi(u)}} - {v}_{u}^{{\psi(u)}}  + {v}_{u}^{{\psi(u)}}\|^2 \leq 2\mathbb{E}\|\widetilde{v}_{u}^{{\psi(u)}} - {v}_{u}^{{\psi(u)}} \|^2 + 2\mathbb{E}\|{v}_{u}^{{\psi(u)}}\|^2 $. First we give the upper bound to $\mathbb{E}\|{v}_{u}^{{\psi(u)}}\|^2$ as follows.
\begin{align}\label{svrg-8}
&\mathbb{E}\|v_{u}^{{\psi(u)}}\|^{2}
\nonumber\\
&=\mathbb{E}\left\|\nabla_{\mathcal{G}_{\psi(u)}} f_{i}\left(w_{u}^{s}\right)-\nabla_{\mathcal{G}_{\psi(u)}} f_{i}\left(w^{s}\right)+\nabla_{\mathcal{G}_{\psi(u)}} f\left(w^{s}\right)\right\|^{2}
\nonumber\\&
=\mathbb{E}\left\|\nabla_{\mathcal{G}_{\psi(u)}} f_{i}\left(w_{u}^{s}\right)-\nabla_{\mathcal{G}_{\psi(u)}} f_{i}\left(w^{*}\right)-\nabla_{\mathcal{G}_{\psi(u)}} f_{i}\left(w^{s}\right)+\nabla_{\mathcal{G}_{\psi(u)}} f_{i}\left(w^{*}\right)+\nabla_{\mathcal{G}_{\psi(u)}} f\left(w^{s}\right)\right\|^{2}
 \nonumber\\&
\stackrel{(a)}{\leq} 2 \mathbb{E}\left\|\nabla_{\mathcal{G}_{\psi(u)}} f_{i}\left(w_{u}^{s}\right)-\nabla_{\mathcal{G}_{\psi(u)}} f_{i}\left(w^{*}\right)\right\|^{2}+2 \mathbb{E}\left\|\nabla_{\mathcal{G}_{\psi(u)}} f_{i}\left(w^{s}\right)-\nabla_{\mathcal{G}_{\psi(u)}} f_{i}\left(w^{*}\right)-\nabla_{\mathcal{G}_{\psi(u)}} f\left(w^{s}\right)+\nabla_{\mathcal{G}_{\psi(u)}} f\left(w^{*}\right)\right\|^{2}
\nonumber\\&
\leq 2 \mathbb{E}\left\|\nabla_{\mathcal{G}_{\psi(u)}} f_{i}\left(w_{u}^{s}\right)-\nabla_{\mathcal{G}_{\psi(u)}} f_{i}\left(w^{*}\right)\right\|^{2}+2 \mathbb{E}\left\|\nabla_{\mathcal{G}_{\psi(u)}} f_{i}\left(w^{s}\right)-\nabla_{\mathcal{G}_{\psi(u)}} f_{i}\left(w^{*}\right)\right\|^{2}
\nonumber\\&
\stackrel{(b)}{\leq} 2 L^{2} \mathbb{E}\left\|w_{u}^{s}-w^{*}\right\|^{2}+2 L^{2} \mathbb{E}\left\|w^{s}-w^{*}\right\|^{2}
\nonumber\\&
=2 L^{2} \mathbb{E}\left\|w_{u}^{s}-w_{t}^{s}+w_{t}^{s}-w^{*}\right\|^{2}+2 L^{2} \mathbb{E}\left\|w^{s}-w^{*}\right\|^{2} \nonumber\\&
\stackrel{(c)}{\leq} 4 L^{2} \mathbb{E}\left\|w_{u}^{s}-w_{t}^{s}\right\|^{2}+4 L^{2} \mathbb{E}\left\|w_{t}^{s}-w^{*}\right\|^{2}+2 L^{2} \mathbb{E}\left\|w^{s}-w^{*}\right\|^{2} \nonumber\\&
\stackrel{(d)}{=} \quad 4 L^{2} \gamma^{2} \mathbb{E}\left\|\sum_{u' \in\{t, \ldots, u\}} \mathbf{U}_{\psi(u)} \widetilde{v}_{u'}^{\psi(u')}\right\|^{2}+4 L^{2} \mathbb{E}\left\|w_{t}^{s}-w^{*}\right\|^{2}+2 L^{2} \mathbb{E}\left\|w^{s}-w^{*}\right\|^{2} \nonumber\\&
\stackrel{(e)}{\leq} \frac{8 L^{2}}{\mu} \mathbb{E}\left(f\left(w_{t}^{s}\right)-f\left(w^{*}\right)\right)+\frac{4 L^{2}}{\mu} \mathbb{E}\left(f\left(w^{s}\right)-f\left(w^{*}\right)\right)+4 L^{2} \gamma^{2} \eta_{1} \sum_{u' \in\{t, \ldots, u\}} \mathbb{E}\left\|\widetilde{v}_{u'}^{\psi(u')}\right\|^{2},
\end{align}
where (a) and (c) follow from  $\|a+b\|^2 \leq 2\|a\|^2 + 2\|b\|^2$, (b) follows from Assumption~\ref{assum1}, (d) follows from Eq.~\ref{global-up1}, and (e) follows from Assumption~\ref{assumc1}.
Next we give the upper bound of $\mathbb{E}\|\widetilde{v}_{u}^{{\psi(u)}} - {v}_{u}^{{\psi(u)}} \|^2$. Following the proof of Lemma~\ref{lem-csgd-2}, we have
\begin{align}\label{111}
  \mathbb{E}\|\widetilde{v}_{u}^{{\psi(u)}} - {v}_{u}^{{\psi(u)}} \|^2 \leq  2{ L_{{*}}^2  \gamma^2 \tau_1}  \sum_{t' \in D(t)} \mathbb{E} \|   \widetilde{v}^{\psi(t')}_{t'} \|^2
 + 8 { L_{*}^2  \gamma^2 \tau_2  \sum_{t' \in D^\prime(t)}} \mathbb{E} \|   \widetilde{v}^{\psi(t')}_{t'} \|^2
\end{align}
combing above two equalities, we have
\begin{align}\label{112}
  &\mathbb{E}\|\widetilde{v}_{u}^{{\psi(u)}} \|^2   \nonumber \\
  &\leq
  \frac{16 L^{2}}{\mu} \mathbb{E}\left(f\left(w_{t}^{s}\right)-f\left(w^{*}\right)\right)+\frac{8 L^{2}}{\mu} \mathbb{E}\left(f\left(w^{s}\right)-f\left(w^{*}\right)\right)+ 8 L^{2} \gamma^{2} \eta_{1} \sum_{u' \in\{t, \ldots, u\}} \mathbb{E}\left\|\widetilde{v}_{u'}^{\psi(u')}\right\|^{2}
  \nonumber \\
   & + 4{ L_{{*}}^2  \gamma^2 \tau_1}  \sum_{u' \in D(u)} \mathbb{E} \|   \widetilde{v}^{\psi(u')}_{u'} \|^2
 + 16 { L_{*}^2  \gamma^2 \tau_2  \sum_{u' \in D^\prime(u)}} \mathbb{E} \|   \widetilde{v}^{\psi(u')}_{u'} \|^2
\end{align}
This completes the proof.
\end{proof}
\begin{proof}[\textbf{Proof of  Theorem~\ref{thm-svrgconvex}:}]
Similar to Eq.~\ref{csgd-8}, for $u \in K(t)$ at $s$-th outer loop, we have that
\begin{eqnarray}\label{csvrg-5}
&& \mathbb{E} f (w_{u+1}^{s})
\\ \nonumber &\stackrel{ (a) }{\leq}&  \mathbb{E} \left ( f (w_{u}^{s}) + \langle \nabla f(w_{u}^{s}), w_{u+1}^{s}-w_{u}^{s}  \rangle + \frac{L}{2} \|w_{u+1}^{s}-w_{u}^{s}   \|^2  \right )
\\ \nonumber &=&  \mathbb{E} \left ( f (w_{u}^{s}) -  \gamma \langle \nabla f(w_{u}^{s}),
\widetilde{v}^{\psi(u)}_{u}  \rangle + \frac{L\gamma^2}{2} \|  \widehat{v}^{\psi(u)}_{u}  \|^2  \right )
\\ \nonumber &{=}&  \mathbb{E} \left ( f (w_{u}^{s}) -  \gamma \langle \nabla f(w_{u}^{s}),  \widetilde{v}^{\psi(u)}_{u} + {v}^{\psi(u)}_{u}- {v}^{\psi(u)}_{u} \rangle  + \frac{L \gamma^2}{2} \|  \widetilde{v}^{\psi(u)}_{u}  \|^2  \right )
\\ \nonumber &\stackrel{(b)}{=}&  \mathbb{E}  f (w_{u}^{s}) -  \gamma \mathbb{E} \langle \nabla f(w_{u}^{s}),  \nabla_{\mathcal{G}_{\psi(u)}} f ({w}_{u}^{s}) \rangle  + \frac{L \gamma^2}{2} \mathbb{E} \|  \widetilde{v}^{\psi(u)}_{u}  \|^2  + \gamma \mathbb{E} \langle \nabla f(w_{u}^{s}), {v}^{\psi(u)}_{u} - \widetilde{v}^{\psi(u)}_{u} \rangle
 \\ \nonumber &\stackrel{ (c) }{\leq}&  \mathbb{E}  f (w_{u}^{s}) -  \gamma \mathbb{E} \|  \nabla_{\mathcal{G}_{\psi(u)}} f ({w}_{u}^{s}) \|^2  + \frac{\gamma}{2} \mathbb{E} \| \nabla_{\mathcal{G}_{\psi(u)}} f ({w}_{u}^{s}) \|^2 + \frac{L \gamma^2}{2} \mathbb{E} \|  \widetilde{v}^{\psi(u)}_{u}  \|^2 + \frac{\gamma}{2} \mathbb{E} \| \widetilde{v}^{\psi(u)}_{u} - {v}^{\psi(u)}_{u} \|^2
  \\  &\stackrel{(d)}{\leq}&  \mathbb{E}  f (w_{u}^{s}) -  \frac{\gamma}{2} \mathbb{E} \|  \nabla_{\mathcal{G}_{\psi(u)}} f ({w}_{u}^{s}) \|^2  + \frac{L_* \gamma^2}{2} \mathbb{E} \|  \widetilde{v}^{\psi(u)}_{u}  \|^2
   + ( L_{{*}}^2  \gamma^2 \tau_1 + 4  L_{*}^2  \gamma^2 \tau_2)  \sum_{t' \in D^\prime(t)} \mathbb{E} \|   \widetilde{v}^{\psi(t')}_{t'} \|^2
 \end{eqnarray}
Summing  Eq.~(\ref{csvrg-5}) over all $ u \in K(t) $, we obtain
\begin{eqnarray}\label{svrg-6}
&& \mathbb{E} \left[f (w_{t+|K(t)|}^{s} - f (w_{t}^{s}) \right]
\\ \nonumber
&\leq&
 -\frac{\gamma}{2}\sum_{u \in K(t)}  \mathbb{E} \|  \nabla_{\mathcal{G}_{\psi(u)}} f ({w}_{u}^{s}) \|^2
 + \frac{L_* \gamma^2}{2}\sum_{u \in K(t)} \mathbb{E} \|  \widetilde{v}^{\psi(u)}_{ u }\|^2
+ ( L_{{*}}^2  \gamma^2 \tau_1 + 4  L_{*}^2  \gamma^2 \tau_2)  \sum_{t' \in D^\prime(t)} \mathbb{E} \|   \widetilde{v}^{\psi(t')}_{t'} \|^2
\\ \nonumber
 &\stackrel{(a)}{\leq}&
 -\frac{\gamma}{2}\left( \frac{1}{2}  \sum_{u\in K(t)} \mathbb{E}\| \nabla_{\mathcal{G}_{\psi(u)}} f({w}_{t}^{s})\|^2
  -  L^2 \gamma^2 \eta_1 \sum_{u\in K(t)} \sum_{u' \in \{t,\cdots,u\}} \mathbb{E} \|\widetilde{v}_{u'}^{\psi(u')}\|^2 \right)
\\ \nonumber &&
 + \frac{L_* \gamma^2}{2}\sum_{u \in K(t)}  \mathbb{E} \|\widetilde{v}^{\psi(u)}_{ u }\|^2
+ ( L_{{*}}^2  \gamma^2 \tau_1 + 4  L_{*}^2  \gamma^2 \tau_2) \sum_{u \in K(t)}  \sum_{u' \in D^\prime(u)} \mathbb{E} \|   \widetilde{v}^{\psi(u')}_{u'} \|^2
\\ \nonumber &=&
 -\frac{\gamma}{4} \sum_{u\in K(t)} \mathbb{E}\| \nabla_{\mathcal{G}_{\psi(u)}} f({w}_{t}^{s})\|^2
 +\frac{L^2 \gamma^3 \eta_1}{2} \sum_{u \in K(t)}\sum_{u' \in \{t,...,u\}} \mathbb{E} \|\widetilde{v}_{u'}^{\psi(u')}\|^2
\\ \nonumber &&
 + \frac{L_* \gamma^2}{2}\sum_{u \in K(t)} \mathbb{E} \|\widetilde{v}^{\psi(u)}_{ u }\|^2
+ ( L_{{*}}^2  \gamma^2 \tau_1 + 4  L_{*}^2  \gamma^2 \tau_2) \sum_{u \in K(t)}  \sum_{u' \in D^\prime(u)} \mathbb{E} \|   \widetilde{v}^{\psi(u')}_{u'} \|^2
\\ \nonumber &\stackrel{(b)}{\leq}&
 -\frac{\gamma\mu}{2}(f(w_{t}^{s}) - f(w^*))
 + ( L_{{*}}^2  \gamma^2 \tau_1 + 4  L_{*}^2  \gamma^2 \tau_2) \sum_{u \in K(t)}  \sum_{u' \in D^\prime(u)} \mathbb{E} \|   \widetilde{v}^{\psi(u')}_{u'} \|^2
 \\ \nonumber
 && + \underbrace{(\frac{L_*^2 \gamma^3 \tau}{2} + \frac{L_*\gamma^2}{2})}_{C}\sum_{u\in K(t)}
 \biggl( \frac{16 L^{2}}{\mu} \mathbb{E}\left(f\left(w_{t}^{s}\right)-f\left(w^{*}\right)\right)
 +\frac{8 L^{2}}{\mu} \mathbb{E}\left(f\left(w^{s}\right)-f\left(w^{*}\right)\right)
  \nonumber \\
   & & + 8 L^{2} \gamma^{2} \eta_{1} \sum_{u' \in\{t, \ldots, u\}} \mathbb{E}\left\|\widetilde{v}_{u'}^{\psi(u')}\right\|^{2}
   + 4{ L_{{*}}^2  \gamma^2 \tau_1}  \sum_{u' \in D(u)} \mathbb{E} \|   \widetilde{v}^{\psi(u')}_{u'} \|^2
 + 16 { L_{*}^2  \gamma^2 \tau_2  \sum_{u' \in D^\prime(u)}} \mathbb{E} \|   \widetilde{v}^{\psi(u')}_{u'} \|^2 \biggr) \nonumber
\end{eqnarray}
where (a) \ff Lemma~\ref{lem-csgd-3}, (b) \ff Lemma~\ref{lem-csvrg-2}.

Let $e_t^s = \mathbb{E}(f(w_t^s)-f(w^*))$ and $e^s = \mathbb{E}(f(w^s)-f(w^*))$, we have
\begin{align}\label{svrg--1}
& e_{t + |K(t)|}^s
\nonumber \\ &
\leq(1-\frac{\gamma\mu}{2}+ \frac{16L^2\eta_1C}{\mu})e_t^s
  + \frac{8L^2\eta_1C}{\mu} e^s
  + 8 C L^{2} \gamma^{2} \eta_{1}\sum_{u\in K(t)} \sum_{u' \in\{t, \ldots, u\}} \mathbb{E}\left\|\widetilde{v}_{u'}^{\psi(u')}\right\|^{2}
  \nonumber \\
   & +(5 { L_{*}^2 \gamma^3 \tau^{1/2}} + 4C{ L_{{*}}^2}\gamma^2\tau_1 + 16C { L_{*}^2}  \gamma^2 \tau_2 ) \sum_{u\in K(t)}  \sum_{u' \in D^\prime(u)} \mathbb{E} \|   \widetilde{v}^{\psi(u')}_{u'} \|^2
   \nonumber \\
   & \leq(1-\frac{\gamma\mu}{2}+ \frac{16L^2\eta_1C}{\mu})e_t^s
  + \frac{8L^2\eta_1C}{\mu} e^s
  +(28 C L^{2} \gamma^{2}\tau^{3/2}+ 5 { L_{*}^2 \gamma^3 \tau^{3/2}}  )  \frac{18G}{1 - 2 L_{*}^2\gamma^2\tau}
\end{align}
We carefully choose $\gamma$ such that $\frac{\gamma\mu}{2}- \frac{16L^2\eta_1C}{\mu}\stackrel{def}{=}\rho >0$. Assume that $\cup_{\kappa \in P(t)}=\{0,1, \ldots, t\}$, applying above, we have
\begin{align}\label{svrg-111}
 & e_t^s  \nonumber  \\
 & \leq (1-\rho)^{v(t)}e^s + \left( \frac{8L^2\eta_1C}{\mu} e^s
  +(28 C L^{2} \gamma^{2}\tau^{3/2}+ 5 { L_{*}^2 \gamma^3 \tau^{3/2}}  )  \frac{18G}{1 - 2 L_{*}^2\gamma^2\tau}\right)\sum_{i=0}^{v(t)}(1-\rho)^i
  \nonumber \\
  & \leq (1-\rho)^{v(t)}e^s + \left( \frac{8L^2\eta_1C}{\mu} e^s
  +(28 C L^{2} \gamma^{2}\tau^{3/2}+ 5 { L_{*}^2 \gamma^3 \tau^{3/2}}  )  \frac{18G}{1 - 2 L_{*}^2\gamma^2\tau}\right)\frac{1}{\rho}
  \nonumber \\
  &\leq \left((1-\rho)^{v(t)} + \frac{8L^2\eta_1C}{\rho\mu}\right)e^s
  + (28 C L^{2} \gamma^{2}\tau^{3/2}+ 5 { L_{*}^2 \gamma^3 \tau^{3/2}}  )  \frac{18G}{\rho(1 - 2 L_{*}^2\gamma^2\tau)}
\end{align}
Thus, to achieve the accuracy $\epsilon$ of, for VF{${\textbf{B}}^2$}-SVRG, i.e., $\mathbb{E} f\left(w_{S}\right)-f\left(w^{*}\right)\leq \epsilon$, we can carefully choose $\gamma$ such that
\begin{align}\label{saga-111}
  \frac{8L^2\eta_1C}{\rho\mu}& \leq 0.05
  \nonumber \\
  (28 C L^{2} \gamma^{2}\tau^{3/2}+ 5 { L_{*}^2 \gamma^3 \tau^{3/2}}  )  \frac{18G}{\rho(1 - 2 L_{*}^2\gamma^2\tau)} & \leq \frac{\epsilon}{8}
\end{align}
And then let $(1-\rho)^{v(t)}\leq 0.25$, i.e., $v(t)\geq \frac{{\text {log}} 0.25}{{\text {log}}(1-\rho)}$, we have that
\[e^{s+1} \leq 0.75 e^s + \frac{\epsilon}{8}\]
Recursively apply above equality, we have that
\[e^S \leq (0.75)^Se^0 + \frac{\epsilon}{2}\]
Finally, the outer loop number $S$ should satisfy the condition of $S\geq \frac{{\text {log}}\frac{2e^0}{\epsilon}}{{\text {log}}\frac{4}{3}}$ and epoch number $v(t)$ in an outer loop should satisfy $v(t)\geq \frac{{\text {log}} 0.25}{{\text {log}}(1-\rho)}$. This completes the proof.
\end{proof}
\subsection{Proof of Theorem~\ref{thm-sagaconvex}}
First we introduce following notations. $\phi(t)$  denotes the corresponding local time counter on the party $\psi(t)$. Given a local time counter $u$ and $\ell$-th party, $\xi(u,\ell)$ denotes the corresponding global time counter not only satisfying $\phi \left(\xi(u,\ell)\right)=u$ but also $\psi\left(\xi(u,\ell)\right) = \ell$.
\begin{lemma} \label{lem-csaga-1} For VF{${\textbf{B}}^2$}-SAGA,  we have that
\begin{align}\label{csaga-1}
&\mathbb{E} \|  \hat{\alpha}_{i_t}^{t,\psi(t)} -  \nabla_{\mathcal{G}_{\psi(t)}} f_{i_t}(w^{*})  \|^2
  \nonumber \\
&  \leq \frac{1}{n} \sum_{t'=1}^{\phi(t)-1} \sum_{i=1}^n \left ( \frac{1}{n} \left ( 1 -\frac{1}{n} \right )^{\phi(t)-t'-1}  \mathbb{E}  \left \|  \nabla_{\mathcal{G}_{\psi(t)}} f_i(\hat{w}_{{\xi(t',\psi(t))}}) -  \nabla_{\mathcal{G}_{\psi(t)}} f_i(w^{*}) \right \|^2 \right )
\nonumber  \\
&      +  \frac{1}{n}  \sum_{i=1}^n   \left ( 1 -\frac{1}{n} \right )^{\phi(t)-1}  \mathbb{E}  \left \|  \nabla_{\mathcal{G}_{\psi(t)}} f_i({w}_{{0}}) -  \nabla_{\mathcal{G}_{\psi(t)}} f_i(w^{*}) \right \|^2
\\
& \mathbb{E} \left \| {\alpha}_{i_t}^{t,\psi(t)} -  \nabla_{\mathcal{G}_{\psi(t)}} f_{i_t}(w^{*})  \right \|^2
\nonumber \\
&  \leq \frac{1}{n} \sum_{t'=1}^{\phi(t)-1} \sum_{i=1}^n \left ( \frac{1}{n} \left ( 1 -\frac{1}{n} \right )^{\phi(t)-t'-1}  \mathbb{E}  \left \|  \nabla_{\mathcal{G}_{\psi(t)}} f_i({w}_{{\xi(t',\psi(t))}}) -  \nabla_{\mathcal{G}_{\psi(t)}} f_i(w^{*}) \right \|^2 \right )
\nonumber  \\
&      +  \frac{1}{n}  \sum_{i=1}^n   \left ( 1 -\frac{1}{n} \right )^{\phi(t)-1}  \mathbb{E}  \left \|  \nabla_{\mathcal{G}_{\psi(t)}} f_i(\hat{w}_{{0}}) -  \nabla_{\mathcal{G}_{\psi(t)}} f_i(w^{*}) \right \|^2
   \\
  &\mathbb{E}\left \|  {\alpha}_{i_t}^{t,\psi(t)} - \widehat{\alpha}_{i_t}^{t,\psi(t)} \right \|^2
\leq \frac{\tau_1 L^2 \gamma^2}{n} \sum_{t'=1}^{\phi(t)-1}  \sum_{{u} \in D(\xi(t',\psi(t)))} \left ( 1 -\frac{1}{n} \right )^{\phi(t)-t'-1} \mathbb{E}  \left \|       \widetilde{v}^{\psi({u})}_{{u}} \right \|^2
   \\
  &\mathbb{E}\left \|  \widetilde{\alpha}_{i_t}^{t,\psi(t)} - \widehat{\alpha}_{i_t}^{t,\psi(t)} \right \|^2
\leq \frac{4\tau_2 L^2 \gamma^2}{n} \sum_{t'=1}^{\phi(t)-1}  \sum_{{u} \in D'(\xi(t',\psi(t)))} \left ( 1 -\frac{1}{n} \right )^{\phi(t)-t'-1} \mathbb{E}  \left \|       \widetilde{v}^{\psi({u})}_{{u}} \right \|^2
\end{align}
\end{lemma}
\begin{proof}[\textbf{Proof of Lemma~\ref{lem-csaga-1}:}] Firstly, we have that
\begin{eqnarray}\label{csaga-2}
&& \mathbb{E} \left \| \hat{\alpha}_{i_t}^{t,\psi(t)} -  \nabla_{\mathcal{G}_{\psi(t)}} f_{i_t}({w}_{t})  \right \|^2
= \frac{1}{n} \sum_{i=1}^n \mathbb{E} \left \|  \hat{\alpha}_{i}^{t,\psi(t)} -  \nabla_{\mathcal{G}_{\psi(t)}} f_{i}({w}_{t})  \right \|^2
\\  &  = & \nonumber \frac{1}{n} \sum_{i=1}^n   \mathbb{E} \sum_{t'=0}^{\phi(t) -1} \mathbf{1}_{ \{ \textbf{u}_{i}^u =t' \}} \left \|   \nabla_{\mathcal{G}_\ell} f_i(\hat{w}_{\xi(t',\psi(t))}) -  \nabla_{\mathcal{G}_\ell} f_{i}(w^{*})  \right \|^2
\\  &  = & \nonumber \frac{1}{n} \sum_{t'=0}^{\phi(t)-1} \sum_{i=1}^n  \mathbb{E}  \mathbf{1}_{ \{ \textbf{u}_{i}^u =t' \}} \left \|  \nabla_{\mathcal{G}_\ell} f_i(\hat{w}_{\xi(t',\psi(t))}) -  \nabla_{\mathcal{G}_\ell} f_{i}(w^{*})  \right \|^2
\end{eqnarray}
where  $\textbf{u}_{i}^u$ denote the last iterate  to update the $\widehat{\alpha}_i^{t,\psi(t)}$. Note that, we do not distinguish $\psi(t)$ and $\psi(t')$ because they correspond to the same party. We consider two cases including $t'>0$ and $t'=0$ as follows.

For $t'>0$,  we have that
\begin{eqnarray}\label{csaga-3}
&& \mathbb{E} \left ( \mathbf{1}_{ \{ \textbf{u}_{i}^u =t' \}} \left \|   \nabla_{\mathcal{G}_\ell} f_{i}(\hat{w}_{{\xi(t',\psi(t))}}) -  \nabla_{\mathcal{G}_\ell} f_{i}(w_{t,\psi(t)}) \right \|^2 \right )
\\  &  \stackrel{ (a) }{\leq}  & \nonumber
\mathbb{E} \left ( \mathbf{1}_{ \{ i_{t'} = i \}} \mathbf{1}_{ \{ i_v \neq i, \forall v \ s.t. \ t'+ 1 \leq v \leq \phi(t) -1 \}}  \left \|   \nabla_{\mathcal{G}_\ell} f_{i}(\hat{w}_{{\xi(t',\psi(t))}}) -  \nabla_{\mathcal{G}_\ell} f_{i}(w^{*}) \right \|^2 \right )
\\  &  \stackrel{ (b) }{\leq}  & \nonumber
 P{ \{ i_{t'} = i \}}  P { \{ i_v \neq i, \forall v
 \ s.t. \ t'+ 1 \leq v \leq \phi(t) -1 \}}  \mathbb{E}   \left \|   \nabla_{\mathcal{G}_\ell} f_{i}(\hat{w}_{{\xi(t',\psi(t))}}) -  \nabla_{\mathcal{G}_\ell} f_{i}(w^{*}) \right \|^2
\\  &  \stackrel{ (c) }{\leq}  & \nonumber \frac{1}{n} \left ( 1 -\frac{1}{n} \right )^{\phi(t) -1-t' }  \mathbb{E}  \left \|  \nabla_{\mathcal{G}_\ell} f_{i}(\hat{w}_{{\xi(t',\psi(t))}}) -  \nabla_{\mathcal{G}_\ell} f_{i}(w^{*}) \right \|^2
\end{eqnarray}
where the inequality (a) uses the fact $i_{t'}$ and $i_v$ are independent for $v \neq t'$, the inequality (b) uses the fact that $P{ \{ i_t = i \}} = \frac{1}{n}$ and $P { \{ i_v \neq i\} } =1-\frac{1}{n}$.

For $t'=0$, we have that
\begin{eqnarray}\label{csaga-4}
&& \mathbb{E} \left ( \mathbf{1}_{ \{ \textbf{u}_{i}^u =0 \}}\left \|   \nabla_{\mathcal{G}_{\psi(t)}} f_{i}(\hat{w}_{0}) -  \nabla_{\mathcal{G}_{\psi(t)}} f_{i}(w_{t,\psi(t)}) \right \|^2   \right )
\\  &  \leq & \nonumber
\mathbb{E} \left ( \mathbf{1}_{ \{ i_v \neq i, \forall v \
 s.t. \ 0 \leq v \leq \phi(t)-1 \}} \left \|   \nabla_{\mathcal{G}_{\psi(t)}} f_{i}(\hat{w}_{{0}}) -  \nabla_{\mathcal{G}_{\psi(t)}} f_{i}(w^{*}) \right \|^2 \right )
\\  &  \leq & \nonumber
P { \{ i_v \neq i, \forall v \
s.t. \ 0 \leq v \leq \phi(t-\tau_3)-1 \}} \mathbb{E}  \left \|   \nabla_{\mathcal{G}_{\psi(t)}} f_{i}(\hat{w}_{0}) -  \nabla_{\mathcal{G}_{\psi(t)}} f_{i}(w^{*}) \right \|^2
\\  &  \leq & \nonumber  \left ( 1 -\frac{1}{n} \right )^{\phi(t)}  \mathbb{E}  \left \|   \nabla_{\mathcal{G}_{\psi(t)}} f_{i}(\hat{w}_{{0}}) -  \nabla_{\mathcal{G}_{\psi(t)}} f_{i}(w^{*}) \right \|^2
\end{eqnarray}

Substituting Eqs.~\ref{csaga-4} and \ref{csaga-3} into \ref{csaga-2}, we have:
\begin{eqnarray}\label{csaga-5}
&& \mathbb{E} \left \| \hat{\alpha}_{i_t}^{t,\psi(t)} -  \nabla_{\mathcal{G}_{\psi(t)}} f_{i_t}(w^{*})  \right \|^2\nonumber
\\  &  = & \nonumber
\frac{1}{n} \sum_{t'=0}^{\phi(t)-1}\sum_{i=1}^n  \mathbb{E}  \mathbf{1}_{ \{ \textbf{u}_{i}^u =t' \}} \left \|  \nabla_{\mathcal{G}_{\psi(t)}} f_i(\hat{w}_{{\xi(t',\psi(t))}}) -  \nabla_{\mathcal{G}_{\psi(t)}} f_i(w^{*}) \right \|^2
\\  &  \stackrel{ (a) }{\leq} & \nonumber \frac{1}{n} \sum_{t'=1}^{\phi(t)-1} \sum_{i=1}^n \left ( \frac{1}{n} \left ( 1 -\frac{1}{n} \right )^{\phi(t)-t'-1}  \mathbb{E}  \left \|  \nabla_{\mathcal{G}_{\psi(t)}} f_i(\hat{w}_{{\xi(t',\psi(t))}}) -  \nabla_{\mathcal{G}_{\psi(t)}} f_i(w^{*}) \right \|^2 \right )
\\  &    & \nonumber  +  \frac{1}{n}  \sum_{i=1}^n   \left ( 1 -\frac{1}{n} \right )^{\phi(t)-1}  \mathbb{E}  \left \|  \nabla_{\mathcal{G}_{\psi(t)}} f_i(\hat{w}_{{0}}) -  \nabla_{\mathcal{G}_{\psi(t)}} f_i(w^{*}) \right \|^2
\end{eqnarray}
 Combing the fact that $w$ represents the consistent read and thus $\tau_3=0$ with (\ref{csaga-5}), we have
\begin{eqnarray}\label{csaga-6}
&& \mathbb{E} \left \| {\alpha}_{i_t}^{t,\psi(t)} -  \nabla_{\mathcal{G}_{\psi(t)}} f_{i_t}(w^{*})  \right \|^2
\\  &  \leq & \nonumber \frac{1}{n} \sum_{t'=1}^{\phi(t)-1} \sum_{i=1}^n \left ( \frac{1}{n} \left ( 1 -\frac{1}{n} \right )^{\phi(t)-t'-1}  \mathbb{E}  \left \|  \nabla_{\mathcal{G}_{\psi(t)}} f_i({w}_{{\xi(t',\psi(t))}}) -  \nabla_{\mathcal{G}_{\psi(t)}} f_i(w^{*}) \right \|^2 \right )
\\  &    & \nonumber  +  \frac{1}{n}  \sum_{i=1}^n   \left ( 1 -\frac{1}{n} \right )^{\phi(t)-1}  \mathbb{E}  \left \|  \nabla_{\mathcal{G}_{\psi(t)}} f_i({w}_{{0}}) -  \nabla_{\mathcal{G}_{\psi(t)}} f_i(w^{*}) \right \|^2
\\  &  \leq & \nonumber \frac{L^2}{n} \sum_{t'=1}^{\phi(t)-1}  \left ( 1 -\frac{1}{n} \right )^{\phi(t)-t'-1}  \sigma ({w}_{{\xi(t',\psi(t))}}) + L^{2}\left ( 1 -\frac{1}{n} \right )^{\phi(t)} \sigma ({w}_{0}),
\end{eqnarray}
where $\sigma\left(w_{u}\right)=\mathbb{E}\left\|w_{u}-w^{*}\right\|^{2}$.
Similarly, we have that
\begin{eqnarray}\label{cSAGA-7}
&&  \mathbb{E}\left \|  {\alpha}_{i_t}^{t,\psi(t)} - \widehat{\alpha}_{i_t}^{t,\psi(t)} \right \|^2
=  \frac{1}{n} \sum_{i=1}^n \mathbb{E} \left \| {\alpha}_{i}^{t,\psi(t)} - \widehat{\alpha}_{i}^{t,\psi(t)} \right \|^2
\\ \nonumber  &  = & \frac{1}{n} \sum_{t'=0}^{\phi(t)-1}  \sum_{i=1}^n \mathbb{E} \mathbf{1}_{ \{ \textbf{u}_{i}^u =t'\} }   \left \| {\alpha}_{i}^{t',\psi(t')} - \widehat{\alpha}_{i}^{t',\psi(t')} \right \|^2
\\  &  \stackrel{ (a) }{\leq} & \nonumber \frac{1}{n}\sum_{i=1}^n \sum_{t'=1}^{\phi(t)-1} \left (  \frac{1}{n} \left ( 1 -\frac{1}{n} \right )^{\phi(t)-t'-1}  \mathbb{E}  \left \| {\alpha}_{i}^{t',\psi(t')} - \widehat{\alpha}_{i}^{t',\psi(t')} \right \|^2 \right )
\\  &    & \nonumber  +  \frac{1}{n}\sum_{i=1}^n  \left ( 1 -\frac{1}{n} \right )^{\phi(t)-1}  \mathbb{E}   \left \| {\alpha}_{i}^{0,\ell} - \widehat{\alpha}_{i}^{0,\ell} \right \|^2
\\  & \stackrel{ (b) }{=} & \nonumber \frac{1}{n} \sum_{t'=1}^{\phi(t)-1} \sum_{i=1}^n \left ( \frac{1}{n} \left ( 1 -\frac{1}{n} \right )^{\phi(t)-t'-1}  \mathbb{E}  \left \| {\alpha}_{i}^{t',\psi(t')} - \widehat{\alpha}_{i}^{t',\psi(t)} \right \|^2 \right )
\\  &  = & \nonumber
\frac{1}{n} \sum_{t'=1}^{\phi(t)-1} \left ( 1 -\frac{1}{n} \right )^{\phi(t)-t'-1} \mathbb{E}  \left \|  \nabla_{\mathcal{G}_\ell} f_i(\widehat{w}_{\xi(t',\psi(t))}) -  \nabla_{\mathcal{G}_\ell} f_i(w_{\xi(t',\psi(t))}) \right \|^2
\\  &  \stackrel{ (c) }{\leq} & \nonumber
\frac{L^2}{n} \sum_{t'=1}^{\phi(t)-1} \left ( 1 -\frac{1}{n} \right )^{\phi(t)-t'-1} \mathbb{E}  \left \|  \widehat{w}_{\xi(t',\psi(t))} - w_{\xi(t',\psi(t))} \right \|^2
\\  &  = & \nonumber \frac{L^2 \gamma^2}{n}
\sum_{t'=1}^{\phi(t)-1} \left ( 1 -\frac{1}{n} \right )^{\phi(t)-t'-1} \mathbb{E}  \left \|   \sum_{{u} \in D(\xi(t',\psi(t)))}    \textbf{U}_{\psi({u})} \widetilde{v}^{\psi({u})}_{{u}} \right \|^2
\\  &  \stackrel{ (d) }{\leq} & \nonumber
\frac{\tau_1 L^2 \gamma^2}{n} \sum_{t'=1}^{\phi(t)-1}  \sum_{{u} \in D(\xi(t',\psi(t)))} \left ( 1 -\frac{1}{n} \right )^{\phi(t)-t'-1} \mathbb{E}  \left \|       \widetilde{v}^{\psi({u})}_{{u}} \right \|^2
\end{eqnarray}
where the inequality (a) can be obtained similar to (\ref{csaga-5}) (note that $\widehat{\alpha}_i^{t',\psi(t')}$ is an inconsistent read of ${\alpha}_i^{t',\psi(t')}$ its time interval can be overlapped by that of ${\alpha}_i^{t',\psi(t')}$  ), the equality (b) uses the fact of ${\alpha}_{i}^{0,\ell} = \widehat{\alpha}_{i}^{0,\ell}$, the inequality (c) uses Assumption {3}, and the  inequality (d) uses Assumption~\ref{assum4}. Moreover, we have
\begin{eqnarray}\label{saga-7}
&&  \mathbb{E}\left \|  \widetilde{\alpha}_{i_t}^{t,\psi(t)} - \widehat{\alpha}_{i_t}^{t,\psi(t)} \right \|^2
=  \frac{1}{n} \sum_{i=1}^n \mathbb{E} \left \| \widetilde{\alpha}_{i}^{t,\psi(t)}
- \widehat{\alpha}_{i}^{t,\psi(t)} \right \|^2
\\ \nonumber  &  = & \frac{1}{n} \sum_{t'=0}^{\phi(t)-1}  \sum_{i=1}^n \mathbb{E} \mathbf{1}_{ \{ \textbf{u}_{i}^u =t'\} }   \left \| \widetilde{\alpha}_{i}^{t',\psi(t')}
- \widehat{\alpha}_{i}^{t',\psi(t)} \right \|^2
\\  & \leq& \nonumber
\frac{1}{n}\sum_{i=1}^n \sum_{t'=1}^{\phi(t)-1} \left (  \frac{1}{n} \left ( 1 -\frac{1}{n} \right )^{\phi(t)-t'-1}  \mathbb{E}  \left \| \widetilde{\alpha}_{i}^{t',\psi(t)}
 - \widehat{\alpha}_{i}^{t',\psi(t)} \right \|^2 \right )
+  \frac{1}{n}\sum_{i=1}^n  \left ( 1 -\frac{1}{n} \right )^{\phi(t)-1}  \mathbb{E}   \left \| \widetilde{\alpha}_{i}^{0,\ell}
- \widehat{\alpha}_{i}^{0,\ell} \right \|^2
\\  & = & \nonumber
\frac{1}{n} \sum_{t'=1}^{\phi(t)-1} \sum_{i=1}^n \left ( \frac{1}{n} \left ( 1 -\frac{1}{n} \right )^{\phi(t)-t'-1}  \mathbb{E}  \left \| \widetilde{\alpha}_{i}^{t',\psi(t)}
- \widehat{\alpha}_{i}^{t',\psi(t)} \right \|^2 \right )
 \\  &  \stackrel{(a)}{\leq} & \nonumber
 \frac{1}{n} \sum_{t'=1}^{\phi(t)-1} \left ( 1 -\frac{1}{n} \right )^{\phi(t)-t'-1} \biggl(2 \mathbb{E}   \|  \nabla_{\mathcal{G}_\ell} f_i(\widehat{w}_{\xi(t',\psi(t))})
 - \nabla_{\mathcal{G}_\ell} f_i(\bar{w}_{\xi(t',\psi(t))}) \|^2
\\
 &  & + \nonumber 2\mathbb{E} \|\nabla_{\mathcal{G}_\ell} g(\widehat{w}_{\xi(t',\psi(t))})  - \nabla_{\mathcal{G}_\ell} g(\bar{w}_{\xi(t',\psi(t))})\|^2 \biggr)
 \\  &  \stackrel{(b)}{\leq} & \nonumber
 \frac{1}{n} \sum_{t'=1}^{\phi(t)-1} \left ( 1 -\frac{1}{n} \right )^{\phi(t)-t'-1} \biggl(2 L^2\mathbb{E}   \|  \widehat{w}_{\xi(t',\psi(t))}
 - \bar{w}_{\xi(t',\psi(t))} \|^2
+ 2 L_g^2\mathbb{E} \| \widehat{w}_{\xi(t',\psi(t))}  - \bar{w}_{\xi(t',\psi(t))}\|^2 \biggr)
\\  &  \stackrel{ (c) }{\leq} & \nonumber
\frac{4L_*^2}{n} \sum_{t'=1}^{\phi(t)-1} \left ( 1 -\frac{1}{n} \right )^{\phi(t)-t'-1} \mathbb{E}  \left \|  \widehat{w}_{\xi(t',\psi(t))} - \bar{w}_{\xi(t',\psi(t))} \right \|^2
\\  &  = & \nonumber
\frac{4L_*^2 \gamma^2}{n} \sum_{t'=1}^{\phi(t)-1} \left ( 1 -\frac{1}{n} \right )^{\phi(t)-t'-1} \mathbb{E}  \left \|   \sum_{{u} \in D'(\xi(t',\psi(t)))}    \textbf{U}_{\psi({u})} \widetilde{v}^{\psi({u})}_{{u}} \right \|^2
\\  &  \stackrel{ (d) }{\leq} & \nonumber
\frac{4\tau_2 L_*^2 \gamma^2}{n} \sum_{t'=1}^{\phi(t)-1}  \sum_{{u} \in D'(\xi(t',\psi(t)))} \left ( 1 -\frac{1}{n} \right )^{\phi(t)-t'-1} \mathbb{E}  \left \|       \widetilde{v}^{\psi({u})}_{{u}} \right \|^2
\end{eqnarray}
(a)-(d) can be obtained from the analyses of Lemma \ref{lem-csgd-2}. This completes the proof.
\end{proof}
\begin{lemma}\label{lem-csaga-3}
Given a global iteration number $u$, we let $\left\{\bar{u}_{0}, \bar{u}_{1}, \ldots, \bar{u}_{v(u)-1\}}\right.$ be the all start iteration number for the global time counters from 0 to u. Thus, for VF{${\textbf{B}}^2$}-SAGA, we have that
\begin{align}\label{saga-lem2}
  \mathbb{E}\|v_u^{\psi(u)}\|^2 \leq   4 \frac{L^{2} \eta_{1}}{l} \sum_{k^{\prime}=1}^{v(u)}\left(1-\frac{1}{l}\right)^{v(u)-k^{\prime}} \sigma\left(w_{\bar{u}_{k^{\prime}}}\right)
     + 2 L^{2}\left(1-\frac{1}{l}\right)^{v(u)} \sigma\left(w_{0}\right)+4 L^{2} \sigma\left(w_{\varphi(u)}\right)+8 L^{2} \gamma^{2} \eta_{1}^{2} q G
\end{align}
\end{lemma}
\begin{proof}[\bf{Proof of Lemma~\ref{lem-csaga-3}}]
W have that
\begin{align}\label{saga-111}
  & \mathbb{E}\|v_u^{\psi(u)}\|^2
  \nonumber \\
  & = \mathbb{E}\|\nabla_{\mathcal{G}_{\psi(t)}} f_i (w^{*}) - \alpha_i^{{\psi(t)}} +  \frac{1}{n} \sum_{i=1}^n \alpha_i^{{\psi(t)}}\|^2
  \nonumber \\
  & = \mathbb{E}\left\|\nabla_{\mathcal{G}_{\psi(u)}} f_{i_{u}}\left(w_{u}\right)-\nabla_{\mathcal{G}_{\psi(u)}} f_{i_{u}}\left(w^{*}\right)-\alpha_{i_{u}}^{u, \ell}+\nabla_{\mathcal{G}_{\psi(u)}} f_{i_{u}}\left(w^{*}\right)+\frac{1}{n} \sum_{i=1}^{n} \alpha_{i}^{u, \ell}-\nabla_{\mathcal{G}_{\psi(u)}} f\left(w^{*}\right)+\nabla_{\mathcal{G}_{\psi(u)}} f\left(w^{*}\right)\right\|^{2}
  \nonumber \\
  & \stackrel{(a)}{\leq}
 2 \mathbb{E}\left\|\nabla_{\mathcal{G}_{\psi(u)}} f_{i_{u}}\left(w^{*}\right)-\alpha_{i_{u}}^{u, \ell}+\frac{1}{n} \sum_{i=1}^{n} \alpha_{i}^{u, \ell}-\nabla_{\mathcal{G}_{\psi(u)}} f\left(w^{*}\right)\right\|^{2}+2 \mathbb{E}\left\|\nabla_{\mathcal{G}_{\psi(u)}} f_{i}\left(w_{u}\right)-\nabla_{\mathcal{G}_{\psi(u)}} f_{i_{t}}\left(w^{*}\right)\right\|^{2}
  \nonumber \\
  &\stackrel{(b)}{\leq}
   2 \mathbb{E}\left\|\alpha_{i_{u}}^{u, \ell}-\nabla_{\mathcal{G}_{\psi(u)}} f_{i_{t}}\left(w^{*}\right)\right\|^{2}+2 \mathbb{E}\left\|\nabla_{\mathcal{G}_{\psi(u)}} f_{i}\left(w_{u}\right)-\nabla_{\mathcal{G}_{\psi(u)}} f_{i_{t}}\left(w^{*}\right)\right\|^{2}
   \nonumber \\
   & \stackrel{(c)}{\leq}
 2 \frac{L^{2}}{n} \sum_{u^{\prime}=1}^{\phi(u)-1}\left(1-\frac{1}{n}\right)^{\phi(u)-u^{\prime}-1} \mathbb{E}\left\|w_{\xi\left(u^{\prime}, \ell\right)}-w^{*}\right\|^{2}+2 L^{2}\left(1-\frac{1}{n}\right)^{\phi(u)} \sigma\left(w_{0}\right)
   \nonumber \\
   & + 2L^2\mathbb{E}\|w_u-w^*\|^2
   \nonumber \\
   & =
   2 \frac{L^{2}}{n} \sum_{u^{\prime}=1}^{\phi(u)-1}\left(1-\frac{1}{n}\right)^{\phi(u)-u^{\prime}-1} \mathbb{E}\left\|w_{\xi\left(u^{\prime}, \ell\right)}-w_{\varphi\left(\xi\left(u^{\prime}, \ell\right)\right)}+w_{\varphi\left(\xi\left(u^{\prime}, \ell\right)\right)}-w^{*}\right\|^{2}
   \nonumber \\ &
   +2 L^{2}\left(1-\frac{1}{n}\right)^{\phi(u)} \sigma\left(w_{0}\right)+2 L^{2} \mathbb{E}\left\|w_{u}-w_{\varphi(u)}+w_{\varphi(u)}-w^{*}\right\|^{2}
   \nonumber \\ &
   \stackrel{(d)}{\leq}
   2 \frac{L^{2}}{n} \sum_{u^{\prime}=1}^{\phi(u)-1}\left(1-\frac{1}{n}\right)^{\phi(u)-u^{\prime}-1} \mathbb{E}\left(2\left\|w_{\xi\left(u^{\prime}, \ell\right)}-w_{\varphi\left(\xi\left(u^{\prime}, \ell\right)\right)}\right\|^{2}+2\left\|w_{\varphi\left(\xi\left(u^{\prime}, \ell\right)\right)}-w^{*}\right\|^{2}\right)
   \nonumber \\ &
   + 2 L^{2}\left(1-\frac{1}{n}\right)^{v(u)} \sigma\left(w_{0}\right)+ 4 L^{2} \mathbb{E}\left\|w_{\varphi(u)}-w^{*}\right\|^{2}+4 L^{2} \gamma^{2} \mathbb{E}\left\|\sum_{v \in\{\varphi(u), \ldots, u\}} \mathbf{U}_{\psi(v)} \widehat{v}_{v}^{\psi(v)}\right\|
   \nonumber \\ & \stackrel{(e)}{\leq}
   2 \frac{L^{2}}{n} \sum_{u^{\prime}=1}^{\phi(u)-1}\left(1-\frac{1}{n}\right)^{\phi(u)-u^{\prime}-1} \mathbb{E}\left(2 \eta_{1} \gamma^{2} \sum_{v \in\left\{\varphi\left(\xi\left(u^{\prime}, \ell\right)\right), \ldots, \xi\left(u^{\prime}, \ell\right)\right\}}\left\|\widehat{v}_{v}^{\psi(v)}\right\|^{2}+2\left\|w_{\varphi\left(\xi\left(u^{\prime}, \ell\right)\right)}-w^{*}\right\|^{2}\right)
   \nonumber \\ & +
   2 L^{2}\left(1-\frac{1}{n}\right)^{v(u)} \sigma\left(w_{0}\right)+4 L^{2} \mathbb{E}\left\|w_{\varphi(u)}-w^{*}\right\|^{2}+4 L^{2} \gamma^{2} \eta_{1} \sum_{v \in\{\varphi(u), \ldots, u\}} \mathbb{E}\left\|\widehat{v}_{v}^{\psi(v)}\right\|^{2}
   \nonumber \\ & \stackrel{(f)}{\leq}
   2 \frac{L^{2}}{n} \sum_{u^{\prime}=1}^{\phi(u)-1}\left(1-\frac{1}{n}\right)^{\phi(u)-u^{\prime}-1} \mathbb{E}\left(2 \eta_{1}^{2} \gamma^{2}  \lambda_{\gamma}G
   +2\left\|w_{\varphi\left(\xi\left(u^{\prime}, \ell\right)\right)}-w^{*}\right\|^{2}\right)
   \nonumber \\ &
   +2 L^{2}\left(1-\frac{1}{n}\right)^{v(u)} \sigma\left(w_{0}\right)+4 L^{2} \mathbb{E}\left\|w_{\varphi(u)}-w^{*}\right\|^{2}+4 L^{2} \gamma^{2} \eta_{1}^{2} \lambda_{\gamma}G
   \nonumber \\ & \stackrel{(g)}{\leq}
   4 \frac{L^{2}}{n} \sum_{u^{\prime}=1}^{\phi(u)-1}\left(1-\frac{1}{n}\right)^{\phi(u)-u^{\prime}-1} \mathbb{E}\left\|w_{\varphi\left(\xi\left(u^{\prime}, \ell\right)\right)}-w^{*}\right\|^{2}
    \nonumber \\ &
    +2 L^{2}\left(1-\frac{1}{n}\right)^{v(u)} \sigma\left(w_{0}\right)+4 L^{2} \mathbb{E}\left\|w_{\varphi(u)}-w^{*}\right\|^{2}+8 L^{2} \gamma^{2} \eta_{1}^{2} \lambda_{\gamma}G
    \nonumber \\ &  \stackrel{(h)}{\leq}
    4 \frac{L^{2} \eta_{1}}{n} \sum_{k^{\prime}=1}^{v(u)}\left(1-\frac{1}{n}\right)^{v(u)-k^{\prime}} \sigma\left(w_{\bar{u}_{k^{\prime}}}\right)
  + 2 L^{2}\left(1-\frac{1}{n}\right)^{v(u)} \sigma\left(w_{0}\right)+4 L^{2} \sigma\left(w_{\varphi(u)}\right)+8 L^{2} \gamma^{2} \eta_{1}^{2} \lambda_{\gamma}G
\end{align}
where (a) and (d) uses $\|\sum_{i=1}^{n}a_i^2\|^2\leq n \sum_{i=1}^{n}\|a_i\|^2$, (b) \ff $\mathbb{E}\|x-\mathbb{E} x\|^{2} \leq \mathbb{E}\|x\|^{2}$, (c) uses Lemma~\ref{lem-csaga-1}, (e) \ff the bound of $|K(t)|$, and (f) follows from Assumption~\ref{assum1}, (g) uses the fact $\sum_{u^{\prime}=1}^{\phi(u)-1}\left(1-\frac{1}{n}\right)^{\phi(u)-u^{\prime}-1}\leq n$.
\end{proof}

\begin{lemma}\label{lem-csaga-2}
For all $\forall $ $\psi (t)$, there are
\begin{equation}\label{csaga-lem2}
 \mathbb{E} ||\widetilde{v}_{t}^{\psi(t)}||^2 \leq \lambda_{\gamma} G
\end{equation}
where $\lambda_{\gamma}= \frac{18}{1 - 72L_{*}^2\gamma^2\tau }$
\end{lemma}
\begin{proof}[\bf{Proof of Lemma~\ref{lem-csaga-2}}]
we give the upper bound to $\mathbb{E}    \|  \widetilde{v}^{ {\psi(u)} }_t - \widehat{v}^{{\psi(u)}}_t  \|^2$ as follows.
We have that
\begin{eqnarray}\label{csaga-8}
&&   \mathbb{E} \left \| \widetilde{v}^{ {\psi(u)} }_t - \widehat{v}^{{\psi(u)}}_t \right \|^2
\\  &  = & \nonumber
\mathbb{E} \left \|   \left(\nabla_{\mathcal{G}_{\psi(t)}}\mathcal{L}(\bar{w})
+ \nabla_{\mathcal{G}_{\psi(t)}} g((\widehat{w}_t)_{\mathcal{G}_{\psi(t)}}) \right)
- \nabla_{\mathcal{G}_{\psi(t)}} f(\widehat{w}_t)
- \widehat{\alpha}_{i}^{\psi(u)}   + \widetilde{\alpha}_{i}^{\psi(u)}
+ \frac{1}{n} \sum_{i=1}^n \widehat{\alpha}_{i}^{\psi(u)}   - \frac{1}{n} \sum_{i=1}^n \widetilde{\alpha}_{i}^{\psi(u)}  \right \|^2
\\  &  \stackrel{ (a) }{\leq} & \nonumber
3  \mathbb{E} Q_1
+ 3\mathbb{E} \underbrace{\left \| \widetilde{\alpha}_{i}^{\psi(u)}
+  \widehat{\alpha}_{i}^{\psi(u)} \right \|^2 }_{Q_2}
+ 3\mathbb{E} \underbrace{\left \|  \frac{1}{n} \sum_{i=1}^n \widetilde{\alpha}_{i}^{t,\psi(t)} - \frac{1}{n} \sum_{i=1}^n \widehat{\alpha}_{i}^{t,\psi(t)}  \right \|^2}_{Q_3}
\end{eqnarray}
where $Q_1 = \left \|  \left(\nabla_{\mathcal{G}_{\psi(t)}}\mathcal{L}(\bar{w})
+ \nabla_{\mathcal{G}_{\psi(t)}} g((\widehat{w}_t)_{\mathcal{G}_{\psi(t)}}) \right)
- \nabla_{\mathcal{G}_{\psi(t)}} f(\widehat{w}_t) \right\|$ and  inequality (a) uses $\| \sum_{i=1}^n a_i \|^2 \leq n \sum_{i=1}^n \| a_i \|^2 $.
We will give the upper bounds for the expectations  of $Q_1$, $Q_2$ and $Q_3$  respectively.
\begin{eqnarray}\label{csaga-9}
 \nonumber \mathbb{E} Q_1 &=& \mathbb{E} \left \|  \left(\nabla_{\mathcal{G}_{\psi(t)}}\mathcal{L}(\bar{w})
+ \nabla_{\mathcal{G}_{\psi(t)}} g((\widehat{w}_t)_{\mathcal{G}_{\psi(t)}}) \right)
- \nabla_{\mathcal{G}_{\psi(t)}} f(\widehat{w}_t) \right\|
 \\&\stackrel{(a)}{\leq}&
  \mathbb{E}  ||\nabla_{\mathcal{G}_{\psi(t)}} f(\bar{w}_t) - \nabla_{\mathcal{G}_{\psi(t)}} f(\widehat{w}_t) + \nabla_{\mathcal{G}_{\psi(t)}} g((\widehat{w}_t)_{\mathcal{G}_{\psi(t)}}) - \nabla_{\mathcal{G}_{\psi(t)}} g((\bar{w}_t)_{\mathcal{G}_{\psi(t)}})||^2
  \nonumber \\
 &\stackrel{(b)}{\leq}& 4{ L_{{*}}^2  \gamma^2 \tau_2}  \sum_{t' \in D'(t)} \mathbb{E} \|   \widetilde{v}^{\psi(t')}_{t'} \|^2
\end{eqnarray}
above inequality can be obtained by following the proof of Lemma~\ref{lem-csgd-2}.
\begin{eqnarray}\label{csaga-10}
\mathbb{E} Q_2 &=&   \mathbb{E}\left \| \widetilde {\alpha}_{i_t}^{t,\psi(t)} - \widehat{\alpha}_{i_t}^{t,\psi(t)} \right \|^2
\\ \nonumber &\leq & \frac{4\tau_2 L_*^2 \gamma^2}{n} \sum_{t'=1}^{\phi(t)-1}  \sum_{{u} \in D'(\xi(t',\psi(t)))} \left ( 1 -\frac{1}{n} \right )^{\phi(t)-t'-1} \mathbb{E}  \left \|       \widetilde{v}^{\psi({u})}_{{u}} \right \|^2
\end{eqnarray}
where the inequality uses Lemma \ref{lem-csaga-1}.
\begin{eqnarray}\label{cSAGA-12}
&& \mathbb{E} Q_3 =    \mathbb{E} \left \|  \frac{1}{n} \sum_{i=1}^n \widetilde{\alpha}_{i}^{t,\psi(t)} - \frac{1}{n} \sum_{i=1}^n \widehat{\alpha}_{i}^{t,\psi(t)}   \right \|^2
\\  &  \leq & \nonumber
\frac{1}{n} \sum_{i=1}^n \mathbb{E} \left \|  \widehat{\alpha}_{i}^{t,\psi(t)}  - \widehat{\alpha}_{i}^{t,\psi(t)}   \right \|^2
\\  &  \leq &
\frac{4\tau_2 L_*^2 \gamma^2}{n} \sum_{t'=1}^{\phi(t)-1}  \sum_{{u} \in D'(\xi(t',\psi(t)))} \left ( 1 -\frac{1}{n} \right )^{\phi(t)-t'-1} \mathbb{E}  \left \|       \widetilde{v}^{\psi({u})}_{{u}} \right \|^2 \nonumber
\end{eqnarray}
where  the first inequality uses $\| \sum_{i=1}^n a_i \|^2 \leq n \sum_{i=1}^n \| a_i \|^2 $, the second inequality uses Lemma \ref{lem-csaga-1}. Combining \ref{csaga-9}, \ref{csaga-10}, and \ref{cSAGA-12}, one can obtain:
\begin{eqnarray}\label{cSAGA-13}
&&    \mathbb{E} \left \| \widetilde{v}_{t}^{\psi(t)}  - \widehat{v}_t^{\psi(t)}  \right \|^2
\\  &  \leq & \nonumber    3  \mathbb{E} {Q_1} + 3 \mathbb{E} {Q_2}  + 3\mathbb{E} {Q_3}
\\  &  \leq & \nonumber {12 L_*^2\gamma^2\tau_2 } \sum_{u \in D'(\xi(t',\psi(t)))} \mathbb{E}  ||\widetilde{v}^{\psi(u)}_{u}||^2
+ \frac{24\tau_2 L_*^2 \gamma^2}{n} \sum_{t'=1}^{\phi(t)-1}  \sum_{u \in D'(\xi(t',\psi(t)))} \left ( 1 -\frac{1}{n} \right )^{\phi(t)-t'-1} \mathbb{E}  \left \|       \widetilde{v}^{\psi({u})}_{{u}} \right \|^2
\end{eqnarray}
Combining $\mathbb{E}\| \widetilde{v}_{t}^{\psi(t)}\|^2 \leq 2\mathbb{E} \left \| \widetilde{v}_{t}^{\psi(t)}  - \widehat{v}_t^{\psi(t)}  \right \|^2 + 2\mathbb{E} \left \|\widehat{v}_t^{\psi(t)}  \right \|^2$ with Eq.~\ref{cSAGA-13} and following the analyses of Lemma~\ref{lem-csvrg-1}, we have
\begin{equation}\label{csaga-13}
 \mathbb{E} ||\widetilde{v}_{t}^{\psi(t)}||^2 \leq \frac{18G}{1 - 72L_{*}^2\gamma^2\tau },
\end{equation}
This completes the proof.
\end{proof}
Moreover,
define ${v}^{{\psi(t)}}_t= \nabla_{\mathcal{G}_{\psi(t)}} f_i (w^{*}) - \alpha_i^{{\psi(t)}} +  \frac{1}{n} \sum_{i=1}^n \alpha_i^{{\psi(t)}}$.
And then, we give the upper bound to $\mathbb{E}   \left \|   \widehat{v}^{ {\psi(t)} }_t - {v}^{{\psi(t)}}_t \right  \|^2$ as follows.
We have that
\begin{eqnarray}\label{csaga-14}
&&   \mathbb{E} \left \| \widehat{v}_{t}^{\psi(t)} -  v_{t}^{\psi(t)}  \right \|^2
\\  &  = & \nonumber  \mathbb{E} \left \|  \nabla_{\mathcal{G}_{\psi(t)}} f_{i_t} (\widehat{w}_{t})- \widehat{\alpha}_{i_t}^{t,\psi(t)}  + \frac{1}{n} \sum_{i=1}^n \widehat{\alpha}_{i}^{t,\psi(t)}  - \nabla_{\mathcal{G}_{\psi(t)}} f_{i_t} (w^{*}) + \alpha_{i_t}^{t,\psi(t)} - \frac{1}{n} \sum_{i=1}^n \alpha_{i}^{t,\psi(t)} \right \|^2
\\  &  \stackrel{ (a) }{\leq} & \nonumber  3  \mathbb{E} \underbrace{\left \|  \nabla_{\mathcal{G}_{\psi(t)}} f_{i_t} (\widehat{w}_{t})-  \nabla_{\mathcal{G}_{\psi(t)}} f_{i_t} ({w}_{t}) \right \|^2 }_{Q_4}
+ 3\mathbb{E} \underbrace{\left \| {\alpha}_{i_t}^{t,\psi(t)} - \widehat{\alpha}_{i_t}^{t,\psi(t)} \right \|^2 }_{Q_5}
+ 3\mathbb{E} \underbrace{\left \|  \frac{1}{n} \sum_{i=1}^n \alpha_{i}^{t,\psi(t)} - \frac{1}{n} \sum_{i=1}^n \widehat{\alpha}_{i}^{t,\psi(t)}  \right \|^2}_{Q_6}
\end{eqnarray}
where the  inequality (a) uses $\| \sum_{i=1}^n a_i \|^2 \leq n \sum_{i=1}^n \| a_i \|^2 $.
We will give the upper bounds for the expectations  of $Q_4$, $Q_5$ and $Q_6$  respectively.
\begin{eqnarray}\label{csaga-15}
 \nonumber \mathbb{E} Q_4 &=& \mathbb{E} \left \|   \nabla_{\mathcal{G}_{{\psi(t)}}} f_{i_t} (\widehat{w}_{t})-  \nabla_{\mathcal{G}_{{\psi(t)}}} f_{i_t} ({w}_{t}) \right \|^2 \\
 &\stackrel{(a)}{\leq}&  { L_{{\psi(t)}}^2}
\mathbb{E} ||w^{*} - \widehat{w}_{t}||^2
\nonumber \\
 &=& {L_{{\psi(t)}}^2\gamma^2}  \mathbb{E} ||\sum_{t' \in D(u)}    \textbf{U}_{\psi(t')} \widetilde{v}^{\psi(t')}_{t'} ||^2
 \nonumber \\
 &\stackrel{(b)}{\leq}& { \tau_1 L_*^2\gamma^2 }    \sum_{t' \in D(u)} \mathbb{E} \left[ ||\widetilde{v}^{\psi(t')}_{t'}||^2 \right] \nonumber
\end{eqnarray}
where (a) uses Assumption~2, (b) uses $\| \sum_{i=1}^n a_i \|^2 \leq n \sum_{i=1}^n \| a_i \|^2 $. Similar to the analyses of $Q_2$ and $Q_3$, we have
\begin{eqnarray}\label{csaga-16}
\mathbb{E} Q_5 =   \mathbb{E}\left \|  {\alpha}_{i_t}^{t,\psi(t)} - \widehat{\alpha}_{i_t}^{t,\psi(t)} \right \|^2
 \leq   \frac{\tau_1 L^2 \gamma^2}{n} \sum_{t'=1}^{\phi(t)-1}  \sum_{\widetilde{u} \in D(\xi(t',\psi(t')))} \left ( 1 -\frac{1}{n} \right )^{\phi(t)-t'-1} \mathbb{E}  \left \|       \widetilde{v}^{\psi(\widetilde{u})}_{\widetilde{u}} \right \|^2
\end{eqnarray}
where the inequality uses Lemma \ref{lem-csaga-1}.
\begin{eqnarray}\label{csaga-17}
 &&\mathbb{E} Q_6 =    \mathbb{E} \left \|  \frac{1}{n} \sum_{i=1}^n {\alpha}_{i}^{t,\psi(t)} - \frac{1}{n} \sum_{i=1}^n \widehat{\alpha}_{i}^{t,\psi(t)}   \right \|^2
 \nonumber \\
 &\leq&  \frac{\tau_1 L^2 \gamma^2}{n} \sum_{t'=1}^{\phi(t)-1}  \sum_{\widetilde{u} \in D(\xi(t',\psi(t')))} \left ( 1 -\frac{1}{n} \right )^{\phi(t)-t'-1} \mathbb{E}  \left \|       \widetilde{v}^{\psi(\widetilde{u})}_{\widetilde{u}} \right \|^2 \nonumber
\end{eqnarray}
Based on above formulations, we have
\begin{align}\label{csaga-18}
 \mathbb{E} \left \| \widetilde{v}_{t}^{\psi(t)} \right \|^2 & =  \mathbb{E} \left \| \widetilde{v}_{t}^{\psi(t)} -{v}_{t}^{\psi(t)} + {v}_t^{\psi(t)}  \right \|^2
   \leq   2\mathbb{E} \| \widetilde{v}_{t}^{\psi(t)} -{v}_{t}^{\psi(t)} \|^2+ 2\mathbb{E} \|{v}_t^{\psi(t)}   \|^2
   \nonumber \\
   & \leq 2\left( 2\mathbb{E} \| \widetilde{v}_{u}^{\psi(t)} -\widehat{v}_{u}^{\psi(t)}\|^2 + 2 \mathbb{E} \|\widehat{v}_{u}^{\psi(t)} - {v}_t^{\psi(t)}  \|^2\right) + 2\mathbb{E} \|{v}_t^{\psi(t)}   \|^2
   \nonumber \\
   & \leq  \frac{360L_{*}^2\gamma^2\tau G}{1 - 72L_{*}^2\gamma^2\tau } + 2\mathbb{E} \|{v}_t^{\psi(t)}   \|^2
\end{align}
\begin{proof}[\bf{Proof of Theorem \ref{thm-sagaconvex}}]
First, we upper bound $\mathbb{E}f(w_{t+1}) $ for $t =0,\cdots,S-1$:
\begin{eqnarray}\label{csaga-19}
&& \mathbb{E} \left[f (w_{t+|K(t)|} - f (w_{t}) \right]
\\ \nonumber
&\leq&
 -\frac{\gamma}{2}\sum_{u \in K(t)}  \mathbb{E} \|  \nabla_{\mathcal{G}_{\psi(u)}} f ({w}_{u}) \|^2
 + \frac{L_* \gamma^2}{2}\sum_{u \in K(t)} \mathbb{E} \|  \widetilde{v}^{\psi(u)}_{ u }\|^2
 \nonumber \\
&+& (\frac{\tau_1 L^2 \gamma^3}{n} + \frac{4\tau_2 L^2 \gamma^3}{n})\sum_{u\in K(t)}\sum_{t'=1}^{\phi(u)-1}  \sum_{{u'} \in D'(\xi(t',\psi(t)))} \left ( 1 -\frac{1}{n} \right )^{\phi(u)-t'-1} \mathbb{E}  \left \|       \widetilde{v}^{\psi({u'})}_{{u'}} \right \|^2
\\ \nonumber
 &\stackrel{(a)}{\leq}&
 -\frac{\gamma}{2}\left( \frac{1}{2}  \sum_{u\in K(t)} \mathbb{E}\| \nabla_{\mathcal{G}_{\psi(u)}} f({w}_{t})\|^2
  -  L^2 \gamma^2 \eta_1 \sum_{u\in K(t)} \sum_{u' \in \{t,\cdots,u\}} \mathbb{E} \|\widetilde{v}_{u'}^{\psi(u')}\|^2 \right)
\\ \nonumber &&
 + \frac{L_* \gamma^2}{2}\sum_{u \in K(t)}  \mathbb{E} \|\widetilde{v}^{\psi(u)}_{ u }\|^2
+  \frac{5\tau^{1/2} L^2 \gamma^3}{n} \sum_{u \in K(t)} \sum_{t'=1}^{\phi(u)-1}  \sum_{{u'} \in D'(\xi(t',\psi(t)))} \left ( 1 -\frac{1}{n} \right )^{\phi(u)-t'-1} \mathbb{E}  \left \|       \widetilde{v}^{\psi({u'})}_{{u'}} \right \|^2
\\ \nonumber &=&
 -\frac{\gamma}{4} \sum_{u\in K(t)} \mathbb{E}\| \nabla_{\mathcal{G}_{\psi(u)}} f({w}_{t}^{s})\|^2
 +\frac{L^2 \gamma^3 \eta_1}{2} \sum_{u \in K(t)}\sum_{u' \in \{t,...,u\}} \mathbb{E} \|\widetilde{v}_{u'}^{\psi(u')}\|^2
\\ \nonumber &&
 + \frac{L_* \gamma^2}{2}\sum_{u \in K(t)} \mathbb{E} \|\widetilde{v}^{\psi(u)}_{ u }\|^2
+ \frac{5\tau^{1/2} L^2 \gamma^3}{n}\sum_{u \in K(t)}\sum_{t'=1}^{\phi(u)-1}  \sum_{{u'} \in D'(\xi(t',\psi(t)))} \left ( 1 -\frac{1}{n} \right )^{\phi(u)-t'-1} \mathbb{E}  \left \|       \widetilde{v}^{\psi({u'})}_{{u'}} \right \|^2
\\ \nonumber &\stackrel{(b)}{\leq}&
 -\frac{\gamma}{4}  \mathbb{E}\| \nabla f({w}_{t})\|^2
  + {(\frac{L_*^2 \gamma^3 \tau}{2} + \frac{L_*\gamma^2}{2})}\sum_{u\in K(t)}
 \mathbb{E} \|   \widetilde{v}^{\psi(u)}_{u} \|^2
 \\ \nonumber &+&
 \frac{5\tau^{1/2}L^2 \gamma^3}{n}\sum_{u \in K(t)}\sum_{t'=1}^{\phi(u)-1}  \sum_{{u'} \in D'(\xi(t',\psi(t)))} \left ( 1 -\frac{1}{n} \right )^{\phi(u)-t'-1} \mathbb{E}  \left \|       \widetilde{v}^{\psi({u'})}_{{u'}} \right \|^2
 \\ \nonumber &\stackrel{(c)}{\leq}&
 -\frac{\gamma}{4}  \mathbb{E}\| \nabla f({w}_{t})\|^2
 + 5 { L_{*}^2 \gamma^3 \tau^{1/2}} \sum_{u \in K(t)}  \sum_{u' \in D^\prime(u)} \mathbb{E} \|   \widetilde{v}^{\psi(u')}_{u'} \|^2
 \nonumber \\
  && (\frac{L_*^2 \gamma^3 \tau}{2} + \frac{L_*\gamma^2}{2})\sum_{u\in K(t)}(180L_{*}^2\gamma^2\tau \frac{18G}{1 - 72L_{*}^2\gamma^2\tau } + 2\mathbb{E} \|{v}_t^{\psi(t)}   \|^2)
     \\ \nonumber &\stackrel{(d)}{\leq}&
-\frac{\gamma}{4}  \mathbb{E}\| \nabla f({w}_{t})\|^2
 + \left(2 { L_{*}^2 \gamma^3 \tau^{1/2}\eta_1\tau_2} + (\frac{L_*^2 \gamma^3 \tau}{2} + \frac{L_*\gamma^2}{2})\eta_1180L_{*}^2\gamma^2\tau\right)\frac{18G}{1 - 72L_{*}^2\gamma^2\tau }
 \nonumber \\
  && (\frac{L_*^2 \gamma^3 \tau}{2} + \frac{L_*\gamma^2}{2})\sum_{u\in K(t)}
  2\biggl(  4 \frac{L^{2} \eta_{1}}{n} \sum_{k^{\prime}=1}^{v(u)}\left(1-\frac{1}{n}\right)^{v(u)-k^{\prime}} \sigma\left(w_{\bar{u}_{k^{\prime}}}\right)
  \nonumber \\
  && + 2 L^{2}\left(1-\frac{1}{n}\right)^{v(u)} \sigma\left(w_{0}\right)
  +4 L^{2} \sigma\left(w_{\varphi(u)}\right)
  + 8 L^{2} \gamma^{2} \eta_{1}^{2} \lambda_{\gamma}G\biggr)
       \\ \nonumber &\stackrel{(e)}{\leq}&
 -\frac{\gamma\mu}{4}e(w_t) - \frac{\gamma\mu}{4}\sigma(w_t)
 + ({L_*^2 \gamma^3 \tau}+ {L_*\gamma^2})
 \eta_1\left( 2 L^{2}\left(1-\frac{1}{n}\right)^{v(u)} \sigma\left(w_{0}\right)
 +4 L^{2} \sigma\left(w_{t}\right)\right)
 \nonumber \\
  &&
   \left(2 { L_{*}^2 \gamma^3 \tau^{1/2}\eta_1\tau_2} + ({L_*^2 \gamma^3 \tau} + {L_*\gamma^2})\eta_1180L_{*}^2\gamma^2\tau + 8 L^{2} \gamma^{2} \eta_{1}^{2}\right)\frac{18G}{1 - 72L_{*}^2\gamma^2\tau }
  \nonumber \\
  && +   4 ({L_*^2 \gamma^3 \tau}+ {L_*\gamma^2})
   \frac{L^{2} \eta_{1}^2}{n} \sum_{k^{\prime}=1}^{v(u)}\left(1-\frac{1}{n}\right)^{v(u)-k^{\prime}} \sigma\left(w_{\bar{u}_{k^{\prime}}}\right)
\end{eqnarray}
where (a) \ff Eq.~\ref{csvrg-5}, (b) uses Lemma \ref{lem-csaga-3}, (c)  \ff Eq. \ref{csaga-18},  (d) \ff Lemma \ref{lem-csaga-3}, (e) \ff Assumption \ref{assumc1}. Thus, we have
\begin{align}\label{csaga-20}
  & e(w_{t+|K(t)|})
  \nonumber \\
   & \leq (1-\frac{\gamma\mu}{4})e(w_t) + \underbrace{2L^2\eta_1({L_*^2 \gamma^3 \tau}+ {L_*\gamma^2})}_{c_1} \left( \left(1-\frac{1}{n}\right)^{v(u)} \sigma\left(w_{0}\right)
 + 2\sigma\left(w_{t}\right)\right)
 \nonumber \\ &
+  \underbrace{ \left(2 { L_{*}^2 \gamma^3 \tau^{1/2}\eta_1\tau_2} + ({L_*^2 \gamma^3 \tau} + {L_*\gamma^2})\eta_1180L_{*}^2\gamma^2\tau + 8 L^{2} \gamma^{2} \eta_{1}^{2}\right)\frac{18G}{1 - 72L_{*}^2\gamma^2\tau }}_{c_0}
  \nonumber \\ &
  + \underbrace{ 4 ({L_*^2 \gamma^3 \tau}+ {L_*\gamma^2})
   \frac{L^{2} \eta_{1}^2}{n} }_{c_2} \sum_{k^{\prime}=1}^{v(u)}\left(1-\frac{1}{n}\right)^{v(u)-k^{\prime}} \sigma\left(w_{\bar{u}_{k^{\prime}}}\right) - \frac{\gamma\mu^2}{4}\sigma(w_t)
   \nonumber \\ &
   = (1-\frac{\gamma\mu}{4})e(w_t)
   + \left(- \frac{\gamma\mu^2}{4}+ 2c_1 + c_2\right)\sigma(w_t)
   + c_1 (1-\frac{1}{n})^{v(t)}\sigma(w_0)
   \nonumber \\
   & + c_2\sum_{k^{\prime}=1}^{v(u)}\left(1-\frac{1}{n}\right)^{v(u)-k^{\prime}} \sigma\left(w_{\bar{u}_{k^{\prime}}}\right) + c_0
\end{align}
where $\left\{\bar{u}_{0}, \bar{u}_{1}, \ldots, \bar{u}_{v(u)-1}\right\}$ are the all start time counters for the global time counters from 0 to $u$.

We define the Lyapunov function as$\mathcal{L}_{t}=\sum_{k=0}^{v(t)} \rho^{v(t)-k} e\left(w_{\bar{u}_{k}}\right)$ where $\rho \in\left(1-\frac{1}{n}, 1\right)$, we have that
\begin{align}\label{csaga-21}
 & L_{t+|K(t)|}  \\
 & =
 \rho^{v(t)+1}e(w_0) + \sum_{k=0}^{v(t)}\rho^{v(t)-k}e(w_{\bar{u}_{k+1}})
 \nonumber \\
 &\stackrel{(a)}{\leq}
 \rho^{v(t)+1}e(w_0)  + \sum_{k=0}^{v(t)}\rho^{v(t)-k}\biggl[
 (1-\frac{\gamma\mu}{4})e(w_{\bar{u}_{k}})
   + \left(- \frac{\gamma\mu^2}{4}+ 2c_1 + c_2\right)\sigma(w_{\bar{u}_{k}})
   \nonumber \\
   &+ c_1 (1-\frac{1}{n})^{v(t)}\sigma(w_0) + c_2\sum_{k^{\prime}=1}^{v(u)}\left(1-\frac{1}{n}\right)^{v(u)-k^{\prime}} \sigma\left(w_{\bar{u}_{k^{\prime}}}\right) + c_0\biggr]
 \nonumber \\
 &=
 \rho^{v(t)+1}e(w_0)  +  (1-\frac{\gamma\mu}{4})L_t + \sum_{k=0}^{v(t)}\rho^{v(t)-k}\biggl[
\left(- \frac{\gamma\mu^2}{4}+ 2c_1 + c_2\right)\sigma(w_{\bar{u}_{k}})
   \\
   &+ c_1 (1-\frac{1}{n})^{v(t)}\sigma(w_0) + c_2\sum_{k^{\prime}=1}^{v(u)}\left(1-\frac{1}{n}\right)^{v(u)-k^{\prime}} \sigma\left(w_{\bar{u}_{k^{\prime}}}\right) \biggr] +  \sum_{k=0}^{v(t)}\rho^{v(t)-k}c_0
   \nonumber \\
 &\stackrel{(b)}{\leq}
 \rho^{v(t)+1}e(w_0)  +  (1-\frac{\gamma\mu}{4})L_t + \left(- \frac{\gamma\mu^2}{4}+ 2c_1 + c_2\right)\sigma(w_{\bar{u}_{k}}) + \frac{c_0}{1-\rho}
    \nonumber \\
 &\stackrel{(c)}{\leq}
 \rho^{v(t)+1}e(w_0)  +  (1-\frac{\gamma\mu}{4})L_t - \left(\frac{\gamma\mu^2}{4}- 2c_1 -c_2\right)\frac{2}{L}e(w_{\bar{u}_{k}}) + \frac{c_0}{1-\rho} \nonumber
\end{align}
where (a) \ffe \ref{csaga-20}, (b) holds by approximately choosing $\gamma$ such that the terms related to $\sigma\left(w_{\pi_{k}}\right)(k=0, \cdots, v(t)-1)$ are negative, because the signs related to the lowest orders of $\sigma\left(w_{\bar{u}_{k}}\right)(k=0, \cdots, v(t)-1)$ are negative. In the following we give the detailed analysis of choosing a suitable $\gamma$ such that terms related to  $\sigma\left(w_{\bar{u}_{k}}\right)(k=0, \cdots, v(t)-1)$ are negative. We first consider $k=0$. Assume that  $C(\sigma(w_0))$ is the coefficient term of $\sigma(w_0)$ in \ffe \ref{csaga-21} , we have that
\begin{align}\label{csaga-22}
& C(\sigma(w_0))
\nonumber \\
& = \rho^{v(t)}\left(-\frac{\gamma \mu^{2}}{4}+2 c_{1}+c_{2}\right)+c_{1} \sum_{k=0}^{v(t)} \rho^{v(t)-k}\left(1-\frac{1}{n}\right)^{k}
\nonumber \\
& = \rho^{v(t)}\left(-\frac{\gamma \mu^{2}}{4}+2 c_{1}+c_{2}+c_{1} \sum_{k=0}^{v(t)}\left(\frac{1-\frac{1}{n}}{\rho}\right)^{k}\right)
\nonumber \\
& \leq \rho^{v(t)}\left(-\frac{\gamma \mu^{2}}{4}+2 c_{1}+c_{2}+c_{1} \frac{1}{1-\frac{1-\frac{1}{n}}{\rho}}\right)
\nonumber \\
& = \rho^{v(t)}\left(-\frac{\gamma \mu^{2}}{4}+c_{2}+c_{1} \left(2 + \frac{1}{1-\frac{1-\frac{1}{n}}{\rho}}\right)\right)
\end{align}
Based on Eq.~\ref{csaga-22}, we can carefully choose $\gamma$ such that $-\frac{\gamma \mu^{2}}{4}+c_{2}+c_{1} \left(2 + \frac{1}{1-\frac{1-\frac{1}{n}}{\rho}}\right)\leq0$.

Assume that $C(\sigma(w_{\bar{u}_k}))$ is the coefficient term of $\sigma(w_{\bar{u}_k})$ $(k=1,\cdots, v(t)-1)$ in the big square brackets of \ffe77, we have that
\begin{align}\label{csaga-23}
& C(\sigma(w_{\bar{u}_{k}}))
\nonumber \\
& = \rho^{v(t)-k}\left(-\frac{\gamma \mu^{2}}{4}+2 c_{1}+c_{2}\right)
+c_{2} \sum_{i=k+1}^{v(t)-1} \rho^{v(t)-i}\left(1-\frac{1}{n}\right)^{i-k}
\nonumber \\
& = \rho^{v(t)-k}\left(-\frac{\gamma \mu^{2}}{4}+2 c_{1}+c_{2}
+c_{2} \sum_{i=k+1}^{v(t)-1} \rho^{k-i}\left(1-\frac{1}{n}\right)^{i-k}\right)
\nonumber \\
& \leq \rho^{v(t)-k}\left(-\frac{\gamma \mu^{2}}{4}+2 c_{1}+c_{2}
+c_{2} \sum_{i=k+1}^{v(t)-1} \left(\frac{1-\frac{1}{n}}{\rho}\right)^{i-k}\right)
\nonumber \\
& = \rho^{v(t)-k}\left(-\frac{\gamma \mu^{2}}{4} + 2c_{1} + c_2\left(1 + \frac{1}{1-\frac{1-\frac{1}{n}}{\rho}}\right)\right)
\end{align}
Based on Eq.~\ref{csaga-23}, we can carefully choose $\gamma$ such that $-\frac{\gamma \mu^{2}}{4} + 2c_{1} + c_2\left(1 + \frac{1}{1-\frac{1-\frac{1}{n}}{\rho}}\right)\leq 0$.

Thus, based on Eq.~\ref{csaga-21}, we have that
\begin{align}\label{csaga-24}
 & \left(\frac{\gamma \mu^{2}}{4}-2 c_{1}-c_{2}\right) \frac{2}{L} e\left(w_{\bar{u}_{k}}\right)
 \nonumber \\
  & \leq \left(\frac{\gamma \mu^{2}}{4}-2 c_{1}-c_{2}\right) \frac{2}{L} e\left(w_{\bar{u}_{k}}\right)+\mathcal{L}_{t+|K(t)|}
  \nonumber \\
  &\stackrel{(a)}{\leq}  \rho^{v(t)+1} e\left(w_{0}\right)+\left(1-\frac{\gamma \mu}{4}\right) \mathcal{L}_{t}+\frac{c_{0}}{1-\rho}
    \nonumber \\
  &\stackrel{(b)}{\leq}
  \left(1-\frac{\gamma \mu}{4}\right)^{v(t)+1} \mathcal{L}_{0}+\rho^{v(t)+1} e\left(w_{0}\right) \sum_{k=0}^{v(t)+1}\left(\frac{1-\frac{\gamma \mu}{4}}{\rho}\right)^{k}+\frac{c_{0}}{1-\rho} \sum_{k=0}^{v(t)}\left(1-\frac{\gamma \mu}{4}\right)^{k}
  \nonumber \\
  &\leq \left(1-\frac{\gamma \mu}{4}\right)^{v(t)+1} e\left(w_{0}\right)+\rho^{v(t)+1} e\left(w_{0}\right) \frac{1}{1-\frac{1-\frac{\gamma \mu}{4}}{\rho}}+\frac{c_{0}}{1-\rho} \frac{4}{\gamma \mu}
  \nonumber\\
 &\stackrel{(c)}{\leq}
 \frac{2 \rho-1+\frac{\gamma \mu}{4}}{\rho-1+\frac{\gamma \mu}{4}} \rho^{v(t)+1} e\left(w_{0}\right)+\frac{c_{0}}{1-\rho} \frac{4}{\gamma \mu}
\end{align}
where (a) \ffe\ref{csaga-21}, (b) holds by using Eq.~\ref{csaga-21} recursively, (c) uses the fact that $1-\frac{\gamma\mu}{4}\leq \rho$
According to Eq. \ref{csaga-24}, we have that
\[e\left(w_{{u_0}_{k}}\right) \leq \frac{2 \rho-1+\frac{\gamma \mu}{4}}{\left(\rho-1 |+\frac{\gamma \mu}{4}\right)\left(\frac{\gamma \mu^{2}}{4}-2 c_{1}-c_{2}\right)} \rho^{v(t)+1} e\left(w_{0}\right)+\frac{4 c_{0}}{\gamma \mu(1-\rho)\left(\frac{\gamma \mu^{2}}{4}-2 c_{1}-c_{2}\right)}\]

Thus, under , to obtain the accuracy $\epsilon$ of Problem~\ref{P} for VF{${\textbf{B}}^2$}-SAGA, we can carefully choose $\gamma$ such that
\begin{align}\label{csaga-25}
1-72L^2\gamma^2\tau>0& \\
  \frac{4 c_{0}}{\gamma \mu(1-\rho)\left(\frac{\gamma \mu^{2}}{4}-2 c_{1}-c_{2}\right)} \leq \frac{\epsilon}{2} &  \\
0<1-\frac{\gamma \mu}{4}<1 &\\
-\frac{\gamma \mu^{2}}{4}+2 c_{1}+c_{2}\left(1+\frac{1}{1-\frac{1-\frac{1}{n}}{\rho}}\right) \leq 0 &\\
-\frac{\gamma \mu^{2}}{4}+c_{2}+c_{1}\left(2+\frac{1}{1-\frac{1-\frac{1}{n}}{\rho}}\right) \leq 0&\\
\end{align}
and let $\frac{2 \rho-1+\frac{\gamma \mu}{4}}{\left(\rho-1+\frac{\gamma \mu}{4}\right)\left(\frac{2 \mu^{2}}{4}-2 c_{1}-c_{2}\right)} \rho^{v(t)+1} e\left(w_{0}\right) \leq \frac{\epsilon}{2}$, we have that
\begin{align}\label{csaga-26}
v(t) \geq \frac{\log \frac{2\left(2 \rho-1+\frac{\gamma \mu}{4}\right) e\left(w_{0}\right)}{\epsilon\left(\rho-1+\frac{\gamma \mu}{4}\right)\left(\frac{\gamma \mu^{2}}{4}-2 c_{1}-c_{2}\right)}}{\log \frac{1}{\rho}}\
\end{align}
This completes the proof.
\end{proof}
\section{Convergence Analyses of Nonconvex Problems}
\subsection{Proof of Theorem~\ref{thm-sgdnonconvex}}
\begin{lemma}\label{lem-ncsgd-1}
For $\forall t$ (whether the $t$-th global iteration is a dominated or collaborative update), there is
\begin{equation}\label{lemeq-ncsgd-1}
 \sum_{t=0}^{S-1} \mathbb{E} ||\widetilde{v}_{t}^{\psi(t)}||^2 \leq \frac{4}{1-\lambda_{1}}\sum_{t=0}^{S-1}\mathbb{E} || \widehat{v}_{t}^{\psi(t)}\|^2,
\end{equation}
where $S$ denotes the total number of iterations, $\lambda_{1}=6L_*^2\gamma^2\tau$.
\end{lemma}
\begin{proof}[\textbf{Proof of Lemma \ref{lem-ncsgd-1}:}] First, when the $t$-th global iteration corresponds to collaborative update, we have
\begin{eqnarray}\label{sgd-1}
\mathbb{E} ||\widetilde{v}_{t}^{\psi(t)}||^2  &=& \mathbb{E} || \vartheta \cdot\left(x_{i}\right)_{\mathcal{G}_{\psi(t)}}
+ \nabla_{\mathcal{G}_{\psi(t)}} g((\widehat{w}_t)_{\mathcal{G}_{\psi(t)}} )||^2 \nonumber \\
&=&  \mathbb{E} ||\vartheta \cdot\left(x_{i}\right)_{\mathcal{G}_{\psi(t)}}
 + \nabla_{\mathcal{G}_{\psi(t)}} g((\bar{w}_{t})_{\mathcal{G}_{\psi(t)}})
 - \nabla_{\mathcal{G}_{\psi(t)}} g((\bar{w}_{t})_{\mathcal{G}_{\psi(t)}})
 + \nabla_{\mathcal{G}_{\psi(t)}} g((\widehat{w}_t)_{\mathcal{G}_{\psi(t)}}) ||^2 \nonumber \\
 &\stackrel{(a)}{\leq}& 2  \mathbb{E} || \bar{v}_{t}^{\psi(t)}\|^2
+ 2 \mathbb{E}\|\nabla_{\mathcal{G}_{\psi(t)}} g((\bar{w}_{t})_{\mathcal{G}_{\psi(t)}})
- \nabla_{\mathcal{G}_{\psi(t)}} g((\widehat{w}_t)_{\mathcal{G}_{\psi(t)}})||^2
\nonumber \\
&\stackrel{(b)}{\leq}& 2  \mathbb{E} || \bar{v}_{t}^{\psi(t)}\|^2
 + 2{L_{g}^2} \mathbb{E}\|(\bar{w}_{t})_{\mathcal{G}_{\psi(t)}}
 - (\widehat{w}_t)_{\mathcal{G}_{\psi(t)}}||^2
\nonumber \\
&\stackrel{(c)}{=}& 2  \mathbb{E} || \bar{v}_{t}^{\psi(t)}\|^2
+ 2{L_{g}^2}\gamma^2 \mathbb{E}\|\sum_{t^\prime\in D'(t), \psi(t^\prime)=\psi(t)} \widetilde{v}_{t^\prime}^{\psi(t^\prime)}||^2
\nonumber \\
&\stackrel{(d)}{\leq}& 2  \mathbb{E} || \bar{v}_{t}^{\psi(t)}\|^2
+ 2{L_{g}^2}\gamma^2 \tau_2\sum_{t^\prime\in D'(t)} \mathbb{E}\|\widetilde{v}_{t^\prime}^{\psi(t^\prime)}||^2
 \nonumber \\
&\stackrel{(e)}{\leq}& 2 \mathbb{E} || \bar{v}_{t}^{\psi(t)}\|^2
+ 2{L_{*}^2}\gamma^2 \tau_2 \sum_{t^\prime\in D'(t)}\mathbb{E}\|\widetilde{v}_{t^\prime}^{\psi(t^\prime)}||^2
\end{eqnarray}
where (a) follows from $\|a+b\|^2\leq 2\|a\|^2 + 2\|b\|^2$, (b) follows from Assumption~\ref{assum2}, (c) follows from the Eq.~\ref{Dt2}, (d) follows from Assumption~\ref{assum4} and $\|\sum_{i=1}^{n}a_i\|^2 \leq n \sum_{i=1}^{n} \|a_i\|^2$, (e) follows from definition of $L_{*}$. Then we bound the  $ \mathbb{E} || \bar{v}_{t}^{\psi(t)}\|^2$ as follow
\begin{align}\label{sgd-13}
 \mathbb{E} || \bar{v}_{t}^{\psi(t)}\|^2
 & =  \mathbb{E} || \bar{v}_{t}^{\psi(t)} - \widehat{v}_{t}^{\psi(t)}
 + \widehat{v}_{t}^{\psi(t)}\|^2
 \nonumber \\
 & \stackrel{a}{\leq} 2 \mathbb{E}\|\nabla_{\mathcal{G}_{\psi(t)}} f_{i_t}(\bar{w}_{t})
 - \nabla_{\mathcal{G}_{\psi(t)}} f_{i_t}(\widehat{w}_t)\|^2
 + 2 \mathbb{E} \|\widehat{v}_{t}^{\psi(t)}\|^2
 \nonumber \\
 & \stackrel{(b)}{\leq} 2{L_{\psi(t)}^2} \mathbb{E}\|\bar{w}_{t}
 - \widehat{w}_t||^2
 + 2 \mathbb{E} \|\widehat{v}_{t}^{\psi(t)}\|^2
  \nonumber \\
 & \stackrel{(c)}{\leq} 2{L_{\psi(t)}^2}\gamma^2 \mathbb{E}\|\sum_{t^\prime\in D'(t)} \widetilde{v}_{t^\prime}^{\psi(t^\prime)}||^2
 + 2\mathbb{E} \|\widehat{v}_{t}^{\psi(t)}\|^2
 \nonumber \\
 & \stackrel{(d)}{\leq} 2{L_{*}^2}\gamma^2 \tau_2 \sum_{t^\prime\in D'(t)}\mathbb{E}\|\widetilde{v}_{t^\prime}^{\psi(t^\prime)}||^2
 + 2\mathbb{E} \|\widehat{v}_{t}^{\psi(t)}\|^2
\end{align}
where (a) follows from $\|a+b\|^2\leq 2\|a\|^2 + 2\|b\|^2$, (b) follows from  Assumption~\ref{assum2}, (c) follows from the Eq.~\ref{Dt2}, (d) follows from the definition of $L_*$, Assumption~\ref{assum4}, and $ \|\sum_{i=1}^{n}a_i\|^2 \leq n \sum_{i=1}^{n} \|a_i\|^2$. Combining Eqs.~\ref{sgd-1} and \ref{sgd-13} we have
\begin{align}\label{sgd-14}
  \mathbb{E} ||\widetilde{v}_{t}^{\psi(t)}||^2 & \leq 4 \mathbb{E} || \widehat{v}_{t}^{\psi(t)}\|^2
+ 6{L_{*}^2}\gamma^2 \tau_2 \sum_{t^\prime\in D'(t)}\mathbb{E}\|\widetilde{v}_{t^\prime}^{\psi(t^\prime)}||^2
\end{align}
Summing Eq.~(\ref{sgd-14}) for all iterations (assume the number of total iterations is $S$ ), there is
\begin{eqnarray}\label{sgd-2}
\sum_{t=0}^{S-1}\mathbb{E} ||\widetilde{v}_{t}^{\psi(t)}||^2 & \leq & 4 \sum_{t=0}^{S-1} \mathbb{E} || \widehat{v}_{t}^{\psi(t)}\|^2
+ 6{L_{*}^2}\gamma^2 \tau_2 \sum_{t=0}^{S-1} \sum_{t^\prime\in D'(t)}\mathbb{E}\|\widetilde{v}_{t^\prime}^{\psi(t^\prime)}||^2
\nonumber \\
&\stackrel{(a)}{\leq}& 4 \sum_{t=0}^{S-1} \mathbb{E} || \widehat{v}_{t}^{\psi(t)}\|^2
+ 6{L_{*}^2}\gamma^2 \tau \sum_{t=0}^{S-1} \mathbb{E}\|\widetilde{v}_{t}^{\psi(t)}||^2,
\end{eqnarray}
where (a) follows from Assumption~\ref{assum4}. When the $t$-th global iteration corresponds to dominated update, it is obviously that
\begin{equation}\label{sgd-3}
 \sum_{t=0}^{S-1} \mathbb{E}  \| \widetilde{v}^{\psi(t)}_{t} \|^2 = \sum_{t=0}^{S-1} \mathbb{E}  \| \widehat{v}^{\psi(t)}_{t} \|^2 < 4 \sum_{t=0}^{S-1} \mathbb{E} || \widehat{v}_{t}^{\psi(t)}\|^2
+ 6{L_{*}^2}\gamma^2 \tau \sum_{t=0}^{S-1} \mathbb{E}\|\widetilde{v}_{t}^{\psi(t)}||^2
\end{equation}
 Combining Eqs.~{\ref{sgd-2}} and {\ref{sgd-3}}, there is
\begin{equation}\label{sgd-4}
 \sum_{t=0}^{S-1} \mathbb{E}  \| \widetilde{v}^{\psi(t)}_{t} \|^2 < 4 \sum_{t=0}^{S-1} \mathbb{E} || \widehat{v}_{t}^{\psi(t)}\|^2
+ 6{L_{*}^2}\gamma^2 \tau \sum_{t=0}^{S-1} \mathbb{E}\|\widetilde{v}_{t}^{\psi(t)}||^2
\end{equation}
whether the $t$-th global iteration corresponds to collaborative update or dominated one, which implies that if $1-\lambda_1 > 0$ there is
\begin{equation}\label{sgd-5}
 \sum_{t=0}^{S-1} \mathbb{E}  \| \widetilde{v}^{\psi(t)}_{t} \|^2 < \frac{4}{1-6{L_{*}^2}\gamma^2 \tau} \sum_{t=0}^{S-1} \mathbb{E} || \widehat{v}_{t}^{\psi(t)}\|^2 =  \frac{4}{1-\lambda_1} \sum_{t=0}^{S-1} \mathbb{E} || \widehat{v}_{t}^{\psi(t)}\|^2
\end{equation}
where $\lambda_1 = 6{L_{*}^2}\gamma^2 \tau$, this completes the proof.
\end{proof}
\begin{lemma}\label{lem-ncsgd-2}
For $\forall t$ (whether the $t$-th global iteration is a dominated or collaborative update), there is
\begin{equation}\label{lemeq-ncsgd-2}
 \mathbb{E} \| v_{t}^{\psi(t)} - \widetilde{v}_{t}^{\psi(t)} \|^2 \leq 2{ L_{{*}}^2  \gamma^2 \tau_1}  \sum_{t' \in D(t)} \mathbb{E} \|   \widetilde{v}^{\psi(t')}_{t'} \|^2
 + 8 { L_{*}^2  \gamma^2 \tau_2  \sum_{t' \in D^\prime(t)}} \mathbb{E} \|   \widetilde{v}^{\psi(t')}_{t'} \|^2.
\end{equation}
\end{lemma}
\begin{proof}[\textbf{Proof of  Lemma \ref{lem-ncsgd-2}:}]
First, we give the bound of $ \mathbb{E} \| \widehat{v}_{t}^{\psi(t)} - \widetilde{v}_{t}^{\psi(t)} \|^2$ as follow
\begin{align}\label{2-1}
 \mathbb{E} \| \widehat{v}_{t}^{\psi(t)} - \widetilde{v}_{t}^{\psi(t)} \|^2
 &  \stackrel{ (a) }{\leq} \mathbb{E} \| \nabla_{\mathcal{G}_{\psi(t)}} f(\bar{w}_t) - \nabla_{\mathcal{G}_{\psi(t)}} f(\widehat{w}_t) + \nabla_{\mathcal{G}_{\psi(t)}} g((\widehat{w}_t)_{\mathcal{G}_{\psi(t)}}) - \nabla_{\mathcal{G}_{\psi(t)}} g((\bar{w}_t)_{\mathcal{G}_{\psi(t)}}\|^2
\nonumber \\
& \stackrel{ (c) }{\leq} 2\mathbb{E} \| \nabla_{\mathcal{G}_{\psi(t)}} f_{i_t} (\bar{w}_t) - \nabla_{\mathcal{G}_{\psi(t)}} f_{i_t} (\widehat{w}_t) \|^2
+ 2 \mathbb{E} \| \nabla_{\mathcal{G}_{\psi(t)}} g ((\bar{w}_t)_{\mathcal{G}_\psi(t)})
- \nabla_{\mathcal{G}_{\psi(t)}} g ((\widehat{w}_t))_{\mathcal{G}_\psi(t)} \|^2
\nonumber \\
& \stackrel{ (d) }{\leq} 2{L^2} \mathbb{E} \| \bar{w}_t - \widehat{w}_t \|^2
 + 2{L_{g}^2} \mathbb{E} \| (\bar{w}_t)_{\mathcal{G}_\psi(t)} - (\widehat{w}_t)_{\mathcal{G}_\psi(t)} \|^2
\nonumber \\
& \stackrel{ (e) }{=} 2{ L^2 \gamma^2}  \mathbb{E} \|  \sum_{t' \in D'(t)} \textbf{U}_{\psi(t')} \widetilde{v}^{\psi(t')}_{t'} \|^2 + 2{ L_{g}^2 \gamma^2 }  \mathbb{E} \| \sum_{t' \in D^\prime(t), \psi(t^\prime)=\psi(t)} \textbf{U}_{\psi(t')} \widetilde{v}^{\psi(t')}_{t'} \|^2
\nonumber \\ & \stackrel{ (f) }{\leq} 2{ L^2  \gamma^2 \tau_2}  \sum_{t' \in D'(t)} \mathbb{E} \|   \widetilde{v}^{\psi(t')}_{t'} \|^2 + 2{ L_{g}^2  \gamma^2 \tau_2  \sum_{t' \in D^\prime(t)}} \mathbb{E} \|   \widetilde{v}^{\psi(t')}_{t'} \|^2
\nonumber \\ & \stackrel{ (g) }{\leq} 4{ L_{{*}}^2  \gamma^2 \tau_2}  \sum_{t' \in D'(t)} \mathbb{E} \|   \widetilde{v}^{\psi(t')}_{t'} \|^2
\end{align}
where (a) follows from the definition of $\bar{v}_{t}^{\psi(t)}$ and the definitions of $\widetilde{v}_{t}^{\psi(t)}$ in different type of updates (dominated or collaborative one), (b) follows from $\|a+b\|^2 \leq 2\|a\|^2 + 2\|b\|^2$, (c) follows from the definition of $\bar{v}_{t}^{\psi(t)}$ and $\widehat{v}_{t}^{\psi(t)}$, (d) follows from Assumptions~\ref{assum2}, (e) follows from Eqs.~\ref{Dt1} and \ref{Dt2}, (f) follows from Assumptions~\ref{assum4} to \ref{assum4} and $\| \sum_{i=1}^{n} a_i \|^2 \leq n\sum_{i=1}^{n} \|a_i\|^2$, (g) follows from the definition of $L_{*}$.  Then we consider the bound 
\begin{align}\label{2}
 \mathbb{E} \| v_{t}^{\psi(t)} - \widetilde{v}_{t}^{\psi(t)} \|^2 & = \mathbb{E} \| v_{t}^{\psi(t)} - \widehat{v}_{t}^{\psi(t)} + \widehat{v}_{t}^{\psi(t)} - \widetilde{v}_{t}^{\psi(t)} \|^2
\nonumber \\
 & \stackrel{ (a) }{\leq} 2\mathbb{E} \| v_{t}^{\psi(t)} - \widehat{v}_{t}^{\psi(t)}\|^2+ 2 \mathbb{E} \| \widehat{v}_{t}^{\psi(t)} - \widetilde{v}_{t}^{\psi(t)} \|^2
\nonumber \\
& \leq 2\mathbb{E} \| \nabla_{\mathcal{G}_{\psi(t)}} f_{i_t} ({w}_t) - \nabla_{\mathcal{G}_{\psi(t)}} f_{i_t} (\widehat{w}_t) \|^2
+ 2 \mathbb{E} \| \widehat{v}_{t}^{\psi(t)} - \widetilde{v}_{t}^{\psi(t)} \|^2
\nonumber \\
& \stackrel{ (b) }{\leq} 2{L^2} \mathbb{E} \| {w}_t - \widehat{w}_t \|^2
 + 2 \mathbb{E} \| \widehat{v}_{t}^{\psi(t)} - \widetilde{v}_{t}^{\psi(t)} \|^2
\nonumber \\
& \stackrel{ (c) }{=} 2{ L^2 \gamma^2}  \mathbb{E} \|  \sum_{t' \in D(t)} \textbf{U}_{\psi(t')} \widetilde{v}^{\psi(t')}_{t'} \|^2
+2 \mathbb{E} \| \widehat{v}_{t}^{\psi(t)} - \widetilde{v}_{t}^{\psi(t)} \|^2
\nonumber \\
& \stackrel{ (d) }{\leq} 2{ L^2  \gamma^2 \tau_1}  \sum_{t' \in D(t)} \mathbb{E} \|   \widetilde{v}^{\psi(t')}_{t'} \|^2
+ 2 \mathbb{E} \| \widehat{v}_{t}^{\psi(t)} - \widetilde{v}_{t}^{\psi(t)} \|^2
\nonumber \\
& \stackrel{ (e) }{\leq} 2{ L_{{*}}^2  \gamma^2 \tau_1}  \sum_{t' \in D(t)} \mathbb{E} \|   \widetilde{v}^{\psi(t')}_{t'} \|^2
+ 8{ L_{*}^2  \gamma^2 \tau_2  \sum_{t' \in D^\prime(t)}} \mathbb{E} \|   \widetilde{v}^{\psi(t')}_{t'} \|^2
\end{align}
where (a) follows from $\|a+b\|^2 \leq 2\|a\|^2 + 2\|b\|^2$, (b) follows from Assumptions~\ref{assum2}, (c) follows from Eqs.~\ref{Dt1} and \ref{Dt2}, inequalities, (d) follows from Assumptions~\ref{assum4} to \ref{assum4} and $\| \sum_{i=1}^{n} a_i \|^2 \leq n\sum_{i=1}^{n} \|a_i\|^2$, (e) follows from the definition of $L_{*}$ and Eq. \ref{2-1}. This completes the proof.
\end{proof}
\begin{lemma}\label{lem-ncsgd-3}
Assuming $S=qc$, where $c>0$ is an integer, we have
\begin{align}\label{lemeq-ncsgd-3}
\sum_{t \in \mathcal{A}(S)}\mathbb{E} \| \nabla f({w}_{t})\|^2
    \stackrel{(a)}{\leq} 2 L^2 \gamma^2 \eta_2^2 \sum_{u=0}^{S-1} \mathbb{E} \|\widetilde{v}_{u}^{\psi(u)}\|^2 + 2 \sum_{u=0}^{S-1} \mathbb{E} \|\nabla_{\mathcal{G}_{\psi(u)}} f({w}_{u})\|^2
\end{align}
\end{lemma}
\begin{proof}[\textbf{Proof of  Lemma \ref{lem-ncsgd-3}:}]
Given $t$ denotes a global iteration number, if the $t$-th global iteration is a dominated one, then for any $t' \in K'(t)$, there is
\begin{align}\label{17}
 \mathbb{E} \| \nabla_{\mathcal{G}_{\psi(t')}} f({w}_{t})\|^2
 & = \mathbb{E} \| \nabla_{\mathcal{G}_{\psi(t')}} f({w}_{t}) - \nabla_{\mathcal{G}_{\psi(t')}} f({w}_{t'}) + \nabla_{\mathcal{G}_{\psi(t')}} f({w}_{t'})\|^2
 \nonumber \\
 & \stackrel{(a)}{\leq}  2\mathbb{E} \| \nabla_{\mathcal{G}_{\psi(t')}} f({w}_{t}) - \nabla_{\mathcal{G}_{\psi(t')}} f({w}_{t'})\|^2
 +  2 \mathbb{E} \|\nabla_{\mathcal{G}_{\psi(t')}} f({w}_{t'})\|^2
  \nonumber \\
 & \leq  2\mathbb{E} \| \nabla f({w}_{t}) - \nabla f({w}_{t'})\|^2
 +  2 \mathbb{E} \|\nabla_{\mathcal{G}_{\psi(t')}} f({w}_{t'})\|^2
   \nonumber \\
 & \stackrel{(b)}{\leq}  2 L^2 \gamma^2 \mathbb{E} \| {w}_{t} - {w}_{t'}\|^2
 +  2 \mathbb{E} \|\nabla_{\mathcal{G}_{\psi(t')}} f({w}_{t'})\|^2
    \nonumber \\
 & \stackrel{(c)}{=}  2 L^2 \gamma^2 \mathbb{E} \|\sum_{u \in \{t,...,t'\}}\textbf{U}_{\psi(u)}\widetilde{v}_{u}^{\psi(u)}\|^2
 +  2 \mathbb{E} \|\nabla_{\mathcal{G}_{\psi(t')}} f({w}_{t'})\|^2
     \nonumber \\
 & \stackrel{(d)}{\leq}  2 L^2 \gamma^2 \eta_2 \sum_{u \in \{t,...,t'\}} \mathbb{E} \|\widetilde{v}_{u}^{\psi(u)}\|^2
 +  2 \mathbb{E} \|\nabla_{\mathcal{G}_{\psi(t')}} f({w}_{t'})\|^2
\end{align}
where (a) follows from $\|a+b\|^2 \leq 2\|a\|^2 + 2\|b\|^2$, (b) follows from Assumptions~\ref{assum2}, (c) follows from Eq.~\ref{Dt1} , (d) follows from the bound of $|K'(t)|$ and $\| \sum_{i=1}^{n} a_i \|^2 \leq n\sum_{i=1}^{n} \|a_i\|^2$. Summing above for $t' \in K'(t)$ and all $t \in \mathcal{A}(S)$ we have
\begin{align}\label{189}
\sum_{t \in \mathcal{A}(S)}\sum_{t' \in K'(t)}\mathbb{E} \| \nabla_{\mathcal{G}_{\psi(t')}} f({w}_{t})\|^2
   &\leq  2 L^2 \gamma^2 \eta_2 \sum_{t \in \mathcal{A}(S)} \sum_{t' \in K'(t)} \sum_{u \in \{t,...,t'\}} \mathbb{E} \|\widetilde{v}_{u}^{\psi(u)}\|^2
   \nonumber \\
    & +  2 \sum_{t \in \mathcal{A}(S)} \sum_{t' \in K'(t)} \mathbb{E} \|\nabla_{\mathcal{G}_{\psi(t')}} f({w}_{t'})\|^2
    \nonumber \\
    & \stackrel{(a)}{\leq} 2 L_*^2 \gamma^2 \eta_2^2 \sum_{t=0}^{S-1} \mathbb{E} \|\widetilde{v}_{t}^{\psi(t)}\|^2 + 2 \sum_{t=0}^{S-1} \mathbb{E} \|\nabla_{\mathcal{G}_{\psi(t)}} f({w}_{t})\|^2
\end{align}
where (a) follows from the bound of $|K'(t)|$ and the definition of $\mathcal{A}(S)$. For $\forall u \in \mathcal{A}(S)$ there is
\begin{align}\label{111}
 \mathbb{E} \|\nabla f({w}_{u})\|^2& = \mathbb{E} \|\sum_{u'\in K'(u)}\nabla_{\mathcal{G}_{\psi(u')}} f({w}_{u})\|^2
  \stackrel{(a)} {=} \sum_{u'\in K'(u)} \mathbb{E} \|\nabla_{\mathcal{G}_{\psi(u')}} f({w}_{u})\|^2
\end{align}
where (a) follows from the  orthogonality between all coordinates. Combing Eqs.~\ref{189} and \ref{111}, there is
\begin{align}\label{188}
\sum_{u \in \mathcal{A}(S)}\mathbb{E} \| \nabla f({w}_{u})\|^2
    {\leq} 2 L_*^2 \gamma^2 \eta_2^2 \sum_{t=0}^{S-1} \mathbb{E} \|\widetilde{v}_{t}^{\psi(t)}\|^2 + 2 \sum_{t=0}^{S-1} \mathbb{E} \|\nabla_{\mathcal{G}_{\psi(t)}} f({w}_{t})\|^2
\end{align}
This completes the proof.
\end{proof}

\begin{proof}[\textbf{Proof of Theorem \ref{thm-sgdnonconvex}:}]
 For $\forall t$ denotes a global iteration, we have that
\begin{eqnarray}\label{ncsgd-5}
&& \mathbb{E} f (w_{t+1})
\\ \nonumber &\stackrel{ (a) }{\leq}&  \mathbb{E} \left ( f (w_{t}) + \langle \nabla f(w_{t}), w_{t+1}-w_{t}  \rangle + \frac{L}{2} \|w_{t+1}-w_{t}   \|^2  \right )
\\ \nonumber &=&  \mathbb{E} \left ( f (w_{t}) -  \gamma \langle \nabla f(w_{t}),  \widetilde{v}^{\psi(t)}_t  \rangle + \frac{L\gamma^2}{2} \|  \widehat{v}^{\psi(t)}_t  \|^2  \right )
\\ \nonumber &{=}&  \mathbb{E} \left ( f (w_{t}) -  \gamma \langle \nabla f(w_{t}),  \widetilde{v}^{\psi(t)}_t + {v}^{\psi(t)}_t- {v}^{\psi(t)}_t \rangle  + \frac{L \gamma^2}{2} \|  \widetilde{v}^{\psi(t)}_t  \|^2  \right )
\\ \nonumber &\stackrel{(b)}{=}&  \mathbb{E}  f (w_{t}) -  \gamma \mathbb{E} \langle \nabla f(w_{t}),  \nabla_{\mathcal{G}_{\psi(t)}} f ({w}_t) \rangle  + \frac{L \gamma^2}{2} \mathbb{E} \|  \widetilde{v}^{\psi(t)}_t  \|^2  + \gamma \mathbb{E} \langle \nabla f(w_{t}), {v}^{\psi(t)}_t - \widetilde{v}^{\psi(t)}_t \rangle
 \\ \nonumber &\stackrel{ (c) }{\leq}&  \mathbb{E}  f (w_{t}) -  \gamma \mathbb{E} \|  \nabla_{\mathcal{G}_{\psi(t)}} f ({w}_t) \|^2  + \frac{\gamma}{2} \mathbb{E} \| \nabla_{\mathcal{G}_{\psi(t)}} f ({w}_t) \|^2 + \frac{L \gamma^2}{2} \mathbb{E} \|  \widetilde{v}^{\psi(t)}_t  \|^2 + \frac{\gamma}{2} \mathbb{E} \| \widetilde{v}^{\psi(t)}_t - {v}^{\psi(t)}_t \|^2
  \\  &\stackrel{(d)}{\leq}&  \mathbb{E}  f (w_{t}) -  \frac{\gamma}{2} \mathbb{E} \|  \nabla_{\mathcal{G}_{\psi(t)}} f ({w}_t) \|^2  + \frac{L_* \gamma^2}{2} \mathbb{E} \|  \widetilde{v}^{\psi(t)}_t  \|^2 \nonumber \\
  & &+  { L_{*}^2  \gamma^3  \tau_1} \sum_{t' \in D(t)} \mathbb{E} \|   \widetilde{v}^{\psi(t')}_{t'} \|^2
  + 4{ L_{*}^2  \gamma^3}  \tau_2 \sum_{t' \in D^\prime(t)} \mathbb{E} \|   \widetilde{v}^{\psi(t')}_{t'} \|^2 \nonumber
 \end{eqnarray}
where the  inequalities (a) follows form Assumption~\ref{assum1}, (b) follows from that $ {v}^{\psi(t)}_t = \nabla_{\mathcal{G}_{\psi(t)}} f_{i_t} ({w}_t)$ for a specific party, (c) follows from $\langle a,b \rangle\leq \frac{1}{2}(\|a\|^2+\|b\|^2)$, (d) follows from Lemma~\ref{lem-ncsgd-2} and the definition of $L_*$.
Summing  Eq.~(\ref{ncsgd-5}) over all $ 0\leq t \leq S-1 $, we obtain
\begin{eqnarray}\label{ncsgd-6}
&& \frac{\gamma}{2} \sum_{t=0}^{S-1} \mathbb{E} \|  \nabla_{\mathcal{G}_{\psi(t)}} f ({w}_t) \|^2
\\ \nonumber &\leq&  \sum_{t=0}^{S-1}\mathbb{E} \left[f (w_{t}) - f (w_{u+1}) \right]  +  {L_{*}^2  \gamma^3}\tau_1 \sum_{t=0}^{S-1} \sum_{t' \in D(t)} \mathbb{E} \|   \widetilde{v}^{\psi(t')}_{t'} \|^2
\\ \nonumber &&
+ 4 { L_{*}^2  \gamma^3}  \tau_2  \sum_{t=0}^{S-1}  \sum_{t' \in D^\prime(t)} \mathbb{E} \|   \widetilde{v}^{\psi(t')}_{t'} \|^2
+ \frac{L_* \gamma^2}{2} \sum_{t=0}^{S-1}  \mathbb{E} \|  \widetilde{v}^{\psi(t)}_t  \|^2
\end{eqnarray}
Note that $\nabla_{\mathcal{G}_{\psi(t)}} f ({w}_t)$ is the gradient of coordinate $\psi(t)$, while to obtain the global convergence rate it is necessary to focus on the gradient of all coordinates \emph{i.e.,} $\nabla  f ({w}_t)$.
Combining Eq.~\ref{ncsgd-6} with Lemma~\ref{lem-ncsgd-3}, there is
\begin{align}\label{18}
\sum_{u \in \mathcal{A}(S)}\mathbb{E} \| \nabla f({w}_{u})\|^2
    &\stackrel{(a)}{\leq}  \sum_{t=0}^{S-1} \frac{4\mathbb{E} \left[ f (w_{t}) - f (w_{t+1}) \right]}{\gamma}
     + {2L_* \gamma} \sum_{t=0}^{S-1}  \mathbb{E} \|  \widetilde{v}^{\psi(t)}_t  \|^2
     \nonumber \\
    & + 2 L_*^2 \gamma^2 (2\tau_1^2 + 8\tau_2^2+\eta_2^2) \sum_{t=0}^{S-1}   \mathbb{E} \|   \widetilde{v}^{\psi(t)}_{t} \|^2
    \nonumber \\
    & \stackrel{(b)}{\leq}  \frac{4\mathbb{E} \left[ f (w^0) - f (w^*) \right]}{\gamma}
    +  (22 L_*^2 \gamma^2 \tau + 2L_*\gamma) \sum_{t=0}^{S-1} \mathbb{E} \|   \widetilde{v}^{\psi(t)}_{t} \|^2
        \nonumber \\
    & \stackrel{(c)}{\leq}  \frac{4\mathbb{E} \left[ f (w^0) - f (w^*) \right]}{\gamma}
    +  \sum_{t=0}^{S-1} (22 L_*^2 \gamma^2 \tau + 2L_*\gamma) \frac{4}{1-6L_*^2\gamma^2\tau} {G}
\end{align}
where (a) follows from Assumptions~\ref{assum4} and \ref{assum4}, (b) follows from the definition of $\tau$, (c) follows from  Lemma~\ref{lem-ncsgd-1} and . Which implies that
\begin{align}\label{sgd-8}
 \frac{1}{S} \sum_{u \in \mathcal{A}(S)}\mathbb{E} \| \nabla f({w}_{u})\|^2
\leq \frac{  4 \mathbb{E}\left[ f (w^0) - f (w^*) \right] }{S\gamma}  +  (22 L_*^2 \gamma^2 \tau + 2L_*\gamma) \frac{4}{1-6L_*^2\gamma^2\tau} {G}
\end{align}
Note that for $\forall$ $S=qc$, where $c$ is an integer, there is  $|\mathcal{A}(S)| = \frac{S}{q}$, and then we have
\begin{align}\label{sgd-113}
 \frac{1}{|\mathcal{A}(S)|} \sum_{u \in \mathcal{A}(S)}\mathbb{E} \| \nabla f({w}_{u})\|^2
\leq \frac{4\mathbb{E}\left[ f (w^0) - f (w^*) \right] }{|\mathcal{A}(S)|\gamma}
 +  (22 L_*^2 \gamma^2 \tau + 2L_*\gamma) \frac{4}{1-6L_*^2\gamma^2\tau} {G_1},
\end{align}
where $G_1=qG$.
To obtain the $\epsilon$-first-order stationary solution one can choose suitable $\gamma$, such that
\begin{align}\label{ncsgd-91-1}
1-6L_*^2\gamma^2\tau & >0
\end{align}
\begin{align}\label{ncsgd-91-2}
  \frac{4\mathbb{E}\left[ f (w^0) - f (w^*) \right] }{|\mathcal{A}(S)|\gamma} & \leq \frac{\epsilon}{2}
\end{align}
\begin{align}\label{ncsgd-91-3}
  (22 L_*^2 \gamma^2 \tau + 2L_*\gamma) \frac{4}{1-6L_*^2\gamma^2\tau }{G_1} & \leq \frac{\epsilon}{2}
\end{align}
which implies that if $\tau$ is upper bounded, i.e. $\tau \leq \frac{512qG}{3\epsilon^2}$ (one can obtain this by combining Eqs. \ref{ncsgd-91-1} and \ref{ncsgd-91-3}, and assuming Eq. \ref{ncsgd-91-2} holds), we can carefully choose the stepsize as
 \[ \gamma = \frac{\epsilon}{32L_*G}\]
 and if the total epoches number (i.e., $v'(S)$) of $S$ global iterations denoted as $T$ satisfying
\begin{equation}\label{sgd-10}
T\geq \frac{256L_*qG\mathbb{E}(f(w_0)-f(w^*))}{\epsilon^2}
\end{equation}
the $\epsilon$-first-order stationary solution is obtained:
\begin{align}\label{sgd-9}
 \frac{1}{T} \sum_{t=0}^{T-1}\mathbb{E} \| \nabla f({w}_{t})\|^2  \leq   \epsilon
\end{align}
 this completes the proof.
\end{proof}
\subsection{Proof of Theorem~\ref{thm-svrgnonconvex}}
\begin{lemma}\label{lem-ncsvrg-1}
For all outer loop $s = 1,\cdots,S$ we define $\mathcal{A'}(s)$ as all epoches during this outer loop, there is
\begin{equation}\label{ncsvrg-lem1}
 \sum\limits_{u\in \mathcal{A'}(s)} \sum\limits_{t\in K'(u)} \mathbb{E} ||\widetilde{v}_{t}^{\psi(t)}||^2 \leq \lambda_{\gamma}  \sum\limits_{u\in \mathcal{A'}(s)} \sum\limits_{t\in K'(u)} \mathbb{E} || {v}_{t}^{\psi(t)}\|^2,
\end{equation}
where $\lambda_{\gamma}=\frac{2}{1 - 20 L_{*}^2\gamma^2\tau}  > 0$.
\end{lemma}
\begin{proof}[\textbf{Proof of  Lemma~\ref{lem-ncsvrg-1}:}]
First, we have
\begin{eqnarray}\label{ncsvrg-1}
\mathbb{E} ||\widetilde{v}_{t}^{{\psi(t)}}||^2
&=&  \mathbb{E} || \widetilde{v}^{{\psi(t)}}_t - {v}^{{\psi(t)}}_t + {v}^{{\psi(t)}}_t||^2  \\
&\stackrel{(a)}{\leq}&  2 \mathbb{E} ||\widetilde{v}^{{\psi(t)}}_t - {v}^{{\psi(t)}}_t ||^2
+  2 \mathbb{E} ||{v}^{{\psi(t)}}_t||^2
\nonumber \\
&\stackrel{(b)}{\leq}&
2\left( 2{ L_{{*}}^2  \gamma^2 \tau_1}  \sum_{t' \in D(t)} \mathbb{E} \|   \widetilde{v}^{\psi(t')}_{t'} \|^2
 + 8 { L_{*}^2  \gamma^2 \tau_2  \sum_{t' \in D^\prime(t)}} \mathbb{E} \|   \widetilde{v}^{\psi(t')}_{t'} \|^2 \right) + 2 \mathbb{E} ||{v}^{{\psi(t)}}_t||^2 \nonumber
\end{eqnarray}
where (a) follows from $\|a+b\|^2\leq 2\|a\|^2+ 2\|b\|^2$, (b) follows from Lemma~\ref{lem-csgd-2}.
Summing Eq.\ref{ncsvrg-1} over an outer loop $s$, then there is
\begin{equation}\label{ncsvrg-5}
    \sum\limits_{u\in \mathcal{A'}(s)} \sum\limits_{t\in K'(u)}\mathbb{E} ||\widetilde{v}_{t}^{\psi(t)}||^2 \stackrel{(a)}
    {\leq}
    \frac{2}{1-20L_*^2\gamma^2\tau}  \sum\limits_{u\in \mathcal{A'}(s)} \sum\limits_{t\in K'(u)} \mathbb{E} ||{v}^{{\psi(t)}}_t ||^2
\end{equation}
where (a) uses Assumptions~\ref{assum4}, \ref{assum4} and the definition of $\tau$.
Thus, if  $\lambda_{\gamma}=\frac{2}{1 - 20 L_{*}^2\gamma^2\tau}  > 0$, then $\mathbb{E}||\widetilde{v}_t^{{\psi(t)}}||^2$ is upper bounded:
\begin{eqnarray}\label{ncsvrg-6}
 \sum\limits_{u\in \mathcal{A'}(s)} \sum\limits_{t\in K'(u)} \mathbb{E} ||\widetilde{v}^{{\psi(t)}}_t||^2
&\leq& \lambda_{\gamma}  \sum\limits_{u\in \mathcal{A'}(s)} \sum\limits_{t\in K'(u)}\mathbb{E} ||{v}^{{\psi(t)}}_t ||^2
\end{eqnarray}
This completes the proof
\end{proof}
\begin{proof}[\textbf{Proof of  Theorem~\ref{thm-svrgnonconvex}:}]
Similar to the proof of Theorem~\ref{thm-sgdnonconvex}, we first apply Lemma~\ref{lem-csgd-3} to an epoch (or an outer loop) $s$, and there is
\begin{align}\label{ncsvrg-7}
 \sum\limits_{u\in \mathcal{A'}(s)} \mathbb{E} \| \nabla f({w}_{u})\|^2
    \stackrel{(a)}{\leq}
    2 L_*^2 \gamma^2 \tau_1^2  \sum\limits_{u\in \mathcal{A'}(s)} \sum\limits_{t\in K'(u)} \mathbb{E} \|\widetilde{v}_{t}^{\psi(t)}\|^2
    + 2  \sum\limits_{u\in \mathcal{A'}(s)} \sum\limits_{t\in K'(u)} \mathbb{E} \|\nabla_{\mathcal{G}_{\psi(t)}} f({w}_{t})\|^2
\end{align}
Summing Eq.~\ref{ncsvrg-7} over outer loops $1, \cdots, S$ we have
\begin{align}\label{ncsvrg-8}
\sum_{s=1}^{S}  \sum\limits_{u\in \mathcal{A'}(s)} \sum\limits_{t\in K'(u)}  \mathbb{E} \| \nabla_{\mathcal{G}_{\psi(u_t)}} f({w}_{u_0}^s)\|^2
    {\leq}
   & 2 L_*^2 \gamma^2 \tau_1^2 \sum_{s=1}^{S}  \sum\limits_{u\in \mathcal{A'}(s)} \sum\limits_{t\in K'(u)} \mathbb{E} \|\widetilde{v}_{u_t}^{\psi(u_t)}\|^2
    \nonumber \\
     & + 2 \sum_{s=1}^{S}  \sum\limits_{u\in \mathcal{A'}(s)} \sum\limits_{t\in K'(u)} \mathbb{E} \|\nabla_{\mathcal{G}_{\psi(u_t)}} f({w}_{u_t}^s)\|^2
\end{align}
Then we bound R.H.S. as follow.
First,  we consider the bound of $\mathbb{E}  \| {v}^{{\psi(t)}}_t \|^2$, and  definite
\begin{equation}\label{ncsvrg-9}
\zeta_{t}^{s}=\nabla_{\mathcal{G}_{\psi(t)}} f_{i_t}\left({w}_{t}^{s}\right)-\nabla_{\mathcal{G}_{\psi(t)}} f_{i_t}\left({w}^s\right)
\end{equation}
where $w_t^s$ denotes $w_t$ at outer loop $s$. From the definition of $v_t^{\psi(t)}$ one can get:
\begin{equation}\label{ncsvrg-10}
\begin{array}{l}{\mathbb{E}\left\|v_{t}^{{\psi(t)}}\right\|^{2}
=\mathbb{E}\left\|\zeta_{t}^{s}+\nabla_{\mathcal{G}_{\psi(t)}} f\left({w}^s\right)\right\|^{2}} \\
 {=\mathbb{E}\left\|\zeta_{t}^{s}+\nabla_{\mathcal{G}_{\psi(t)}} f\left({w}^s\right)-\nabla_{\mathcal{G}_{\psi(t)}} f\left(w_{t}^{s}\right)+\nabla_{\mathcal{G}_{\psi(t)}} f\left(w_{t}^{s}\right)\right\|^{2} } \\
 {\stackrel{(a)}{\leq} 2 \mathbb{E}\left\|\nabla_{\mathcal{G}_{\psi(t)}} f\left(w_{t}^{s}\right)\right\|^{2}+2 \mathbb{E}\left\|\zeta_{t}^{s}-\mathbb{E}\zeta_{t}^{s}\right\|^{2}} \\ {=2 \mathbb{E}\left\|\nabla_{\mathcal{G}_{\psi(t)}} f\left(w_{t}^{s}\right)\right\|^{2}+2 \mathbb{E}\left\|\left(\nabla_{\mathcal{G}_{\psi(t)}} f_{i_t}\left(w_{t}^{s}\right)-\nabla_{\mathcal{G}_{\psi(t)}} f_{i_t}\left({w}^s\right)-\mathbb{E}\zeta_{t}^{s}\right)\right\|^{2}}\end{array}
\end{equation}
where  (a) follows from $\|a+b\|^2\leq 2\|a\|^2+ 2\|b\|^2$, and $\mathbb{E}\left[\zeta_{t}^{s}\right]=\nabla f\left(w_{t}^{s}\right)- \nabla f\left({w}^s\right)$. From the above equality, we have
\begin{equation}\label{ncsvrg-11}
\begin{array}{l}{\mathbb{E}\left\|v_{t}^{{\psi(t)}}\right\|^{2}} \\ {\stackrel{(a)}{\leq} 2 \mathbb{E}\left\|\nabla_{\mathcal{G}_{\psi(t)}} f\left(w_{t}^{s}\right)\right\|^{2}+{2} \mathbb{E}\left\|\nabla_{\mathcal{G}_{\psi(t)}} f_{i_t}\left(w_{t}^{s}\right)-\nabla_{\mathcal{G}_{\psi(t)}} f_{i_t}\left({w}^{s}\right)\right\|^{2}} \\ {\stackrel{(b)}{\leq}2 \mathbb{E}\left\|\nabla_{\mathcal{G}_{\psi(t)}} f\left(w_{t}^{s}\right)\right\|^{2}+{2 L_*^{2}} \mathbb{E}\left\|w_{t}^{s}-{w}^{s}\right\|^{2}}\end{array}
\end{equation}
where (a) follows from that $\mathbb{E}\|\zeta-\mathbb{E}[\zeta]\|^2\leq \mathbb{E}\|\zeta\|^2$, (b) follows from Assumption~\ref{assum2}. We define $\widetilde{\nabla}_{\mathcal{G}_{\psi(t)}}^s  = \nabla_{\mathcal{G}_{\psi(t)}} \mathcal{L} (\bar{w}_{t}^s) + \nabla_{\mathcal{G}_{\psi(t)}} g ((\widehat{w}_t^s))_{\mathcal{G}_\psi(t)}
= \vartheta_1 \cdot (x_i)_{\mathcal{G}_{\psi(t)}}
+ \nabla_{\mathcal{G}_{\psi(t)}} g ((\widehat{w}_t^s))_{\mathcal{G}_\psi(t)} $ when the $t$-th global iteration denotes a collaborative update, while $\widetilde{\nabla}_{\mathcal{G}_{\psi(t)}}^s = \nabla_{\mathcal{G}_{\psi(t)} } f(\widehat{w}_t)$ if a  dominated update.
Then we derive the upper bound of $\mathbb{E} ||w_{t+1}^{s} - {w}^s ||^2 $
\begin{align}\label{ncsvrg-12}
\mathbb{E}   ||w^s_{t+1} - w^s||^2 &
= \mathbb{E}  ||w^s_{t+1}- w^s_t + w^s_t - w^s||^2
 \nonumber \\
&= \mathbb{E} ||w^s_{t+1}-w^s_t ||^2 + \mathbb{E}||w^s_t - w^s||^2 - 2\mathbb{E}\left<w^s_{t+1}- w^s_t, w^s_t-w^s\right>
\nonumber \\
&= \gamma^2 \mathbb{E} ||\widetilde{v}^{{\psi(t)}}_t ||^2 + \mathbb{E} ||w^s_t - w^s||^2
- 2\gamma \mathbb{E}\left< \widetilde{\nabla}_{\mathcal{G}_{\psi(t)}}^s , w^s_t - w^s\right>
\nonumber \\
&\stackrel{(a)}{\leq}  \gamma^2\mathbb{E} ||\widetilde{v}^{{\psi(t)}}_t||^2
 +    \mathbb{E} ||w^s_t - w^s||^2 + 2\gamma \mathbb{E}\left[  \frac{1}{2\beta_t}||\widetilde{\nabla}_{\mathcal{G}_{\psi(t)}}^s ||^2 + \frac{\beta_t}{2}||w^s_t - w^s||^2  \right]
 \nonumber \\
&= \gamma^2\mathbb{E} ||\widetilde{v}^{{\psi(t)}}_t||^2
+ \frac{\gamma}{\beta_t} \mathbb{E} ||\widetilde{\nabla}_{\mathcal{G}_{\psi(t)}}^s ||^2 + (1+ \gamma \beta_t)  \mathbb{E} ||w^s_t - w^s||^2
\end{align}
where  (a) follows from Yong-Equation. For $\forall t\in K'(u)$, where $u \in \mathcal{A'}(s)$ there is
\begin{align}\label{ncsvrg-13}
\mathbb{E} f(w^s_{t+1})
& \stackrel{(a)}{\leq}\mathbb{E}\left[ f(w^s_t) + \left< \nabla f(w^s_t) , w^s_{t+1}-w^s_t \right> + \frac{L}{2} ||  w^s_{t+1}- w^s_t||^2 \right]
\nonumber \\
&= \mathbb{E} f(w^s_t) - \gamma\mathbb{E}\left< \nabla f(w^s_t),  \widetilde{\nabla}_{\mathcal{G}_{\psi(t)}}^s  \right>  + \frac{\gamma^2 L}{2} \mathbb{E} ||\widetilde{v}^{{\psi(t)}}_t||^2
\nonumber \\
&\stackrel{(b)}{=} \mathbb{E} f(w^s_t) - \frac{\gamma}{2} \mathbb{E}\biggl[  ||\nabla_{\mathcal{G}_{\psi(t)}} f(w^s_t)||^2 +  || \widetilde{\nabla}_{\mathcal{G}_{\psi(t)}}^s ||^2 \nonumber \\
&- ||\nabla_{\mathcal{G}_{\psi(t)}} f(w^s_t) -  \widetilde{\nabla}_{\mathcal{G}_{\psi(t)}}^s  ||^2   \biggr] + \frac{\gamma^2 L_{*}}{2} \mathbb{E} ||\widetilde{v}^{{\psi(t)}}_t||^2
\end{align}
where the (a) follows from Assumption~2\ref{assum1}, (b) follows form $ \left<a,b\right>=\|a\|^2+\|b\|^2-\|a-b\|^2$. Next, we give the upper bound of the term $\mathbb{E}||\nabla_{\mathcal{G}_{\psi(t)}} f(w^s_t) -  \widetilde{\nabla}_{\mathcal{G}_{\psi(t)}}^s  ||^2$ :
 \begin{eqnarray}\label{ncsvrg-14}
\mathbb{E}||\nabla_{\mathcal{G}_{\psi(t)}} f(w^s_t) -  \widetilde{\nabla}_{\mathcal{G}_{\psi(t)}}^s  ||^2 \leq 2{L_{{*}}^2  \gamma^2 \tau_1}  \sum_{u' \in D(t)} \mathbb{E} \|   \widetilde{v}^{\psi(u')}_{u'} \|^2
+ 8{ L_{{*}}^2  \gamma^2 \tau_2}  \sum_{u' \in D'(t)} \mathbb{E} \|   \widetilde{v}^{\psi(u')}_{u'} \|^2.
\end{eqnarray}
Above result can be obtained by applying Lemma~\ref{lem-csgd-2} with $v_{t}^{\psi(t)}$ and $\widetilde{v}_{t}^{\psi(t)}$ defined in SVRG-based algorithm. From (\ref{ncsvrg-13}) and (\ref{ncsvrg-14}), it is easy to derive the following inequality:
\begin{eqnarray}\label{ncsvrg-15}
\mathbb{E} f(w_{t+1}^{s})  &\leq& \mathbb{E}f(w^s_t) - \frac{\gamma}{2} \mathbb{E}
||\nabla_{\mathcal{G}_{\psi(t)}} f(w^s_t)||^2  -  \frac{\gamma}{2} \mathbb{E} || \widetilde{\nabla}_{\mathcal{G}_{\psi(t)}}^s ||^2 + \frac{\gamma^2L}{2} \mathbb{E} || \widetilde{v}^{{\psi(t)}}_t||^2
\nonumber \\
&+& {{L_{*}^2\gamma^3 }}  \left(\tau_1 \sum_{t' \in D(t)} \mathbb{E} \|   \widetilde{v}^{\psi(u')}_{u'} \|^2 r
+  4\tau_2 \sum_{t' \in D^\prime(t)} \mathbb{E} \|   \widetilde{v}^{\psi(t')}_{t'} \|^2 \right)
\end{eqnarray}
Similar to many convergence analyses of nonconvex optimization, we define the Lyapunov function as
\begin{eqnarray}\label{ncsvrg-16}
R_t^{s} = \mathbb{E}\left[ f(w^s_t) + c_t ||w^s_t - w^s||^2 \right],
\end{eqnarray}
then there is
\begin{align}\label{ncsvrg-17}
R_{t+1}^{s}
&= \mathbb{E}\left[ f(w_{t+1}^{s}) + c_{t+1} ||w^s_{t+1}- w^s ||^2 \right] \nonumber \\
&\stackrel{(a)}{\leq}  \mathbb{E}f(w^s_t) - \frac{\gamma}{2} \mathbb{E}
||\nabla_{\mathcal{G}_{\psi(t)}} f(w^s_t)||^2  -  \frac{\gamma}{2} \mathbb{E} || \widetilde{\nabla}_{\mathcal{G}_{\psi(t)}}^s ||^2 + \frac{\gamma^2L}{2} \mathbb{E} || \widetilde{v}^{{\psi(t)}}_t||^2
\nonumber \\
&+ {{L_{*}^2\gamma^3 }}  \left( \tau_1 \sum_{t' \in D(t)} \mathbb{E} \|   \widetilde{v}^{\psi(t')}_{t'} \|^2 +  4 \tau_2 \sum_{t' \in D^\prime(t)} \mathbb{E} \|   \widetilde{v}^{\psi(t')}_{t'} \|^2  \right)
\nonumber \\
&+ c_{t+1} \left[ \gamma^2\mathbb{E} ||\widetilde{v}^{{\psi(t)}}_t||^2    + \frac{\gamma}{\beta_t} \mathbb{E} ||\widetilde{\nabla}_{\mathcal{G}_{\psi(t)}}^s ||^2 + (1+ \gamma \beta_t)  \mathbb{E} ||w^s_t - w^s||^2  \right]
 \nonumber \\
&\stackrel{(b)}{\leq} \mathbb{E} f(w^s_t)  - \frac{\gamma}{2 } \mathbb{E} ||\nabla_{\mathcal{G}_{\psi(t)}} f(w^s_t)||^2  - (\frac{\gamma}{2 } - \frac{c_{t+1}\gamma}{ \beta_t})\mathbb{E} || \widetilde{\nabla}_{\mathcal{G}_{\psi(t)}}^s||^2
\nonumber \\
&  +   {{L_*^2\gamma^3\tau_1 }}   \sum_{t' \in D(t)} \mathbb{E} ||\widetilde{v}^{\psi(t')}_{t'}||^2 + (\frac{\gamma^2L}{2 }
+ c_{t+1}\gamma^2)  \mathbb{E} ||\widetilde{v}^{{\psi(t)}}_t||^2
\nonumber \\
&+  4 L_*^2\gamma^3\tau_2 \sum_{t' \in D^\prime(t)} \mathbb{E} \|   \widetilde{v}^{\psi(t')}_{t'} \|^2 +  c_{t+1}(1+ \gamma\beta_t)  \mathbb{E}  ||w^s_t - w^s||^2 \nonumber \\
\end{align}
where (a) follow from Eqs.~(\ref{ncsvrg-12}) and (\ref{ncsvrg-15}). Summing this over an outer loop $s$ one can obtain:
\begin{align}\label{ncsvrg-18}
 & \sum\limits_{u\in \mathcal{A'}(s)} \sum\limits_{t\in K'(u)}R_{u_t+1}^{s} \\
&=
\sum\limits_{u\in \mathcal{A'}(s)} \sum\limits_{t\in K'(u)}\left( \mathbb{E} f(w^s_{u_t})
 - \frac{\gamma}{2 } \mathbb{E} ||\nabla_{\mathcal{G}_{\psi({u_t})}} f(w^s_{u_t})||^2
 - (\frac{\gamma}{2 } - \frac{c_{{u_t}+1}\gamma}{ \beta_{u_t}})\mathbb{E} || \widetilde{\nabla}_{\mathcal{G}_{\psi({u_t})}}^s||^2\right)
\nonumber \\
 & \quad +  \sum\limits_{u\in \mathcal{A'}(s)} \sum\limits_{t\in K'(u)} {{L_*^2\gamma^3\tau_1 }}   \sum_{t' \in D({u_t})} \mathbb{E} ||\widetilde{v}^{\psi(t')}_{t'}||^2
 +  \sum\limits_{u\in \mathcal{A'}(s)} \sum\limits_{t\in K'(u)}(\frac{\gamma^2L}{2 } + c_{{u_t}+1}\gamma^2)  \mathbb{E} ||\widetilde{v}^{{\psi({u_t})}}_{u_t}||^2
 \nonumber \\
& \quad +   \sum\limits_{u\in \mathcal{A'}(s)} \sum\limits_{t\in K'(u)}4 L_*^2\gamma^3\tau_2 \sum_{t' \in D^\prime({u_t})} \mathbb{E} \|   \widetilde{v}^{\psi(t')}_{t'} \|^2
+   \sum\limits_{u\in \mathcal{A'}(s)} \sum\limits_{{t}\in K'(u)}c_{{u_t}+1}(1+ \gamma\beta_{u_t})  \mathbb{E}  ||w^s_{u_t} - w^s||^2 \nonumber
\end{align}
Summing above inequality over all outer loops $s=1,\cdots,S$ and reorganize it we have
\begin{align}\label{ncsvrg-19}
&\sum\limits_{s=1}^{S} \sum\limits_{u\in \mathcal{A'}(s)} \sum\limits_{t\in K'(u)}R_{{{u_t}}+1}^{s}
\\
&= \sum\limits_{s=1}^{S} \sum\limits_{u\in \mathcal{A'}(s)} \sum\limits_{t\in K'(u)}
\left( \mathbb{E} f(w^s_{{u_t}})
- \frac{\gamma}{2 } \mathbb{E}  ||\nabla_{\mathcal{G}_{\psi({{u_t}})}} f(w^s_{{u_t}})||^2
+ c_{{{u_t}}+1}(1+ \gamma\beta_{{u_t}})  \mathbb{E}  ||w^s_{{u_t}} - w^s||^2 \right)
\nonumber \\
&  + \sum\limits_{s=1}^{S} \sum\limits_{u\in \mathcal{A'}(s)} \sum\limits_{t\in K'(u)}\left(   {{L_*^2\gamma^3\tau_1 }}   \sum_{t' \in D({u_t})} \mathbb{E} ||\widetilde{v}^{\psi(t')}_{t'}||^2
+ (\frac{\gamma^2L}{2 } + c_{{{u_t}}+1}\gamma^2)  \mathbb{E} ||\widetilde{v}^{{\psi({u_t})}}_{{u_t}}||^2\right)
 \nonumber \\
& - \sum\limits_{s=1}^{S} \sum\limits_{u\in \mathcal{A'}(s)} \sum\limits_{t\in K'(u)} (\frac{\gamma}{2 }
- \frac{c_{{{u_t}}+1}\gamma}{ \beta_{{u_t}}})\mathbb{E} || \widetilde{\nabla}_{\mathcal{G}_{\psi({{u_t}})}}^s ||^2
+ \sum\limits_{s=1}^{S}\sum\limits_{t=0 }^{|v(s)|-1}4L_*^2\gamma^3\tau_2 \sum_{t' \in D^\prime({{u_t}})} \mathbb{E} \|   \widetilde{v}^{\psi(t')}_{t'} \|^2
\nonumber \\
& \stackrel{(a)}{\leq} \sum\limits_{s=1}^{S} \sum\limits_{u\in \mathcal{A'}(s)} \sum\limits_{t\in K'(u)} \left( \mathbb{E} f(w^s_{{u_t}})
- \frac{\gamma}{2 } \mathbb{E}  ||\nabla_{\mathcal{G}_{\psi({{u_t}})}} f(w^s_{{u_t}})||^2
+ c_{{{u_t}}+1}(1+ \gamma\beta_{{u_t}})  \mathbb{E}  ||w^s_{{u_t}} - w^s||^2 \right)
 \nonumber \\
&  + \sum\limits_{s=1}^{S} \sum\limits_{u\in \mathcal{A'}(s)} \sum\limits_{t\in K'(u)}\left(   {{L_*^2\gamma^3(\tau_1^2 + 4\tau_2^2) }} \mathbb{E} ||\widetilde{v}^{\psi({u_t})}_{{u_t}}||^2  + (\frac{\gamma^2L_{*}}{2 }
+ c_{{u_t}+1}\gamma^2)  \mathbb{E} ||\widetilde{v}^{\psi({u_t})}_{{u_t}}||^2\right)
\nonumber \\
&\stackrel{(b)}{\leq}\sum\limits_{s=1}^{S} \sum\limits_{u\in \mathcal{A'}(s)} \sum\limits_{t\in K'(u)}\left( \mathbb{E} f(w^s_{{u_t}})  - \frac{\gamma}{2 } \mathbb{E}  ||\nabla_{\mathcal{G}_{\psi({{u_t}})}} f(w^s_{{u_t}})||^2 + c_{{{u_t}}+1}(1+ \gamma\beta_{{u_t}})  \mathbb{E}  ||w^s_{{u_t}} - w^s||^2 \right)
\nonumber \\
&  +  \lambda_{\gamma}\sum\limits_{s=1}^{S} \sum\limits_{u\in \mathcal{A'}(s)} \sum\limits_{t\in K'(u)}\left(   {{5L_*^2\gamma^3 }} \tau  + \frac{\gamma^2L_{*}}{2 } + c_{{{u_t}}+1}\gamma^2  \right)\mathbb{E} ||{v}^{\psi({{u_t}})}_{{u_t}} ||^2
\end{align}
where (a) follows from Assumptions~\ref{assum4} to \ref{assum4} and assuming $\frac{1}{2} \geq \frac{c_{{{u_t}}+1}}{\beta_{{u_t}}}$, (b) follows from Lemma~\ref{lem-ncsvrg-1}. Denote $ {{10L_*^2\gamma^3 }} \tau  + {\gamma^2L_{*}} + 2c_{{{u_t}}+1}\gamma^2$ as $\lambda_{{u_t}}$, we have
\begin{align}\label{ncsvrg-20}
&\sum\limits_{s=1}^{S} \sum\limits_{u\in \mathcal{A'}(s)} \sum\limits_{t\in K'(u)}
R_{{{u_t}}+1}^s \\
&\stackrel{(a)}{\leq}
\sum\limits_{s=1}^{S} \sum\limits_{u\in \mathcal{A'}(s)} \sum\limits_{t\in K'(u)}
\left( \mathbb{E} f(w^s_{{u_t}})  - \frac{\gamma}{2 } \mathbb{E}  ||\nabla_{\mathcal{G}_{\psi({{u_t}})}} f(w^s_{{u_t}})||^2 + c_{{{u_t}}+1}(1+ \gamma\beta_{{u_t}})  \mathbb{E}  ||w^s_{{u_t}} - w^s||^2 \right)
\nonumber \\
&  +  \lambda_{\gamma}\sum\limits_{s=1}^{S} \sum\limits_{u\in \mathcal{A'}(s)} \sum\limits_{t\in K'(u)}
 \lambda_{{u_t}}
\left(  \mathbb{E}\left\|\nabla_{\mathcal{G}_{\psi({{u_t}})}} f\left(w_{{{u_t}}}^{s}\right)\right\|^{2}+{ L_{*}^{2}} \mathbb{E}\left\|w_{{{u_t}}}^{s}-{w}^{s}\right\|^{2} \right)  \nonumber \\
&{=}\sum\limits_{s=1}^{S} \sum\limits_{u\in \mathcal{A'}(s)} \sum\limits_{t\in K'(u)}
\left( \mathbb{E} f(w^s_{{u_t}})+ (c_{{{u_t}}+1}(1+ \gamma\beta_{{u_t}})+\lambda_{\gamma}\lambda_{{{u_t}}}L_{*}^2 ) \mathbb{E}  ||w^s_{{u_t}} - w^s||^2 \right)  \nonumber \\
&  -  \sum\limits_{s=1}^{S} \sum\limits_{u\in \mathcal{A'}(s)} \sum\limits_{t\in K'(u)}
(\frac{\gamma}{2 } -\lambda_{\gamma}\lambda_{u_t})\mathbb{E}  ||\nabla_{\mathcal{G}_{\psi({{u_t}})}} f(w^s_{{u_t}})||^2 \nonumber
\end{align}
where (a) follows from Eq.~\ref{ncsvrg-11} and the definitions of $\lambda_{u_t}$ and $L_*$. Then we return to Eq.~\ref{ncsvrg-8}:
 \begin{align}\label{ncsvrg-133}
&\sum_{s=1}^{S}  \sum\limits_{u\in \mathcal{A'}(s)} \sum\limits_{t\in K'(u)}  \mathbb{E} \| \nabla_{\mathcal{G}_{\psi(u_t)}} f({w}_{u_0}^s)\|^2 \\
    &\stackrel{(a)}{\leq} 2 \lambda_{\gamma}L_*^2 \gamma^2 \tau_1^2 \sum\limits_{s=1}^{S} \sum\limits_{u\in \mathcal{A'}(s)} \sum\limits_{t\in K'(u)} \mathbb{E} \|{v}_{{u_t}}^{\psi({u_t})}\|^2
     + 2 \sum\limits_{s=1}^{S} \sum\limits_{u\in \mathcal{A'}(s)} \sum\limits_{t\in K'(u)} \mathbb{E} \|\nabla_{\mathcal{G}_{\psi({u_t})}} f({w}_{{u_t}})\|^2
     \nonumber \\
     & \stackrel{(b)}{\leq} 4 \lambda_{\gamma}L_*^4 \gamma^2 \tau \sum\limits_{s=1}^{S} \sum\limits_{u\in \mathcal{A'}(s)} \sum\limits_{t\in K'(u)} \mathbb{E}\left\|w_{{u_t}}^{s}-{w}^{s}\right\|^{2}
     \nonumber \\
     & + (2 + 4 \lambda_{\gamma}L_*^2 \gamma^2 \tau) \sum\limits_{s=1}^{S} \sum\limits_{u\in \mathcal{A'}(s)} \sum\limits_{t\in K'(u)} \mathbb{E} \|\nabla_{\mathcal{G}_{\psi({u_t})}} f({w}_{{u_t}}^s)\|^2
\end{align}
where (a) follows from  Lemma \ref{lem-ncsvrg-1}, (b) follows from Eq.~\ref{ncsvrg-11} and $u_0$ denotes the start iteration during epoch $u$. This implies that
 \begin{align}\label{ncsvrg-22}
&\frac{\frac{\gamma}{2 } -\lambda_{\gamma}\lambda_{u_t}}{2 + 4 \lambda_{\gamma}L_*^4 \gamma^2 \tau}
\sum_{s=1}^{S}  \sum\limits_{u\in \mathcal{A'}(s)} \sum\limits_{t\in K'(u)}  \mathbb{E} \| \nabla_{\mathcal{G}_{\psi(u_t)}} f({w}_{u_0}^s)\|^2
\nonumber \\
 & \stackrel{(a)}{\leq} L_*^2 (\frac{\gamma}{2 } -\lambda_{\gamma}\lambda_{u_t} )
  \sum\limits_{s=1}^{S} \sum\limits_{u\in \mathcal{A'}(s)} \sum\limits_{t\in K'(u)} \mathbb{E}\left\|w_{{u_t}}^{s}-{w}^{s}\right\|^{2}
 + (\frac{\gamma}{2 } -\lambda_{\gamma}\lambda_{u_t}) \sum\limits_{s=1}^{S} \sum\limits_{u\in \mathcal{A'}(s)} \sum\limits_{t\in K'(u)} \mathbb{E} \|\nabla_{\mathcal{G}_{\psi({u_t})}} f({w}_{{u_t}}^s)\|^2
\end{align}
where (a) follows from the definition of $L_*$. Combining Eq.~\ref{ncsvrg-22} with \ref{ncsvrg-20} we have
 \begin{align}\label{ncsvrg-23}
& \frac{\frac{\gamma}{2 } -\lambda_{\gamma}\lambda_{u_t}}{2 + 4 \lambda_{\gamma}L_*^4 \gamma^2 \tau}
\sum_{s=1}^{S}  \sum\limits_{u\in \mathcal{A'}(s)} \sum\limits_{t\in K'(u)}  \mathbb{E} \| \nabla_{\mathcal{G}_{\psi(u_t)}} f({w}_{u_0}^s)\|^2
\nonumber \\
& \stackrel{(a)}{\leq} L_*^2 (\frac{\gamma}{2 } -\lambda_{\gamma}\lambda_{u_t} )  \sum\limits_{s=1}^{S} \sum\limits_{u\in \mathcal{A'}(s)} \sum\limits_{t\in K'(u)}  \mathbb{E}\left\|w_{{u_t}}^{s}-{w}^{s}\right\|^{2}
+  \sum\limits_{s=1}^{S} \sum\limits_{u\in \mathcal{A'}(s)} \sum\limits_{t\in K'(u)} R_{{{u_t}}+1}^s
\nonumber \\ &
+  \sum\limits_{s=1}^{S} \sum\limits_{u\in \mathcal{A'}(s)} \sum\limits_{t\in K'(u)} \left(\mathbb{E} f(w^s_{{u_t}})
+ (c_{{{u_t}}+1}(1+ \gamma\beta_{t})+\lambda_{\gamma}\lambda_{{t}}L_{*}^2 ) \mathbb{E}  ||w^s_{{u_t}} - w^s||^2 \right)
\nonumber \\
 & =  \sum\limits_{s=1}^{S} \sum\limits_{u\in \mathcal{A'}(s)} \sum\limits_{t\in K'(u)} R_{{{u_t}}+1}^s
 + \sum\limits_{s=1}^{S} \sum\limits_{u\in \mathcal{A'}(s)} \sum\limits_{t\in K'(u)}
  \left(\mathbb{E} f(w^s_{{u_t}}) + (c_{{{u_t}}+1}(1+ \gamma\beta_{{u_t}})+\frac{\gamma}{2}L_{*}^2 )
 \mathbb{E}  ||w^s_{{u_t}} - w^s||^2\right)
\end{align}
Rearrange Eq.~\ref{ncsvrg-23} we have
\begin{align}\label{ncsvrg-24}
\sum\limits_{s=1}^{S} \sum\limits_{u\in \mathcal{A'}(s)} \sum\limits_{t\in K'(u)}R_{{{u_t}}+1}^s
&\leq
\sum\limits_{s=1}^{S} \sum\limits_{u\in \mathcal{A'}(s)} \sum\limits_{t\in K'(u)} \left(\mathbb{E} f(w^s_{{u_t}}) + (c_{{{u_t}}+1}(1+ \gamma\beta_{{u_t}})+\frac{\gamma}{2}L_{*}^2 ) \mathbb{E}  ||w^s_{{u_t}} - w^s||^2 \right)
\nonumber \\
& - \frac{\frac{\gamma}{2 } -\lambda_{\gamma}\lambda_{u_t}}{2 + 4 \lambda_{\gamma}L_*^4 \gamma^2 \tau}\sum\limits_{s=1}^{S} \sum\limits_{u\in \mathcal{A'}(s)} \mathbb{E} \| \nabla f({w}_{u})\|^2
\nonumber \\
&= \sum\limits_{s=1}^{S} \sum\limits_{u\in \mathcal{A'}(s)} \sum\limits_{t\in K'(u)}R_{{{u_t}}}^s
 -  \sum_{s=1}^{S}  \sum\limits_{u\in \mathcal{A'}(s)} \sum\limits_{t\in K'(u)}   \Gamma_{u_t}   \mathbb{E} \| \nabla_{\mathcal{G}_{\psi(u_t)}} f({w}_{u_0}^s)\|^2
\end{align}
where
\begin{align}\label{ncsvrg-25}
c_{{u_t}} & = c_{{u_t}+1}(1+ \gamma \beta_{u_t}) +  \frac{\gamma}{2}L_*^2
\end{align}
and
\begin{eqnarray}\label{ncsvrg-26}
 \Gamma_{{u_t}} = \frac{ \frac{\gamma}{2} - \frac{2}{1 - 20L_{*}^2\gamma^2\tau}( 10L_{*}^2 \gamma^3 \tau  + {\gamma^2L_{*}} + 2c_{{u_t}+1}\gamma^2)} {2 + 4 \lambda_{\gamma}L_*^4 \gamma^2 \tau}
\end{eqnarray}
Denote the subscript of the last iteration in $s$-th outer loop as $c_{\bar{s}}$ and set it as  0, and set \[{w}^{s+1} = w^{s}_{{ {u_t}=\bar{s}}}\] then there is
\[R_{{ {u_t}=\bar{s}}}^{s} =  \mathbb{E}  f(w_{{ {u_t}=\bar{s}}}^{s})= \mathbb{E} f({w}^{s+1}).\]
 Applying these to \ref{ncsvrg-20} we can get,
\begin{eqnarray}\label{ncsvrg-27}
\sum_{s=1}^{S}  \sum\limits_{u\in \mathcal{A'}(s)} \sum\limits_{t\in K'(u)}   \mathbb{E} \| \nabla_{\mathcal{G}_{\psi(u_t)}} f({w}_{u_0}^s)\|^2
 \leq \frac{\mathbb{E}\left[  f( w^{s})  -  f(w^{s+1}) \right] }{ \Gamma_*}
\end{eqnarray}
where $\Gamma_* = min \{\Gamma_{u_t}\}$, $u_0$ denotes the start iteration during epoch $u$. Using the update rule of VF{${\textbf{B}}^2$}-SVRG
 and summing up all outer loops, and defining $w_0$ as initial point and $w^*$ as optimal solution, we have the final inequality:
\begin{eqnarray}\label{ncsvrg-28}
\sum_{s=1}^{S}  \sum\limits_{u\in \mathcal{A'}(s)} \sum\limits_{t\in K'(u)}   \mathbb{E} \| \nabla_{\mathcal{G}_{\psi(u_t)}} f({w}_{u_0}^s)\|^2
 \leq \frac{\mathbb{E}\left[  f( w_{0})  -  f(w^{*}) \right] }{\Gamma_*}
\end{eqnarray}
since $\sum\limits_{t\in K'(u)}   \mathbb{E} \| \nabla_{\mathcal{G}_{\psi(u_t)}} f({w}_{u_0}^s)\|^2 = \mathbb{E} \| \nabla f({w}_{u_0}^s)\|^2 $, we have
\begin{eqnarray}\label{ncsvrg-28}
\frac{1}{T}\sum_{s=1}^{S}  \sum\limits_{u\in \mathcal{A'}(s)}   \mathbb{E} \| \nabla f({w}_{u_0}^s)\|^2
 \leq \frac{\mathbb{E}\left[  f( w_{0})  -  f(w^{*}) \right] }{T \Gamma_*}
\end{eqnarray}
where $T$ denotes the total number of epoches, $u_0$ denotes the start iteration during epoch $u$.

To prove Theorem~\ref{thm-svrgnonconvex}, set $\{c_{u_t}^s\}_{u_t=\bar{s}} = 0$, $ \gamma = \frac{m_0}{L_{*}n^\alpha}$, $\beta_t = \beta = {2L_{*}}$, where  $0<m_0<1$, and $0<\alpha<1$. And there is
\begin{eqnarray}\label{ncsvrg-29}
\theta &= & {\gamma \beta} =  \frac{2m_0}{n^{{\alpha}}}
\end{eqnarray}
From the recurrence formula of $c_t$, we have:
\begin{eqnarray}\label{ncsvrg-30}
c_0& =& \frac{\gamma L_*^2}{2} \frac{(1+\theta)^{{N} } - 1}{\theta} \nonumber \\
&=& \frac{m_0L_*^2}{2L_*n^\alpha} \frac{n^\alpha}{2m_0} \left( (1+\theta)^{{N}} -1 \right) \nonumber \\
&\stackrel{(a)}{\leq}& \frac{L_*}{4}  \left( (1+\theta)^\frac{1}{\theta} -1 \right) \nonumber \\
&\stackrel{(b)}{\leq} & \frac{L_*}{4}(e-1)
\end{eqnarray}
where (a) follow form ${{N}} \leq \lfloor  \frac{n^{{\alpha}}}{2m_0}  \rfloor$, (b) follows from that $(1+\frac{1}{l})^l$ is increasing for $l>0$, and $\lim\limits_{l\rightarrow  \infty}(1 + \frac{1}{l})^l = e$. Since $e-1<2$, there is $c_0 \leq \frac{L_*}{2}$ which satisfies $c_t \leq \frac{\beta}{2}=L_*$ (used in \ref{ncsvrg-19}). Therefore, $c_t$ is decreasing with respect to $t$, and $c_0$ is also upper bounded.
\begin{eqnarray}
\Gamma_* &=& \min_t \Gamma_t \nonumber \\
&\stackrel{(a)}{\geq} & \frac{ \frac{\gamma}{2} - \frac{2}{1 - 20L_{*}^2\gamma^2\tau}( 10L_{*}^2 \gamma^3 \tau  + {\gamma^2L_{*}} + 2c_{0}\gamma^2)} {2 + 4 \lambda_{\gamma}L_*^4 \gamma^2 \tau}
 \nonumber \\
& =&\frac{ \frac{\gamma}{2} - \frac{2n^{2\alpha}}{n^{2\alpha} - 20m_0^2\tau}
( \frac{10m_0^2\tau}{n^{2\alpha}}  + \frac{2m_0}{n^\alpha})\gamma}
{2 + \frac{8L_*^2m_0^2\tau}{n^{2\alpha}-20m_0^2\tau}}
\nonumber \\
&\stackrel{(b)}{\geq}&  \frac{\left( \frac{1}{2}-(20m_0^2\tau + 4m_0)\right)\gamma}{2 + {8L_*^2m_0^2\tau}}
\nonumber \\
&\stackrel{(c)}{\geq}& \frac{\sigma }{L_{*}n^{\alpha}}
\end{eqnarray}
where (a) follows from $c_0=max\{c_t\}$, (b) follow form $n^{{\alpha}} \leq n^{2\alpha} - 20m_0^2 \tau$ (we assume $n \geq \frac{1 + \sqrt{1+80m_0^2\tau}}{2}$, this is easy to satisfy when $n$ is large) and $n^\alpha > 1$,
 (c) follows from that if $\frac{1}{2} > 20m_0^2\tau + 4m_0$ and  $\sigma$ is a small value which is independent of $n$.

Above all, if $\tau < \text{min} \{\frac{n^{2\alpha}}{20m_0^2},\frac{1-8m_0}{40m_0^2} \} $ (where $\tau < \frac{n^{2\alpha}}{20m_0^2}$  denotes $\lambda_\gamma >0$), where $1-8m_0>0$, and $N$, satisfies $N \leq \lfloor  \frac{n^{{\alpha}}}{2m_0}  \rfloor$  we have the conclusion:
\begin{eqnarray} \label{ncsvrg-final}
\frac{1}{T}\sum\limits_{s=1}^{S}\sum\limits_{t=0}^{N-1}\mathbb{E}  ||\nabla f(w^s_{t})||^2  \leq \frac{L_{*}n^{\alpha}\mathbb{E}\left[  f( w_{0})  -  f( w^{*}) \right] }{T \sigma }
\end{eqnarray}
where $T$ denotes the total number of epoches.
Let R.H.S. of \ref{ncsvrg-final} $\leq \epsilon$, one can obtain that
\begin{eqnarray} \label{ncsvrg-final2}
T\geq \frac{L_{*}n^{\alpha}\mathbb{E}\left[  f( w_{0})  -  f( w^{*}) \right] }{\epsilon \sigma }
\end{eqnarray}
This completes the proof.
\end{proof}
\subsection{Proof of Theorem~\ref{thm-saganonconvex}}
\begin{lemma}\label{lem-ncsaga-2}
For all $\forall $ $\psi (t)$, there are
\begin{equation}\label{lemeq-ncsaga-2}
 \sum\limits_{t=0}^{S-1} \mathbb{E} ||\widetilde{v}_{t}^{\psi(t)}||^2 \leq \lambda_{\gamma} \sum\limits_{t=0}^{S-1} \mathbb{E} || {v}_{t}^{\psi(t)}\|^2,
\end{equation}
where  $\lambda_\gamma = \frac{2}{1 - 180L_{*}^2\gamma^2\tau }> 0$.
\end{lemma}
\begin{proof}[\bf{Proof of Lemma~\ref{lem-ncsaga-2}}]
First,
we give the upper bound to $\mathbb{E}    \|  \widetilde{v}^{ {\psi(u)} }_t - \widehat{v}^{{\psi(u)}}_t  \|^2$ as follows.
We have that
\begin{eqnarray}\label{ncsaga-8}
&&   \mathbb{E} \left \| \widetilde{v}^{ {\psi(u)} }_t - \widehat{v}^{{\psi(u)}}_t \right \|^2
\\  &  = & \nonumber
\mathbb{E} \left \|   \left(\nabla_{\mathcal{G}_{\psi(t)}}\mathcal{L}(\bar{w})
+ \nabla_{\mathcal{G}_{\psi(t)}} g((\widehat{w}_t)_{\mathcal{G}_{\psi(t)}}) \right)
- \nabla_{\mathcal{G}_{\psi(t)}} f(\widehat{w}_t)
- \widehat{\alpha}_{i}^{\psi(u)}   + \widetilde{\alpha}_{i}^{\psi(u)}
+ \frac{1}{n} \sum_{i=1}^n \widehat{\alpha}_{i}^{\psi(u)}   - \frac{1}{n} \sum_{i=1}^n \widetilde{\alpha}_{i}^{\psi(u)}  \right \|^2
\\  &  \stackrel{ (a) }{\leq} & \nonumber
3  \mathbb{E} Q_1
+ 3\mathbb{E} \underbrace{\left \| \widetilde{\alpha}_{i}^{\psi(u)}
+  \widehat{\alpha}_{i}^{\psi(u)} \right \|^2 }_{Q_2}
+ 3\mathbb{E} \underbrace{\left \|  \frac{1}{n} \sum_{i=1}^n \widetilde{\alpha}_{i}^{t,\psi(t)} - \frac{1}{n} \sum_{i=1}^n \widehat{\alpha}_{i}^{t,\psi(t)}  \right \|^2}_{Q_3}
\end{eqnarray}
where $Q_1 = \left \|  \left(\nabla_{\mathcal{G}_{\psi(t)}}\mathcal{L}(\bar{w})
+ \nabla_{\mathcal{G}_{\psi(t)}} g((\widehat{w}_t)_{\mathcal{G}_{\psi(t)}}) \right)
- \nabla_{\mathcal{G}_{\psi(t)}} f(\widehat{w}_t) \right\|$ and  inequality (a) uses $\| \sum_{i=1}^n a_i \|^2 \leq n \sum_{i=1}^n \| a_i \|^2 $.
We will give the upper bounds for the expectations  of $Q_1$, $Q_2$ and $Q_3$  respectively.
\begin{eqnarray}\label{ncsaga-9}
 \nonumber \mathbb{E} Q_1 &=& \mathbb{E} \left \|  \left(\nabla_{\mathcal{G}_{\psi(t)}}\mathcal{L}(\bar{w})
+ \nabla_{\mathcal{G}_{\psi(t)}} g((\widehat{w}_t)_{\mathcal{G}_{\psi(t)}}) \right)
- \nabla_{\mathcal{G}_{\psi(t)}} f(\widehat{w}_t) \right\|
 \\&\stackrel{(a)}{\leq}&
  \mathbb{E}  ||\nabla_{\mathcal{G}_{\psi(t)}} f(\bar{w}_t) - \nabla_{\mathcal{G}_{\psi(t)}} f(\widehat{w}_t) + \nabla_{\mathcal{G}_{\psi(t)}} g((\widehat{w}_t)_{\mathcal{G}_{\psi(t)}}) - \nabla_{\mathcal{G}_{\psi(t)}} g((\bar{w}_t)_{\mathcal{G}_{\psi(t)}})||^2
  \nonumber \\
 &\stackrel{(b)}{\leq}& 4{ L_{{*}}^2  \gamma^2 \tau_2}  \sum_{t' \in D'(t)} \mathbb{E} \|   \widetilde{v}^{\psi(t')}_{t'} \|^2
\end{eqnarray}
above inequality can be obtained by following the proof of Lemma~\ref{lem-csgd-2}.
\begin{eqnarray}\label{ncsaga-10}
\mathbb{E} Q_2 &=&   \mathbb{E}\left \| \widetilde {\alpha}_{i_t}^{t,\psi(t)} - \widehat{\alpha}_{i_t}^{t,\psi(t)} \right \|^2
\\ \nonumber &\leq & \frac{4\tau_2 L_*^2 \gamma^2}{n} \sum_{t'=1}^{\phi(t)-1}  \sum_{{u} \in D'(\xi(t',\psi(t)))} \left ( 1 -\frac{1}{n} \right )^{\phi(t)-t'-1} \mathbb{E}  \left \|       \widetilde{v}^{\psi({u})}_{{u}} \right \|^2
\end{eqnarray}
where the inequality uses Lemma \ref{lem-csaga-1}.
\begin{eqnarray}\label{ncsaga-11}
&& \mathbb{E} Q_3 =    \mathbb{E} \left \|  \frac{1}{n} \sum_{i=1}^n \widetilde{\alpha}_{i}^{t,\psi(t)} - \frac{1}{n} \sum_{i=1}^n \widehat{\alpha}_{i}^{t,\psi(t)}   \right \|^2
\\  &  \leq & \nonumber
\frac{1}{n} \sum_{i=1}^n \mathbb{E} \left \|  \widehat{\alpha}_{i}^{t,\psi(t)}  - \widehat{\alpha}_{i}^{t,\psi(t)}   \right \|^2
\\  &  \leq &
\frac{4\tau_2 L_*^2 \gamma^2}{n} \sum_{t'=1}^{\phi(t)-1}  \sum_{{u} \in D'(\xi(t',\psi(t)))} \left ( 1 -\frac{1}{n} \right )^{\phi(t)-t'-1} \mathbb{E}  \left \|       \widetilde{v}^{\psi({u})}_{{u}} \right \|^2 \nonumber
\end{eqnarray}
where  the first inequality uses $\| \sum_{i=1}^n a_i \|^2 \leq n \sum_{i=1}^n \| a_i \|^2 $, the second inequality uses Lemma \ref{lem-csaga-1}. Combining \ref{ncsaga-9}, \ref{ncsaga-10}, and \ref{ncsaga-11}, one can obtain:
\begin{eqnarray}\label{SAGA-13}
&&    \mathbb{E} \left \| \widetilde{v}_{t}^{\psi(t)}  - \widehat{v}_t^{\psi(t)}  \right \|^2
\\  &  \leq & \nonumber    3  \mathbb{E} {Q_1} + 3 \mathbb{E} {Q_2}  + 3\mathbb{E} {Q_3}
\\  &  \leq & \nonumber {12 L_*^2\gamma^2\tau_2 } \sum_{u \in D'(\xi(t',\psi(t)))} \mathbb{E}  ||\widetilde{v}^{\psi(u)}_{u}||^2
+ \frac{24\tau_2 L_*^2 \gamma^2}{n} \sum_{t'=1}^{\phi(t)-1}  \sum_{u \in D'(\xi(t',\psi(t)))} \left ( 1 -\frac{1}{n} \right )^{\phi(t)-t'-1} \mathbb{E}  \left \|       \widetilde{v}^{\psi({u})}_{{u}} \right \|^2
\end{eqnarray}
Summing above equality over all iterations, we have
\begin{eqnarray}\label{SAGA-14}
&&   \sum_{t=0}^{S-1} \mathbb{E} \left \| \widetilde{v}_{t}^{\psi(t)}  - \widehat{v}_t^{\psi(t)}  \right \|^2
\\  &  \leq & \nonumber
{12 L_*^2\gamma^2\tau_2^2 } \sum_{t=0}^{S-1} \mathbb{E}  ||\widetilde{v}^{\psi(t)}_{t}||^2
+ \frac{24\tau_2^2 L_*^2 \gamma^2}{n} \sum_{t=0}^{S-1} \sum_{t'=1}^{\phi(t)-1}  \left ( 1 -\frac{1}{n} \right )^{\phi(t)-t'-1} \mathbb{E}  \left \|       \widetilde{v}^{\psi(t')}_{t'} \right \|^2
\\ &  \leq & \nonumber
 {36 L_*^2\gamma^2\tau_2^2 } \sum_{t=0}^{S-1} \mathbb{E}  ||\widetilde{v}^{\psi(t)}_{t}||^2
\end{eqnarray}
Define ${v}^{{\psi(t)}}_t= \nabla_{\mathcal{G}_{\psi(t)}} f_i (w_{t}) - \alpha_i^{{\psi(t)}} +  \frac{1}{n} \sum_{i=1}^n \alpha_i^{{\psi(t)}}$.
And then, we give the upper bound to $\mathbb{E}   \left \|   \widehat{v}^{ {\psi(t)} }_t - {v}^{{\psi(t)}}_t \right  \|^2$ as follows.
We have that
\begin{eqnarray}\label{ncsaga-15}
&&   \mathbb{E} \left \| \widehat{v}_{t}^{\psi(t)} -  v_{t}^{\psi(t)}  \right \|^2
\\  &  = & \nonumber  \mathbb{E} \left \|  \nabla_{\mathcal{G}_{\psi(t)}} f_{i_t} (\widehat{w}_{t})- \widehat{\alpha}_{i_t}^{t,\psi(t)}  + \frac{1}{n} \sum_{i=1}^n \widehat{\alpha}_{i}^{t,\psi(t)}  - \nabla_{\mathcal{G}_{\psi(t)}} f_{i_t} (w_{t}) + \alpha_{i_t}^{t,\psi(t)} - \frac{1}{n} \sum_{i=1}^n \alpha_{i}^{t,\psi(t)} \right \|^2
\\  &  \stackrel{ (a) }{\leq} & \nonumber  3  \mathbb{E} \underbrace{\left \|  \nabla_{\mathcal{G}_{\psi(t)}} f_{i_t} (\widehat{w}_{t})-  \nabla_{\mathcal{G}_{\psi(t)}} f_{i_t} ({w}_{t}) \right \|^2 }_{Q_4}
+ 3\mathbb{E} \underbrace{\left \| {\alpha}_{i_t}^{t,\psi(t)} - \widehat{\alpha}_{i_t}^{t,\psi(t)} \right \|^2 }_{Q_5}
+ 3\mathbb{E} \underbrace{\left \|  \frac{1}{n} \sum_{i=1}^n \alpha_{i}^{t,\psi(t)} - \frac{1}{n} \sum_{i=1}^n \widehat{\alpha}_{i}^{t,\psi(t)}  \right \|^2}_{Q_6}
\end{eqnarray}
where the  inequality (a) uses $\| \sum_{i=1}^n a_i \|^2 \leq n \sum_{i=1}^n \| a_i \|^2 $.
We will give the upper bounds for the expectations  of $Q_4$, $Q_5$ and $Q_6$  respectively.
\begin{eqnarray}\label{ncsaga-16}
 \mathbb{E} Q_4 &=& \mathbb{E} \left \|   \nabla_{\mathcal{G}_{{\psi(t)}}} f_{i_t} (\widehat{w}_{t})-  \nabla_{\mathcal{G}_{{\psi(t)}}} f_{i_t} ({w}_{t}) \right \|^2 \\
 &\stackrel{(a)}{\leq}&  { L_{{\psi(t)}}^2}
\mathbb{E} \left[ ||w_{t} - \widehat{w}_{t}||^2 \right]
\nonumber \\
 &=& {L_{{\psi(t)}}^2\gamma^2}  \mathbb{E} \left[ ||\sum_{t' \in D(u,i_t)}    \textbf{U}_{\psi(t')} \widetilde{v}^{\psi(t')}_{t'} ||^2 \right]
 \nonumber \\
 &\stackrel{(b)}{\leq}& { \tau_1 L_*^2\gamma^2 }    \sum_{t' \in D(u,i_t)} \mathbb{E} \left[ ||\widetilde{v}^{\psi(t')}_{t'}||^2 \right] \nonumber
\end{eqnarray}
where (a) uses Assumption~\ref{assum2}, (b) uses $\| \sum_{i=1}^n a_i \|^2 \leq n \sum_{i=1}^n \| a_i \|^2 $. Similar to the analyses of $Q_2$ and $Q_3$, we have
\begin{eqnarray}\label{ncsaga-17}
\mathbb{E} Q_5 =   \mathbb{E}\left \|  {\alpha}_{i_t}^{t,\psi(t)} - \widehat{\alpha}_{i_t}^{t,\psi(t)} \right \|^2
 \leq   \frac{\tau_1 L^2 \gamma^2}{n} \sum_{t'=1}^{\phi(t)-1}  \sum_{\widetilde{u} \in D(\xi(t',\psi(t')))} \left ( 1 -\frac{1}{n} \right )^{\phi(t)-t'-1} \mathbb{E}  \left \|       \widetilde{v}^{\psi(\widetilde{u})}_{\widetilde{u}} \right \|^2
\end{eqnarray}
where the inequality uses Lemma \ref{lem-csaga-1}.
\begin{eqnarray}\label{ncsaga-18}
 &&\mathbb{E} Q_6 =    \mathbb{E} \left \|  \frac{1}{n} \sum_{i=1}^n {\alpha}_{i}^{t,\psi(t)} - \frac{1}{n} \sum_{i=1}^n \widehat{\alpha}_{i}^{t,\psi(t)}   \right \|^2
 \nonumber \\
 &\leq&  \frac{\tau_1 L^2 \gamma^2}{n} \sum_{t'=1}^{\phi(t)-1}  \sum_{\widetilde{u} \in D(\xi(t',\psi(t')))} \left ( 1 -\frac{1}{n} \right )^{\phi(t)-t'-1} \mathbb{E}  \left \|       \widetilde{v}^{\psi(\widetilde{u})}_{\widetilde{u}} \right \|^2 \nonumber
\end{eqnarray}
Summing above inequality for $t=0,\cdots,S-1$ and follow the analyses of Eq.~\ref{ncsaga-15} one can have
\begin{eqnarray}\label{ncsaga-19}
&&   \sum_{t=0}^{S-1} \mathbb{E} \left \| \widetilde{v}_{u}^{\psi(t)}  - \widehat{v}_t^{\psi(t)}  \right \|^2
\\  &  \leq & \nonumber
{3 L_*^2\gamma^2\tau_1^2 } \sum_{t=0}^{S-1} \mathbb{E}  ||\widetilde{v}^{\psi(t')}_{t'}||^2
+ \frac{6\tau_1 L^2 \gamma^2}{n} \sum_{t=0}^{S-1} \sum_{t'=1}^{\phi(t)-1}  \left ( 1 -\frac{1}{n} \right )^{\phi(t)-t'-1} \mathbb{E}  \left \|       \widetilde{v}^{\psi(\widetilde{u})}_{\widetilde{u}} \right \|^2
\\ &  \leq & \nonumber {9 L_*^2\gamma^2\tau_1^2 } \sum_{t=0}^{S-1} \mathbb{E}  ||\widetilde{v}^{\psi(t')}_{t'}||^2
\end{eqnarray}
Based on above formulations, we have
\begin{align}\label{ncsaga-20}
   \sum_{t=0}^{S-1} \mathbb{E} \left \| \widetilde{v}_{u}^{\psi(t)}  - {v}_t^{\psi(t)}  \right \|^2 & = \sum_{t=0}^{S-1} \mathbb{E} \left \| \widetilde{v}_{u}^{\psi(t)} -\widehat{v}_{u}^{\psi(t)} + \widehat{v}_{u}^{\psi(t)} - {v}_t^{\psi(t)}  \right \|^2
   \nonumber \\
    &\leq \sum_{t=0}^{S-1} \left( 2\mathbb{E} \| \widetilde{v}_{u}^{\psi(t)} -\widehat{v}_{u}^{\psi(t)}\|^2 + 2 \mathbb{E} \|\widehat{v}_{u}^{\psi(t)} - {v}_t^{\psi(t)}  \|^2\right)
       \nonumber \\
    &\stackrel{(a)}{\leq} {18 L_*^2\gamma^2\tau_1^2 } \sum_{t=0}^{S-1} \mathbb{E}  ||\widetilde{v}^{\psi(t)}_{t}||^2 +  {72 L_*^2\gamma^2\tau_2^2 } \sum_{t=0}^{S-1} \mathbb{E}  ||\widetilde{v}^{\psi(t)}_{t}||^2
\end{align}
then we have
\begin{align}\label{ncsaga-21}
   \sum_{t=0}^{S-1} \mathbb{E} \left \| \widetilde{v}_{t}^{\psi(t)} \right \|^2 & = \sum_{t=0}^{S-1} \mathbb{E} \left \| \widetilde{v}_{t}^{\psi(t)} -{v}_{t}^{\psi(t)} + {v}_t^{\psi(t)}  \right \|^2
   \leq  \sum_{t=0}^{S-1} \left( 2\mathbb{E} \| \widetilde{v}_{t}^{\psi(t)} -{v}_{t}^{\psi(t)} \|^2+ 2\mathbb{E} \|{v}_t^{\psi(t)}   \|^2 \right)
   \nonumber \\
   & \leq { L_*^2\gamma^2(36\tau_1^2 + 144\tau_2^2) } \sum_{t=0}^{S-1} \mathbb{E}  ||\widetilde{v}^{\psi(t')}_{t'}||^2  + 2\sum_{t=0}^{S-1}{v}_t^{\psi(t)}
\end{align}
which implies that
if $\lambda_\gamma = \frac{2}{1 - 180\eta_2L_{*}^2\gamma^2\tau }> 0$, we hae
 \begin{equation}\label{ncsaga-22}
 \sum_{t=1}^{S} \mathbb{E} ||\widetilde{v}_{t}^{{\psi(t)}}||^2 \leq \lambda_{\gamma}\sum_{t=1}^{S} \mathbb{E} || {v}_{t}^{{\psi(t)}}\|^2,
\end{equation}
This completes the proof.
\end{proof}

Similar to the proof of Theorem~\ref{thm-sgdnonconvex}, we first apply Lemma~\ref{lem-csgd-3} to all $S$ iterations and there is
\begin{align}\label{ncsaga-23}
\sum_{t \in \mathcal{A}(S)}\mathbb{E} \| \nabla f({w}_{t})\|^2
    \stackrel{(a)}{\leq} 2 L_*^2 \gamma^2 \eta_2^2 \sum_{t=0}^{S-1} \mathbb{E} \|\widetilde{v}_{t}^{\psi(t)}\|^2 + 2 \sum_{t=0}^{S-1} \mathbb{E} \|\nabla_{\mathcal{G}_{\psi(t)}} f({w}_{t})\|^2
\end{align}
Then, we give the upper bound to $\mathbb{E}  \| {v}^{{\psi(t)}}_{t} \|^2$ as follows. We definite:
\begin{equation}\label{ncsaga-24}
\zeta_{t}^{{\psi(t)}}=\nabla_{\mathcal{G}_{{\psi(t)}}}f_{i_t}(w_{t})-{\alpha}_{i_t}^{t,\psi(t)}
\end{equation}
and use the definition of $v_t^{\psi(t)}$ to get
\begin{equation}\label{ncsaga-25}
\begin{array}{l}{\mathbb{E}\left\|v_{t}^{{\psi(t)}}\right\|^{2}=\mathbb{E}\left\|\zeta_{t}^{{\psi(t)}}+ \frac{1}{n}\sum_{i=1}^{n}\alpha_i^{t,\psi(t)}\right\|^{2}} \\ {=\mathbb{E}\left\|\zeta_{t}^{{\psi(t)}}+\frac{1}{n}\sum_{i=1}^{n}\alpha_i^{t,\psi(t)}-\nabla_{\mathcal{G}_{{\psi(t)}}}f({w_{t}})+ \nabla_{\mathcal{G}_{{\psi(t)}}}f({w_{t}})\right\|^{2} }
\\ { \stackrel{(a)}{\leq}2 \mathbb{E}\left\|\nabla_{\mathcal{G}_{{\psi(t)}}}f({w_{t}})\right\|^{2}+2 \mathbb{E}\left\|\zeta_{t}^{{\psi(t)}}-\mathbb{E}\left[\zeta_{t}^{{\psi(t)}}\right]\right\|^{2}} \\ { \stackrel{(b)}{\leq} 2 \mathbb{E}\left\|\nabla_{\mathcal{G}_{{\psi(t)}}}f({w_{t}})\right\|^{2}+2 \mathbb{E}\left\|\zeta_{t}^{{\psi(t)}}\right\|^{2}}\end{array}
\end{equation}
where (a) follows from $\|a+b\|^2 \leq 2\|a\|^2+ 2\|b\|^2$ and $\mathbb{E}\left[\zeta_{t}^{{\psi(t)}}\right]=\nabla_{\mathcal{G}_{{\psi(t)}}} f\left(w_{t}\right)- \frac{1}{n}\sum_{i=1}^{n}\alpha_i^{t,\psi(t)}$, (b) follows from $\mathbb{E}\left\|\zeta_{t}^{{\psi(t)}}-\mathbb{E}\left[\zeta_{t}^{{\psi(t)}}\right]\right\|^{2}\leq \mathbb{E} \|\zeta_t^{{\psi(t)}}\|^2$, we have
\begin{align}\label{ncsaga-26}
&{\mathbb{E}\left\|v_{t}^{{\psi(t)}}\right\|^{2}}
\nonumber \\
&\leq 2 \mathbb{E}\left\|\nabla_{\mathcal{G}_{{\psi(t)}}} f\left(w_{t}\right)\right\|^{2}+{2} \mathbb{E}\left\|{\alpha}_{i_t}^{t,\psi(t)}-\nabla_{\mathcal{G}_{{\psi(t)}}}f_{i_t}({w_t})\right\|^{2}
\nonumber \\
&\stackrel{(a)}{\leq} 2\mathbb{E}\left\|\nabla_{\mathcal{G}_{{\psi(t)}}} f\left(w_{t}\right)\right\|^{2}+  2\left ( 1 -\frac{1}{n} \right )^{\phi(t)-1}  \mathbb{E}  \left \|  \nabla_{\mathcal{G}_{\psi(t)}} f_i({w}_{{0}}) -  \nabla_{\mathcal{G}_{\psi(t)}} f_i(w_{t}) \right \|^2
\nonumber\\   &\quad + 2\frac{1}{n} \sum_{t'=1}^{\phi(t)-1} \left ( 1 -\frac{1}{n} \right )^{\phi(t)-t'-1} \mathbb{E}  \left \|  \nabla_{\mathcal{G}_{\psi(t)}} f_i({w}_{{\xi(t',\psi(t))}}) - \nabla_{\mathcal{G}_{\psi(t)}} f_i(w_{t}) \right \|^2 \nonumber \\
&\stackrel{(b)}{\leq} 2\mathbb{E}\left\|\nabla_{\mathcal{G}_{{\psi(t)}}} f\left(w_{t}\right)\right\|^{2}
+  2L^2\left ( 1 -\frac{1}{n} \right )^{\phi(t)-1}  \mathbb{E}  \left \|{w}_{{0}} -  w_{t} \right \|^2
\nonumber\\
&\quad + 2L^2\frac{1}{n} \sum_{t'=1}^{\phi(t)-1} \left ( 1 -\frac{1}{n} \right )^{\phi(t)-t'-1} \mathbb{E}  \left \|  {w}_{{\xi(t',\psi(t))}} - w_{t} \right \|^2 \nonumber \\
&\stackrel{(c)}{\leq} 2\mathbb{E}\left\|\nabla_{\mathcal{G}_{{\psi(t)}}} f\left(w_{t}\right)\right\|^{2}
+  2L^2\left ( 1 -\frac{1}{n} \right )^{\phi(t)-1}  \mathbb{E}  \left \|{w}_{{0}} -  w_{t} \right \|^2
\nonumber\\
 &\quad + 2L^2\frac{1}{n} \sum_{t'=1}^{\phi(t)-1} \left ( 1 -\frac{1}{n} \right )^{\phi(t)-t'-1} \left(2\mathbb{E}   \|  {w}_{{\xi(t',\psi(t))}} - w_0\|^2 + 2\mathbb{E}\| w_0- w_{t} \|^2\right)
  \nonumber \\
& \stackrel{(d)}{\leq}2\mathbb{E}\left\|\nabla_{\mathcal{G}_{{\psi(t)}}} f\left(w_{t}\right)\right\|^{2}
+ 6L^2\mathbb{E}\| w_0- w_{t} \|^2
+ 4L^2\frac{1}{n} \sum_{t'=1}^{\phi(t)-1} \left ( 1 -\frac{1}{n} \right )^{\phi(t)-t'-1} \mathbb{E}   \|  {w}_{{\xi(t',\psi(t))}} - w_0\|^2
\end{align}
where (a) follows from Lemma \ref{lem-csaga-1}, (b) follows from Assumption 2, (c) follows from $\|a+b\|^2\leq 2\|a\|^2 + 2\|b\|^2$, (d) follows from $\sum_{t'=1}^{\phi(t)-1} \frac{1}{n}\left ( 1 -\frac{1}{n} \right )^{\phi(t)-t'-1}<1$ and $\left ( 1 -\frac{1}{n} \right )^{\phi(t)-1}<1$.
Moreover, there is
\begin{align}\label{ncsaga-27}
  \sum_{t=0}^{S-1}\frac{1}{n} \sum_{t'=1}^{\phi(t)-1} \left ( 1 -\frac{1}{n} \right )^{\phi(t)-t'-1} \mathbb{E}   \|  {w}_{{\xi(t',\psi(t))}} - w_0\|^2 & \leq \sum_{t=0}^{S-1} \mathbb{E}   \|  {w}_{t} - w_0\|^2
\end{align}
which follows from that $\frac{1}{n} \sum_{t'=1}^{\phi(t)-1} \left ( 1 -\frac{1}{n} \right )^{\phi(t)-t'-1} \leq 1$. As for $\mathbb{E}\| w_0- w_{t} \|^2 $ there is
\begin{align}\label{ncsaga-28}
  \mathbb{E}\| w_0- w_{t+1} \|^2 & = \mathbb{E}\| w_0- w_t + w_t -w_{t+1} \|^2
  \nonumber \\
  & =  \mathbb{E}\| w_0- w_t\|^2 + \mathbb{E}\|w_t -w_{t+1} \|^2 -2\mathbb{E}\left<w_0- w_t ,w_t -w_{t+1}\right>
    \nonumber \\
  & =  \mathbb{E}\| w_0- w_t\|^2 + \mathbb{E}\|w_t -w_{t+1} \|^2 -2\gamma\mathbb{E}\left<w_0- w_t ,\widetilde{\nabla}_{\mathcal{G}_{{\psi(t)}}}^s\right>
  \nonumber \\
  & \leq \mathbb{E}\| w_0- w_t\|^2 + \mathbb{E}\|w_t -w_{t+1} \|^2 +2\gamma(\frac{1}{2\beta_t}\mathbb{E}\|\widetilde{\nabla}_{\mathcal{G}_{{\psi(t)}}}^s\|^2 + \frac{\beta_t}{2}\mathbb{E}\| w_0- w_t\|^2)
 \nonumber \\
  & =(1 + \gamma \beta_t) \mathbb{E}\| w_0- w_t\|^2 + \gamma^2\mathbb{E}\|w_t -w_{t+1} \|^2 + \frac{\gamma}{\beta_t}\mathbb{E}\|\widetilde{\nabla}_{\mathcal{G}_{{\psi(t)}}}^s\|^2
\end{align}

\begin{proof}[\bf{Proof of Theorem \ref{thm-saganonconvex}}]
First, we upper bound $\mathbb{E}f(w_{t+1}) $ for $t =0,\cdots,S-1$:
\begin{align}\label{ncsaga-29}
\mathbb{E} f(w_{t+1})
& \stackrel{(a)}{\leq}\mathbb{E}\left[ f(w_t) + \left< \nabla f(w_t) , w_{t+1}-w_t \right> + \frac{L}{2} ||  w_{t+1}- w_t||^2 \right]
 \nonumber \\
&= \mathbb{E} f(w_t) - \gamma\mathbb{E}\left< \nabla f(w_t),  \widetilde{\nabla}_{\mathcal{G}_{{\psi(t)}},i_t}  \right>  + \frac{\gamma^2 L}{2} \mathbb{E} ||\widetilde{v}^{{{\psi(t)}}}_t||^2
\nonumber \\
&\stackrel{(b)}{=} \mathbb{E} f(w_t) - \frac{\gamma}{2} \mathbb{E}\biggl[  ||\nabla_{\mathcal{G}_{{\psi(t)}}} f(w_t)||^2 +  || \widetilde{\nabla}_{\mathcal{G}_{{\psi(t)}}} ||^2
\nonumber \\
&- ||\nabla_{\mathcal{G}_{{\psi(t)}}} f(w_t) -  \widetilde{\nabla}_{\mathcal{G}_{{\psi(t)}}} ||^2   \biggr] + \frac{\gamma^2 L_{*}}{2} \mathbb{E} ||\widetilde{v}^{{{\psi(t)}}}_t||^2
\end{align}
where the (a) follows from Assumption~2, (b) follows form $ \left<a,b\right>=\|a\|^2+\|b\|^2-\|a-b\|^2$.
 Next, we give the upper bound of the term $\mathbb{E}||\nabla_{\mathcal{G}_{{\psi(t)}}} f(w_t) -  \widetilde{\nabla}_{\mathcal{G}_{{\psi(t)}}}  ||^2$ :
 \begin{eqnarray}\label{ncsaga-30}
\mathbb{E}||\nabla_{\mathcal{G}_{\psi(t)}} f(w_t) -  \widetilde{\nabla}_{\mathcal{G}_{\psi(t)}}^s  ||^2 \leq 2{L_{{*}}^2  \gamma^2 \tau_1}  \sum_{u' \in D(t)} \mathbb{E} \|   \widetilde{v}^{\psi(u')}_{u'} \|^2
+ 8{ L_{{*}}^2  \gamma^2 \tau_2}  \sum_{u' \in D'(t)} \mathbb{E} \|   \widetilde{v}^{\psi(u')}_{u'} \|^2.
\end{eqnarray}
Above result can be obtained by following the analyses of  Lemma~\ref{lem-csgd-2}. From Eqs.~(\ref{ncsaga-29}) and (\ref{ncsaga-30}), it is easy to derive the following inequality:
\begin{eqnarray}\label{ncsaga-31}
\mathbb{E} f(w_{t+1}^{s})  &\leq& \mathbb{E}f(w_t) - \frac{\gamma}{2} \mathbb{E}
||\nabla_{\mathcal{G}_{{\psi(t)}}} f(w_t)||^2  -  \frac{\gamma}{2} \mathbb{E} || \widetilde{\nabla}_{\mathcal{G}_{{\psi(t)}},i_t}^s ||^2 + \frac{\gamma^2L}{2} \mathbb{E} || \widetilde{v}^{{{\psi(t)}}}_t||^2   \nonumber \\
&+& {{L_{*}^2\gamma^3 }}  \left(\tau_1 \sum_{t' \in D(u)} \mathbb{E} \|   \widetilde{v}^{\psi(t')}_{t'} \|^2 +  4\tau_2 \sum_{t' \in D^\prime(u)} \mathbb{E} \|   \widetilde{v}^{\psi(t')}_{t'} \|^2 \right)
\end{eqnarray}
Here, we define a Lyapunov function:
\begin{eqnarray}\label{ncsaga-32}
R_t = \mathbb{E} f(w_{t}) + c_t \mathbb{E} \|w_0-w_t\|^2  \,.
\end{eqnarray}
From the definition of  Lyapunov function,  and (\ref{ncsaga-32}):
\begin{align}\label{ncsaga-33}
R_{t+1}
=& \mathbb{E}\left[ f(w_{t+1}) + c_{t+1}\mathbb{E}  \|w_0-w_{t+1}\|^2  \right]
\nonumber \\
\leq& \mathbb{E}f(w_t) - \frac{\gamma}{2} \mathbb{E}
||\nabla_{\mathcal{G}_{{\psi(t)}}} f(w_t)||^2  -  \frac{\gamma}{2} \mathbb{E} || \widetilde{\nabla}_{\mathcal{G}_{{\psi(t)}},i_t}^s ||^2 + \frac{\gamma^2L}{2} \mathbb{E} || \widetilde{v}^{{{\psi(t)}}}_t||^2
\nonumber \\
&+ {{L_{*}^2\gamma^3 }}  \left(\tau_1 \sum_{t' \in D(u)} \mathbb{E} \|   \widetilde{v}^{\psi(t')}_{t'} \|^2 +  4 \tau_2 \sum_{t' \in D^\prime(u)} \mathbb{E} \|   \widetilde{v}^{\psi(t')}_{t'} \|^2 \right)+ c_{t+1} \mathbb{E} \|w_0-w_{t+1}\|^2
 \nonumber \\
\stackrel{(a)}{\leq}& \mathbb{E}f(w_t) - \frac{\gamma}{2} \mathbb{E}
||\nabla_{\mathcal{G}_{{\psi(t)}}} f(w_t)||^2  -  \frac{\gamma}{2} \mathbb{E} || \widetilde{\nabla}_{\mathcal{G}_{{\psi(t)}}} ||^2 + \frac{\gamma^2L}{2} \mathbb{E} || \widetilde{v}^{{{\psi(t)}}}_t||^2
\nonumber \\
&+ {{L_{*}^2\gamma^3 }}  \left(\tau_1 \sum_{t' \in D(u)} \mathbb{E} \|   \widetilde{v}^{\psi(t')}_{t'} \|^2 +  4\tau_2 \sum_{t' \in D^\prime(u)} \mathbb{E} \|   \widetilde{v}^{\psi(t')}_{t'} \|^2 \right)
\nonumber \\
&+ c_{t+1}\left((1+{\gamma\beta_t})\mathbb{E} \|w_0-w_{t}\|^2
 +\frac{\gamma }{\beta_t}\mathbb{E} \|\widetilde{\nabla}_{\mathcal{G}_{{\psi(t)}}}\|^2  + \gamma^2\mathbb{E} \|\widetilde{v}^{{\psi(t)}}_t  \|^2\right)
\end{align}
where (a) follows from Eq.~\ref{ncsaga-28}.
Summing above inequality for all iterations then we have that:
\begin{align}\label{ncsaga-34}
\sum\limits_{t=0}^{S-1}R_{{{t}}+1}
 &= \sum\limits_{t=0}^{S-1}\left( \mathbb{E} f(w_{{t}})  - \frac{\gamma}{2 } \mathbb{E} ||\nabla_{\mathcal{G}_{\psi({t})}} f(w_{{t}})||^2
 - (\frac{\gamma}{2 }-\frac{\gamma c_{{t}+1} }{\beta_{t}})\mathbb{E} || \widetilde{\nabla}_{\mathcal{G}_{{\psi(t)}}}||^2 \right)
 \nonumber \\
& \quad +  \sum\limits_{t=0}^{S-1}\biggl( {{L_{*}^2\gamma^3 (\tau_1^2 +4 \tau_2^2)}}   \mathbb{E} \|   \widetilde{v}^{\psi(t)}_{t} \|^2
+ \frac{\gamma^2L}{2} \mathbb{E} || \widetilde{v}^{{{\psi(t)}}}_t||^2
 \nonumber\\
&\quad + c_{{t}+1}\gamma^2\mathbb{E}\|\widetilde{v}^{{\psi(t)}}_{t}  \|^2
 + c_{{t}+1}(1+{\gamma\beta_{t}})\mathbb{E} \|w_0-w_{t}\|^2\biggr)
\nonumber\\
&\stackrel{(a)}{\leq} \sum\limits_{t=0}^{S-1}\left( \mathbb{E} f(w_{{t}})  - \frac{\gamma}{2 } \mathbb{E} ||\nabla_{\mathcal{G}_{\psi({t})}} f(w_{{t}})||^2
 - (\frac{\gamma}{2 }-\frac{\gamma c_{{t}+1} }{\beta_{t}})\mathbb{E} || \widetilde{\nabla}_{\mathcal{G}_{{\psi(t)}}}||^2 \right)
 \nonumber \\
& \quad +  \sum\limits_{t=0}^{S-1}\biggl( \left( {5{L_{*}^2\gamma^3 \tau +  \frac{\gamma^2L}{2} + c_{{t}+1}\gamma^2 }}\right)   \mathbb{E} \|   \widetilde{v}^{\psi(t)}_{t} \|^2
 + c_{{t}+1}(1+{\gamma\beta_{t}})\mathbb{E} \|w_0-w_{t}\|^2\biggr)
\nonumber\\
&\stackrel{(b)}{\leq} \sum\limits_{t=0}^{S-1}\left( \mathbb{E} f(w_{{t}})
 - \frac{\gamma}{2 } \mathbb{E} ||\nabla_{\mathcal{G}_{\psi({t})}} f(w_{{t}})||^2
  - (\frac{\gamma}{2 }-\frac{\gamma c_{{t}+1} }{\beta_{t}})\mathbb{E} || \widetilde{\nabla}_{\mathcal{G}_{{\psi(t)}}}||^2 \right)
  \nonumber \\
&\quad + \sum\limits_{t=1}^{S}\lambda_{\gamma}\lambda_{{t}} \left( \mathbb{E}\left\|\nabla_{\mathcal{G}_{\psi(u_t)}} f\left(w_{{t}}\right)\right\|^{2}
+ 5L_{*}^2 \|w_0-w_{t}\|^2\right)+ \sum\limits_{t=0}^{S-1}c_{{t}+1}(1+{\gamma\beta_{t}})\mathbb{E} \|w_0-w_{t}\|^2
\nonumber\\
&= \sum\limits_{t=0}^{S-1} \left( \mathbb{E} f(w_{{t}}) +\left(c_{{t}+1} (1+{\gamma\beta})
+ 5L_{*}^2\lambda_{\gamma}\lambda_{{t}}\right) \|w_0-w_{t}\|^2 \right)
\nonumber \\
& \quad - \sum\limits_{t=0}^{S-1} (\frac{\gamma}{2}-\lambda_{\gamma}\lambda_{{t}}) \mathbb{E}\left\|\nabla_{\mathcal{G}_{\psi({t})}} f\left(w_{{t}}\right)\right\|^{2}
- \sum\limits_{t=0}^{S-1} (\frac{\gamma}{2 }-\frac{\gamma c_{{t}+1} }{\beta_{t}})\mathbb{E} || \widetilde{\nabla}_{\mathcal{G}_{{\psi(t)}}}||^2
\nonumber\\
&\stackrel{(c)}{\leq} \sum\limits_{t=0}^{S-1} \left( \mathbb{E} f(w_{{t}}) +\left(c_{{t}+1} (1+{\gamma\beta})+ 5L_{*}^2\lambda_{\gamma}\lambda_{{t}}\right)\|w_0-w_{t}\|^2 \right)
 - \sum\limits_{t=0}^{S-1} (\frac{\gamma}{2}-\lambda_{\gamma}\lambda_{{t}}) \mathbb{E}\left\|\nabla_{\mathcal{G}_{\psi({t})}} f\left(w_{{t}}\right)\right\|^{2}
\end{align}
where (a) follow from the definition of $\tau$, (b) uses Eqs.~\ref{ncsaga-26} and \ref{ncsaga-27},  $\lambda_{t} ={ 10 L_{*}^2 \gamma^3}  \tau  + {\gamma^2L} + 2c_{{t}+1}\gamma^2 $, and the definition of $L_{*}$, (b) follows from assuming $\frac{\gamma}{2 }-\frac{\gamma c_{t+1} }{\beta_{t}}>0$.

Similar to the proof of Theorem~\ref{thm-svrgnonconvex}, we have
\begin{align}\label{ncsaga-35}
&\sum_{u \in \mathcal{A}(S)}\sum\limits_{t\in K'(u)}  \mathbb{E} \| \nabla_{\mathcal{G}_{\psi(t)}} f({w}_{{u_0}})\|^2
 \\
    &\stackrel{(a)}{\leq} 2 \lambda_{\gamma}L_*^2 \gamma^2 \tau_1^2 \sum_{u \in \mathcal{A}(S)}\sum\limits_{t\in K'(u)}  \mathbb{E} \|{v}_{t}^{\psi(t)}\|^2
     + 2 \sum_{u \in \mathcal{A}(S)}\sum\limits_{t\in K'(u)}  \mathbb{E} \|\nabla_{\mathcal{G}_{\psi(t)}} f({w}_{t})\|^2
     \nonumber \\
     & \stackrel{(b)}{\leq} 20 \lambda_{\gamma}L_*^4 \gamma^2 \tau \sum_{u \in \mathcal{A}(S)}\sum\limits_{t\in K'(u)} \mathbb{E}\|w_0-w_{t}\|^2
      + (2 + 4 \lambda_{\gamma}L_*^2 \gamma^2 \tau) \sum_{u \in \mathcal{A}(S)}\sum\limits_{t\in K'(u)}   \mathbb{E} \|\nabla_{\mathcal{G}_{\psi(t)}} f({w}_{t})\|^2 \nonumber
\end{align}
where ${u_0}$ denotes the start global iteration of epoch $u$, (a) follows from  Eq. \ref{ncsaga-23}, (b) follows from Eqs.~\ref{ncsaga-26} and \ref{ncsaga-27}. This implies that
 \begin{align}\label{ncsaga-36}
&\frac{\frac{\gamma}{2 } -\lambda_{\gamma}\lambda_{t}}{2 + 4 \lambda_{\gamma}L_*^4 \gamma^2 \tau}\sum_{u \in \mathcal{A}(S)}\sum\limits_{t\in K'(u)}  \mathbb{E} \| \nabla_{\mathcal{G}_{\psi(t)}} f({w}_{{u_0}})\|^2
 \\
     & \stackrel{(a)}{\leq} 5L_*^2 (\frac{\gamma}{2 } -\lambda_{\gamma}\lambda_{t} ) \sum_{u \in \mathcal{A}(S)}\sum\limits_{t\in K'(u)}  \mathbb{E}\|w_0-w_{t}\|^2
 + (\frac{\gamma}{2 } -\lambda_{\gamma}\lambda_{t}) \sum_{u \in \mathcal{A}(S)}\sum\limits_{t\in K'(u)}   \mathbb{E} \|\nabla_{\mathcal{G}_{\psi(t)}} f({w}_{t})\|^2 \nonumber
\end{align}
where (a) follows from the definition of $L_*$. Combining Eq.~\ref{ncsaga-36} with \ref{ncsaga-34} we have
 \begin{align}\label{saga-30}
& \frac{\frac{\gamma}{2 } -\lambda_{\gamma}\lambda_{t}}{2 + 4 \lambda_{\gamma}L_*^4 \gamma^2 \tau} \sum_{u \in \mathcal{A}(S)}\sum\limits_{t\in K'(u)}  \mathbb{E} \| \nabla_{\mathcal{G}_{\psi(t)}} f({w}_{{u_0}})\|^2
\nonumber \\
     & \stackrel{(a)}{\leq} 5L_*^2 (\frac{\gamma}{2 } -\lambda_{\gamma}\lambda_{t} ) \sum_{u \in \mathcal{A}(S)}\sum\limits_{t\in K'(u)} \mathbb{E} \|w_0-w_{t}\|^2
      + \sum_{u \in \mathcal{A}(S)}\sum\limits_{t\in K'(u)} R_{{t}+1}
      \nonumber \\
   &  +  \sum_{u \in \mathcal{A}(S)}\sum\limits_{t\in K'(u)}  \left( \mathbb{E} f(w_{{t}}) +\left(c_{{t}+1} (1+{\gamma\beta})
     + 5L_{*}^2\lambda_{\gamma}\lambda_{{t}}\right)\mathbb{E} \|w_0-w_{t}\|^2 \right)
     \nonumber \\
     & = \sum_{u \in \mathcal{A}(S)}\sum\limits_{t\in K'(u)} R_{{t}+1}
     + \sum_{u \in \mathcal{A}(S)}\sum\limits_{t\in K'(u)}  \left( \mathbb{E} f(w_{{t}}) +\left(c_{{t}+1}(1+{\gamma\beta})
     + \frac{5}{2}{\gamma L_*^2}\right)\mathbb{E} \|w_0-w_{t}\|^2 \right)
\end{align}
Rearrange Eq.~\ref{saga-30} we have
\begin{align}\label{saga-31}
\sum_{u \in \mathcal{A}(S)}\sum\limits_{t\in K'(u)} R_{{t}+1}
&\leq
 \sum_{u \in \mathcal{A}(S)}\sum\limits_{t\in K'(u)}  \left( \mathbb{E} f(w_{{t}}) +\left(c_{{t}+1} (1+{\gamma\beta})
 + \frac{5}{2}{\gamma L_*^2}\right)\mathbb{E} \|w_0-w_{t}\|^2 \right)
\nonumber \\
& - \frac{\frac{\gamma}{2 } -\lambda_{\gamma}\lambda_{t}}{2 + 4 \lambda_{\gamma}L_*^4 \gamma^2 \tau}\sum_{u \in \mathcal{A}(S)}\sum\limits_{t\in K'(u)} \mathbb{E} \| \nabla_{\mathcal{G}_{\psi(u')}} f({w}_{{u_0}})\|^2
\nonumber \\
&=\sum_{u \in \mathcal{A}(S)}\sum\limits_{t\in K'(u)} R_{{t}}
 - \sum_{u \in \mathcal{A}(S)}\sum\limits_{t\in K'(u)} \Gamma_t \mathbb{E} \| \nabla_{\mathcal{G}_{\psi(t)}} f({w}_{{u_0}})\|^2
\end{align}
where
\begin{align}\label{saga-32}
c_{t} & = c_{{t}+1}\left(1 + \gamma \beta_{t}\right)
+ \frac{5}{2}{\gamma L_*^2}
\end{align}
 and
\begin{eqnarray}\label{saga-33}
 \Gamma_{t} = \frac{ \frac{\gamma}{2} - \frac{2}{1 - 180L_{*}^2\gamma^2\tau}( 10L_{*}^2 \gamma^3 \tau  + {\gamma^2L_{*}} + 2c_{t+1}\gamma^2)} {2 + 4 \lambda_{\gamma}L_*^4 \gamma^2 \tau}
\end{eqnarray}
Let $\bar{S}$ be the subscript of the final global iteration, and one can set $\{ c_{t}\}_{t=\bar{S}} = 0$,  define $w_0$ as initial point and $w^*$ as optimal solution, we have
\begin{eqnarray}\label{SAGA-34}
\frac{1}{S}\sum_{u \in \mathcal{A}(S)}\sum\limits_{t\in K'(u)} \Gamma_t \mathbb{E} \| \nabla_{\mathcal{G}_{\psi(t)}} f({w}_{{u_0}})\|^2
 \leq \frac{\mathbb{E}\left[  f( w_0)  -  f(w^{*}) \right] }{S \Gamma_*}
\end{eqnarray}
where $\bar{t}$ denotes the start global iteration of epoch $t$ and use $\Gamma_* = min \{\Gamma_t\}$.

To prove Theorem~\ref{thm-saganonconvex}, set $\{c_{t}\}_{t=S-1} = 0$,
$ \gamma = \frac{m_0}{L_{*}n^\alpha}$, $\beta_t = \beta = {4L_{*}}$, where  $0<m_0<1$, and $0<\alpha<1$. And there is
\begin{align}\label{saga-35}
\theta = \gamma \beta_{t}= \frac{4m_0}{n^{{\alpha}}}
\end{align}
 Then following the analysis of Eq. \ref{ncsvrg-30}, we have that the total epoch number $T$ should satisfy $T\leq \lfloor  \frac{n^{{\alpha}}}{4m_0}  \rfloor$.
\begin{eqnarray}
\Gamma_* &=& \min_t \Gamma_t \nonumber \\
&\stackrel{(a)}{\geq} & \frac{ \frac{\gamma}{2} - \frac{2}{1 - 180L_{*}^2\gamma^2\tau}( 10L_{*}^2 \gamma^3 \tau  + {\gamma^2L_{*}} + 2c_{0}\gamma^2)} {2 + 4 \lambda_{\gamma}L_*^4 \gamma^2 \tau}
 \nonumber \\
& =&\frac{ \frac{\gamma}{2} - \frac{2n^{2\alpha}}{n^{2\alpha} - 180m_0^2\tau}
( \frac{10m_0^2\tau}{n^{2\alpha}}  + \frac{5m_0}{n^\alpha})\gamma}
{2 + \frac{8L_*^2m_0^2\tau}{n^{2\alpha}-180m_0^2\tau}}
\nonumber \\
&\stackrel{(b)}{\geq}&  \frac{\left( \frac{1}{2}-(20m_0^2\tau + 10m_0)\right)\gamma}{2 + {8L_*^2m_0^2\tau}}
\nonumber \\
&\stackrel{(c)}{\geq}& \frac{\sigma }{L_{*}n^{\alpha}}
\end{eqnarray}
where (a) follows from $c_0=max\{c_t\}$, (b) follow form $n^{{\alpha}} \leq n^{2\alpha} - 180m_0^2 \tau$ (which is satisfied when $n \geq \frac{1 + \sqrt{1+720m_0^2\tau}}{2}$, this is easy to satisfy when $n$ is large) and $n^\alpha > 1$, (c) follow from that if $\frac{1}{2} > 20m_0^2\tau + 10m_0$ and  $\sigma$ is a small value which is independent of $n$.

Based on above analyses, we have the conclusion:
\begin{eqnarray} \label{ncsaga-final}
\frac{1}{T}\sum\limits_{t=0}^{T-1}\mathbb{E}  ||\nabla f(w_{t_0})||^2  \leq \frac{L_{*}n^{\alpha}\mathbb{E}\left[  f( w_{0})  -  f( w^{*}) \right] }{T \sigma }
\end{eqnarray}
where, $T$ denotes the number of total epoches, $t_0$ is the start iteration of epoch $t$.
This completes the proof.
\end{proof}
\end{document}